\documentclass{osudissert96}

\usepackage{graphicx} 
\usepackage{amsmath,amssymb,amsfonts,amstext,amsthm,bm}
\usepackage{graphics,epsf,epsfig,psfrag}
\usepackage{color,subfigure}

\usepackage{algorithm}
\usepackage{algorithmic}

\usepackage{dsfont,booktabs}

\usepackage{enumitem}
\usepackage{cleveref}

\usepackage{colortbl}

\newcommand*{\belowrulesepcolor}[1]{%
	\noalign{%
		\kern-\belowrulesep
		\begingroup
		\color{#1}%
		\hrule height\belowrulesep
		\endgroup
	}%
}
\newcommand*{\aboverulesepcolor}[1]{%
	\noalign{%
		\begingroup
		\color{#1}%
		\hrule height\aboverulesep
		\endgroup
		\kern-\aboverulesep
	}%
}

\DeclareMathAlphabet{\mathsfit}{\encodingdefault}{\sfdefault}{m}{sl}
\SetMathAlphabet{\mathsfit}{bold}{\encodingdefault}{\sfdefault}{bx}{n}

\def\gB{{\mathcal{B}}}

\def\gD{{\mathcal{D}}}

\def\gL{{\mathcal{L}}}

\def\gR{{\mathcal{R}}}
\def\gS{{\mathcal{S}}}
\def\gT{{\mathcal{T}}}

\DeclareMathOperator*{\argmin}{arg\,min}

\newtheorem{theorem}{Theorem}
\newtheorem{proposition}{Proposition} 
\newtheorem{lemma}{Lemma} 
\newtheorem{assum}{Assumption} 
\newtheorem{corollary}{Corollary} 
\newtheorem{definition}{Definition} 
\newtheorem{example}{Example}

\definecolor{light-gray}{gray}{0.89}

\usepackage{multirow}
\usepackage{adjustbox}

\usepackage{threeparttable}
\usepackage[table,dvipsnames]{xcolor}
\usepackage{hhline}
\usepackage{makecell}

\usepackage{url}

\allowdisplaybreaks[4]

\renewcommand\typesetChapterTitle[1]{#1}

\begin{document}

%
%

\author{Kaiyi Ji}
\title{Bilevel Optimization for Machine Learning: Algorithm Design and Convergence Analysis}
\authordegrees{B.S.}  
\unit{Department of Electrical and Computer Engineering}

\advisorname{Prof. Yingbin Liang}
\member{Prof. Ness B. Shroff}
\member{Prof. Philip Schniter}    
\member{Prof. Cathy Xia}

%
%

\maketitle

\disscopyright

\begin{abstract}

Bilevel optimization has become a powerful framework in a variety  of  machine learning applications including signal processing, meta-learning, hyperparameter optimization, reinforcement learning and network architecture search. There are generally two classes of bilevel optimization formulations for modern machine learning: 1) problem-based bilevel optimization, whose inner-level problem is formulated as finding a minimizer of a given loss function; and 2) algorithm-based bilevel optimization, whose inner-level solution is an output of a fixed  algorithm. For the first problem class, two popular types of gradient-based algorithms have been proposed to estimate the gradient of the outer-level objective (hypergradient) via approximate implicit differentiation (AID) and iterative differentiation (ITD). Algorithms for the second problem class include the popular model-agnostic meta-learning (MAML) and almost no inner loop (ANIL). Although bilevel optimization algorithms have been widely used, their convergence rate and fundamental limitations have not been well explored.

In this thesis,  we provide a comprehensive theory for bilevel algorithms in the aforementioned two classes. We further propose enhanced and principled algorithm designs for bilevel optimization with higher efficiency and scalability in practice. For problem-based bilevel optimization, we first provide a comprehensive convergence theory for AID- and ITD-based algorithms for the nonconvex-strongly-convex setting. For the AID-based methods, we orderwisely improve the previous computational complexities, and for the ITD-based methods we establish the first theoretical convergence rate. Our analysis also provides a quantitative comparison between ITD- and AID-based methods. 
We further provide the theoretical guarantee for ITD- and AID-based methods in meta-learning. 

Second, we propose a new accelerated bilevel optimizer named AccBiO, 
for which we provide the first-known complexity bounds without the gradient boundedness assumption (which was made in existing analyses) respectively for strongly-convex-strongly-convex and convex-strongly-convex bilevel optimizations. Our analysis controls the finiteness of all iterates as the algorithm runs via an induction proof to ensure that the hypergradient estimation error will not explode after the acceleration steps. We also provide significantly tighter upper bounds than the existing complexity when the bounded gradient assumption does hold. 

We then provide the first-known lower bounds for strongly-convex-strongly-convex and convex-strongly-convex bilevel optimizations. We demonstrate the optimality of our results by showing that AccBiO achieves the optimal results (i.e., the upper and lower bounds match) up to logarithmic factors when the inner-level problem takes a quadratic form with a constant-level condition number. Interestingly, our lower bounds under both geometries are larger than the corresponding optimal complexities of minimax optimization, establishing that bilevel optimization is provably more challenging than minimax optimization.

We finally propose a novel stochastic bilevel optimization algorithm named stocBiO, which features a sample-efficient hypergradient estimator using efficient Jacobian- and Hessian-vector product computations. We provide the convergence rate guarantee for stocBiO, and show that stocBiO outperforms the best known computational complexities orderwisely with respect to the condition number $\kappa$ and the target accuracy $\epsilon$. We further validate our theoretical results and demonstrate the efficiency of stocBiO by the experiments on hyperparameter optimization.

For algorithm-based bilevel optimization, we first develop a new theoretical framework for analyzing MAML  for two types of objective functions that are of interest in practice: (a) resampling case (e.g., reinforcement learning), where loss functions take the form in expectation; and (b) finite-sum case (e.g., supervised learning), where loss functions take the finite-sum form with given samples. For both cases, we characterize the convergence rate and complexity to attain an $\epsilon$-accurate solution for multi-step MAML in the general nonconvex setting. In particular, our results suggest choosing the inner-stage stepsize to be inversely proportional to the number $N$ of inner-stage steps in order for $N$-step MAML to have guaranteed convergence. Technically, we develop novel techniques to deal with the nested structure of the meta gradient for multi-step MAML, which can be of independent interest.

 We then characterize the convergence rate and the computational complexity for ANIL under two representative inner-loop loss geometries, i.e., strongly-convexity and nonconvexity. Our results show that such a geometric property can significantly affect the overall convergence performance of ANIL. For example, ANIL achieves a faster convergence rate for a strongly-convex inner-loop loss as the number $N$ of inner-loop gradient descent steps increases, but a slower convergence rate for a nonconvex inner-loop loss as $N$ increases. Moreover, our complexity analysis provides a theoretical quantification on the improved efficiency of ANIL over MAML. The experiments on standard few-shot meta-learning benchmarks validate our theoretical findings. 
\end{abstract}

\dedication{Dedicated to my parents, girlfriend and beloved.}

\begin{acknowledgements}
First of all, I want to express my deepest gratitude to my advisors Prof.~Tan and Prof.~Liang for their great support on my Ph.D. study. I would not have finished this dissertation and grow up from a fresh Ph.D. to a mature researcher without their suggestions, supervision and supports. In my first two years, I worked closely with Prof.~Tan and learned a lot from him about how to come up with new research ideas, do critical thinking, and write a paper. His enthusiasm to research encouraged me to dive into all of my research projects in this two-year long journey. 
Since then, I have been in collaboration with Prof. Liang in a number of research projects, and these experiences have greatly broadened my research scope and taught me how to conduct independent research. At each time when I came up with new ideas or struggled with technical questions, I was always able to receive valuable suggestions and technical supports from the discussion with her. I am really grateful to my two advisors for their advice and huge help on my Ph.D. study.

I would like to thank Prof.~H.~Vincent Poor and Prof.~Jason D.~Lee for their valuable suggestions and instructions 
on my project during my visit in Princeton University. This period of study greatly enhance my ability to communicate and gave me the opportunity to explore interesting problems in meta-learning and bilevel optimization, which have become two important topics  along my research direction.

I also would like to thank my collaborators, Guocong Quan, Tengyu Xu, Junjie Yang, Ziwei Guan, Yi Zhou, Zhe Wang, Bowen Weng, Prof.~Yuejie Chi, Prof.~Jingfen Xu and Prof.~Ness B. Shroff for their  valuable suggestions and instructions on the writing, analysis and experiments of my papers. Their professional attitude and broad knowledge have impressed me so deeply and greatly broaden my research view for my academic career.

I would like to thank my defense committee: Prof.~Cathy Xia, Prof.~Philip Schniter and Prof. Ness B. Shroff   for their precious 
time and valuable comments on my dissertation. I am also grateful to my lab mates: Shaofeng Zou, Yi Zhou, Zhe Wang, Huaqing Xiong, Haoyu Fu, Tengyu Xu, Ziwei Guan, Junjie Yang, and Davis Sow for their great help in my research and life.

I would like to thank my parents for their selfless love and supports on my Ph.D. study in another country. I am grateful for the accompany of my girlfriend Jingyi during my Ph.D. study and especially the struggling time during COVID-19.

I would like to thank University Fellowship and Presidential Fellowship awarded by The Ohio State University to support my research and life. I am also very grateful for the strong software and hardware supports from Ohio Supercomputer Center.

I acknowledge the great support from the ECE department and grants NSF No.~1717060, NSF CCF-1801855, 
NSF CCF-1761506, NSF CCF-1801846, NSF ECCS-1818904, NSF CCF-1900145 and NSF CCF-1909291.
\end{acknowledgements} 

\begin{vita}

\dateitem{Sep 22nd, 1993}{Born - Zigong, China}

\dateitem{2016}{B.S., Electronic Engineering and Information Science, \\University of Science and Technology of China} 
\dateitem{2020}{Visiting student,\\ Princeton University}

\dateitem{2016-present}{Graduate Research Associate,\\
			 The Ohio State University.}

\begin{publist}


\noindent\textbf{Journal Publications}\newline

\pubitem{Y. Zhang, Y. Zhou, {\bf K. Ji}, M. Zavlanos.
	\newblock ``Improving the Convergence Rate of One-Point Zeroth-Order Optimization using Residual Feedback'.
	\newblock Accepted provisionally by {\em Automatica, 2021}.}

\pubitem{T. Xu, Y. Zhou, {\bf K. Ji}, Y. Liang.
	\newblock ``When Will Gradient Methods Converge to Max-margin Classifier under ReLU Models?''
	\newblock Accepted by {\em Stat, 2021}.}

\pubitem{{\bf K. Ji}, J. Yang, Y. Liang.
	\newblock ``Theoretical Convergence of Multi-Step Model-Agnostic Meta-Learning''.
	\newblock accepted by {\em Journal of Machine Learning Research (JMLR), 2021}.}

\pubitem{{\bf K. Ji}, Y. Zhou, Y. Liang.
	\newblock ``Understanding Estimation and Generalization Error of Generative Adversarial Networks''.
	\newblock Accepted by {\em IEEE Transactions on Information Theory (TIT), 2021}.}
	
\pubitem{{\bf K. Ji}, J. Tan, Y. Chi, J. Xu
	\newblock ``Learning Latent Features with Pairwise Penalties in Matrix Completion''.
	\newblock Accepted by {\em IEEE Transactions on Signal Processing (TSP), 2020}.}
	
\pubitem{J. Tan, G. Quan, {\bf K. Ji}, N. Shroff.
	\newblock ``On Resource Pooling and Separation for LRU Caching''.
	\newblock  In {\em PACM on Measurement and Analysis of Computing Systems, 2018}.}

\noindent\textbf{Conference Publications}\newline

\pubitem{{\bf K. Ji}, J. Yang, Y. Liang.
	\newblock ``Bilevel Optimization: Nonasymptotic Analysis and Enhanced Design''.
	\newblock {\em In Proc. International Conference on Machine Learning (ICML), 2021.}}

\pubitem{{\bf K. Ji}, J. Lee, Y. Liang, H. Poor.
	\newblock ``Convergence of Meta-Learning with Task-Specific Adaptation over Partial Parameters''.
	\newblock {\em In Proc. Neural Information Processing Systems (NeurIPS), 2020.}}

\pubitem{ {\bf K. Ji}, Z. Wang,  Y. Zhou, Y. Liang.
	\newblock ``History-Gradient Aided Batch Size Adaptation for Variance Reduced Algorithms''.
	\newblock {\em In Proc. International Conference on Machine Learning (ICML), 2020.}}

\pubitem{Y. Zhou, Z. Wang,  {\bf K. Ji}, Y. Liang.
	\newblock ``Proximal Gradient Algorithm with Momentum and Flexible Parameter Restart for Nonconvex Optimization''.
	\newblock {\em International Joint Conference on Artificial Intelligence (IJCAI), 2020.}}

\pubitem{{\bf K. Ji}, J. Tan, Y. Chi, J. Xu.
	\newblock ``Learning Latent Features with Pairwise Penalties in Matrix Completion''.
	\newblock {\em IEEE Sensor Array and Multichannel Signal Processing Workshop (SAM), 2020.}}

\pubitem{Z. Guan, {\bf K. Ji},  D. Bucci Jr, T. Hu, J. Palombo, M. Liston, Y. Liang.
	\newblock ``Robust Stochastic Bandit Algorithms under Probabilistic Unbounded Adversarial Attack''.
	\newblock {\em In AAAI Conference on Artificial Intelligence (AAAI),  2020.}}

\pubitem{{\bf K. Ji}, Z. Wang,  Y. Zhou, Y. Liang.
	\newblock ``Improved Zeroth-Order Variance Reduced Algorithms and Analysis for Nonconvex Optimization''.
	\newblock {\em  In Proc. International Conference on Machine Learning (ICML), 2019.}}

\pubitem{Z. Wang,  {\bf K. Ji},  Y. Zhou, Y. Liang, V. Tarokh.
	\newblock ``SpiderBoost and Momentum: Faster Stochastic Variance Reduction Algorithms''.
	\newblock {\em In Proc. Neural Information Processing Systems (NeurIPS), 2019.}}

\pubitem{\textbf{K. Ji}, Y. Liang.
	\newblock ``Minimax Estimation of Neural Net Distance''.
	\newblock {\em In Proc. Neural Information Processing Systems (NeurIPS), 2018.}}

\pubitem{{\bf K. Ji}, G. Quan, J. Tan.
	\newblock ``Miss Ratio for LRU Caching with Consistent Hashing''.
	\newblock {\em In IEEE International Conference on Computer Communications (INFOCOM), 2018}}

\pubitem{G. Quan, {\bf K. Ji}, J. Tan.
	\newblock ``LRU Caching with Dependent Competing Requests''.
	\newblock {\em In IEEE International Conference on Computer Communications (INFOCOM), 2018}}

\pubitem{J. Tan, G. Quan, {\bf K. Ji}, N. Shroff.
	\newblock ``On Resource Pooling and Separation for LRU Caching''.
	\newblock {\em In ACM Special Interest Group on Measurement and Evaluation (SIGMETRICS), 2018.}}

\end{publist}

\newpage

\begin{fieldsstudy}

\noindent Major fields: Electrical and Computer Engineering. \newline
\noindent Concentrations: machine learning, optimization, networking.



\end{fieldsstudy}

\end{vita}

\tableofcontents
\listoftables
\listoffigures

%
%

\chapter{Introduction}\label{intro.ch}
Bilevel optimization has  received significant attention recently and become an influential framework in signal processing~\cite{kunapuli2008classification,flamary2014learning}, meta-learning~\cite{franceschi2018bilevel,bertinetto2018meta,rajeswaran2019meta,ji2020convergence}, hyperparameter optimization~\cite{franceschi2018bilevel,shaban2019truncated,feurer2019hyperparameter}, reinforcement learning~\cite{konda2000actor,hong2020two} and network architecture search~\cite{liu2018darts,he2020milenas}. Bilevel optimization for modern machine learning takes two major formulations: 1) problem-based bilevel optimization, whose inner-level problem is formulated as finding a minimizer of a given objetive function; and 2) algorithm-based bilevel optimization, whose inner-level problem is find the $N$-step output of a fixed iterative algorithm such as gradient descent. 
 The first class of problems occur in various applications including meta-learning with shared embedding model~\cite{bertinetto2018meta}, hyperparameter optimization via implicit differentiation~\cite{domke2012generic,pedregosa2016hyperparameter}, reinforcement learning~\cite{hong2020two} and network architecture search~\cite{he2020milenas}. 
Two popular types of gradient-based algorithms have been proposed to estimate the gradient of the outer-level objective (hypergradient) via approximate implicit differentiation (AID) and iterative differentiation (ITD). The second class of problems are often involved in application such as meta-initialization learning~\cite{finn2017model,grant2018recasting,mi2019meta,collins2020distribution} and hyperparameter optimization via dynamic system~\cite{franceschi2018bilevel}. 
Algorithms for this problem class include the popular model-agnostic meta-learning (MAML)~\cite{finn2017model}and meta-learning with task-specific adaptation on partial parameters such as almost no inner loop (ANIL)~\cite{raghu2020rapid}. Although bilevel optimization algorithms have been widely used in practice, their convergence rate analysis and fundamental limitations have not been well explored.
In addition, with the advent of large-scale neural networks and datasets, it is increasingly important to design more efficient bilevel optimization methods.

This thesis provides a comprehensive nonasymptotic analysis for bilevel algorithms in the aforementioned two classes. We further propose enhanced and principled algorithm designs for bilevel optimization with higher efficiency and scalability in applications such as meta-learning and hyperparameter optimization. In specific, for the problem-based bilevel optimization, we first provide a comprehensive convergence rate analysis for AID- and ITD-based bilevel optimization algorithms. We then develop acceleration algorithms for bilevel optimization, for which we provide novel convergence analysis with relaxed assumptions and significantly lower complexity. We also provide the first lower bounds for bilevel optimization, and establish the optimality by providing matching upper bounds under certain conditions. We finally propose new stochastic bilevel optimization algorithms with lower computational complexity and higher efficiency in practice.  For the algorithm-based formulation, we develop a theoretical convergence for general multi-step MAML for the resampling and finite-sum cases. We then analyze the convergence for meta-learning with task-specific adaptation on partial parameters, and characterize the impact of parameter selections and loss geometries on the complexity.  
   
In the following, we summarize our specific motivations and main contributions of the above studies sequentially. 
\section{Convergence for Problem-Based Bilevel Optimization}
 A general problem-based bilevel optimization takes the following formulation. 
\begin{align}\label{objective_deter}
&\min_{x\in\mathbb{R}^{p}} \Phi(x):=f(x, y^*(x)) \nonumber
\\&\;\;\mbox{s.t.} \quad y^*(x)= \argmin_{y\in\mathbb{R}^{q}} g(x,y),
\end{align}
where the upper- and inner-level functions $f$ and $g$ are both jointly continuously differentiable. The goal of~\cref{objective_deter} is to minimize the objective function $\Phi(x)$ with respect to (w.r.t.)~$x$, where $y^*(x)$ is obtained by solving the lower-level minimization problem. In this thesis, we focus on the setting where the lower-level function $g$ is strongly convex w.r.t.~$y$, and the upper-level objective function $\Phi(x)$ is nonconvex. Such geometrics commonly exist in many applications such as meta-learning and hyperparameter optimization, where $g$ corresponds to an empirical loss  with a strongly-convex regularizer and $x$ are parameters of neural networks.

A broad collection of algorithms have been proposed to solve bilevel optimization problems. For example, \cite{hansen1992new,shi2005extended,moore2010bilevel} reformulated the bilevel problem in~\cref{objective_deter} into a single-level constrained problem based on the optimality conditions of the lower-level problem. However, such type of methods often involve  a large number of constraints, and  are hard to implement in machine learning applications. Recently, more efficient gradient-based bilevel optimization algorithms have been proposed, which can be generally categorized into the approximate implicit differentiation (AID) based approach~\cite{domke2012generic,pedregosa2016hyperparameter,gould2016differentiating,liao2018reviving,ghadimi2018approximation,grazzi2020iteration,lorraine2020optimizing} and the iterative differentiation (ITD) based approach~\cite{domke2012generic,maclaurin2015gradient,franceschi2017forward,franceschi2018bilevel,shaban2019truncated,grazzi2020iteration}. However, most of these studies have focused on the asymptotic convergence analysis, and the nonasymptotic convergence rate analysis (that characterizes how fast an algorithm converges) has not been well explored except a few attempts recently. \cite{ghadimi2018approximation} provided the convergence rate analysis for the AID-based approach. \cite{grazzi2020iteration} provided the iteration complexity for the hypergradient computation via ITD and AID, but did not characterize the  convergence rate for the entire execution of algorithms. Thus, the first focus of this thesis is to develop a {\em comprehensive and sharper} theory, which covers a broader class of bilevel optimizers via ITD and AID techniques, and more importantly, improves existing  analysis with a more practical parameter selection and orderwisely lower computational complexity.

\begin{table*}[!t]
 \centering
 \caption{Comparison of bilevel deterministic optimization algorithms.}
 \vspace{0.1cm}
 \begin{threeparttable}
  \begin{tabular}{|c|c|c|c|c|c|}
   \hline
Algorithm & Gc($f,\epsilon$) & Gc($g,\epsilon$) & JV($g,\epsilon$) &  HV($g,\epsilon$)    \\\hline\hline
AID-BiO \cite{ghadimi2018approximation}  & $\mathcal{O}(\kappa^4\epsilon^{-1})$ & $\mathcal{O}(\kappa^5\epsilon^{-5/4})$ &  $\mathcal{O}\left(\kappa^4\epsilon^{-1}\right)$& $\mathcal{\widetilde O}\left(\kappa^{4.5}\epsilon^{-1}\right)$ \\ \hline
   \cellcolor{blue!15}{AID-BiO (this thesis)} & \cellcolor{blue!15}{$\mathcal{O}(\kappa^3\epsilon^{-1})$} & \cellcolor{blue!15}{$\mathcal{ O}(\kappa^4\epsilon^{-1})$ } & \cellcolor{blue!15}{$\mathcal{O}\left(\kappa^{3}\epsilon^{-1}\right)$} & \cellcolor{blue!15}{ $\mathcal{ O}\left(\kappa^{3.5}\epsilon^{-1}\right)$} \\ \hline
      \cellcolor{blue!15}{ITD-BiO (this thesis)} & \cellcolor{blue!15}{$\mathcal{O}(\kappa^3\epsilon^{-1})$} & \cellcolor{blue!15}{$\mathcal{\widetilde O}(\kappa^4\epsilon^{-1})$ } & \cellcolor{blue!15}{$\mathcal{\widetilde O}\left(\kappa^4\epsilon^{-1}\right)$} & \cellcolor{blue!15}{ $\mathcal{\widetilde O}\left(\kappa^4\epsilon^{-1}\right)$} \\ \hline
  \end{tabular}\label{tab:determinstic}
  \vspace{0.1cm}
        {\small
   \begin{tablenotes}
  \item $\mbox{\normalfont Gc}(f,\epsilon)$ and $\mbox{\normalfont Gc}(g,\epsilon)$: 
 number of gradient evaluations w.r.t.~$f$ and $g$. 
 \item $\mbox{\normalfont JV}(g,\epsilon)$: number of Jacobian-vector products $\nabla_x\nabla_y g(x,y)v$. 
  \item  $\mbox{\normalfont HV}(g,\epsilon)$: number of Hessian-vector products $\nabla_y^2g(x,y) v$.
   \item $\kappa:$ condition number. \; Notation $\mathcal{\widetilde O}$: omit $\log\frac{1}{\epsilon}$ terms.
   \end{tablenotes}}
 \end{threeparttable}
  \vspace{-0.6cm}
\end{table*}

\vspace{0.2cm}
\noindent {\bf Main Contributions.} Our main contributions lie in developing a shaper theory for the nonconvex-strongly-convex bilevel optimization problem. 

We first provide a unified convergence rate and complexity analysis for both ITD and  AID based bilevel optimizers, which we call as ITD-BiO and AID-BiO. Compared to existing analysis in~\cite{ghadimi2018approximation} for AID-BiO that requires a continuously increasing number of inner-loop steps to achieve the guarantee, our analysis allows a constant number of inner-loop steps as often used in practice. In addition, we introduce a warm start initialization for the inner-loop updates and the outer-loop hypergradient estimation,  which allows us to backpropagate the tracking errors to previous loops, and yields an improved computational complexity. \Cref{tab:determinstic} shows that the gradient complexities Gc($f,\epsilon$), Gc($g,\epsilon$), and Jacobian- and Hessian-vector product complexities JV($g,\epsilon$) and HV($g,\epsilon$) of AID-BiO to attain an $\epsilon$-accurate stationary point improve those of~\cite{ghadimi2018approximation} by the order of $\kappa$, $\kappa\epsilon^{-1/4}$, $\kappa$, and $\kappa$, respectively, where $\kappa$ is the condition number. Our analysis also shows that AID-BiO requires less computations of Jacobian- and Hessian-vector products than ITD-BiO by an order of $\kappa$ and $\kappa^{1/2}$. Our results further provide the theoretical guarantee for AID-BiO and ITD-BiO in meta-learning. 

\section{Acceleration for Problem-Based Bilevel Optimization}\label{intro:sec:acc}
The {\em finite-time} (convergence) analysis of problem-based bilevel optimization algorithms has been studied recently. \cite{grazzi2020iteration} provided the iteration complexity for hypergradient approximation with ITD and AID. \cite{ghadimi2018approximation} proposed an AID-based bilevel approximation (BA) algorithm as well as an accelerated variant ABA, and analyzed their finite-time complexities under different loss geometries. In particular, the complexity upper bounds of BA  and ABA are given by $\mathcal{\widetilde O}(\frac{1}{\mu_y^{6}\mu_x^{2}})$ and $\mathcal{ \widetilde O}(\frac{1}{\mu_y^{3}\mu_x })$ for the strongly-convex-strongly-convex setting where $\Phi(\cdot)$ is $\mu_x$-strongly-convex and $g(x,\cdot)$ is $\mu_y$-strongly-convex,  $\mathcal{O}\big(\frac{1}{\mu_y^{11.25}\epsilon^{1.25}}\big)$ and $\mathcal{ O}\big(\frac{1}{\mu_y^{6.75}\epsilon^{0.75}} \big)$ for the convex-strongly-convex setting, and $\mathcal{ O}\big(\frac{1}{\mu_y^{6.25}\epsilon^{1.25}} \big)$ for the nonconvex-strongly-convex setting.  \cite{ji2020bilevel} further improved the bound for the nonconvex-strongly-convex setting to $\mathcal{O}\big(\frac{1}{\mu_y^{4}\epsilon }\big)$. However, these analyses reply on a strong assumption on the boundedness of the outer-level gradient $\nabla_y f(x,\cdot)$\footnote{\cite{grazzi2020iteration} assume the inner-problem solution $y^*(x)$ is uniformly bounded for all $x$ so that $\nabla_y f(x,y^*(x))$ is bounded.} 
to guarantee that the smoothness parameter of $\Phi(\cdot)$ and the hyperparameter estimation error are bounded as the algorithm runs. Then the following question needs to be adressed. 

\begin{list}{$\bullet$}{\topsep=0.1in \leftmargin=0.2in \rightmargin=0.1in \itemsep =0.01in}
 \item[1.] \textit{Can we design a new acceleration bilevel optimization algorithm, which provably converges without the gradient boundedness?}
 \end{list} 
 In addition, even when the boundedness assumption holds, existing complexity bounds show pessimistic dependences on the condition numbers, e.g., $\mathcal{O}(\frac{1}{\mu_y^{6.75}})$ for the convex-strongly-convex case. Then, the following question arises. 
 \begin{list}{$\bullet$}{\topsep=0.1in \leftmargin=0.2in \rightmargin=0.1in \itemsep =0.01in}
 \item[2.] \textit{Under the bounded gradient assumption, can we provide new upper bounds with tighter dependences on the condition numbers for strongly-convex-strongly-convex and convex-strongly-convex
 bilevel optimizations?}
 \end{list} 
 In this thesis, we provide affirmative answers to the above questions.

 \begin{table*}[!t]
\renewcommand{\arraystretch}{1.1}
\centering
\small
\caption{Comparison of complexities for finding an $\epsilon$-approximate point without the gradient boundedness assumption. All listed results are from this thesis. 
}
\label{tab:results}

\vspace{0.3cm}
\begin{tabular}{|c|c|c|} \hline
 \textbf{Type} & \textbf{References} & \textbf{Computational Complexity} \\ \hline 
\multirow{2}{*}{\shortstack{SCSC}} 
&  \textbf{AccBiO} (\Cref{upper_srsr_withnoB}) &{\scriptsize $\mathcal{\widetilde O}\Big(\sqrt{\frac{\widetilde L_y}{\mu_x\mu_y^{3}}}+\Big(\sqrt{\frac{ \rho_{yy}\widetilde L_y}{\mu_x\mu_y^{4}} }+ \sqrt{ \frac{\rho_{xy}\widetilde L_y}{\mu_x\mu_y^{3}}}\Big)\sqrt{\Delta^*_{\text{\normalfont \tiny SCSC}}}\Big)$} \\ \cline{2-3}
& \textbf{AccBiO} (quadratic $g$, \Cref{coro:quadaticSr}) &  {\scriptsize $ \mathcal{\widetilde O}\Big(\sqrt{\frac{\widetilde L_y}{\mu_x\mu_y^{3}}}\Big)$} \\ \cline{2-3} \hline
\multirow{2}{*}{\shortstack{CSC} }  
& \textbf{AccBiO} ( \Cref{th:upper_csc1sc})& {\scriptsize
$\mathcal{\widetilde O}\Big( \sqrt{\frac{\widetilde L_y}{\epsilon\mu_y^3}}+\Big(\sqrt{\frac{\rho_{yy}\widetilde L_y}{\epsilon\mu_y^{4}}} +  \sqrt{\frac{\rho_{xy}\widetilde L_y}{\epsilon\mu_y^3}}\Big)\sqrt{\Delta^*_{\text{\normalfont\tiny CSC}}}\Big)$
}\\ \cline{2-3}
&\textbf{AccBiO} (quadratic $g$, \Cref{coro:quadaticConv})& {\scriptsize$\mathcal{\widetilde O}\Big(\sqrt{\frac{\widetilde L_y}{\epsilon\mu_y^3}}\Big)$}\\ \cline{2-3} \hline
\end{tabular}
\vspace{0.2cm}
 {\small
\begin{list}{$\bullet$}{\topsep=0.1in \leftmargin=0.2in \rightmargin=0.1in \itemsep =0.01in}
\item[*] The complexity is measured by $\tau (n_J+n_H) + n_G$ (\Cref{complexity_measyre}), where $n_G, n_J,n_H$ are the numbers of gradients, Jacobian- and Hessian-vector products, and $\tau$ is a universal constant. 
In the references column, quadratic $g(x,y)$ means that $g$ takes a quadratic form as {\scriptsize $g(x,y)=y^T H y +  x^T J y + b^Ty+h(x)$} for the constant matrices $H,J$ and a constant vector $b$.
In the computational complexity column, $\widetilde L_y$ denotes the smoothness parameter of $g(x,\cdot)$, $\rho_{xy}$ and $\rho_{yy}$ are the Lipschitz parameters of {\scriptsize $\nabla_y^2g(\cdot,\cdot)$} and {\small$\nabla_x\nabla_yg(\cdot,\cdot)$} (see \cref{def:three}), {\scriptsize $\Delta^*_{\text{\normalfont\tiny SCSC}}=\|\nabla_y f( x^*,y^*(x^*))\|+\frac{\|x^*\|}{\mu_y}+\frac{\sqrt{\Phi(0)-\Phi(x^*)}}{\sqrt{\mu_x}\mu_y}$} ($\Delta^*_{\text{\normalfont\tiny CSC}}$ takes the same form as $\Delta^*_{\text{\normalfont\tiny SCSC}}$ but with $\mu_x$ replaced by {\scriptsize $\frac{\epsilon}{(\|x^*\|+1)^2}$}). 
\end{list}
}
\vspace{-0.5cm}
\end{table*}

\begin{table*}[h]
\renewcommand{\arraystretch}{1.1}
\centering
\caption{Comparison of computational complexities for finding an $\epsilon$-approximate point with the gradient boundedness assumption. }
\label{tab:results_withB}

\vspace{0.3cm}
\begin{tabular}{|c|c|c|} \hline
 \textbf{Type} & \textbf{References} & \textbf{Computational Complexity} \\ \hline 
\multirow{3}{*}{\shortstack{SCSC}} 
&BA~\cite{ghadimi2018approximation} & {\small $\mathcal{\widetilde O}\Big(\max\Big\{\frac{1}{\mu_x^2\mu_y^6},\frac{\widetilde L^2_y}{\mu^2_y}\Big\}\Big)$}  \\ \cline{2-3}
&ABA~\cite{ghadimi2018approximation} & {\small $\mathcal{\widetilde O}\Big(\max\Big\{\frac{1}{\mu_x\mu_y^3},\frac{\widetilde L^2_y}{\mu^2_y}\Big\}\Big)$}  \\ \cline{2-3}& \cellcolor{light-gray} \textbf{AccBiO-BG} (this thesis, \Cref{upper_srsr}) & \cellcolor{light-gray} {\small $ \mathcal{\widetilde O}\Big(\sqrt{\frac{\widetilde L_y}{\mu_x\mu_y^4}}\Big)$} \\ \hline \hline
\multirow{3}{*}{\shortstack{CSC} }  
& BA~\cite{ghadimi2018approximation} & {\small $\mathcal{\widetilde O}\Big(\frac{1}{\epsilon^{1.25}}\max\Big\{\frac{1}{\mu_y^{3.75}},\frac{\widetilde L^{10}_y}{\mu^{11.25}_y}\Big\}\Big)$} \\ \cline{2-3}
&ABA~\cite{ghadimi2018approximation} &{\small $\mathcal{\widetilde O}\Big(\frac{1}{\epsilon^{0.75}}\max\Big\{\frac{1}{\mu_y^{2.25}},\frac{\widetilde L^{6}_y}{\mu^{6.75}_y}\Big\}\Big)$}  \\ \cline{2-3}
& \cellcolor{light-gray} \textbf{AccBiO-BG} (this thesis, \Cref{convex_upper_BG})& \cellcolor{light-gray} {\small 
$\mathcal{\widetilde O}\Big( \sqrt{\frac{\widetilde L_y}{\epsilon\mu_y^4}}\Big)$
}\\ 
\hline 
\end{tabular}
\end{table*}

 \vspace{0.2cm}
 \noindent{\bf Main Contributions.}
 We first propose a new accelerated bilevel optimizer named AccBiO. 
 In contrast to existing bilevel optimizers, we show that AccBiO converges to the $\epsilon$-accurate solution  without the requirement on the boundedness of the gradient {\small $\nabla_y f(x,\cdot)$} for any $x$. For the  strongly-convex-strongly-convex bilevel optimization, \Cref{tab:results} shows that AccBiO achieves an upper complexity bound of {\footnotesize $\mathcal{\widetilde O}\Big(\sqrt{\frac{\widetilde L_y}{\mu_x\mu_y^{3}}}+\Big(\sqrt{\frac{ \rho_{yy}\widetilde L_y}{\mu_x\mu_y^{4}} }+ \sqrt{ \frac{\rho_{xy}\widetilde L_y}{\mu_x\mu_y^{3}}}\Big)\sqrt{\Delta^*_{\text{\normalfont\tiny SCSC}}}\Big)$}.  When the inner-level function $g(x,y)$ takes the quadratic form as {\footnotesize $g(x,y)=y^T H y +  x^T J y + b^Ty+h(x)$}, we further improve the upper bounds to {\footnotesize$ \mathcal{\widetilde O}\Big(\sqrt{\frac{\widetilde L_y}{\mu_x\mu_y^{3}}}\Big)$}. 
For the convex-strongly-convex bilevel optimization, AccBiO achieves an upper bound of {\footnotesize 
$\mathcal{\widetilde O}\Big( \sqrt{\frac{\widetilde L_y}{\epsilon\mu_y^3}}+\Big(\sqrt{\frac{\rho_{yy}\widetilde L_y}{\epsilon\mu_y^{4}}} +  \sqrt{\frac{\rho_{xy}\widetilde L_y}{\epsilon\mu_y^3}}\Big)\sqrt{\Delta^*_{\text{\normalfont\tiny CSC}}}\Big)$}, which is further improved to  {\footnotesize$\mathcal{\widetilde O}\Big(\sqrt{\frac{\widetilde L_y}{\epsilon\mu_y^3}}\Big)$} for the quadratic $g(x,y)$.
Technically, our analysis controls the finiteness of all iterates $x_k,k=0,....$ as the algorithm runs via an induction proof to ensure that the hypergradient estimation error will not explode after the acceleration steps. 
 
Furthermore, when the gradient $\nabla_y f(x,\cdot)$ is bounded, as assumed by existing studies, we provide new upper bounds with significantly tighter dependence on the condition numbers. In specific, as shown in \Cref{tab:results_withB}, our upper bounds outperform the best known results by a factor of $\frac{1}{\mu_x^{0.5}\mu_y}$ and $\frac{1}{\epsilon^{0.25}\mu_y^{4.75}}$ for the strongly-convex-strongly-convex and  convex-strongly-convex 
cases, respectively. 

\section{Lower Bounds for Problem-Based Bilevel Optimization}
Although recent studies have characterized the convergence rate for several problem-based bilevel optimization algorithms, as shown in \Cref{intro:sec:acc}, it is still unclear how much further these convergence rates can be improved. Furthermore, existing complexity results on  {\bf bilevel} optimization are much worse than those on {\bf minimax} optimization, which is a special case of bilevel optimization with $f(x,y)=g(x,y)$. For example, for the convex-strongly-convex case, it was shown in~\cite{lin2020near}  that the optimal complexity for {\bf minimax} optimization is given by $\mathcal{\widetilde O}\big(\frac{1}{\epsilon^{0.5}\mu_y^{0.5}}\big)$, which is much smaller than the best known $\mathcal{ \widetilde O}\big(\frac{1}{\mu_y^{6.75}\epsilon^{0.75} }\big)$ for {\bf bilevel} optimization. Similar observations hold for the strongly-convex-strongly-convex setting. Therefore, the following fundamental questions arise and need to be addressed.
\begin{list}{$\bullet$}{\topsep=0.1in \leftmargin=0.2in \rightmargin=0.1in \itemsep =0.01in}
 \item[1.] \textit{What is the performance limit of bilevel optimization in terms of computational complexity? Whether bilevel optimization is provably more challenging (i.e., requires more computations) than minimax optimization?}
\item[2.] \textit{Can we establish near-optimal bilevel algorithms under certain conditions? }
 \end{list}
In this thesis, we provide confirmative answers to the above questions. 

 \begin{table*}[!t]
\renewcommand{\arraystretch}{1.1}
\centering
\small
\caption{Comparison of upper and lower bounds for finding an $\epsilon$-approximate point. All listed results come form this thesis.}
\label{tab:results_up_low}

\vspace{0.3cm}
\begin{tabular}{|c|c|c|} \hline
 \textbf{Type} & \textbf{References} & \textbf{Computational Complexity} \\ \hline 
\multirow{3}{*}{\shortstack{SCSC}} 
&  \textbf{AccBiO} (\Cref{upper_srsr_withnoB}) &{\scriptsize $\mathcal{\widetilde O}\Big(\sqrt{\frac{\widetilde L_y}{\mu_x\mu_y^{3}}}+\Big(\sqrt{\frac{ \rho_{yy}\widetilde L_y}{\mu_x\mu_y^{4}} }+ \sqrt{ \frac{\rho_{xy}\widetilde L_y}{\mu_x\mu_y^{3}}}\Big)\sqrt{\Delta^*_{\text{\normalfont \tiny SCSC}}}\Big)$} \\ \cline{2-3}
& \textbf{AccBiO} (quadratic $g$, \Cref{coro:quadaticSr}) &  {\scriptsize $ \mathcal{\widetilde O}\Big(\sqrt{\frac{\widetilde L_y}{\mu_x\mu_y^{3}}}\Big)$} \\ \cline{2-3} 
&  \cellcolor{blue!15} \textbf{Lower bound} (\Cref{thm:low1}) & \cellcolor{blue!15} {$\widetilde{\Omega}\big(\sqrt{\frac{1}{\mu_x\mu_y^2}}\big)$} \\ \hline \hline
\multirow{4}{*}{\shortstack{CSC} }  
& \textbf{AccBiO} ( \Cref{th:upper_csc1sc})& {\scriptsize
$\mathcal{\widetilde O}\Big( \sqrt{\frac{\widetilde L_y}{\epsilon\mu_y^3}}+\Big(\sqrt{\frac{\rho_{yy}\widetilde L_y}{\epsilon\mu_y^{4}}} +  \sqrt{\frac{\rho_{xy}\widetilde L_y}{\epsilon\mu_y^3}}\Big)\sqrt{\Delta^*_{\text{\normalfont\tiny CSC}}}\Big)$
}\\ \cline{2-3}
&\textbf{AccBiO} (quadratic $g$, \Cref{coro:quadaticConv})& {\scriptsize$\mathcal{\widetilde O}\Big(\sqrt{\frac{\widetilde L_y}{\epsilon\mu_y^3}}\Big)$}\\ \cline{2-3} 
& \cellcolor{blue!15} \textbf{Lower bound} $($\Cref{co:co1}, {\scriptsize$\widetilde L_y\leq \mathcal{O}(\mu_y)$}$)$& \cellcolor{blue!15} {$\widetilde \Omega \Big(\sqrt{\frac{1}{\epsilon\mu_y^2}}\Big)$}\\ \cline{2-3}
& \cellcolor{blue!15} \textbf{Lower bound} $($\Cref{co:co2}, {\scriptsize $\widetilde L_y\leq \mathcal{O}(1)$}$)$& \cellcolor{blue!15} {\small $\widetilde \Omega\big(\epsilon^{-0.5}\min\{\mu^{-1}_y,\epsilon^{-1.5}\}\big)$}\\ \hline 
\end{tabular}
\vspace{-0.5cm}
\end{table*}

\vspace{0.2cm}
\noindent{\bf Main Contributions.}  We provide the first-known lower bound of {\small $\widetilde{\Omega}(\frac{1}{\sqrt{\mu_x}\mu_y})$} for solving the strongly-convex-strongly-convex bilevel optimization.  
When the inner-level function $g(x,y)$ takes the quadratic form as {\small $g(x,y)=y^T H y +  x^T J y + b^Ty+h(x)$} with {\small $\widetilde L_y\leq \mathcal{O}(\mu_y)$}, the upper bound achieved by our AccBiO in \Cref{intro:sec:acc}
matches the lower bound up to logarithmic factors, suggesting that AccBiO is near-optimal.   Technically, our analysis of the lower bound involves careful construction of quadratic $f$ and $g$ functions with a properly structured bilinear term, as well as novel characterization of subspaces of iterates for updating $x$ and $y$. 

 We next provide a lower bound for solving convex-strongly-convex bilevel optimization. 
For the quadratic $g(x,y)$ with $\widetilde L_y\leq \mathcal{O}(\mu_y)$, the upper bound achieved by AccBiO matches the lower bound up to logarithmic factors, suggesting the optimality of AccBiO. Technically, the analysis of the lower bound is different from that for the strongly-convex $\Phi(\cdot)$, and exploits the structures of different powers of an unnormalized graph Laplacian matrix $Z$.

To compare between bilevel optimization and minimax optimization, for the strongly-convex-strongly-convex case, our lower bound is larger than the optimal complexity of {\small$\widetilde{\Omega}(\frac{1}{\sqrt{\mu_x\mu_y}})$} for the same type of minimax optimization by a factor of $\frac{1}{\sqrt{\mu_y}}$. Similar observation holds for the convex-strongly-convex case. This establishes that bilevel optimization is fundamentally more challenging than minimax optimization.   

\section{Stochastic Bilevel Optimization}
The {\em stochastic} problem-based bilevel optimization often occurs    
in  applications where fresh data are sampled for algorithm iterations (e.g., in reinforcement learning~\cite{hong2020two}) or the sample size of training data is large (e.g., hyperparameter optimization~\cite{franceschi2018bilevel}, Stackelberg game~\cite{roth2016watch}). Typically, the  objective function is given by 
\begin{align}\label{objective}
&\min_{x\in\mathbb{R}^{p}} \Phi(x)=f(x, y^*(x))= 
\begin{cases}
\frac{1}{n}{\sum_{i=1}^nF(x,y^*(x);\xi_i) } \\
\mathbb{E}_{\xi} \left[F(x,y^*(x);\xi)\right] 
\end{cases}\nonumber\\
& \;\mbox{s.t.} \;y^*(x)= \argmin_{y\in\mathbb{R}^q} g(x,y)=
\begin{cases}
\frac{1}{m}{\sum_{i=1}^{m} G(x,y;\zeta_i)} \\
\mathbb{E}_{\zeta} \left[G(x,y;\zeta)\right]
\end{cases}
\end{align}
where $f(x,y)$ and $g(x,y)$ take either the expectation form w.r.t. the random variables $\xi$ and $\zeta$ or the finite-sum form over given 
data $\gD_{n,m}=\{\xi_i,\zeta_j, i=1,...,n;j=1,...,m\}$ often with large sizes $n$ and $m$. During the optimization process, data batch is sampled via the distributions of $\xi$ and $\zeta$ or from the set $\gD_{n,m}$. For such a stochastic setting, \cite{ghadimi2018approximation} proposed a bilevel stochastic approximation (BSA) method via single-sample gradient and Hessian estimates.  Based on such a method, \cite{hong2020two} further proposed a two-timescale stochastic approximation (TTSA) algorithm, and showed that TTSA achieves a better trade-off between the complexities of inner- and outer-loop optimization stages than BSA. 
{\em Then, the second focus of this thesis is to design a more sample-efficient algorithm for bilevel  stochastic optimization, which is easy to implement, uses efficient Jacobian- and Hessian-vector product computations,  and 
achieves lower computational complexity by orders of magnitude than BSA and TTSA.}

\begin{table*}[!t]
\vspace{-0.25cm}
 \centering
 \caption{Comparison of bilevel stochastic optimization algorithms.}
\small
 \vspace{0.1cm}
 \begin{threeparttable}
  \begin{tabular}{|c|c|c|c|c|c|}
   \hline
Algorithm & Gc($F,\epsilon$) & Gc($G,\epsilon$) & JV($G,\epsilon$) &  HV($G,\epsilon$)  
\\\hline\hline
TTSA \cite{hong2020two}  &$ \mathcal{O}(\text{\scriptsize poly}(\kappa)\epsilon^{-\frac{5}{2}})$\tnote{*} & $\mathcal{O}(\text{\scriptsize poly}(\kappa)\epsilon^{-\frac{5}{2}})$& $\mathcal{O}(\text{\scriptsize poly}(\kappa)\epsilon^{-\frac{5}{2}})$&$\mathcal{O}(\text{\scriptsize poly}(\kappa)\epsilon^{-\frac{5}{2}})$
\\ \hline
BSA \cite{ghadimi2018approximation}  & $\mathcal{O}(\kappa^6\epsilon^{-2})$ & $\mathcal{O}(\kappa^9\epsilon^{-3})$ &  $\mathcal{O}\left(\kappa^6\epsilon^{-2}\right)$& $\mathcal{\widetilde O}\left(\kappa^6\epsilon^{-2}\right)$
\\ \hline
   \cellcolor{blue!15}{stocBiO (this thesis)} & \cellcolor{blue!15}{$\mathcal{O}(\kappa^5\epsilon^{-2})$} & \cellcolor{blue!15}{$\mathcal{O}(\kappa^9\epsilon^{-2})$ } & \cellcolor{blue!15}{$\mathcal{ O}\left(\kappa^5\epsilon^{-2}\right)$} & \cellcolor{blue!15}{$\mathcal{\widetilde O}\left(\kappa^6\epsilon^{-2}\right)$} \\ \hline 
  \end{tabular}\label{tab:stochastic}
   \begin{tablenotes}
  \item[*] We use $\text{poly}(\kappa)$ because \cite{hong2020two} does not provide the explicit dependence on $\kappa$.
 \end{tablenotes}
 \end{threeparttable}
 \vspace{-0.5cm}
\end{table*}

\vspace{0.2cm}
\noindent{\bf Main Contributions.} In this thesis, we propose a stochastic bilevel optimizer (stocBiO) to solve the stochastic bilevel optimization problem in~\cref{objective}. Our algorithm features a {\em mini-batch} hypergradient estimation via implicit differentiation,  where the core design involves 
 a sample-efficient hypergradient estimator via the Neumann series. 
As shown in \Cref{tab:stochastic}, the gradient complexities of our proposed algorithm w.r.t.~$F$  and $G$  improve upon those of BSA~\cite{ghadimi2018approximation} by an order of $\kappa$ and $\epsilon^{-1}$, respectively. In addition, the Jacobian-vector product complexity JV($G,\epsilon$) of our algorithm improves that of BSA by an order of $\kappa$. In terms of the target accuracy $\epsilon$, our computational complexities  improve those of TTSA~\cite{hong2020two} by an order of $\epsilon^{-1/2}$. Our results further provide the theoretical complexity guarantee for stocBiO in hyperparameter optimization. The experiments demonstrate the superior efficiency of stocBiO for stochastic bilevel optimization.

\section{Convergence theory for Model-Agnostic Meta-Learning}
Meta-learning or learning to learn~\cite{thrun2012learning,naik1992meta,bengio1991learning,schmidhuber1987evolutionary} is a powerful tool  for  quickly learning new tasks by using the prior experience from  related tasks. Recent works have empowered this idea with neural networks, and their proposed meta-learning algorithms have been shown to enable fast learning over unseen tasks using only a few samples 
by efficiently extracting the knowledge from a range of observed tasks~\cite{santoro2016meta,vinyals2016matching,finn2017model}.  Current meta-learning algorithms can be generally categorized into metric-learning based~\cite{koch2015siamese,snell2017prototypical},  model-based~\cite{vinyals2016matching,munkhdalai2017meta}, and optimization-based~\cite{finn2017model,nichol2018reptile,rajeswaran2019meta} approaches. Among them, optimization-based meta-learning is  a  simple and effective approach used in a wide range of domains including classification/regression~\cite{rajeswaran2019meta}, reinforcement learning~\cite{finn2017model}, robotics~\cite{al2018continuous}, federated learning~\cite{chen2018federated}, and imitation learning~\cite{finn2017one}.

Model-agnostic meta-learning (MAML)~\cite{finn2017model} is a popular optimization-based approach,  which is simple and compatible generally  with  models  trained with gradient descents. MAML takes a bilevel optimization procedure, where the inner stage runs a few steps of (stochastic) gradient descent for each individual task, and the outer stage updates the meta parameter based on the inner-stage outputs over all the sampled tasks. The goal of MAML is to find a good meta initialization $w^*$ based on the observed tasks such that for a new task, starting from this $w^*$, a few (stochastic) gradient steps suffice to find a good model parameter.  Such an algorithm has been demonstrated to have superior empirical performance~\cite{antoniou2019train,grant2018recasting,zintgraf2018caml,nichol2018first}. Recently, the theoretical convergence of MAML has also been studied. Specifically, \cite{finn2019online} extended MAML to the online setting, and analyzed the regret for the strongly convex objective function. 
\cite{fallah2020convergence} provided an analysis for one-step MAML for nonconvex functions, where each inner stage takes  a single stochastic gradient descent (SGD) step. 

In practice, the MAML training often takes {\em multiple} SGD steps at the inner stage, for example in \cite{finn2017model,antoniou2019train} for supervised learning and in~\cite{finn2017model,fallah2020provably} for reinforcement learning, in order to attain a higher test accuracy (i.e., better generalization performance) even at a price of higher computational cost. 
Compared to the single-step MAML, the multi-step MAML has been shown to achieve better test performance. For example, as shown in Fig. 5 of \cite{finn2017model} and Table 2 of \cite{antoniou2019train},  the test accuracy is improved as the number of inner-loop steps increases. In particular, in the original MAML work \cite{finn2017model}, $5$ inner-loop steps are taken in the training of a 20-way convolutional MAML model. In addition, some important variants of MAML also take multiple inner-loop steps, which include but not limited to ANIL (Almost No Inner Loop)~\cite{raghu2020rapid} and BOIL (Body Only
update in Inner Loop) \cite{oh2021boil}. For these reasons, it is important and meaningful to analyze the convergence of multi-step MAML, and the resulting analysis can be helpful for studying other
MAML-type of variants.  

However, the theoretical convergence of such {\em multi-step} MAML algorithms has not been established yet. In fact, several mathematical challenges will arise in the theoretical analysis if the inner stage of MAML takes multiple steps. First, the meta gradient of multi-step MAML has a nested and recursive structure, which requires the performance analysis of an optimization path over a nested structure. In addition, multi-step update also yields a complicated bias error in the Hessian estimation as well as the statistical correlation between the Hessian and gradient estimators, both of which cause further difficulty in the analysis of the meta gradient. {\em The contribution of this thesis lies in the development of a new theoretical framework for analyzing the general {\em multi-step} MAML with techniques for handling the above challenges. }

\vspace{0.15cm}
\noindent{\bf Main Contributions.} We develop a new theoretical framework, under which we characterize the convergence rate and the computational complexity to attain an $\epsilon$-accurate solution for {\em multi-step} MAML in the general nonconvex setting. Specifically, for the resampling case where each iteration needs sampling of fresh data (e.g., in reinforcement learning), our analysis enables to decouple the Hessian approximation error from the gradient approximation error based on a novel bound on the distance between two different inner optimization paths, which facilitates the analysis of the overall convergence of MAML. For the finite-sum case where the objective is based on pre-assigned samples (e.g., supervised learning), we develop novel techniques to handle the difference between two losses over the training and  test sets in the  analysis. 

Our analysis provides a guideline for choosing the inner-stage stepsize at the order of $\mathcal{O}(1/N)$ and shows that  $N$-step MAML is guaranteed to converge with the gradient and Hessian computation complexites  growing only linearly with $N$, which is consistent with the empirical observations in~\cite{antoniou2019train}.
In addition, for problems where Hessians are small, e.g., most classification/regression meta-learning problems~\cite{finn2017model}, we show that the inner stepsize $\alpha$ can be set  larger while still maintaining the convergence,
which explains the empirical findings for MAML training in~\cite{finn2017model,rajeswaran2019meta}.

\section{Meta-Learning with  Adaptation on Partial Parameters} 


As a powerful meta-learning paradigm, model-agnostic meta-learning (MAML)~\cite{finn2017model} has been successfully applied to a variety of application domains including classification~\cite{rajeswaran2019meta}, reinforcement learning~\cite{finn2017model}, imitation learning~\cite{finn2017one}, etc.  
At a high level, the MAML algorithm takes a bilevel optimization procedure: the inner loop of task-specific adaptation and the outer (meta) loop of initialization training. Since the outer loop often adopts a gradient-based algorithm, which takes the gradient over the inner-loop algorithm (i.e., the inner-loop optimization path), even the simple inner loop of gradient descent updating can result in the Hessian update in the outer loop, which causes significant computational and memory cost. Particularly in deep learning, if all neural network parameters are updated in the inner loop, then the cost for the outer loop is extremely high. Thus, designing simplified MAML, especially the inner loop, is highly motivated. ANIL (which stands for {\em almost no inner loop}) proposed in~\cite{raghu2019rapid} has recently arisen as such an appealing approach. In particular,
\cite{raghu2019rapid} proposed to update only a small subset (often only the last layer) of parameters in the inner loop. Extensive experiments  in \cite{raghu2019rapid} demonstrate that ANIL achieves a significant speedup over MAML without sacrificing the performance.  

Despite extensive empirical results, there has been no theoretical study of ANIL yet, which motivates this work. In particular, we would like to answer several new questions arising in ANIL (but not in the original MAML). While the outer-loop loss function of ANIL is still nonconvex as MAML, the inner-loop loss can be either  {\em strongly convex} or {\em nonconvex} in practice. The strong convexity  occurs naturally if only the last layer of neural networks is updated in the inner loop, whereas the nonconvexity often occurs if more than one layer of neural networks are updated in the inner loop. Thus, our theory will explore how such different geometries affect the convergence rate, computational complexity, as well as the hyper-parameter selections.  We will also theoretically quantify how much computational advantage ANIL achieves over MAML by training only partial parameters in the inner loop. 

\vspace{0.15cm}
\noindent{\bf Main Contributions.} We characterize the convergence rate and the computational complexity for ANIL with $N$-step inner-loop gradient descent, under nonconvex outer-loop loss geometry, and under two representative inner-loop loss geometries, i.e., strongly-convexity and nonconvexity. Our analysis also provides theoretical guidelines for choosing the hyper-parameters such as the stepsize and the number $N$ of inner-loop steps under each geometry. We summarize our specific results as follows.

\begin{list}{$\bullet$}{\topsep=0.2ex \leftmargin=0.26in \rightmargin=0.1in \itemsep =0.05in}

\item {\bf Convergence rate}: ANIL converges sublinearly with the convergence error decaying sublinearly with the number of sampled tasks due to nonconvexity of the meta objective function. The convergence rate is further significantly affected by the geometry of the inner loop. Specifically, ANIL converges exponentially fast with $N$ initially and then saturates under the strongly-convex inner loop, and constantly converges slower as $N$ increases under the nonconvex inner loop.

\item {\bf Computational complexity}: ANIL attains an $\epsilon$-accurate stationary point with the gradient and second-order evaluations at the order of $\mathcal{O}(\epsilon^{-2})$ due to nonconvexity of the meta objective function. The computational cost is also significantly affected by the geometry of the inner loop. Specifically, under the strongly-convex inner loop, its complexity first decreases and then increases with $N$, which suggests a moderate value of $N$ and a constant stepsize in practice for a fast training. But under the nonconvex inner loop, ANIL has higher computational cost as $N$ increases, which suggests a small $N$ and a stepsize at the level of $1/N$ for desirable training.

\item Our experiments validate that ANIL exhibits aforementioned {\em very different} convergence behaviors under the two inner-loop geometries.

\end{list}
From the technical standpoint, we develop new techniques to capture the properties for ANIL, which does not follow from the existing theory for MAML~\cite{fallah2019convergence,ji2020multi}. First, our analysis explores how different geometries of the inner-loop loss (i.e., strongly-convexity and nonconvexity) affect the convergence of ANIL. Such comparison does not exist in MAML. Second, ANIL contains parameters that are updated only in the outer loop, which exhibit {\em special} meta-gradient properties not captured in MAML.

\section{Related Works}
\noindent{\bf Problem-based bilevel optimization approaches}: Bilevel optimization was first introduced by~\cite{bracken1973mathematical}. Since then, a number of bilevel optimization algorithms have been proposed, which include but not limited to constraint-based methods~\cite{shi2005extended,moore2010bilevel} and gradient-based methods~\cite{domke2012generic,pedregosa2016hyperparameter,gould2016differentiating,maclaurin2015gradient,franceschi2018bilevel,ghadimi2018approximation,liao2018reviving,shaban2019truncated,hong2020two,liu2020generic,li2020improved,grazzi2020iteration,lorraine2020optimizing,ji2021lower,liu2021value}. Among them, \cite{ghadimi2018approximation,hong2020two} provided the complexity analysis for their proposed methods for the nonconvex-strongly-convex bilevel optimization problem.
For such a problem, this thesis develops a general and enhanced convergence rate analysis for ITD- and AID-based bilevel optimizers for the deterministic setting, and proposes a novel algorithm named stocBiO for the stochastic setting  with order-level lower computational complexity than the existing results. We also provide the first-known lower  bounds on complexity as well as tighter upper bounds under various loss geometries.  

Other types of loss geometries have also been studied. \cite{liu2020generic,li2020improved} assumed that the lower- and upper-level functions $g(x,\cdot)$ and $f(x,\cdot)$ are convex and strongly convex, and provided an asymptotic analysis for their methods. \cite{ghadimi2018approximation,hong2020two}  studied the setting where $\Phi(\cdot)$ is strongly convex or convex, and  $g(x,\cdot)$ is strongly convex.

After our stocBiO work was posted on arXiv, there were a few subsequent studies on using momentum-based approximation for accelerating SGD-type bilevel optimization algorithms~\cite{chen2021single,guo2021stochastic,khanduri2021near,guo2021randomized,yang2021provably}. In particular, \cite{guo2021stochastic} proposed a single-loop algorithm SEMA by incorporating momentum-based technique~\cite{cutkosky2019momentum} to the updates. \cite{chen2021single} proposed a single-loop method named STABLE by using the similar momentum scheme for the Hessian updates.
SEMA, MSTSA and STABLE achieve the same complexity as our stocBiO w.r.t.~$\epsilon$. \cite{khanduri2021near,guo2021randomized,yang2021provably} improved the dependence on $\epsilon$ of our stocBiO from $\mathcal{O}(\epsilon^{2})$ to $\mathcal{O}(\epsilon^{1.5})$ via recursive momentum and variance reduction. In particular, our proposed hypergradient estimator has been successfully used in MRBO and VRBO proposed by~\cite{yang2021provably}.  We want to emphasize our stocBiO is the first mini-batch SGD-type bilevel optimization algorithm along this direction.

\vspace{0.2cm}
\noindent {\bf Problem-based bilevel optimization in meta-learning}: Problem-based bilevel optimization framework has been successfully applied to meta-learning recently~\cite{snell2017prototypical,franceschi2018bilevel,rajeswaran2019meta,zugner2019adversarial,ji2020convergence,ji2020multi}. For example, \cite{rajeswaran2019meta} reformulated the model-agnostic meta-learning (MAML)~\cite{finn2017model} as problem-based bilevel optimization, and proposed iMAML via implicit gradient. 
Another well-established framework in few-shot meta learning~\cite{bertinetto2018meta,lee2019meta,ravi2016optimization,snell2017prototypical,zhou2018deep} aims to learn good parameters as a common embedding model for all tasks. 
Building on the embedded features, task-specific parameters are then searched as a minimizer of the inner-loop loss function~\cite{bertinetto2018meta,lee2019meta}.
For example,
 \cite{snell2017prototypical}  proposed a bilevel optimization procedure for  meta-learning to learn a common embedding model  for all tasks. Our work provides a theoretical complexity guarantee for two popular types of bilevel optimizer, i.e., AID-BiO and ITD-BiO, for meta-learning. 

\vspace{0.2cm}
\noindent {\bf Problem-based bilevel optimization in hyperparameter optimization}: 
Hyperparameter optimization has become increasingly important as a powerful tool in the automatic machine learning (autoML)~\cite{okuno2018hyperparameter,yu2020hyper}. Recently, various bilevel optimization algorithms have been proposed for hyperparameter optimization, which include AID-based methods~\cite{pedregosa2016hyperparameter,franceschi2018bilevel}, ITD-based  methods~\cite{franceschi2018bilevel,shaban2019truncated,grazzi2020iteration}, self-tuning networks~\cite{mackay2018self,bae2020delta}, penalty-based methods~\cite{mehra2019penalty,sinha2020gradient,liu2021value}, proximal approximation based  method~\cite{jenni2018deep}, etc.  
Our work demonstrates superior efficiency of the proposed principled stocBiO algorithm in hyperparameter optimization.

\vspace{0.2cm}

\noindent{\bf Algorithm-based bilevel optimization in meta-learning.} Algorithm-based bilevel optimization  approaches have been widely used in meta-learning due to its simplicity and efficiency~\cite{li2017meta,ravi2016optimization,finn2017model}. 
 As a pioneer along this line, MAML~\cite{finn2017model} aims to find an initialization such that gradient descent from it achieves fast adaptation. Many follow-up studies~\cite{grant2018recasting,finn2019online,jerfel2018online,finn2018meta,finn2018probabilistic,mi2019meta,liu2019taming,rothfuss2019promp,foerster2018dice,fallah2020convergence,raghu2020rapid, collins2020distribution} have extended MAML from different perspectives. 
 For example, \cite{finn2019online} provided a follow-the-meta-leader extension of MAML for online learning. \cite{raghu2020rapid} proposed an efficient variant of MAML named ANIL (Almost No Inner Loop) by adapting only a small subset (e.g., head) of neural network parameters in the inner loop.  Various Hessian-free MAML algorithms have been proposed to avoid  the costly computation of second-order derivatives, which include but not limited to FOMAML~\cite{finn2017model}, Reptile~\cite{nichol2018reptile}, ES-MAML~\cite{song2020simple}, and  HF-MAML~\cite{fallah2020convergence}. In particular, FOMAML~\cite{finn2017model} omits all second-order derivatives in its meta-gradient computation, HF-MAML~\cite{fallah2020convergence} estimates the meta gradient in one-step MAML using Hessian-vector product approximation.  This thesis focuses on the first MAML algorithms, but the techniques here can be extended to analyze the Hessian-free multi-step MAML. 
 Alternatively to meta-initialization algorithms such as MAML, meta-regularization approaches aim to learn a good bias for a regularized empirical risk  
minimization problem for intra-task learning~\cite{alquier2017regret, denevi2018learning,denevi2018incremental,denevi2019learning,rajeswaran2019meta,balcan2019provable,zhou2019efficient}.  
\cite{balcan2019provable} formalized a connection between meta-initialization and meta-regularization from an online learning perspective.~\cite{zhou2019efficient} proposed an efficient meta-learning approach based on a  minibatch proximal updating procedure. 

%
Theoretical property of MAML was initially established in \cite{finn2018meta}, which showed that MAML is a universal learning algorithm approximator under certain conditions. Then {\em MAML-type algorithms} have been studied recently from the optimization perspective, where the convergence rate and computation complexity is typically characterized. \cite{finn2019online} analyzed online MAML for a strongly convex objective function under a bounded-gradient assumption. 
 \cite{fallah2020convergence} developed a convergence analysis for one-step MAML for a general nonconvex objective in the  resampling case. Our study here provides a new convergence analysis for {\em multi-step} MAML in the {\em nonconvex} setting for both the resampling  and  finite-sum cases.  Since the initial version of our work was posted in arXiv, there have been a few studies on multi-step MAML more recently. \cite{wang2020global2,wang2020global} studied the global optimality of MAML under the over-parameterized neural networks, while our analysis focus on general nonconvex functions.  
\cite{kim2020multi} 
proposed an efficient extension of multi-step MAML by gradient reuse in the inner loop, while our analysis focuses on the most basic MAML algorithm.  
\cite{ji2020convergence} analyzed the convergence and complexity performance of multi-step ANIL algorithm, which is an efficient simplification of MAML by adapting only partial parameters in the inner loop. 
We emphasize that the study here is the first along the line of studies on multi-step MAML. We note that a concurrent work \cite{fallah2020provably} also studies multi-step MAML for reinforcement learning setting, where they design an unbiased multi-step estimator. As a comparison, our estimator is biased due to the data sampling in the inner loop, and hence we need extra developments to control this bias, e.g.,  by bounding the difference between batch-gradient and the stochastic-gradient parameter updates in the inner loop.  

%
%

\vspace{0.2cm}

\noindent{\bf Statistical theory for meta-learning.}  
\cite{zhou2019efficient} statistically demonstrated the importance of prior hypothesis in reducing the excess risk via a regularization approach.  
\cite{du2020few} studied few-shot learning from a representation learning perspective, and showed that representation learning can provide a sufficient rate improvement in both linear regression and learning neural networks. \cite{tripuraneni2020provable} studied a multi-task linear regression problem with shared low-dimensional representation, and proposed a sample-efficient algorithm with performance guarantee. \cite{arora2020provable} proposed a representation learning approach for imitation learning via bilevel optimization, and demonstrated the improved sample complexity brought by representation learning.

\section{Organization of the Dissertation}
The rest of the dissertation is organized as follows. In \Cref{chp_deter_bilevel}, we provide a convergence theory for problem-based bilevel optimization. In \Cref{chp_acc_bilevel}, we provide accelerated bilevel optimization algorithms with lower computational complexity.  In \Cref{chp_lower_bilevel}, we develop lower bounds for two types of problem-based bilevel optimization problems. In \Cref{chp:stoc_bilevel}, we propose an efficient stochastic optimization algorithm with provable performance improvements.  In \Cref{chp: maml}, we provide a comprehensive convergence theory for multi-step MAML under various settings. In \Cref{chp:anil}, we analyze the convergence behaviors of ANIL under different loss landscapes. Lastly, we discuss several future research directions and briefly talk about some other works the author has done during this Ph.D. study in \Cref{end.ch}. 
  
\chapter{Convergence Theory for Problem-Based Bilevel Optimization}\label{chp_deter_bilevel}
In this chapter, we first provide a comprehensive convergence and complexity theory for widely-used problem-based bilevel optimization algorithms. All technical proofs for the results in this chapter are provided in \Cref{sec: append_deter_bilevel}. 

\section{Algorithms for Problem-Based Bilevel Optimization}\label{sec:alg}

As shown in~\Cref{alg:main_deter}, we describe two popular types of problem-based bilevel optimizers respectively based on AID and ITD (referred to as AID-BiO and ITD-BiO) for solving the problem~\cref{objective_deter}.  

Both AID-BiO and ITD-BiO update in a nested-loop manner. In the inner loop, both of them run $D$ steps of gradient decent (GD) to find an approximation point $y_k^D$ close to $y^*(x_k)$. 
Note that we choose the initialization $y_{k}^0$ of each inner loop as the output $y_{k-1}^D$ of the preceding inner loop rather than a random start.  Such a {\em warm start} allows us to backpropagate the tracking error $\|y_k^D-y^*(x_k)\|$ to previous loops, and yields an improved computational complexity.

At the outer loop,  AID-BiO first solves $v_k^N$ from a linear system $\nabla_y^2 g(x_k,y_k^D) v = 
\nabla_y f(x_k,y^D_k)$\footnote{Equivalent to solving a quadratic programming $\min_v\frac{1}{2} v^T\nabla_y^2g(x_k,y_k^D) v-v^T\nabla_y f(x_k,y^D_k).$ } using $N$ steps of conjugate-gradient (CG) starting from $v_k^0$ (where we also adopt a warm start with $v_k^0=v_{k-1}^N$), and then constructs \begin{align}\label{hyper-aid}
\widehat\nabla \Phi(x_k)= \nabla_x f(x_k,y_k^T) -\nabla_x \nabla_y g(x_k,y_k^T)v_k^N
\end{align}
as an estimate of the true hypergradient $\nabla \Phi(x_k)$, whose form is given as follows.  
\begin{proposition}\label{prop:grad}
Hypergradient $\nabla \Phi(x_k)$ takes the forms of 
\begin{align}\label{trueG}
\nabla \Phi(x_k) =&  \nabla_x f(x_k,y^*(x_k)) -\nabla_x \nabla_y g(x_k,y^*(x_k)) v_k^*, 
\end{align}
where $v_k^*$ is the solution of the following linear system $$ \nabla_y^2 g(x_k,y^*(x_k))v=
\nabla_y f(x_k,y^*(x_k)).$$
\end{proposition}
As shown in~\cite{domke2012generic,grazzi2020iteration}, the construction of~\cref{hyper-aid} involves only Hessian-vector products in solving $v_N$ via CG and Jacobian-vector product $\nabla_x \nabla_y g(x_k,y_k^D)v_k^N$, which can be efficiently computed and stored via existing automatic differentiation packages.  

\begin{algorithm}[t]
\small 
	\caption{Bilevel algorithms  via AID or ITD}    
	\label{alg:main_deter}
	\begin{algorithmic}[1]
		\STATE {\bfseries Input:}  $K,D,N$, stepsizes $\alpha, \beta $, initializations $x_0, y_0,v_0$.
		\FOR{$k=0,1,2,...,K$}
		\STATE{Set $y_k^0 = y_{k-1}^{D} \mbox{ if }\; k> 0$ and $y_0$ otherwise  }
		\FOR{$t=1,....,D$}
		\vspace{0.05cm}
		\STATE{Update $y_k^t = y_k^{t-1}-\alpha \nabla_y g(x_k,y_k^{t-1}) $}
		\vspace{0.05cm}
		\ENDFOR
                  \STATE{Hypergradient estimation via 
\vspace{0.1cm}
                  \\\hspace{-0.3cm} {\bf AID}: 1) set $v_k^0 = v_{k-1}^{N} \mbox{ if }\; k> 0$ and $v_0$ otherwise
              \vspace{0.1cm}    \\\hspace{0.9cm}2) solve $v_k^N$ from $\nabla_y^2 g(x_k,y_k^D) v = 
\nabla_y f(x_k,y^D_k)$  via $N$ steps of CG starting at $v_k^0$
 \vspace{0.1cm} \\\hspace{0.9cm}3) get Jacobian-vector product {\small$\nabla_x \nabla_y g(x_k,y_k^D)v_k^N$}  via automatic differentiation
 \vspace{0.1cm} \\\hspace{0.9cm}4) {\small$\widehat\nabla \Phi(x_k)= \nabla_x f(x_k,y_k^D) -\nabla_x \nabla_y g(x_k,y_k^D)v_k^N$} 
             \vspace{0.0cm}      \\ \hspace{-0.3cm} {\bf  ITD}: compute $\widehat\nabla \Phi(x_k)=\frac{\partial f(x_k,y^D_k)}{\partial x_k}$ via backpropagation
                    }
                 \STATE{Update $x_{k+1}=x_k- \beta \widehat\nabla \Phi(x_k)$}
		\ENDFOR
	\end{algorithmic}
	\end{algorithm}

As a comparison, the outer loop of ITD-BiO computes the gradient $\frac{\partial f(x_k,y^D_k(x_k))}{\partial x_k}$  as an approximation of the hyper-gradient $\nabla \Phi(x_k)=\frac{\partial f(x_k,y^*(x_k))}{\partial x_k}$ via backpropagation, where we write  $y^D_k(x_k)$ because the output $y_k^D$ of the inner loop has a dependence on $x_k$ through the inner-loop iterative GD updates. The explicit form of the estimate $\frac{\partial f(x_k,y^D_k(x_k))}{\partial x_k}$ is given by the following proposition via the chain rule. For notation simplification, let $\prod_{j=D}^{D-1} (\cdot)= I$.
   \begin{proposition}\label{deter:gdform}
 $\frac{\partial f(x_k,y^D_k(x_k))}{\partial x_k}$ takes the analytical form of:
\begin{small}
\begin{align*}
\frac{\partial f(x_k,y^D_k)}{\partial x_k}= &\nabla_x f(x_k,y_k^D) -\alpha\sum_{t=0}^{D-1}\nabla_x\nabla_y g(x_k,y_k^{t})\prod_{j=t+1}^{D-1}(I-\alpha  \nabla^2_y g(x_k,y_k^{j}))\nabla_y f(x_k,y_k^D).
\end{align*}
\end{small}
\end{proposition}
 \Cref{deter:gdform} shows that the differentiation involves the computations of second-order derivatives such as Hessian $ \nabla^2_y g(\cdot,\cdot)$. Since efficient Hessian-free methods  have been successfully deployed in the existing automatic differentiation tools, computing these second-order derivatives reduces to more efficient computations of Jacobian- and Hessian-vector products.

\section{Definitions and Assumptions}
Let $z=(x,y)$ denote all parameters. In this thesis, we focus on the following types of loss functions.
\begin{assum}\label{assum:geo}
The lower-level function $g(x,y)$ is $\mu$-strongly-convex w.r.t.~$y$ and the total objective function $\Phi(x)=f(x,y^*(x))$ is nonconvex w.r.t.~$x$.  
\vspace{-0.2cm} 
\end{assum}
Since $\Phi(x)$ is nonconvex, algorithms are expected to find an $\epsilon$-accurate stationary point defined as follows. 
\begin{definition}
We say $\bar x$ is an $\epsilon$-accurate stationary point for the objective function $\Phi(x)$ in~\cref{objective_deter} if $\|\nabla \Phi(\bar x)\|^2\leq \epsilon$, where $\bar x$ is the output of an algorithm.
\end{definition}
In order to compare the performance of different bilevel algorithms, we adopt the following metrics of  complexity. 

\begin{definition}\label{com_measure}
For a function $f(x,y)$ and a vector $v$, let $\mbox{\normalfont Gc}(f,\epsilon)$ be the number of the partial gradient $\nabla_x f$ or $\nabla_y f$, and let $\mbox{\normalfont JV}(g,\epsilon)$ and  $\mbox{\normalfont HV}(g,\epsilon)$ be the number of Jacobian-vector products $\nabla_x\nabla_y g (x,y)v$. 
  and  Hessian-vector products $\nabla_y^2g(x,y) v$. 
 \end{definition} 
We  take the following standard assumptions on the loss functions in~\cref{objective_deter}, which have been widely adopted in bilevel optimization~\cite{ghadimi2018approximation,ji2020convergence}.
\begin{assum}\label{ass:lip}
The loss function $f(z)$ and $g(z)$ satisfy
\begin{list}{$\bullet$}{\topsep=0.2ex \leftmargin=0.2in \rightmargin=0.in \itemsep =0.01in}
\item The function $f(z)$ is $M$-Lipschitz, i.e.,  for any $z,z^\prime$, $|f(z)-f(z^\prime)|\leq M\|z-z^\prime\|.$
\item $\nabla f(z)$ and $\nabla g(z)$ are $L$-Lipschitz, i.e., for any $z,z^\prime$, 
\vspace{-0.1cm}
\begin{align*}
\|\nabla f(z)-\nabla f(z^\prime)\|\leq& L\|z-z^\prime\|,\;\|\nabla g(z)-\nabla g(z^\prime)\|\leq L\|z-z^\prime\|.
\end{align*}
\end{list}
\vspace{-0.2cm}
\end{assum}
As shown in~\Cref{prop:grad}, the gradient of the objective function $\Phi(x)$ involves the second-order derivatives $\nabla_x\nabla_y g(z)$ and $\nabla_y^2 g(z)$. The following assumption imposes the Lipschitz conditions on such high-order derivatives, as also made in~\cite{ghadimi2018approximation}.
\begin{assum}\label{high_lip}
Suppose the derivatives $\nabla_x\nabla_y g(z)$ and $\nabla_y^2 g(z)$ are $\tau$- and $\rho$- Lipschitz, i.e.,
\begin{list}{$\bullet$}{\topsep=0.2ex \leftmargin=0.2in \rightmargin=0.in \itemsep =0.02in}
\item For any $z,z^\prime$, $\|\nabla_x\nabla_y g(z)-\nabla_x\nabla_y g(z^\prime)\| \leq \tau \|z-z^\prime\|$.
\item For any $z,z^\prime$, $\|\nabla_y^2 g(z)-\nabla_y^2 g(z^\prime)\|\leq \rho \|z-z^\prime\|$.
\end{list} 
\end{assum}

\section{Convergence for Bilevel Optimization}\label{main:result_deter}
We first characterize the convergence and complexity of AID-BiO.  Let $\kappa=\frac{L}{\mu}$ denote the condition number. 
\begin{theorem}[AID-BiO]\label{th:aidthem}
Suppose Assumptions~\ref{assum:geo},~\ref{ass:lip}, \ref{high_lip} hold. Define  a smoothness parameter $L_\Phi = L + \frac{2L^2+\tau M^2}{\mu} + \frac{\rho L M+L^3+\tau M L}{\mu^2} + \frac{\rho L^2 M}{\mu^3}=\Theta(\kappa^3)$, choose the stepsizes $\alpha\leq \frac{1}{L}$, $\beta=\frac{1}{8L_\Phi}$, and set the inner-loop iteration number $D\geq\Theta(\kappa)$ and the CG iteration number  $N\geq \Theta(\sqrt{\kappa})$, where the detailed forms of  $D,N$ can be found in \Cref{appen:aid-bio}. 
 Then, the outputs of AID-BiO satisfy
\begin{align*}
\frac{1}{K}\sum_{k=0}^{K-1}\| \nabla \Phi(x_k)\|^2 \leq \frac{64L_\Phi (\Phi(x_0) - \inf_x\Phi(x))+5\Delta_0}{K}, 
\end{align*}
where $\Delta_0=\|y_0-y^*(x_{0})\|^2 + \|v_{0}^*-v_0\|^2>0$.

In order to achieve an $\epsilon$-accurate stationary point, the complexities satisfy 
\begin{list}{$\bullet$}{\topsep=0.ex \leftmargin=0.2in \rightmargin=0.in \itemsep =0.01in}
\item Gradient: {\small$\mbox{\normalfont Gc}(f,\epsilon)=\mathcal{O}(\kappa^3\epsilon^{-1}), \mbox{\normalfont Gc}(g,\epsilon)=\mathcal{O}(\kappa^4\epsilon^{-1}).$}
\item Jacobian- and Hessian-vector: {\small$ \mbox{\normalfont JV}(g,\epsilon)=\mathcal{O}\left(\kappa^3\epsilon^{-1}\right), \mbox{\normalfont HV}(g,\epsilon)=\mathcal{O}\left(\kappa^{3.5}\epsilon^{-1}\right).$}
\end{list}
\end{theorem}
As shown in \Cref{tab:determinstic}, 
the complexities  $\mbox{\normalfont Gc}(f,\epsilon)$, $\mbox{\normalfont Gc}(g,\epsilon)$, $\mbox{\normalfont JV}(g,\epsilon) $ and $\mbox{\normalfont HV}(g,\epsilon)$ of our analysis improves that of \cite{ghadimi2018approximation} (eq.~(2.30) therein) by the order of $\kappa$, $\kappa\epsilon^{-1/4}$, $\kappa$ and $\kappa$.  Such an improvement is achieved by a refined analysis with a constant number of inner-loop steps, and by a warm start strategy to backpropagate the tracking errors $\|y_k^D-y^*(x_k)\|$  and $\|v_k^N-v^*_k\|$ to previous loops, as also demonstrated by our meta-learning experiments.  We next characterize the convergence and complexity performance of the ITD-BiO algorithm. \begin{theorem}[ITD-BiO]\label{th:determin}
Suppose Assumptions~\ref{assum:geo},~\ref{ass:lip}, and \ref{high_lip} hold. Define $L_\Phi $ as in \Cref{th:aidthem}, and choose $\alpha\leq \frac{1}{L}$,
 $\beta=\frac{1}{4L_\Phi}$ and $D\geq \Theta(\kappa\log\frac{1}{\epsilon})$, where the detailed form of $D$ can be found in \Cref{append:itd-bio}. Then, we have 
\begin{align*}
\frac{1}{K}\sum_{k=0}^{K-1}\| \nabla \Phi(x_k)\|^2 \leq \frac{16 L_\Phi (\Phi(x_0)-\inf_x\Phi(x))}{K} + \frac{2\epsilon}{3}.
\end{align*}
In order to achieve an $\epsilon$-accurate stationary point, the complexities satisfy 
\begin{list}{$\bullet$}{\topsep=0.ex \leftmargin=0.1in \rightmargin=0.in \itemsep =0.01in}
\item Gradient: {\small$\mbox{\normalfont Gc}(f,\epsilon)=\mathcal{O}(\kappa^3\epsilon^{-1}), \mbox{\normalfont Gc}(g,\epsilon)=\mathcal{\widetilde O}(\kappa^4\epsilon^{-1}).$}
\item Jacobian- and Hessian-vector complexity: {\small$ \mbox{\normalfont JV}(g,\epsilon)=\mathcal{\widetilde O}\big(\kappa^4\epsilon^{-1}\big), \mbox{\normalfont HV}(g,\epsilon)=\mathcal{\widetilde O}\big(\kappa^4\epsilon^{-1}\big).$}
\end{list}
\end{theorem}
By comparing \Cref{th:aidthem} and \Cref{th:determin}, it can be seen that the complexities $\mbox{\normalfont JV}(g,\epsilon)$ and $\mbox{\normalfont HV}(g,\epsilon)$  of AID-BiO are better than those of ITD-BiO by the order of  $\kappa$ and $\kappa^{0.5}$, which implies that AID-BiO is more computationally and memory efficient than ITD-BiO, as verified in \Cref{fig:strfc100bievlcsaa}.

\section{Applications to Meta-Learning}
Consider the few-shot meta-learning problem with $m$ tasks $\{\mathcal{T}_i,i=1,...,m\}$
 sampled from distribution $P_\gT$. Each task $\mathcal{T}_i$ has a loss function $\gL(\phi,w_i;\xi)$ over each data sample $\xi$, where $\phi$ are the parameters of an embedding model shared by all tasks, and $w_i$ are the task-specific parameters. The goal of this framework is to  find good parameters $\phi$ for all tasks, and building on the embedded features, each task then adapts its own parameters $w_i$ by minimizing its loss.

The model training  takes a bilevel procedure. In the lower-level stage, building on the embedded features, the base learner of task $\mathcal{T}_i$ searches $w_i^*$ as the minimizer of its 
loss 
  over a training set $\gS_i$. In the upper-level stage, the meta-learner evaluates the minimizers $w_i^*,i=1,...,m$ on held-out test sets, and optimizes $\phi$ of the embedding model over all tasks. Let $\widetilde w=(w_1,...,w_m)$ denote all task-specific parameters. Then, the objective function is given by 
 \begin{align}\label{obj:meta}
 &\min_{\phi} \gL_{\gD} (\phi,\widetilde w^{*})=\frac{1}{m}\sum_{i=1}^m\underbrace{\frac{1}{|\gD_i|}\sum_{\xi\in\gD_i}\mathcal{L}(\phi,w_i^*;\xi)}_{\gL_{\gD_i}(\phi,w_i^*)\text{: task-specific upper-level loss}} \nonumber
 \\& \;\mbox{s.t.} \; \widetilde w^* = \argmin_{\widetilde w} \gL_{\gS} (\phi,\widetilde w)=\frac{\sum_{i=1}^m\gL_{\gS_i}(\phi,w_i)
 }{m},
 \end{align} 
 where $\gL_{\gS_i}(\phi,w_i)= \frac{1}{|\gS_i|}\sum_{\xi\in\gS_i}\mathcal{L}(\phi,w_i;\xi) + \gR(w_i)$ with a strongly-convex regularizer $\gR(w_i)$, e.g., $L^2$, and $\gS_i,\gD_i$ are the training and test datasets of task $\mathcal{T}_i$. Note that the lower-level problem is equivalent to solving each $w^*_i$ as a minimizer of the task-specific loss $\gL_{\gS_i}(\phi,w_i)$ for $i=1,...,m$.  In practice, $w_i$ often corresponds to the parameters of the last {\em linear} layer of a neural network and $\phi$ are the parameters of the remaining layers (e.g., $4$ convolutional layers in~\cite{bertinetto2018meta,ji2020convergence}), and hence the lower-level function is {\em strongly-convex} w.r.t. $\widetilde w$ and the upper-level function $\gL_{\gD} (\phi,\widetilde w^{*}(\phi))$ is generally nonconvex w.r.t. $\phi$. In addition, due to the small sizes of datasets $\gD_i$ and $\gS_i$ in few-shot learning, all updates for each task $\gT_i$ use {\em full gradient descent} without data resampling. 
 As a result,  AID-BiO and ITD-BiO in~\Cref{alg:main_deter} can be applied here.  
In some applications where the number $m$ of tasks is large,  it is more efficient to sample a batch $\gB$ of i.i.d.\ tasks from $\{\mathcal{T}_i,i=1,...,m\}$ at each meta (outer) iteration, and optimizes the mini-batch versions $ \gL_{\gD} (\phi,\widetilde w;\gB) = \frac{1}{|\gB|}\sum_{i\in\gB}\gL_{\gD_i}(\phi,w_i)$ and $\gL_{\gS} (\phi,\widetilde w;\gB)=\frac{1}{|\gB|}\sum_{i\in\gB}\gL_{\gS_i}(\phi,w_i)$ instead. 

We next  provide the convergence result of ITD-BiO for this case, and that of AID-BiO can be similarly derived.
\begin{theorem}\label{th:meta_learning}
Suppose Assumptions~\ref{assum:geo},~\ref{ass:lip} and \ref{high_lip} hold and suppose each task loss $\gL_{\gS_i}(\phi,\cdot)$ is $\mu$-strongly-convex. Choose the same parameters $\beta,D$ as in~\Cref{th:determin}. Then, we have
\begin{align*}
\frac{1}{K}\sum_{k=0}^{K-1}\mathbb{E}\| \nabla \Phi(\phi_k)\|^2 \leq& \mathcal{O} \left(\frac{1}{K}+\frac{\kappa^2}{|\gB|}\right).
\end{align*}
\end{theorem}
 \Cref{th:meta_learning} shows  that compared to the full batch case (i.e., without task sampling) in~\cref{obj:meta}, task sampling introduces a variance term $\mathcal{O}(\frac{1}{|\gB|})$ due to the stochastic nature of the algorithm.  

 \begin{figure*}[ht]
  \vspace{-2mm}
	\centering    
	\subfigure[dataset: miniImageNet  ]{\label{fig1:abilevel}\includegraphics[width=60mm]{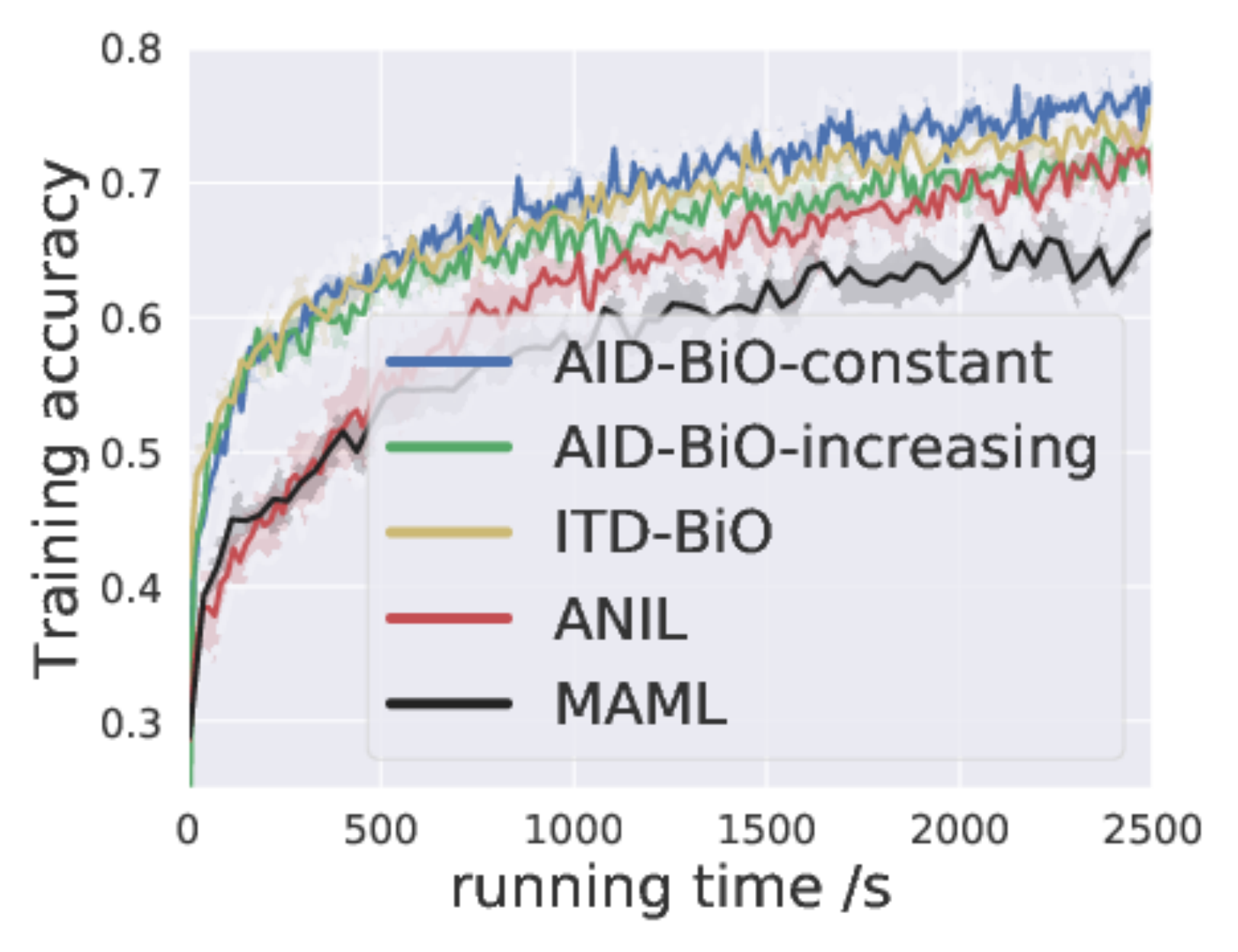}\includegraphics[width=60mm]{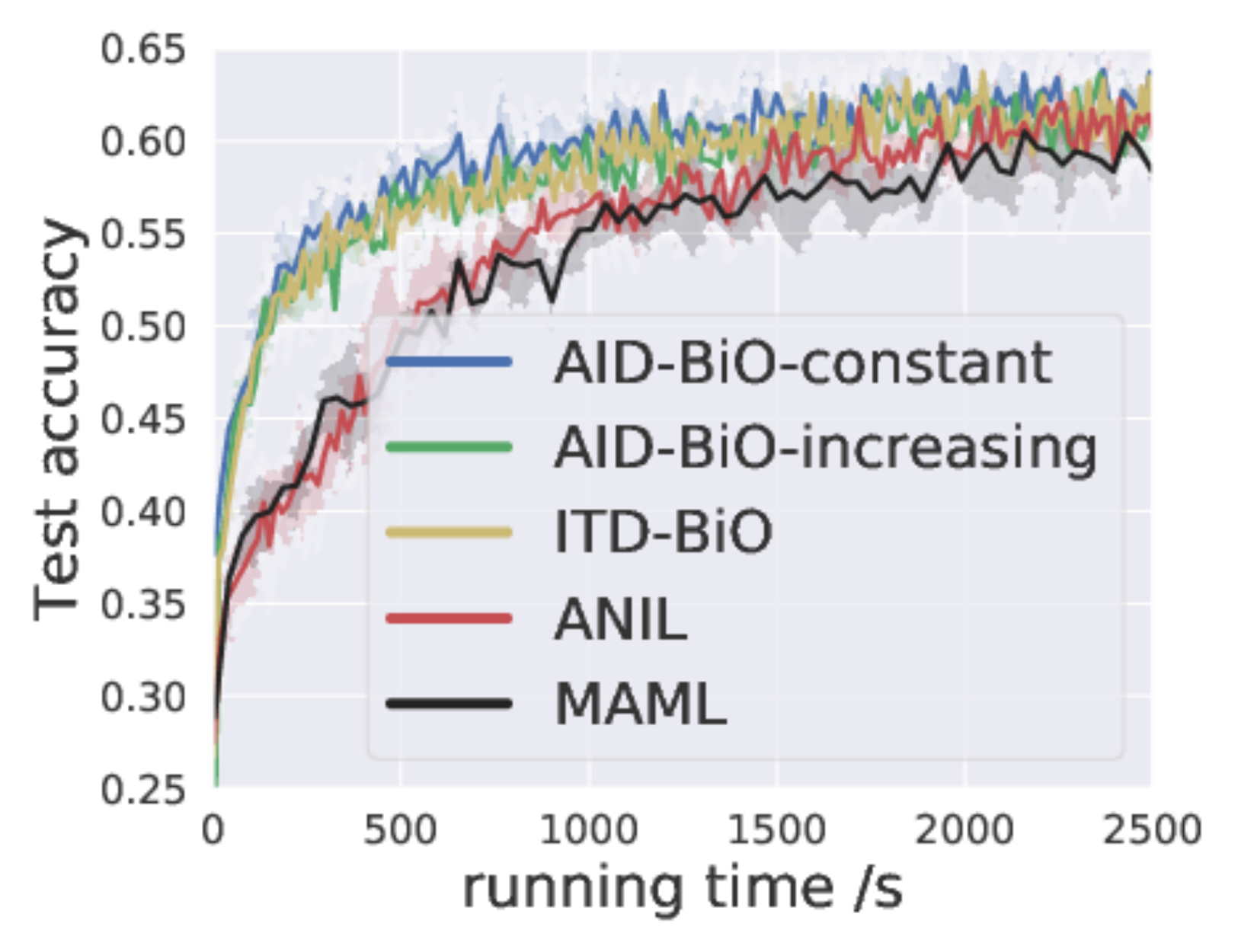}} 
	\subfigure[dataset: FC100]{\label{fig1:bbilevelssc}\includegraphics[width=60mm]{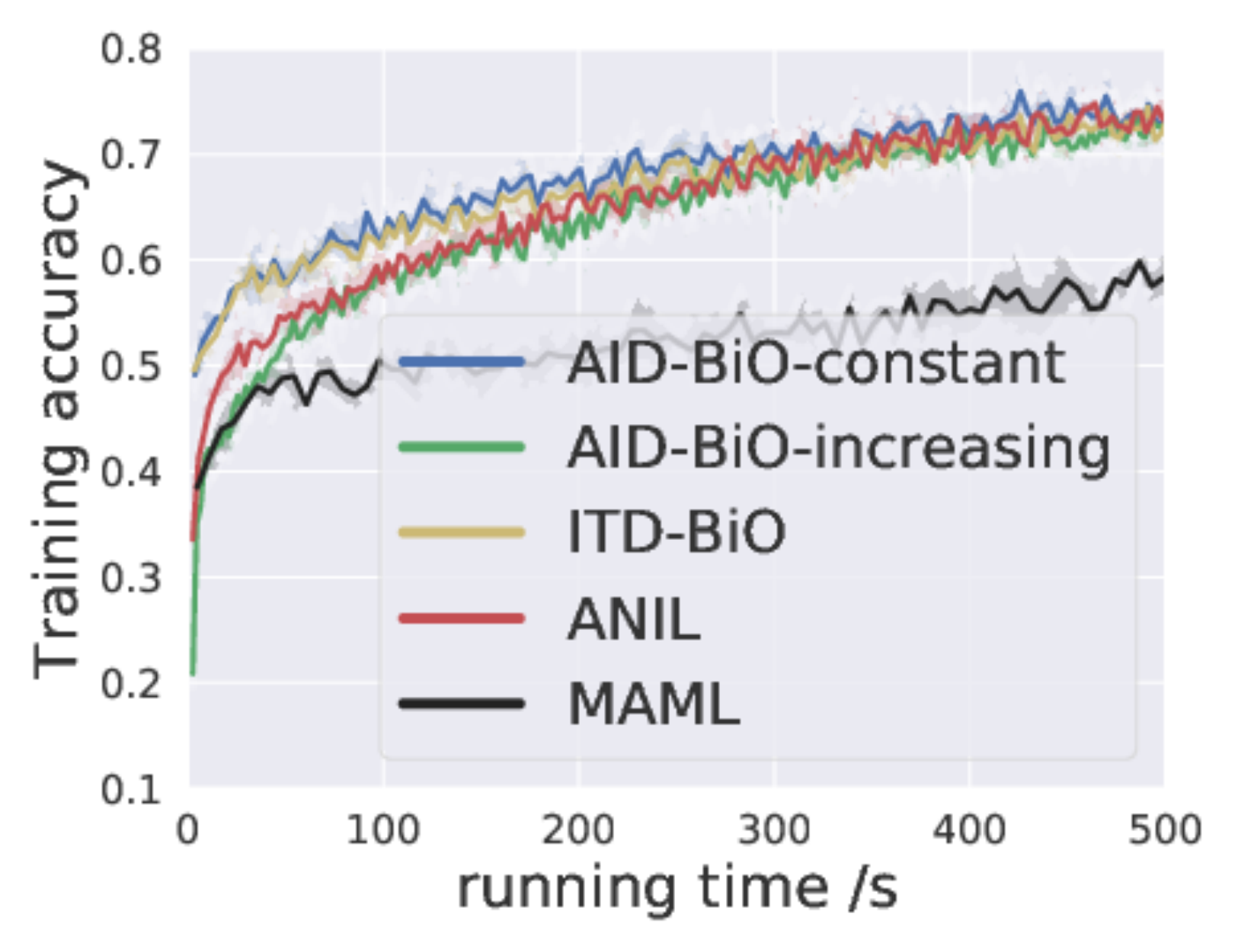}\includegraphics[width=60mm]{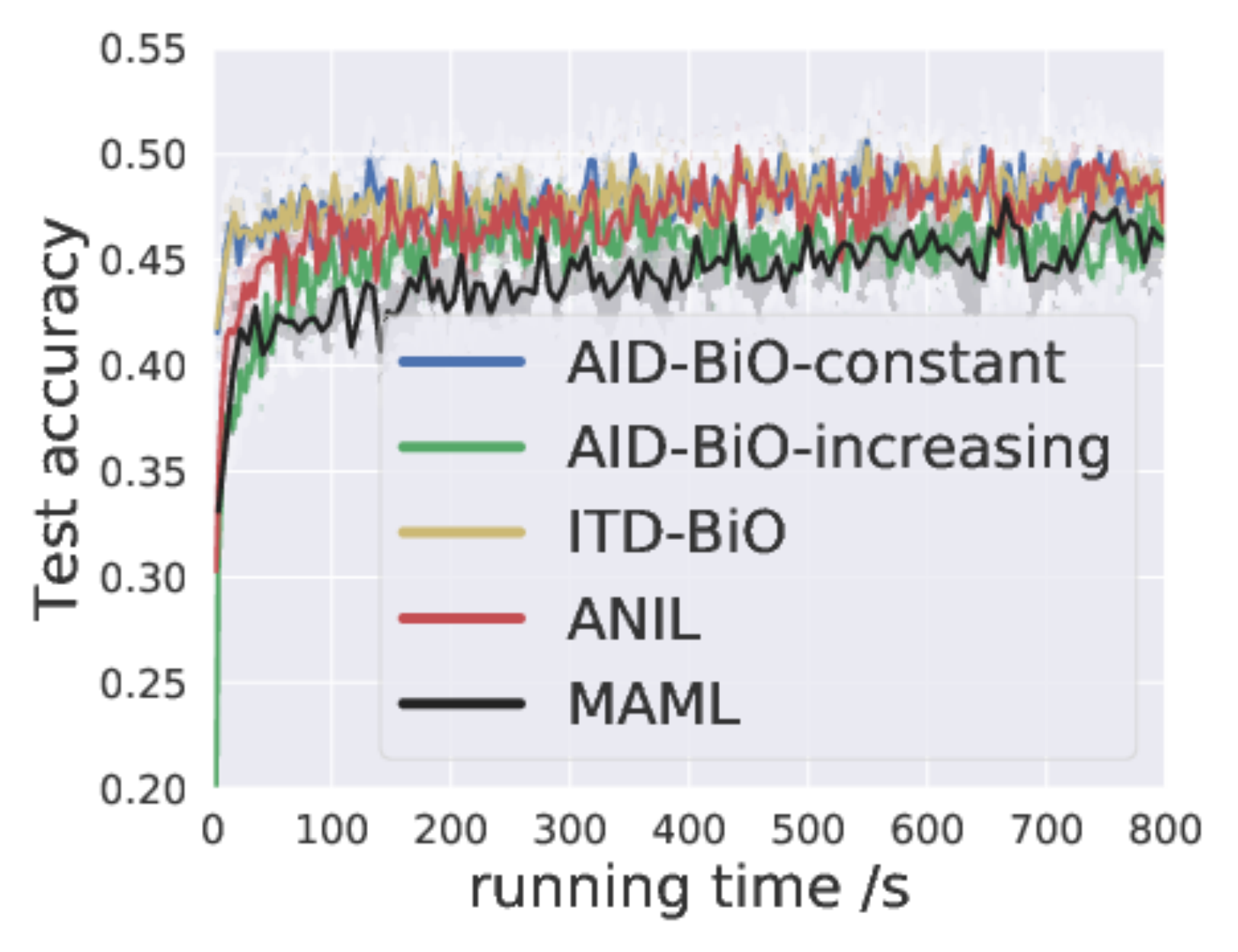}}  
	\vspace{-0.2cm}
	\caption{Comparison of various bilevel algorithms on meta-learning. For each dataset, left plot: training accuracy v.s. running time; right plot: test accuracy v.s. running time.}\label{fig:strfc100bievlcsaa}
	  \vspace{-0.2cm}
\end{figure*}

  \begin{figure*}[ht]
  \vspace{-2mm}
	\centering    
	\subfigure[$T=10$, miniImageNet dataset]{\label{fig1:ci}\includegraphics[width=61mm]{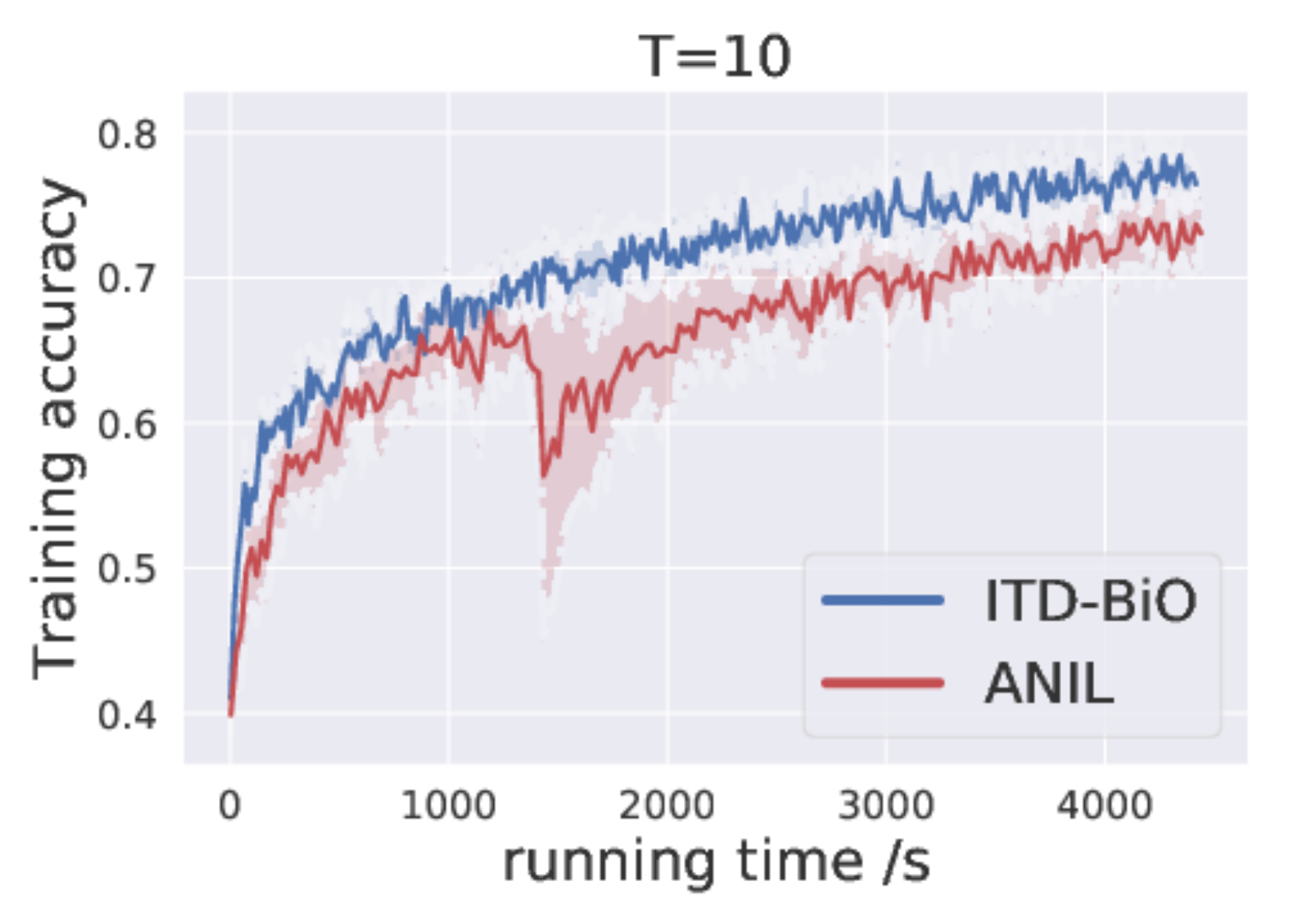}\includegraphics[width=61mm]{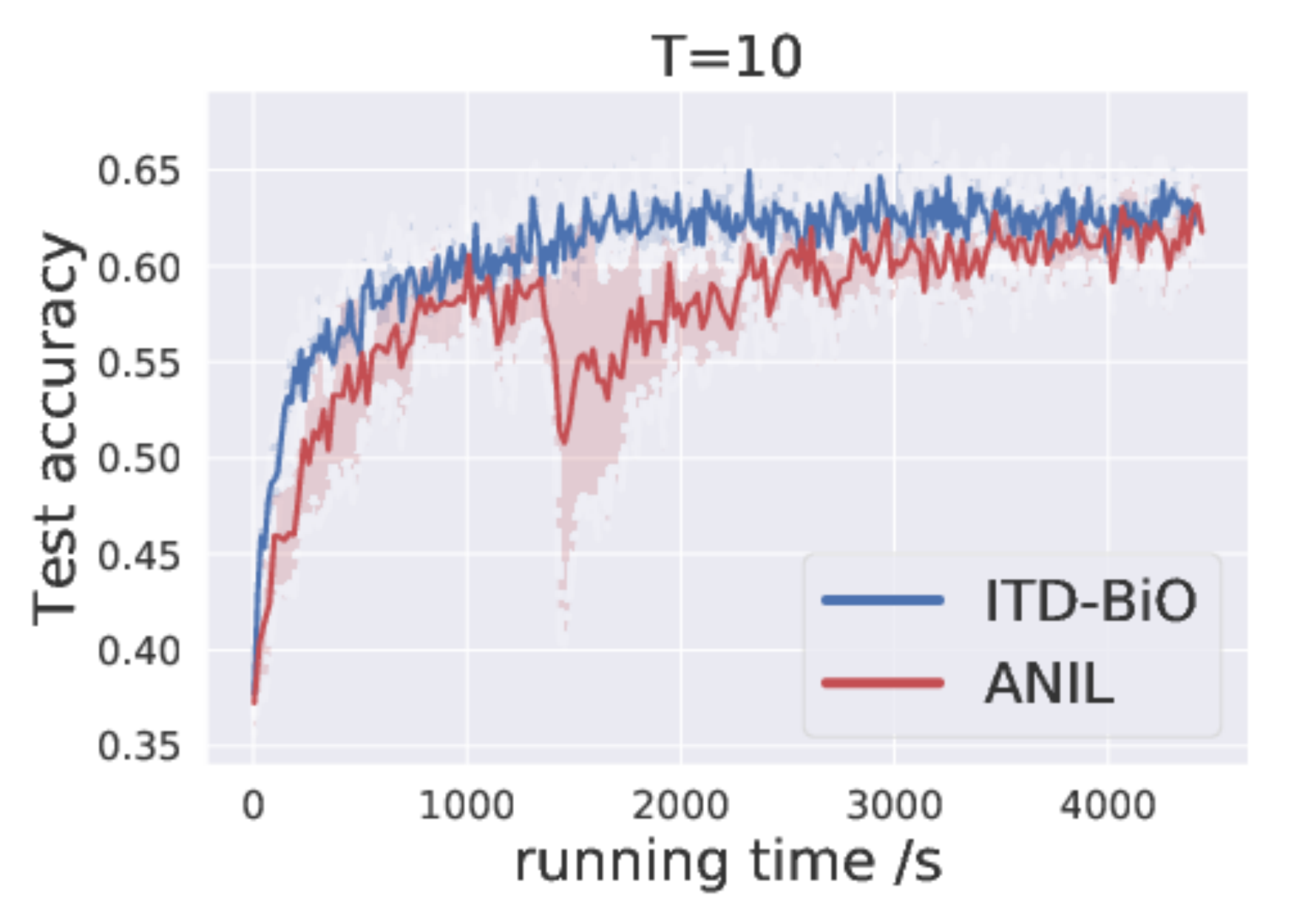}} 
	\subfigure[$T=20$, FC100 dataset]{\label{fig1:di}\includegraphics[width=60mm]{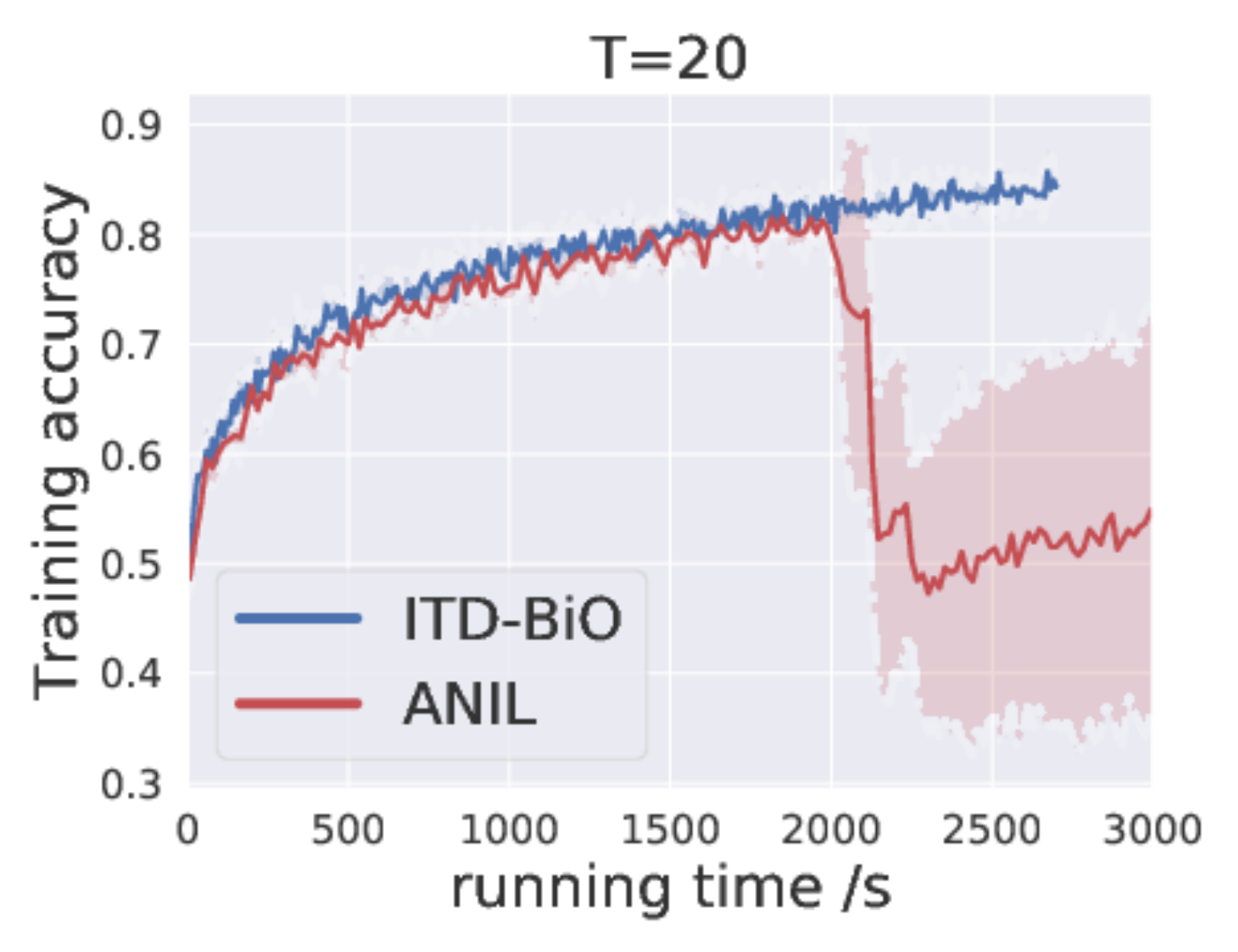}\includegraphics[width=60mm]{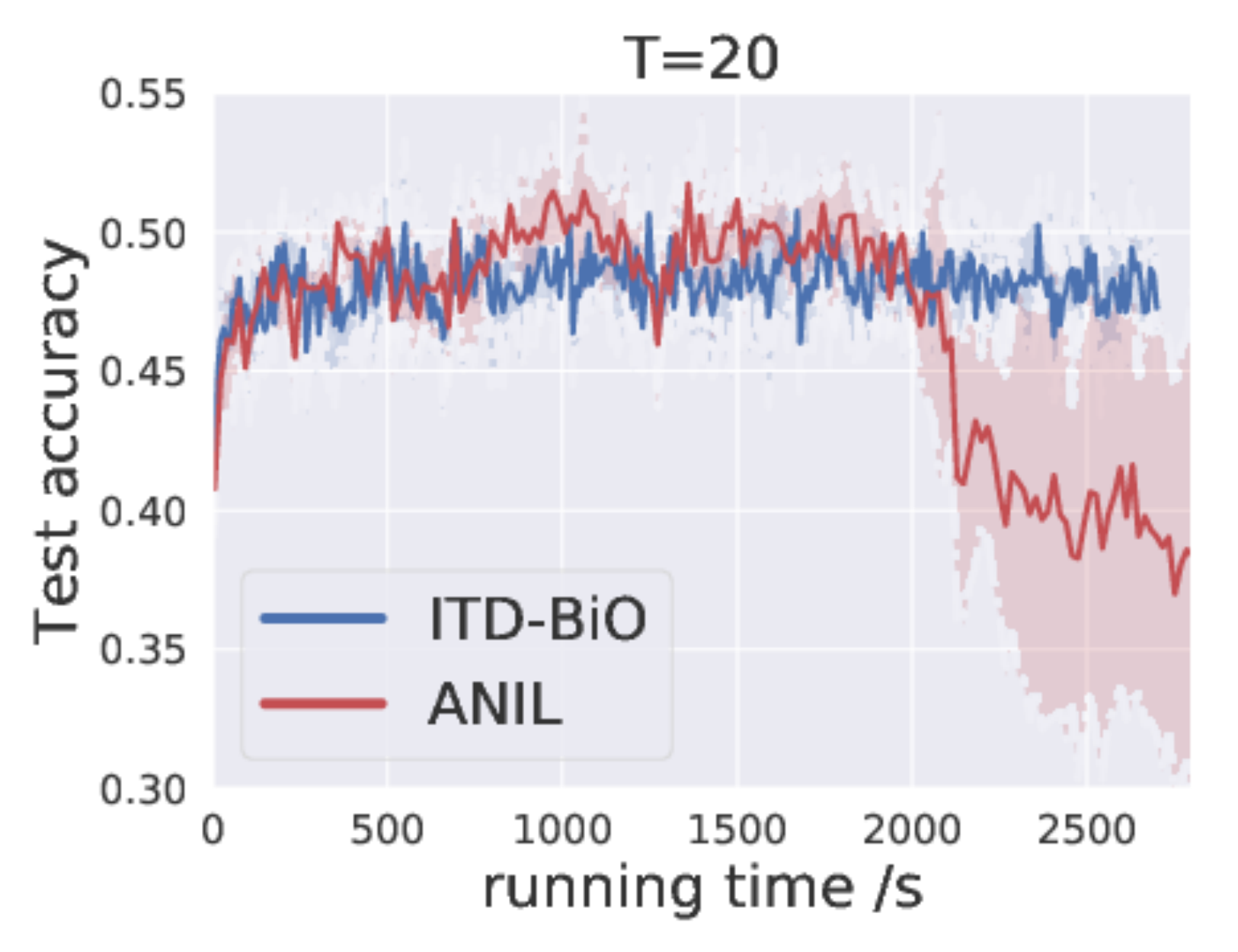}}  
	\vspace{-0.2cm}
	\caption{Comparison of ITD-BiO and ANIL with a relatively large inner-loop iteration number $T$.}\label{figure:resultlg}
	  \vspace{-0.2cm}
\end{figure*}

\subsection*{Experiments on Meta-Learning}
 To validate our theoretical results for deterministic bilevel optimization, we compare the performance among the following four algorithms: ITD-BiO, AID-BiO-constant (AID-BiO with a constant number of inner-loop steps as in our analysis), AID-BiO-increasing (AID-BiO with an increasing number of inner-loop steps under analysis in~\cite{ghadimi2018approximation}), and two popular meta-learning algorithms MAML\footnote{MAML consists of an inner loop for task  adaptation and an outer loop for meta initialization training.}~\cite{finn2017model} and ANIL\footnote{ANIL refers to almost no inner loop, which is an efficient MAML variant with task adaption on the last-layer of parameters.}~\cite{raghu2019rapid}. We conduct experiments over a 5-way 5-shot task on two datasets: FC100 and miniImageNet. The results are averaged over 10 trials with different random seeds. We provide the model architectures and hyperparameter settings in~\Cref{appen:meta_learning}.

It can be seen from \Cref{fig:strfc100bievlcsaa} that for both the miniImageNet and FC100 datasets, AID-BiO-constant converges  faster  than AID-BiO-increasing in terms of both the training accuracy and test accuracy, and achieves a better final test accuracy than ANIL and MAML. This demonstrates the superior improvement of our developed analysis over existing analysis in~\cite{ghadimi2018approximation} for AID-BiO algorithm.  
Moreover,  it can be observed that AID-BiO is slightly faster than ITD-BiO in terms of  the training accuracy and test accuracy. This is in consistence with our theoretical results. 

We also compare the robustness between the bilevel optimizer ITD-BiO (AID-BiO performs similarly to ITD-BiO in terms of the convergence rate) and ANIL when  the number $T$ (i.e., $D$ in \Cref{alg:main_deter})  of inner-loop steps is relatively large. 
It can be seen from~\Cref{figure:resultlg}  that when the number of inner-loop steps is large, i.e., $T=10$ for miniImageNet and $T=20$ for FC100, the bilevel optimizer ITD-BiO converges stably with a small variance, whereas ANIL suffers from a sudden descent at 1500s on miniImageNet  and even diverges after 2000s on FC100. 

\section{Summary of Contributions}
In this chapter, we develop a general and enhanced convergence rate analysis for AID- and ITD-based bilevel optimization algorithm for the nonconvex-strongly-convex bilevel problems. Our results also provide the theoretical guarantee for various bilevel optimizers in meta-learning. Our experiments validate our theoretical results. We anticipate that the convergence rate analysis that we develop will be useful for analyzing other  bilevel optimization problems with different loss geometries. 

\chapter{Acceleration Algorithms for Bilevel Optimization}\label{chp_acc_bilevel}

In this chapter,  we develop novel acceleration algorithms for bilevel optimization. All technical proofs for the results in this chapter are provided in \Cref{appendix:acc_bilevel}. 

\section{Bilevel Problem Class}


In this section, we introduce the problem class we are interested in. First, we suppose functions $f(x,y)$ and $g(x,y)$ in~\cref{objective_deter} satisfy the following standard smoothness property. 
\begin{assum}\label{fg:smooth}
The outer-level function $f$ satisfies, for $\forall x_1,x_2,x\in\mathbb{R}^p$ and $y_1,y_2,y\in\mathbb{R}^q$, there exist constants $L_x,L_{xy},L_y\geq 0$ such that
 \begin{align}\label{def:first}
\| \nabla_x  f(x_1,y)-\nabla_x f(x_2,y)\| \leq &L_x \|x_1-x_2\|, \nonumber
\\\| \nabla_x  f(x,y_1)-\nabla_x f(x,y_2)\| \leq& L_{xy} \|y_1-y_2\|, \nonumber
\\\| \nabla_y  f(x_1,y)-\nabla_y f(x_2,y)\| \leq &L_{xy} \|x_1-x_2\|, \nonumber
\\\| \nabla_y  f(x,y_1)-\nabla_y f(x,y_2)\| \leq &L_{y} \|y_1-y_2\|.
\end{align}
The inner-level function $g$ satisfies that, there exist $\widetilde L_{xy},\widetilde L_{y}\geq 0$ such that 
\begin{align}\label{df:sec} 
\| \nabla_y  g(x_1,y)-\nabla_y g(x_2,y)\| \leq &\widetilde L_{xy} \|x_1-x_2\|, \nonumber
\\\| \nabla_y  g(x,y_1)-\nabla_y g(x,y_2)\| \leq &\widetilde L_{y} \|y_1-y_2\|.
\end{align}
\end{assum}
The hypergradient $\nabla \Phi(x)$ plays an important role for designing bilevel optimization algorithms. The computation of $\nabla\Phi(x)$ involves Jacobians $\nabla_x\nabla_y g(x,y)$ and Hessians $\nabla_y^2 g(x,y)$. In this these, we are interested in the following inner-level problem with general Lipschitz continuous  Jacobians and Hessians, as adopted by~\cite{ghadimi2018approximation,ji2020bilevel,hong2020two}. For notational convenience, let $z:=(x,y)$ denote both variables.  
\begin{assum}\label{g:hessiansJaco}
There exist constants $\rho_{xy},\rho_{yy}\geq 0$ such that for $\forall\,(z_1,z_2)\in\mathbb{R}^p \times \mathbb{R}^q$, 
{\small
\begin{align}\label{def:three}
\|\nabla_x\nabla_y g(z_1)-\nabla_x\nabla_y g(z_2)\| \leq \rho_{xy}\|z_1-z_2\|,\;\; \|\nabla^2_y g(z_1)-\nabla_y^2 g(z_2)\| \leq \rho_{yy}\|z_1-z_2\|. 
\end{align} }
\end{assum} 
In this these, we study the following two classes of bilevel optimization problems. 
\begin{definition}[Bilevel Problem Classes]\label{de:pc}
Suppose $f$ and $g$ satisfy Assumptions~\ref{fg:smooth},~\ref{g:hessiansJaco} and there exists a constant $B>0$ such that  $\|x^*\|=B$, where $x^*\in\argmin_{x\in\mathbb{R}^p}\Phi(x)$.
We define the following two classes of bilevel problems under different geometries.     
\begin{list}{$\bullet$}{\topsep=0.1in \leftmargin=0.15in \rightmargin=0.1in \itemsep =0.01in}
\item {\bf Strongly-convex-strongly-convex class $\mathcal{F}_{scsc}:$}  $\Phi(\cdot)$ is $\mu_x$-strongly-convex and $g(x,\cdot)$ is $\mu_y$-strongly-convex. 
\item {\bf Convex-strongly-convex class $\mathcal{F}_{csc}:$}  $\Phi(\cdot)$ is convex and $g(x,\cdot)$ is $\mu_y$-strongly-convex. 
\end{list}
\end{definition}
A simple but important subclass of the bilevel problem class in \Cref{de:pc} includes the following quadratic inner-level functions $g(x,y)$.
%
\begin{align}\label{quadratic_case}
(\text{Quadratic $g$ subclass:}) \quad g(x,y) = \frac{1}{2} y^T H y +  x^T J y +b^Ty + h(x),   
\end{align}
where the Hessian $H$ and the Jacobian $J$ satisfy $ H \preceq \widetilde L_{y} I$ and $J \preceq \widetilde L_{xy} I$ for $\forall x\in\mathbb{R}^p$ and $\forall y\in\mathbb{R}^q$. 
Note that the above quadratic subclass also covers a large collection of applications such as few-shot meta-learning with shared embedding model~\cite{bertinetto2018meta} and biased regularization in hyperparameter optimization~\cite{grazzi2020iteration}. 

\section{Complexity Measures}
We introduce the criterion for measuring the computational complexity of bilevel optimization algorithms. Note that the updates of $x$ and $y$ of bilevel algorithms involve computing gradients, Jacobian- and Hessian-vector products. In practice, it has been shown in~\cite{griewank1993some,rajeswaran2019meta} that the time and memory cost for computing a  Hessian-vector product $\nabla^2f(\cdot) v $ (similarly for a Jacobian-vector product) via automatic differentiation (e.g., the widely-used reverse mode in PyTorch or TensorFlow)  is no more than a (universal) constant order (e.g., usually $2$-$5$ times) over the cost for computing gradient $\nabla f(\cdot)$. For this reason,  we take the following complexity measures.    
\begin{definition}[Complexity Measure]\label{complexity_measyre}
The total complexity $\mathcal{C}_{\text{\normalfont sub}}(\mathcal{A},\epsilon)$ of a bilevel optimization algorithm $\mathcal{A}$ to find a point $\bar x$ such that the suboptimality gap $f(\bar x)-\min_xf(x)\leq \epsilon$  is given by 
$\mathcal{C}_{\text{\normalfont sub}}(\mathcal{A},\epsilon) = \tau (n_J+n_H) + n_G$, 
where $n_J,n_ H$ and $n_G$ are the total numbers of Jacobian- and Hessian-vector product, and gradient evaluations, and $\tau>0$ is a universal constant.  Similarly, we define $\mathcal{C}_{\text{\normalfont norm}}(\mathcal{A},\epsilon)= \tau (n_J+n_H) + n_G$ as the complexity to find a point $\bar x$ such that the gradient norm $\|\nabla f(\bar x)\|\leq \epsilon$.
\end{definition}

\section{Accelerated Bilevel Optimization Algorithm: AccBiO} \label{upper_withoutB}


As shown in \Cref{alg:bioNoBG}, we propose a new accelerated algorithm named AccBiO for bilevel optimization.  
At the beginning of each outer iteration, we run $N$ steps of accelerated gradient descent (AGD) to get $y_k^N$ as an approximate of $y_k^*=\argmin_y g(x_k,y)$. Then, based on the inner-level output $y_k^N$, we construct a hypergradient estimate via $G_k:= \nabla_x f(x_k,y_k^N) -\nabla_x \nabla_y g(x_k,y_k^N)v_k^M$, where $v_k^M$ is 
  the output of an $M$-step heavy ball method with stepsizes $\eta$ and $\theta$ for solving a quadratic problem as shown in line $7$. Finally, as shown in lines $8$-$9$, we update the variables $z_{k}$ and $x_k$  using  Nesterov's momentum acceleration scheme~\cite{nesterov2018lectures} over the estimated hypergradient $G_k$.
  Next, we analyze  the convergence and complexity performance of AccBiO for the two bilevel optimization classes $\mathcal{F}_{scsc}$ and $\mathcal{F}_{csc}$ in \Cref{de:pc}.

\begin{algorithm}[t]
	\caption{Accelerated Bilevel Optimization (AccBiO) Algorithm} 
	\small
	\label{alg:bioNoBG}
	\begin{algorithmic}[1]
		\STATE {\bfseries Input:}  Initialization $ z_0=x_0=y_0=0$, parameters  $\lambda$ and $\theta$ 
		\FOR{$k=0,1,...,K$}
		\STATE{Set $y_k^0 = 0 $ as initialization}
		\FOR{$t=1,....,N$}
		\STATE{
		
		\vspace{-0.45cm} 
		\begin{align*}
		 \quad y_k^{t} &= s_k^{t-1} - \frac{1}{\widetilde L_y} \nabla_y g(x_k,s_k^{t-1}), \; s_k^{t} = \frac{2\sqrt{\kappa_y}}{\sqrt{\kappa_y}+1} y_k^{t} - \frac{\sqrt{\kappa_y}-1}{\sqrt{\kappa_y}+1} y_k^{t-1}.
		\end{align*}
		\vspace{-0.7cm} }
		\ENDFOR
                 \STATE{ \textit {Hypergradient computation}: \\
               \hspace{0.4cm}1) Get $v_k^M$ after running $M$ steps of heavy-ball method $$v_k^{t+1} = v_k^t-\lambda\nabla Q(v_k^t) +\theta(v_k^t-v_k^{t-1})$$ \\ \hspace{0.7cm} with initialization $v_k^{0}=v_k^{1}=0$ over \vspace{-0.2cm} 
               \begin{align*}
            \min_v Q(v):=\frac{1}{2}v^T\nabla_y^2 g(x_k,y_k^N) v - v^T
\nabla_y f( x_k,y^N_k)
\end{align*}\vspace{-0.4cm}\\
    \hspace{0.4cm}2) Compute $\nabla_x \nabla_y g(x_k,y_k^N)v_k^M $ via automatic differentiation; \\
    \vspace{0.1cm}
    \hspace{0.4cm}3) compute $G_k:= \nabla_x f(x_k,y_k^N) -\nabla_x \nabla_y g( x_k,y_k^N)v_k^M.$ 
              }
               \vspace{0.1cm}
               \STATE{Update $z_{k+1}=x_k -\frac{1}{L_\Phi} G_k$}
               \vspace{0.1cm}
               \STATE{Update $x_{k+1}=\Big(1+\frac{\sqrt{\kappa_x}-1}{\sqrt{\kappa_x}+1}\Big)z_{k+1} - \frac{\sqrt{\kappa_x}-1}{\sqrt{\kappa_x}+1} z_k$}                          
		\ENDFOR
	\end{algorithmic}
	\end{algorithm}

\section{Convergence Analysis for AccBiO}\label{sec:caoaccsas}	
We first consider the strongly-convex-strongly-convex setting, where $\Phi(x)$ is $\mu_x$-strongly-convex and $g(x,\cdot)$ is $\mu_y$-strongly-convex. 
 The following theorem provides a theoretical performance guarantee for AccBiO. Recall $x^*=\argmin_{x}\Phi(x)$.
\begin{theorem}\label{upper_srsr_withnoB}  
Suppose that $(f,g)$ belong to the strongly-convex-strongly-convex class $\mathcal{F}_{scsc}$ in \Cref{de:pc}. Choose stepsizes {\small $\lambda=\frac{4}{(\sqrt{\widetilde L_y}+\sqrt{\mu_y})^2}$} and {\small $\theta=\max\big\{\big(1-\sqrt{\lambda\mu_y}\big)^2,\big(1-\sqrt{\lambda\widetilde L_y}\big)^2\big\}$} for the heavy-ball method.  Let $\kappa_y=\frac{\widetilde L_y}{\mu_y}$ be the condition number for the inner-level function $g(x,\cdot)$ and $L_\Phi =\Theta\big(\frac{1}{\mu_y^2}+\big(\frac{ \rho_{yy}}{\mu_y^3} +  \frac{\rho_{xy}}{\mu_y^2}\big)\big(\Delta^*_{\text{\normalfont\tiny SCSC}}+\frac{\sqrt{\epsilon}}{\sqrt{\mu_x}\mu_y}\big) \big) $ be the smoothness parameter of the objective function $\Phi(\cdot)$, where $\Delta^*_{\text{\normalfont\tiny SCSC}}=\|\nabla_y f( x^*,y^*(x^*))\|+\frac{\|x^*\|}{\mu_y}+\frac{\sqrt{\Phi(0)-\Phi(x^*)}}{\sqrt{\mu_x}\mu_y}$. Then, we have 
\begin{align*}
\Phi(z_K)- \Phi(x^*)\leq  \Big(1 -\frac{1}{\sqrt{\kappa_x}} \Big)^{K}(\Phi(0) -\Phi(x^*)+\frac{\mu_x}{2} \|x^*\|^2) +\frac{\epsilon}{2},
\end{align*}
where $\kappa_x=\frac{L_\Phi}{\mu_x}$ is the condition number for $\Phi(\cdot)$. To achieve $\Phi(z_K)- \Phi(x^*)<\epsilon$, the complexity satisfies 
\begin{align}\label{mybabasohotssca1}
\mathcal{C}_{\text{\normalfont sub}}(\mathcal{A},\epsilon)\leq \mathcal{\widetilde O}\bigg(\sqrt{\frac{\widetilde L_y}{\mu_x\mu_y^{3}}}+\Big(\sqrt{\frac{ \rho_{yy}\widetilde L_y}{\mu_x\mu_y^{4}} }+ \sqrt{ \frac{\rho_{xy}\widetilde L_y}{\mu_x\mu_y^{3}}}\Big)\sqrt{\Delta^*_{\text{\normalfont\tiny SCSC}}}\bigg).
\end{align}
\end{theorem}	
To the best of our knowledge, our result in \Cref{upper_srsr_withnoB} is the first-known upper bound on the computational complexity 
for strongly-convex bilevel optimization under only mild assumptions on the Lipschitz continuity of the first- and second-order derivatives of the outer- and inner-level functions $f,g$. As a comparison, existing results in \cite{ghadimi2018approximation,ji2020bilevel} for bilevel optimization further make a strong assumption that the gradient norm $\|\nabla_y f(x,y)\|$ is bounded for all $(x,y)\in\mathbb{R}^p\times\mathbb{R}^q$ to upper-bound the smoothness parameter $L_{\Phi_k}$ of $\Phi(x_k)$ and the hypergradient estimation error $\|G_k-\nabla \Phi(x_k)\|$ at the $k^{th}$ iteration. This is because  $L_{\Phi_k}$ and $\|G_k-\nabla \Phi(x_k)\|$ turn out to be increasing with the gradient norm $\|\nabla_y f(x_k,y^*(x_k))\|$, for which it is challenging to prove the boundedness given the theoretical frameworks in \cite{ghadimi2018approximation,ji2020bilevel} where no results on bounded iterates are established.  Our analysis does not require such a restrictive assumption because we show by induction that the optimality gap $\|x_k-x^*\|$ is well bounded as the algorithm runs. As a result, we can guarantee the boundedness of the smoothness parameter $L_{\Phi_k}$ and the error $\|G_k-\nabla \Phi(x_k)\|$ during the entire optimization process.  
In \Cref{sec:upper_gBscss}, we further develop tighter upper bounds than existing results under this additional bounded gradient assumption. 


Based on \Cref{upper_srsr_withnoB}, we next study the quadratic $g$ subclass, where the inner-level function $g(x,y)$ takes a quadratic form as in \cref{quadratic_case}. 
The following corollary provides upper bounds on the computational complexity of AccBiO under this case.
\begin{corollary}[\bf Quadratic $g$ subclass]\label{coro:quadaticSr}
Under the same setting of \Cref{upper_srsr_withnoB}, consider the quadratic inner-level function $g(x,y)$ in \cref{quadratic_case}, where $\nabla_x\nabla_y g(\cdot,\cdot)$ and $\nabla_y^2 g(\cdot,\cdot)$ are constant. To achieve $\Phi(z_K)- \Phi(x^*)<\epsilon$, 
the complexity $\mathcal{C}_{\text{\normalfont sub}}(\mathcal{A},\epsilon)$ is at most  
$\mathcal{C}_{\text{\normalfont sub}}(\mathcal{A},\epsilon)\leq \mathcal{\widetilde O}\Big(\sqrt{\frac{\widetilde L_y}{\mu_x\mu_y^{3}}}\Big).$
\end{corollary}
\Cref{coro:quadaticSr} shows that for the quadratic $g$ subclass, the complexity upper bound in \Cref{upper_srsr_withnoB} specializes to $\mathcal{\widetilde O}\Big(\sqrt{\frac{\widetilde L_y}{\mu_x\mu_y^{3}}}\Big)$. This improvement over the complexity for the general case in~\cref{mybabasohotssca1} comes from tighter upper bounds on the smoothness parameter $L_\Phi$ of the objective function $\Phi(x)$ and a smaller hypergradient estimation error $\|G_k-\nabla\Phi(x_k)\|$.  

We next provide an upper bound for the convex-strongly-convex setting, where the  function $\Phi(x)$ is convex. Recall from \Cref{de:pc} that $\|x^*\|=B$ for some constant $B>0$, where $x^*$ is one minimizer of $\Phi(\cdot)$.
For this case, we construct a strongly-convex-strongly-convex function $\widetilde \Phi(\cdot)=\widetilde f(x,y^*(x))$ by adding a small quadratic regularization to the outer-level function $f(x,y)$, i.e., 
\begin{align}\label{regularized_fxy}
\widetilde f(x,y) = f(x,y) +\frac{\epsilon}{2R} \|x\|^2. 
\end{align}
Then, we can apply the results in \Cref{upper_srsr_withnoB} to $\widetilde\Phi(x)$ and obtain the following theorem. 
\begin{theorem}\label{th:upper_csc1sc}
Suppose that $(f,g)$ belong to the convex-strongly-convex class $\mathcal{F}_{csc}$ in \Cref{de:pc}. Let $L_{\widetilde \Phi}$ be the smoothness parameter of function $\widetilde \Phi(\cdot)$, which takes the same form as $L_\Phi$ in \Cref{upper_srsr_withnoB} except that $L_x,f,x^*$ and $\Phi$ become $L_x+\frac{\epsilon}{R},\widetilde f,\widetilde x^*$ and $\widetilde \Phi$, respectively. Let  $\Delta^*_{\text{\normalfont\tiny CSC}} = \|\nabla_y f( x^*,y^*(x^*))\|+\frac{\|x^*\|}{\mu_y}+\frac{(\|x^*\|+1)\sqrt{(\Phi(0)-\Phi(x^*))}}{\sqrt{\epsilon}\mu_y}$.  We consider two widely-used convergence criterions as follows. 
\begin{itemize}
\item {\bf (Suboptimality gap)} Choose $R=B^2$ in \cref{regularized_fxy}, and choose the same parameters as in \Cref{upper_srsr_withnoB} with $\epsilon$ and $\mu_x$ being replaced by $\epsilon/2$ and $\frac{\epsilon}{R}$, respectively.
To achieve $\Phi(z_K) - \Phi(x^*)\leq \epsilon$, the required complexity is at most
 \begin{align*}
\mathcal{C}_{\text{\normalfont sub}}(\mathcal{A},\epsilon) \leq  \mathcal{\widetilde O}\Big( B\Big( \sqrt{\frac{\widetilde L_y}{\epsilon\mu_y^3}}+\Big(\sqrt{\frac{\rho_{yy}\widetilde L_y}{\epsilon\mu_y^{4}}} +  \sqrt{\frac{\rho_{xy}\widetilde L_y}{\epsilon\mu_y^3}}\Big)\sqrt{\Delta^*_{\text{\normalfont\tiny CSC}}}\Big)\Big).
\end{align*}
\item {\bf (Gradient norm)} Choose $R=B$ in \cref{regularized_fxy}, and choose the same parameters as in \Cref{upper_srsr_withnoB} with 
 $\epsilon$ and $\mu_x$ being replaced by $\epsilon^2/(4L_{\widetilde \Phi}+ \frac{8\epsilon}{R})$ and $\frac{\epsilon}{R}$, respectively. 
 To achieve $\|\nabla \Phi (z_k)\|\leq 5\epsilon$, the required complexity is at most 
  \begin{align*}
\mathcal{C}_{\text{\normalfont norm}}(\mathcal{A},\epsilon) \leq \mathcal{\widetilde O}\Big( \Big( \sqrt{\frac{B\widetilde L_y}{\epsilon\mu_y^3}}+\Big(\sqrt{\frac{B\rho_{yy}\widetilde L_y}{\epsilon\mu_y^{4}}} +  \sqrt{\frac{B\rho_{xy}\widetilde L_y}{\epsilon\mu_y^3}}\Big)\sqrt{\Delta^*_{\text{\normalfont\tiny CSC}}}\Big)\Big).
\end{align*}
\end{itemize}
\end{theorem}  
As far as we know, \Cref{th:upper_csc1sc} is the first convergence result for convex-strongly-convex bilevel optimization without the bounded gradient assumption. Then, similarly to \Cref{coro:quadaticSr}, we also study the quadratic $g(x,y)$ case where $g$ takes the quadratic form as given in \cref{quadratic_case}.  
\begin{corollary}[\bf Quadratic $g$ subclass]\label{coro:quadaticConv}
Under the same setting of \Cref{th:upper_csc1sc}, consider the quadratic inner-level function $g(x,y)$ 
where  $\nabla_x\nabla_y g(\cdot,\cdot)$ and $\nabla_y^2 g(\cdot,\cdot)$ are constant. Then, we have
\begin{itemize}
\item {\bf (Suboptimality  gap)} To achieve $\Phi(z_K) - \Phi(x^*)\leq \epsilon$, we have $\mathcal{C}_{\text{\normalfont sub}}(\mathcal{A},\epsilon) \leq  \mathcal{\widetilde O}\Big(B \sqrt{\frac{\widetilde L_y}{\epsilon\mu_y^3}}\Big)$.
\item {\bf (Gradient norm)} To achieve $\|\nabla \Phi (z_k)\|\leq \epsilon$, we have $\mathcal{C}_{\text{\normalfont norm}}(\mathcal{A},\epsilon) \leq \mathcal{\widetilde O}\Big(\sqrt{\frac{B\widetilde L_y}{\epsilon\mu_y^3}}\Big).$
\end{itemize}
\end{corollary}
It can be seen from \Cref{coro:quadaticConv} that for the quadratic $g$ subclass, AccBiO achieves a computational complexity of $ \mathcal{\widetilde O}\Big(\sqrt{\frac{B\widetilde L_y}{\epsilon\mu_y^3}}\Big)$ in term of the gradient norm. For the case where $\widetilde L_y\leq\mathcal{O}(\mu_y)$, the complexity becomes $ \mathcal{\widetilde O}\Big(\sqrt{\frac{B}{\epsilon\mu_y^2}}\Big)$. 

\section{Upper Bounds with Bounded Gradient Assumption} \label{sec:upper_gBscss}

Our analysis in \Cref{sec:caoaccsas} for AccBiO does not make the bounded gradient assumption, which has been commonly taken in the existing studies~\cite{ghadimi2018approximation,ji2020bilevel,hong2020two,ji2020convergence}. 
In this section, we establish tighter upper bounds than those in existing works~\cite{ghadimi2018approximation,ji2020bilevel} under such an additional assumption.

\begin{assum}[{\bf Bounded gradient}]\label{bounded_f_assump}
There exists a constant $U$ such that for any $(x^\prime,y^\prime)\in\mathbb{R}^p\times \mathbb{R}^q$, $\|\nabla_y f(x^\prime,y^\prime)\|\leq U$.
\end{assum}

We propose an accelerated algorithm named AccBiO-BG in \Cref{alg:bio} for bilevel optimization under the additional bounded gradient assumption. Similarly to  AccBiO, AccBiO-BG first runs $N$ steps of accelerated gradient descent (AGD) at each outer iteration. Note that AccBiO-BG here adopts a warm start strategy with $y_k^0=y_{k-1}^N$ so that our analysis does not require the boundedness of $y^*(x_k),k=0...,K$ and reduces the total computational complexity. 
Then, 
AccBiO-BG constructs the hypergradient estimate $G_k:= \nabla_x f(\widetilde x_k,y_k^N) -\nabla_x \nabla_y g(\widetilde x_k,y_k^N)v_k^M$ following the same steps as in AccBiO. 
Finally, we update variables $x_{k},z_{k}$ via two  accelerated gradient steps, where we incorporate a variant~\cite{ghadimi2016accelerated} of Nesterov's momentum. We use this variant instead of vanilla Nesterov's momentum~\cite{nesterov2018lectures} in \Cref{alg:bioNoBG}, because the resulting analysis is easier to handle the warm start strategy, which backpropagates the tracking error $\|y^N_k-y^*(x_k)\|$ to previous loops. 



\begin{algorithm}[t]
	\caption{Accelerated Bilevel Optimizer under Bounded Gradient Assumption (AccBiO-BG) } 
	\small
	\label{alg:bio}
	\begin{algorithmic}[1]
		\STATE {\bfseries Input:}  Initialization $ z_0=x_0=y_0=0$, parameters  $\eta_k,\tau_k.\alpha_k,\beta_k,\lambda$ and $\theta$ 
		\FOR{$k=0,...,K$}
		\STATE{Set $\widetilde x_k = \eta_k x_k + (1-\eta_k)z_k$}
		\STATE{Set $y_k^0 = y_{k-1}^{N} \mbox{ if }\; k> 0$ and $y_0$ otherwise (warm start)}
		\FOR{$t=1,....,N$}
		\STATE{
		
		\vspace{-0.45cm} 
		\begin{align*}
		\mbox{(AGD:)} \quad y_k^{t} &= s_k^{t-1} - \frac{1}{\widetilde L_y} \nabla_y g(\widetilde x_k,s_k^{t-1}), \quad s_k^{t} = \frac{2\sqrt{\kappa_y}}{\sqrt{\kappa_y}+1} y_k^{t} - \frac{\sqrt{\kappa_y}-1}{\sqrt{\kappa_y}+1} y_k^{t-1}.
		\end{align*}
		\vspace{-0.7cm} }
		\ENDFOR
                 \STATE{ \textit {Hypergradient computation}: \\
               \hspace{0.4cm}1) Get $v_k^M$ after running $M$ steps of heavy-ball method $$v_k^{t+1} = v_k^t-\lambda\nabla Q(v_k^t) +\theta(v_k^t-v_k^{t-1})$$ \\ \hspace{0.7cm} with initialization $v_k^{0}=v_k^{1}=0$ over \vspace{-0.3cm} 
               \begin{align*}
             \mbox{(Quadratic programming:)}\;\;  \min_v Q(v):=\frac{1}{2}v^T\nabla_y^2 g(\widetilde x_k,y_k^N) v - v^T
\nabla_y f(\widetilde x_k,y^N_k);
\end{align*}\vspace{-0.4cm}\\
    \hspace{0.4cm}2) Compute Jacobian-vector product $\nabla_x \nabla_y g(\widetilde x_k,y_k^N)v_k^M $ via automatic differentiation; \\
    \vspace{0.1cm}
    \hspace{0.4cm}3) compute {\bf hypergradient estimate} $G_k:= \nabla_x f(\widetilde x_k,y_k^N) -\nabla_x \nabla_y g(\widetilde x_k,y_k^N)v_k^M.$ 
              }
               \vspace{0.1cm}
               \STATE{Update $x_{k+1}=\tau_k \widetilde x_{k}+(1-\tau_k)x_k-\beta_k G_k$}
               \vspace{0.1cm}
               \STATE{Update $z_{k+1}=\widetilde x_{k}-\alpha_k G_k$}\vspace{0.1cm}
                              
		\ENDFOR
	\end{algorithmic}
	\end{algorithm}

\vspace{-0.1cm}	
\section{Convergence Analysis for AccBiO-BG}	
We first consider the strongly-convex-strongly-convex setting under the bounded gradient assumption. The following theorem provides a theoretical convergence guarantee for AccBiO-BG.
\begin{theorem}\label{upper_srsr} 
Suppose that $(f,g)$ belong to the strongly-convex-strongly-convex class $\mathcal{F}_{scsc}$ in \Cref{de:pc} and further suppose Assumption~\ref{bounded_f_assump} is satisfied. 
Choose $\alpha_k=\alpha\leq \frac{1}{2L_\Phi}$, $\eta_k=\frac{\sqrt{\alpha\mu_x}}{\sqrt{\alpha\mu_x}+2}$, $\tau_k=\frac{\sqrt{\alpha\mu_x}}{2}$ and $\beta_k=\sqrt{\frac{\alpha}{\mu_x}}$, where $L_\Phi$ is the smoothness parameter of $\Phi(x)$. Choose stepsizes {\small $\lambda=\frac{4}{(\sqrt{\widetilde L_y}+\sqrt{\mu_y})^2}$} and {\small $\theta=\max\big\{\big(1-\sqrt{\lambda\mu_y}\big)^2,\big(1-\sqrt{\lambda\widetilde L_y}\big)^2\big\}$} for the heavy-ball method.  
Then, to achieve $\Phi(z_K) - \Phi(x^*) \leq \epsilon$, the required complexity $\mathcal{C}_{\text{\normalfont sub}}(\mathcal{A},\epsilon)$ is at most 
\begin{align*}
\mathcal{C}_{\text{\normalfont sub}}(\mathcal{A},\epsilon)\leq  \mathcal{O}\Big(\sqrt{\frac{\widetilde L_y}{\mu_x\mu_y^4}}\log\frac{\mbox{\normalfont\small poly}(\mu_x,\mu_y,U,\Phi(x_{0})-\Phi(x^*))}{\epsilon}  \log \frac{\mbox{\normalfont \small poly}(\mu_x,\mu_y,U)}{\epsilon} \Big).
\end{align*}
\end{theorem}	
The proof of \Cref{upper_srsr} is provided in \Cref{proof:upss_wb}. As shown in \Cref{upper_srsr}, the upper bound achieved by our proposed AccBiO-BG algorithm is {\small$\mathcal{\widetilde O}(\sqrt{\frac{1}{\mu_x\mu_y^4}})$}. This bound improves the best known $\mathcal{\widetilde O}\big(\max\big\{\frac{1}{\mu_x\mu_y^3},\frac{\widetilde L^2_y}{\mu_y^2}\big\}\big)$ (see eq. (2.60) therein)  of the accelerated bilevel approximation algorithm (ABA) in  \cite{ghadimi2018approximation} by a factor of $\mathcal{O}(\mu_x^{-1/2}\mu_y^{-1})$. 



We then study the convex-strongly-convex bilevel optimization under the bounded gradient assumption. Similarly to \Cref{th:upper_csc1sc}, we consider a strongly-convex-strongly-convex function $\widetilde \Phi(\cdot)=\widetilde f(x,y^*(x))$  with $\widetilde f(x,y) = f(x,y) +\frac{\epsilon}{2B^2} \|x\|^2$, where $B=\|x^*\|$ as defined in~\Cref{de:pc}. 
Then, we have the following theorem. 
\begin{theorem}\label{convex_upper_BG}
Suppose that $(f,g)$ belong to the convex-strongly-convex class $\mathcal{F}_{csc}$ in \Cref{de:pc} and further suppose Assumption~\ref{bounded_f_assump} is satisfied. Let $L_{\widetilde\Phi}$ be the smoothness parameter of $\widetilde\Phi(\cdot)$, which takes the same form as $L_\Phi$ in \Cref{upper_srsr} but with $L_x$ being  replaced by  $L_x+ \frac{\epsilon}{B^2}$. 
Choose the same parameter as in \Cref{upper_srsr} with $\alpha=\frac{1}{2L_{\widetilde \Phi}}$ and $\mu_x=\frac{\epsilon}{B^2}$. Then, to achieve $\Phi(z_K) - \Phi(x^*) \leq \epsilon$, the required complexity $\mathcal{C}_{\text{\normalfont sub}}(\mathcal{A},\epsilon)$ is 
\begin{align*}
\mathcal{C}_{\text{\normalfont sub}}(\mathcal{A},\epsilon)\leq \mathcal{O}\Big(B\sqrt{\frac{\widetilde L_y}{\epsilon\mu_y^4}}\log\frac{\mbox{\small \normalfont poly}(\epsilon,\mu_y,B,U,\Phi(x_{0})-\Phi(x^*))}{\epsilon}  \log \frac{\mbox{\small \normalfont poly}(B,\epsilon,\mu_y,U)}{\epsilon} \Big).
\end{align*}
\end{theorem}
As shown in \Cref{convex_upper_BG}, our proposed AccBiO-BG algorithm achieves a complexity  of $\mathcal{\widetilde O}\big(\frac{1}{\epsilon^{0.5}\mu_y^2}\big)$, which improves the best known result $\mathcal{O}\big(\frac{1}{\epsilon^{0.75}\mu_y^{6.75}}\big)$ achieved by the ABA algorithm in \cite{ghadimi2018approximation} (see eq. (2.61) therein) by an order of $\mathcal{\widetilde O}\big(\frac{1}{\epsilon^{0.25}\mu_y^{4.75}}\big)$.

\section{Summary of Contributions}
In this chapter, we propose new acceleration algorithms named AccBiO and AccBiO-BG for bilevel optimization. 
For AccBiO, we provide the first-known convergence analysis without the bounded gradient assumption (which was made in existing studies). We further show AccBiO-BG achieves a significantly lower complexity than existing bilevel algorithms when the bounded gradient assumption does hold. We anticipate the proposed AccBiO and AccBiO-BG can be useful for various applications such as meta-learning and hyperparameter optimization. 

\chapter{Lower Bounds and Optimality for Bilevel Optimization}\label{chp_lower_bilevel}

In this chapter,  we develop lower bounds for problem-based bilevel optimization and discuss the optimality of the   AccBiO algorithm proposed in \Cref{chp_acc_bilevel}.  All technical proofs for the results in this chapter are provided in \Cref{appendix:lower_bilevel}. 


\section{Algorithm Class for Bilevel Optimization}
Compared to  {\bf minimization} and {\bf minimax} problems, the most different and challenging component of  {\bf bilevel optimization} lies in the computation  of the {\em hypergradient} $\nabla\Phi(\cdot)$. In specific, when functions $f$ and $g$ are continuously twice differentiable, it has been shown in~\cite{foo2008efficient} that  $\nabla\Phi(\cdot)$ takes the form of 
\begin{align}\label{hyperG}
\nabla \Phi(x) =&  \nabla_x f(x,y^*(x)) -\nabla_x \nabla_y g(x,y^*(x)) [\nabla_y^2 g(x,y^*(x)) ]^{-1}\nabla_y f(x,y^*(x)). 
\end{align}
In practice, exactly calculating the Hessian inverse  $(\nabla_y^2 g(\cdot) )^{-1}$ in \cref{hyperG} is computationally infeasible, and hence two types of hypergradient  estimation methods named AID and ITD have been proposed, where only efficient  {\bf Hessian- and Jacobian-vector products} need to be computed.
We present ITD- and AID-based bilevel optimization algorithms as follows. 

\begin{example}[ITD-based Bilevel Algorithms]\label{exam:itd}\cite{maclaurin2015gradient,franceschi2017forward,ji2020bilevel,grazzi2020iteration} Such type of algorithms use ITD-based methods for hypergradient computation, and take the following updates.  

\vspace{0.2cm}
\noindent For each outer iteration $m=0,....,Q-1$,
\begin{list}{$\bullet$}{\topsep=0.3ex \leftmargin=0.09in \rightmargin=0.in \itemsep =0.01in}
\item  Update variable $y$ for $N$ times via iterative algorithms (e.g., gradient descent, accelerated gradient methods).
\begin{align}\label{inner_up}
(\text{\normalfont Gradient descent:}) \quad y_{m}^t =  y_m^{t-1} - \eta \nabla_y g(x_m, y_m^{t-1}), t=1,...,N.
\end{align}
\item Compute the hypergradient estimate $G_m=\frac{\partial f(x_m,y_m^N(x_m))}{\partial x_m}$ via backpropagation. Under the gradient updates in~\cref{inner_up}, $G_m$ takes the form of 
\begin{align}\label{g:omamascsa}
G_m =& \nabla_x f(x_m,y^N_m) \nonumber
\\&-\eta\sum_{t=0}^{N-1}\nabla_x\nabla_y g(x_m,y_m^t)\prod_{j=t+1}^{N-1}(I-\eta  \nabla^2_y g(x_m,y_m^{j}))\nabla_y f(x_m,y_m^N).
\end{align}
A similar form holds for case when updating $y$ with accelerated gradient methods. 
\item Update $x$ based on $G_m$ via gradient-based iterative methods.
\end{list}
\end{example}
It can be seen from~\cref{g:omamascsa}  that only Hessian-vector products $ \nabla^2_y g(x_m,y_m^{j})v_j, j=1,...,N$ and Jacobian-vector products $\nabla_x\nabla_y g(x_m,y_m^j)v_{j},j=1,...,N$ are computed, where each $v_j$ is  obtained recursively via   
\begin{align*}
v_{j-1} = \underbrace{(I-\alpha  \nabla^2_y g(x_m,y_m^{j}))v_j}_{\text{Hessian-vector product}} \text{ with } v_N = \nabla_y f(x_m,y_m^N).
\end{align*} 
The same observation applies to the following AID-based bilevel methods.

\begin{example}[AID-based Bilevel Algorithms]\label{exam:aids}\cite{domke2012generic,pedregosa2016hyperparameter,grazzi2020iteration,ji2020bilevel} Such a class of algorithms use AID-based approaches for hypergradient computation, and take the following updates.

\vspace{0.2cm}
\noindent For each outer iteration $m=0,....,Q-1$,
\begin{list}{$\bullet$}{\topsep=0.3ex \leftmargin=0.11in \rightmargin=0.in \itemsep =0.01in}
\item  Update variable $y$ using gradient decent (GD) or accelerated gradient descent (AGD)
\begin{align} 
(\mbox{\normalfont GD:})\quad y_{m}^t &=  y_m^{t-1} - \eta \nabla_y g(x_m, y_m^{t-1}), t=1,...,N \nonumber
\\ (\mbox{\normalfont AGD:}) \quad y_m^{t} &= z_m^{t-1} - \eta \nabla_y g(x_m,z_m^{t-1}),  \nonumber
\\z_m^{t} &= \Big(1+\frac{\sqrt{\kappa_y}-1}{\sqrt{\kappa_y}+1}\Big) y_m^{t} - \frac{\sqrt{\kappa_y}-1}{\sqrt{\kappa_y}+1} y_m^{t-1}, t= 1,...,N 
\end{align}
where $\kappa_y= \widetilde L_y/\mu_y$ be the condition number of the inner-level function $g(x,\cdot)$. 
\item Update $x$ via $x_{m+1} = x_{m}-\beta G_m$, where $G_m$ is constructed via AID and takes the form of 
\begin{align}\label{aid_hgd}
G_m = \nabla_x f(x_m,y_m^N)- \nabla_x \nabla_y g(x_m,y_m^N)v_m^T, 
\end{align}
where vector $v_m^S$ is obtained by running $S$ steps of GD (with initialization $v_m^0=0$) or accelerated gradient methods (e.g., heavy-ball method with $v_m^0 = v_m^1=0$) to solve a quadratic programming
\begin{align}\label{exam:quadratic}
\min_{v} Q(v):= \frac{1}{2}v^T\nabla_y^2 g(x_m,y_m^N)v-v^T\nabla_y f(x_m,y_m^N).
\end{align}
\end{list}
\end{example}
We next verify that \cref{exam:aids} belongs  to the algorithm class defined in \Cref{alg_class}.  
For the case when $S$-steps GD with initialization $\bf{0}$ is applied to solve the quadratic program in~\cref{exam:quadratic}, simple telescoping yields 
\begin{align*}
v_m^{S} = \alpha \sum_{t=0}^{S-1}(I-\alpha\nabla_y^2g(x_m,y_m^N))^{t} \nabla_yf(x_m,y_m^N),
\end{align*}
which, incorporated into \cref{aid_hgd}, implies that $G_m$ falls into the span subspaces in~\cref{x_span}, and hence all updates fall into the subspaces $\mathcal{H}_x^k, \mathcal{H}_y^k, k=0,...,K$ defined in \Cref{alg_class}. For the case when heavy-ball method, i.e., $v_m^{t+1} = v_m^t-\eta_t \nabla Q(v_m^t) +\theta_t(v_m^t-v_m^{t-1})$, with initialization $v_m^0=v_m^1=\bf{0}$ is applied for~\cref{exam:quadratic}, expressing the updates via a dynamic system perspective yields 
\begin{align}\label{vmt_form} 
\begin{bmatrix}
v_m^{S}\\v_m^{S-1} 
\end{bmatrix}
= 
\sum_{s=2}^{S}\prod_{t=s}^{S-1}
\begin{bmatrix}
(1+\theta_t) I- \eta_t \nabla_y^2 g(x_m,y_m^N) & -\theta_t I\\I & \bf{0}
\end{bmatrix}
 \begin{bmatrix}
\eta_t \nabla_y f(x_m,y_m^N)\\\bf{0} 
\end{bmatrix}.
\end{align}


We next introduce a general hypergradient-based algorithm class, which includes the above AID- and ITD-based bilevel optimization algorithms. 
\begin{definition}[Hypergradient-Based Algorithm Class]\label{alg_class}
Suppose there are totally $K$ iterations and $x$ is updated for $Q$ times at iterations indexed by $s_i,i=1,...,Q-1$ with $s_0<...<s_{Q-1}\leq K$. Note that $Q$ is an arbitrary  positive integer in $0,...,K$ and $s_i,i=1,...,Q-1$ are $Q$ arbitrary distinct integers in $0,...,K$. The iterates $\{(x_k,y_k)\}_{k=0,...,K}$ are generated according to $(x_k,y_k)\in\mathcal{H}_x^k, \mathcal{H}_y^k$, where the linear subspaces $\mathcal{H}_x^k, \mathcal{H}_y^k, k=0,...,K$ with $\mathcal{H}_x^{0} = \mathcal{H}_y^0=\{{\bf 0 }\}$ are given as follows. 
\begin{align}\label{hxk}
H_y^{k+1} = \text{\normalfont Span}\left\{y_i,\nabla_y g(\widetilde x_i,\widetilde y_i), \forall \widetilde x_i\in\mathcal{H}_x^i,\forall y_i,\widetilde y_i\in\mathcal{H}_y^i, 1\leq i \leq k \right\}.  
\end{align}
For $x$, we have, for all $m=0,..., Q-1$,  
\begin{align}\label{x_span}
&\mathcal{H}_x^{s_m} =  \text{\normalfont Span} \Big\{x_i,\nabla_xf(\widetilde x_i,\widetilde y_i),\nabla_x\nabla_y g( x_{i}^t, y_{i}^t)\prod_{j=1}^t(I-\alpha \nabla_y^2g(x_{i,j}^t,y_{i,j}^t))\nabla_y f(\hat x_i,\hat y_i), \nonumber
\\&\hspace{0.3cm}t=0,...,T,\forall x_i, \hat x_i, x_{i}^t,x_{i,j}^t \in \mathcal{H}_x^i, \forall \hat y_i, y_{i}^t,y_{i,j}^t \in \mathcal{H}_y^i, 1\leq i \leq s_m-1, \forall \alpha\in\mathbb{R}, T\in \mathbb{N}\Big\}\nonumber
\\ &\mathcal{H}_x^n = \mathcal{H}_x^{s_m}, \forall s_m\leq n\leq s_{m+1}-1 \text{ with } s_Q=K+1. 
\end{align}
\end{definition}
It can be easily verified that the ITD-based methods in \Cref{exam:itd} belong to the algorithm class in  \Cref{alg_class}. 
Combining $v_m^S$ in \cref{vmt_form} with~\cref{aid_hgd}, we can see that the resulting $G_m$ falls into the span subspaces in~\cref{x_span}, and hence the AID-based methods in \Cref{exam:aids} also belong to the algorithm class in  \Cref{alg_class}. 
Note that the algorithm class considered in \Cref{alg_class} also includes single-loop (i.e., updating $x$ and $y$ simultaneously) bilevel optimization algorithms, e.g., by setting $N=1$ in \Cref{exam:itd} and \Cref{exam:aids}.

Note that in this algorithm class, $x$ can be updated at any iteration due to the arbitrary choices of $Q,s_i,i=1,...,Q-1$ and the hypergradient estimate can be constructed using any combination of points in the historical search space (similarly for $y$). 

%

\section{Lower Bound for Strongly-Convex-Strongly-Convex Case}
We first study the case when $\Phi(\cdot)$ is $\mu_x$-strongly-convex and the inner-level function $g(x,\cdot)$ is $\mu_y$-strongly-convex. We present our lower bound result for this case in the following theorem.
\begin{theorem}\label{thm:low1} 
Let $M = K+QT+Q + 2$ with $K, T,Q $ given by~\Cref{alg_class}. 
There exists a problem instance in $\mathcal{F}_{scsc}$ defined in~\Cref{de:pc} with dimensions $p=q=d> \max\big\{2M,M+1+\log_{r}\big(\mbox{\normalfont poly}\big(\mu_x\mu_y^2\big)\big)\big\}$ such that for this problem, any output $x^K$ belonging to the subspace $\mathcal{H}_x^K$, i.e., generated by any algorithm in the hypergradient-based algorithm class
defined in~\Cref{alg_class}, satisfies 
\begin{align}\label{result:first}
\Phi(x^K)-\Phi(x^*) \geq \Omega\Big(\mu_x\mu_y^2(\Phi(x_0)-\Phi(x^*)) r^{2M}\Big),
\end{align}
where $x^*=\argmin_{x\in\mathbb{R}^d}\Phi(x)$ and the parameter $r$ satisfies $1-\Big(\frac{1}{2}+\sqrt{\xi+\frac{1}{4}}\Big)^{-1} < r< 1$ with $\xi$ given by 
$\xi\geq \frac{\widetilde L_y}{4\mu_y} +\frac{L_x}{8\mu_x} +\frac{L_y\widetilde L_{xy}^2}{8\mu_x\mu_y^2} -\frac{3}{8}\geq \Omega \big(\frac{1}{\mu_x\mu_y^2}\big)$. 
To achieve $\Phi(x^K)-\Phi(x^*)\leq \epsilon$,  the total complexity $\mathcal{C}_{\text{\normalfont sub}}(\mathcal{A},\epsilon)$ satisfies 
\begin{align*}
\mathcal{C}_{\text{\normalfont sub}}(\mathcal{A},\epsilon) \geq \Omega\bigg(\sqrt{\frac{L_y\widetilde L_{xy}^2}{\mu_x\mu_y^2}}\log  \frac{\mu_x\mu_y^2(\Phi(x_0)-\Phi(x^*))}{\epsilon}  \bigg).
\end{align*}
\end{theorem}
Note that the inner-level function $g(x,y)$ in our constructed worst-case instance takes the same quadratic form as in \cref{quadratic_case} so that the lower bound in \Cref{thm:low1} also applies to the quadratic $g$ subclass. 
We provide a proof sketch of \Cref{thm:low1} as follows, and present the complete proof in~\Cref{appen:thm1}. 
 
\subsection*{Proof Sketch of \Cref{thm:low1}}
The proof of  \Cref{thm:low1}  can be divided into four main steps: 1) constructing a worst-case instance $(f,g)\in\mathcal{F}_{scsc}$ ; 2) characterizing the optimal point $x^*=\argmin_{x\in\mathbb{R}^d}\Phi(x)$; 3) characterizing the subspaces $\mathcal{H}_x^K,\mathcal{H}_y^K$; and 4) lower-bounding the convergence rate and complexity. 

\vspace{0.1cm}
\noindent {\bf Step 1 (construct a worse-case instance):} We construct the following instance functions $f$ and $g$.
 \begin{align}\label{str_fg}
f(x,y) &= \frac{1}{2} x^T(\alpha Z^2 +\mu_x I) x -\frac{\alpha\beta}{\widetilde L_{xy}}x^TZ^3y +\frac{\bar L_{xy}}{2}x^TZy + \frac{L_y}{2}\|y\|^2 + \frac{\bar L_{xy}}{\widetilde L_{xy}} b^T y, \nonumber
\\g(x,y) &= \frac{1}{2} y^T (\beta Z^2 +\mu_y I)y -\frac{\widetilde L_{xy}}{2} x^TZy + b^Ty, 
\end{align} 
where $\alpha=\frac{L_x-\mu_x}{4}$, $\beta=\frac{\widetilde L_y-\mu_y}{4}$, and the coupling matrices  $Z,Z^2,Z^4$ take the forms of 
{\footnotesize
\begin{align}\label{matrices_coupling}
Z= \begin{bmatrix}
 &   &  & 1\\
 &  &  1& -1  \\
  &\text{\reflectbox{$\ddots$}} &\text{\reflectbox{$\ddots
  $}} &  \\
  1& -1  &  & \\ 
\end{bmatrix},\;
Z^2 =
\begin{bmatrix}
 1& -1 &  &  & \\
 -1& 2  &-1 &  & \\
 &   \ddots & \ddots & \ddots  & \\
  && -1& 2 & -1 \\
  &  & & -1 & 2\\ 
\end{bmatrix},\;
Z^4 =  
\begin{bmatrix}
 2& -3 & 1 &  & &\\
 -3& 6  &-4 &1  & &\\
 1& -4 & 6 & -4 & 1&\\
  &  \ddots&  \ddots&  \ddots &  \ddots & \ddots\\
   &  & 1& -4 &  6& -4 \\
  &  & & 1 & -4 & 5\\ 
\end{bmatrix}.
\end{align}
}
\hspace{-0.13cm}The above matrices play an important role in developing lower bounds due to their following zero-chain properties~\cite{nesterov2003introductory,zhang2019lower}. Let $\mathbb{R}^{k,d}= \{x\in\mathbb{R}^d | x_i=0 \text{ for } k+1\leq i \leq d\}$, where $x_i$ denotes the $i^{th}$ coordinate of the vector $x$. 

\begin{lemma}[Zero-Chain Property]\label{zero_chain}
For any given vector $v\in\mathbb{R}^{k,d}$, we have $Z^2v\in\mathbb{R}^{k+1,d}$.
\end{lemma}
\Cref{zero_chain} indicates  that if a vector $v$ has nonzero entries only at the first $k$ coordinates, then multiplying it with a matrix $Z^2$ has at most one more nonzero entry at position $k+1$. We demonstrate the validity of the constructed instance by 
showing that  $f$ and $g$ in \cref{str_fg} satisfy Assumptions~\ref{fg:smooth} and~\ref{g:hessiansJaco}, and $\Phi(x)$ is $\mu_x$-strongly-convex.  

\vspace{0.1cm}
 \noindent {\bf Step 2 (characterize the minimizer $x^*$):} We show that the unique minimizer $x^*$ satisfies the following equation   
 \begin{align}\label{refere:eqs}
 Z^4 x^* +\lambda Z^2x^* +\tau x^* = \gamma Zb,
\end{align}
where $\lambda=\Theta(1)$ and $\gamma=\Theta(1)$, $\tau = \Theta(\mu_x\mu_y^2)$. We  choose $b$ in \cref{refere:eqs} such that $(Zb)_t=0$  for all $t\geq 3$, which is feasible because we show that $Z$ is invertible. Based on the structure of $Z$ in \cref{matrices_coupling}, we show that there exists a vector $\hat x$ with its $i^{th}$ coordinate $\hat x_i = r^i$ such that 
\begin{align}\label{sketch_p}
\|x^*-\hat x\| \leq \mathcal{O}(r^d),
\end{align}
where $0<r<1$ satisfies  $1-r=\Theta(\mu_x\mu_y^2)$. Then, based on the above  \cref{sketch_p}, we are able to characterize $x^*$, e.g., its norm $\|x^*\|$, using its approximate (exponentially close) $\hat x$. 

\vspace{0.1cm}
 \noindent {\bf Step 3 (characterize the iterate subspaces):}  By exploiting the forms of  the subspaces {\small$\{\mathcal{H}_x^k,\mathcal{H}_y^k\}_{k=1}^K$}
defined in \Cref{alg_class}, we use the induction to show that 
{\small$H_{x}^{K }\subseteq \mbox{Span}\{Z^{2(K+QT+Q)}(Zb),....,Z^2(Zb),(Zb) \}$}.
 Then, noting that $(Zb)_t=0$  for all $t\geq 3$ and using the zero-chain property of $Z^2$, we have the $t^{th}$ coordinate of the output $x^K$ to be zero, i.e., $(x^K)_{t}=0$,  for all $t\geq M+1$.
 
 \vspace{0.1cm}
 \noindent {\bf Step 4 (combine Steps $1,2,3$ and characterize the complexity):} By choosing $d> \max\big\{2M,M+1+\log_{r}\big(\frac{\tau}{4(7+\lambda)}\big)\big\}$, and based on {\bf Steps} 2 and 3, we have  
$ \|x^K-x^*\| \geq \frac{\|x^*-x_0\|}{3\sqrt{2}} r^M$
which, in conjunction with the form of $\Phi(x)$, yields the result in~\cref{result:first}. The complexity result then follows because $1-r=\Theta(\mu_x\mu_y^2)$ and from the definition of the complexity measure in~\Cref{complexity_measyre}.

\vspace{0.2cm}
{\noindent \bf Remark.} We note that the introduction of the term $\frac{\alpha\beta}{\widetilde L_{xy}}x^TZ^3y$ in $f$ is necessary to obtain the lower bound $\widetilde \Omega (\mu_x\mu_y^{2})$. Without such a term, there will be an additional high-order term $\Omega(A^6x)$ at the left hand side of \cref{refere:eqs}. Then, following the same steps as in Step 2, we would obtain a result similar to \cref{sketch_p}, but with a parameter $r$ satisfying 
$ 0<\frac{1}{1-r}<\mathcal{O}\big( {\frac{1}{\mu_x\mu_y}} \big).$
Then, following the same steps as in Steps 3 and 4, the final overall complexity $\mathcal{C}_{\text{\normalfont sub}}(\mathcal{A},\epsilon)\geq \Omega\big( {\frac{1}{\mu_x\mu_y}} \big)$, which is not as tight as 
 $\Omega\big( {\frac{1}{\mu_x\mu^2_y}} \big)$ which we obtain under the selection in \cref{str_fg}.

\section{Lower Bound for Convex-Strongly-Convex Case}
We next characterize the lower complexity bound for the convex-strongly-convex setting, where $\Phi(\cdot)$ is {\bf convex} and the inner-level function $g(x,\cdot)$ is $\mu_y$-strongly-convex.
We state our main result for this case in the following theorem.
\begin{theorem}\label{main:convex}
Let $M = K+QT-Q+3$ with $K, T,Q $ given by~\Cref{alg_class}, and let $x^K$ be an output belonging to the  subspace $\mathcal{H}_x^K$, i.e., generated by any algorithm in the hypergradient-based algorithm class
defined in~\Cref{alg_class}.  
There exists an instance in $\mathcal{F}_{csc}$ defined in~\Cref{de:pc} with dimensions $p=q=d$ such that in order to achieve $\|\nabla\Phi(x^K)\|\leq \epsilon$, it requires 
$M \geq \lfloor r^*\rfloor -3$, 
 where $r^*$ is the solution of the equation 
 \begin{align}\label{eq:rsoltion}
 r^4 +r \Big(\frac{2\beta^4}{\mu_y^4}+\frac{4\beta^3}{\mu_y^3}+\frac{4\beta^2}{\mu_y^2}\Big) = \frac{B^2(\widetilde L^2_{xy}L_y+L_x\mu_y^2)^2}{128\mu_y^4\epsilon^2},
 \end{align}
 where $\beta=\frac{\widetilde L_y-\mu_y}{4}$ and $B$ is defined in \Cref{de:pc}. The total complexity satisfies $\mathcal{C}_{\text{\normalfont norm}}(\mathcal{A},\epsilon)\geq \Omega(r^*)$.
\end{theorem}
Note that \Cref{main:convex} uses the gradient norm $\|\nabla \Phi(x)\|\leq \epsilon$ rather than the suboptimality gap $\Phi(x^K)-\Phi(x^*)$ as the convergence criteria. This is because for the convex-strongly-convex case,  lower-bounding the suboptimality gap requires  the Hessian matrix $A$ in the worst-case construction of the total objective function $\Phi(x)$ to have a nice structure, e.g., the solution of $A^\prime x = e_1$ ($e_1$ has a single non-zero value $1$ at the first coordinate) is explicit, where $A^\prime$ is derived by removing last $k$ columns and rows of $A$.  However, in bilevel optimization, $A$ often contains different powers of the zero-chain matrix $Z$, and does not have such a structure. We will leave the lower bound under the suboptimality criteria for the future study. 

Note that $r^*$ in \Cref{main:convex} has a complicated form. 
The following two corollaries simplify the complexity results by considering specific parameter regimes. 
\begin{corollary}\label{co:co1}
Under the same setting of \Cref{main:convex}, consider the case when $\beta\leq \mathcal{O}(\mu_y)$. Then, we have 
$\mathcal{C}_{\text{\normalfont norm}}(\mathcal{A},\epsilon)\geq \Omega \big(\frac{B^{\frac{1}{2}}(\widetilde L^2_{xy}L_y+L_x\mu_y^2)^{\frac{1}{2}}}{\mu_y\epsilon^{\frac{1}{2}}} \big)$. 
\end{corollary}
\begin{corollary}\label{co:co2}
Under the same setting of \Cref{main:convex}, consider the case when $\beta \leq\mathcal{O}(1)$, i.e., at a constant level. Then, we have {\small $\mathcal{C}_{\text{\normalfont norm}}(\mathcal{A},\epsilon)\geq \widetilde \Omega(\frac{1}{\sqrt{\epsilon}}\min\{\frac{1}{\mu_y},\frac{1}{\sqrt{\epsilon^{3}}}\})$}.
\end{corollary}
The proof sketch of \Cref{main:convex} is provided as follows. The complete proof is provided in \Cref{apep:mainconvexs}. 

\subsection*{Proof Sketch of \Cref{main:convex}}
{\bf Step 1(construct the worst-case instance):} We construct the instance functions $f$ and $g$ as follows.
\begin{align}\label{con_fg}
f(x,y) &= \frac{L_x}{8} x^T Z^2 x+ \frac{L_y}{2}\|y\|^2, \nonumber
\\g(x,y) &= \frac{1}{2} y^T (\beta Z^2 +\mu_y I)y -\frac{\widetilde L_{xy}}{2} x^TZy + b^Ty, 
\end{align}
where $\beta=\frac{\widetilde L_y-\mu_y}{4}$.  Here, the coupling matrix $Z$ is different from that~\cref{matrices_coupling} for the strongly-convex-strongly-convex case, which takes the form of 
\begin{align}\label{matrices_coupling_convex}
Z:=\begin{bmatrix}
   & &  1& -1\\
   &1&  -1&   \\
 \text{\reflectbox{$\ddots $}}  &\text{\reflectbox{$\ddots$}}  &  \\
  -1&   &  & \\ 
\end{bmatrix},\quad
&Z^2 := 
\begin{bmatrix}
 2& -1 &  &  & \\
 -1& 2  &-1 &  & \\
 &   \ddots & \ddots & \ddots  & \\
  && -1& 2 & -1 \\
  &  & & -1 & 1\\ 
\end{bmatrix}.
\end{align}
It can be verified that $Z$ is invertible and $Z^2$ in  \cref{matrices_coupling_convex} also satisfies the zero-chain property, i.e., \Cref{zero_chain}. 
We can further verify that  $\Phi(x)$ is convex and functions $f,g$ satisfy Assumptions~\ref{fg:smooth} and~\ref{g:hessiansJaco}.  

 \vspace{0.1cm}
\noindent {\bf Step 2 (characterize the minimizer $x^*$):} Recall that $x^*\in \argmin_{x\in\mathbb{R}^d}\Phi(x)$. We then show that $x^*$ satisfies the equation 
\begin{align*}
\Big (\frac{L_x\beta^2}{4} Z^6 +  \frac{L_x\beta^2\beta\mu_y}{2} Z^4 +\Big (\frac{L_y\widetilde L_{xy}^2}{4}+\frac{L_x\mu_y^2}{4}\Big) Z^2 \Big) x^* =  \frac{L_y\widetilde L_{xy}}{2} Zb. 
\end{align*}
Let $\widetilde b =  \frac{L_y\widetilde L_{xy}}{2} Zb$ and choose $b$  such that $\widetilde b_t=0$ for all $t\geq 4$. Then, by choosing $\widetilde b_1,\widetilde b_2, \widetilde b_3$ properly, we derive that $x^* = \frac{B}{\sqrt{d}} {\bf 1}$, where ${\bf 1}$ is the all-one vector, and hence $\|x^*\|= B$. 

 \vspace{0.1cm}
\noindent  {\bf Step 3 (characterize the gradient norm):} In this step, we show that for any $x$ whose last three coordinates are zeros, the gradient norm of $\nabla\Phi(x)$ is lower-bounded. Namely, we prove that 
\begin{align}\label{eq:sehcsa}
\min_{x\in\mathbb{R}^d: \;x_{d-2}=x_{d-1}=x_d=0} \|\nabla\Phi(x)\|^2 \geq \frac{B^2\Big(\frac{\widetilde L^2_{xy}L_y}{4}+\frac{L_x\mu_y^2}{4}\Big)^2}{8\mu_y^4d^4 +16d\beta^4+32d\beta^3\mu_y+32d\beta^2\mu_y^2}.
\end{align}

 \vspace{0.1cm}
\noindent{\bf Step 4 (characterize the iterate subspaces):} By exploiting the forms of  the subspaces $\{\mathcal{H}_x^k,\mathcal{H}_y^k\}_{k=1}^K$
defined in \Cref{alg_class} and by induction, we show that 
{\small $H_{x}^{K }\subseteq \mbox{Span}\{Z^{2(K+QT-Q)}(Zb),....,Z^2(Zb),(Zb)\}$}.
Since  $(Zb)_t=0$  for all $t\geq 4$ and using the zero-chain property of $Z^2$, we have the $t^{th}$ coordinate of the output $x^K$ is zero, i.e., $(x^K)_{t}=0$,  for all $t\geq M+1$, where $M=K+QT-Q + 3$. 

 \vspace{0.1cm}
\noindent  {\bf Step 5 (combine Steps $1,2,3,4$ and characterize the complexity):} Choose $d$ such that the right hand side of \cref{eq:sehcsa} equals $\epsilon$ by  solving \cref{eq:rsoltion}. Then, using the results in Steps 3 and 4, it follows that for any $M\leq d-3$, $\|\nabla \Phi(x^K)\|\geq \epsilon$. Thus, to achieve $\|\nabla \Phi(x)\|\leq \epsilon$ , it requires $M>d-3$ and the complexity result follows as  $\mathcal{C}_{\text{\normalfont norm}}(\mathcal{A},\epsilon) \geq  \Omega(M)$.

\section{Optimality of Bilevel Optimization and Discussion}
We compare the lower and upper bounds and make the following remarks on the optimality of bilevel optimization and its comparison to minimax optimization. 

\vspace{0.3cm}
\noindent {\bf Optimality of results for quadratic $g$ subclass.} We compare the developed lower and upper bounds and make a few remarks on the optimality of the proposed AccBiO algorithms. Let us first focus on the quadratic $g$ subclass where $g(x,y)$ takes the quadratic form as in \cref{quadratic_case}. For the strongly-convex-strongly-convex setting, comparison of \Cref{thm:low1} and \Cref{coro:quadaticSr} implies that AccBiO achieves the optimal complexity for $\widetilde L_y\leq \mathcal{O}(\mu_y)$, i.e., the inner-level problem is easy to solve.  For the general case, there is still a gap of $\frac{1}{\sqrt{\mu_y}}$ between lower and upper bounds.  For the convex-strongly-convex setting, comparison of \Cref{main:convex} and \Cref{coro:quadaticConv} shows that AccBiO is optimal for $\widetilde L_y\leq \mathcal{O}(\mu_y)$, and there is a gap for the general case. Such a gap is mainly due to the large smoothness parameter $L_\Phi$ of $\Phi(\cdot)$. We note that a similar issue also occurs for minimax optimization, which has been addressed by \cite{lin2020near} using an accelerated proximal point method for the inner-level problem and exploiting Sion's minimax theorem $\min_x\max_y f(x,y) =\max_y\min_x f(x,y)$. However, such an approach is not applicable for bilevel optimization due to the asymmetry of $x$ and $y$, e.g., $\min_xf(x,y^*(x)) \neq \min_y g(x^*(y),y)$. This gap between lower and upper bounds deserves future efforts.

\vspace{0.3cm}
\noindent {\bf Optimality of results for general $g$.} We now discuss the optimality of our results for a more general $g$ whose second-order derivatives are Lipschitz continuous. For the strongly-convex-strongly-convex setting, it can be seen from the comparison of \Cref{thm:low1} and \Cref{upper_srsr_withnoB} that  there is a gap between the lower and upper bounds. This gap is because the lower bounds construct the bilinearly coupled worst-case $g(x,y)$ whose Hessians and Jacobians are constant, rather than generally $\rho_{yy}$- and $\rho_{xy}$-Lipschitz continuous as considered in the upper bounds. Hence, tighter lower bounds need to be provided for this setting, which requires more sophisticated  worst-case instances with Lipschitz continuous Hessians $\nabla_y^2 g(x,y)$ and Jacobians $\nabla_x\nabla_yg(x,y)$.  For example, it is possible to construct $g(x,y)$ as $g(x,y)=\sigma(y)y^T Zy - x^TZy +b^Ty$, where $\sigma(\cdot):\mathbb{R}^d\rightarrow \mathbb{R}$ satisfies a certain Lipschitz property. For example, if $\sigma$ is Lipchitz continuous, simple calculation shows that $L_\Phi$ scales at an order of $\kappa_y^3$. However, it still requires significant efforts to determine the form of $\sigma$ such that the optimal point of $\Phi(\cdot)$  and the subspaces $\mathcal{H}_x,\mathcal{H}_y$ are easy to characterize and satisfy the properties outlined in the proof of \Cref{upper_srsr_withnoB}.  

\vspace{0.3cm}
\noindent {\bf Comparison to minimax optimization.} We compare the optimality between minimax optimization and bilevel optimization. For the strongly-convex-strongly-convex minimax optimization, \cite{zhang2019lower} developed a lower bound of  {\small$\widetilde{\Omega}(\frac{1}{\sqrt{\mu_x\mu_y}})$} for minimax optimization, which is achieved by the accelerated proximal point method proposed by \cite{lin2020near} up to logarithmic factors. 
For the same type of bilevel optimization,
we provide a lower bound of {\small$\widetilde{\Omega}\big(\sqrt{\frac{1}{\mu_x\mu_y^2}}\big)$} in~\Cref{thm:low1}, which is larger than that of minimax optimization by a factor of $\frac{1}{\sqrt{\mu_y}}$.  Similarly for the convex-strongly-convex bilevel optimization, we provide a lower bound of {\small $\widetilde \Omega\big(\frac{1}{\sqrt{\epsilon}}\min\{\frac{1}{\mu_y},\frac{1}{\epsilon^{1.5}}\}\big)$}, which is larger than the optimal complexity of {\small$\widetilde{\Omega}(\frac{1}{\sqrt{\epsilon\mu_y}})$} for the same type of minimax optimization~\cite{lin2020near} in a large regime of $\mu_y\geq\Omega(\epsilon^3)$. 
This establishes that bilevel optimization is fundamentally more challenging than minimax optimization. This is because bilevel optimization needs to handle the different structures of the outer- and inner-level functions $f$ and $g$ (e.g., second-order derivatives in the hypergradient), whereas for minimax optimization, the fact of $f=g$ simplifies the problem (e.g., no second-order derivatives) and allows more efficient algorithm designs.

\section{Summary of Contributions} 
In this chapter, we develop the first lower bounds for  the convex-strongly-convex and strongly-convex-strongly-convex  bilevel optimizations. We further show that the upper bounds achieved by AccBiO in \Cref{chp_acc_bilevel} match the lower bounds for a quadratic inner problem with a constant-level condition number. We anticipate that the analysis can be extended to other algorithm classes such as penalty-based algorithm class and other problems such as minimax optimization.  

\chapter{Enhanced Design for Stochastic Bilevel Optimization}\label{chp:stoc_bilevel}
In this chapter,  we provide faster and sample-efficient stochastic bilevel optimization algorithms with performance guarantee. All technical proofs for the results in this chapter are provided in \Cref{appendix:stoc_bilevel}. 

\section{Algorithm for Stochastic Bilevel Optimization}
  \begin{figure*}[t]
	\centering  
	\includegraphics[width=150mm]{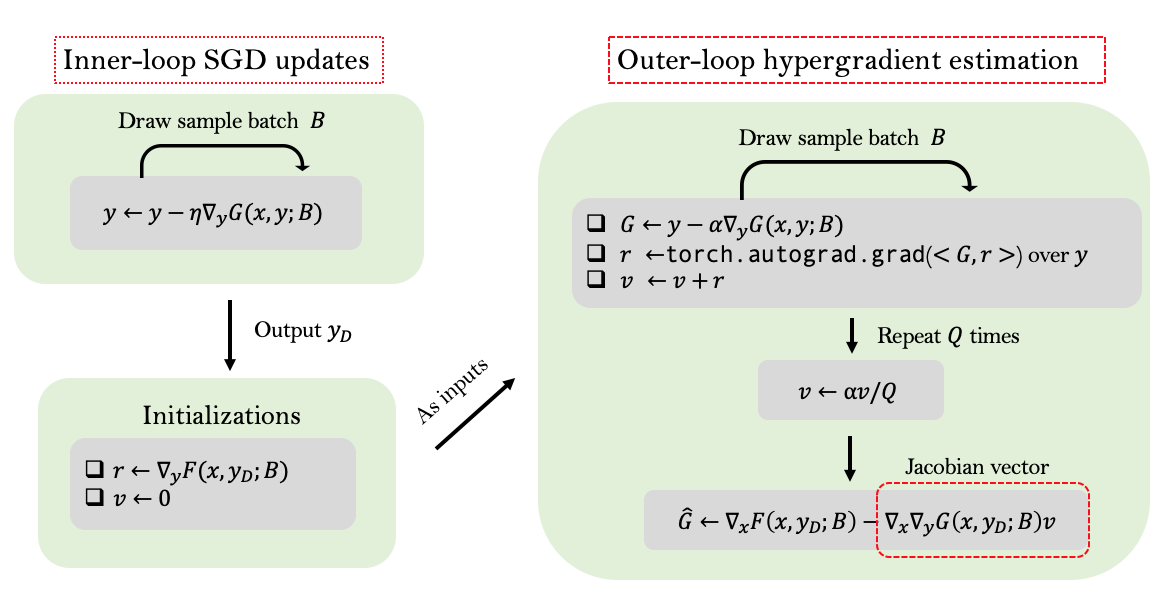}
	\vspace{-0.2cm}
	\caption{Illustration of hyperparameter estimation in our proposed stocBiO algorithm. Note that the  hyperparameter estimation (lines 9-10 in \Cref{alg:main}) involves only computations of automatic differentiation over scalar $<G_j(y),r_i>$ w.r.t.~$y$. In addition, our implementation applies the function {\em torch.autograd.grad} in PyTorch, which automatically determines the size of Jacobians. 
}\label{fig:stocBiO_illustrate}
\end{figure*}

We  propose a new stochastic bilevel optimizer (stocBiO) in~\Cref{alg:main} to solve the problem~\cref{objective}. 
It has a double-loop structure similar to \Cref{alg:main_deter}, but  runs $D$ steps of stochastic gradient decent (SGD) at the inner loop to obtain an approximated solution  $y_k^D$.  
Based on the output $y_k^D$ of the inner loop, stocBiO first computes a gradient {\small $\nabla_y F(x_k,y_k^D;\gD_F)$} over a sample batch $\gD_F$, and then computes a vector $v_Q$ as an estimated solution  of the linear system {\small $ \nabla_y^2 g(x_k,y^*(x_k))v=
\nabla_y f(x_k,y^*(x_k))$} via \Cref{alg:hessianEst}. Here, $v_Q$  takes a form of 
\begin{align}\label{ours:est}
v_Q =& \eta \sum_{q=-1}^{Q-1}\prod_{j=Q-q}^Q (I - \eta \nabla_y^2G(x_k,y_k^D;\gB_j)) v_0,
\end{align}
where {\small$v_0 = \nabla_y F(x_k,y_k^D;\gD_F)$} and $\gB_j,j=1,...,Q$ are mutually-independent sample sets, $Q$ and $\eta$ are constants, and we let $\prod_{Q+1}^Q (\cdot)= I$ for notational simplification.  
Our construction of $v_Q$, i.e., \Cref{alg:hessianEst}, is motived by the Neumann series $\sum_{i=0}^\infty U^k=(I-U)^{-1}$, and involves only Hessian-vector products rather than Hessians, and hence is computationally and memory efficient.  {\bf This procedure is illustrated in \Cref{fig:stocBiO_illustrate}.} Then, we construct 
 \begin{align}\label{estG}
 \widehat \nabla \Phi(x_k) =&  \nabla_x F(x_k,y_k^D;\gD_F)-\nabla_x \nabla_y G(x_k,y_k^D;\gD_G)v_Q
 \end{align}
as an estimate of hypergradient $\nabla \Phi(x_k)$. 
Compared to the deterministic case, it is more challenging to design a sample-efficient Hypergradient estimator in the stochastic case. For example, instead of choosing the same batch sizes for all $\gB_j,j=1,...,Q$  in~\cref{ours:est}, 
 our analysis captures the different impact of components $\nabla_y^2G(x_k,y_k^D;\gB_j)$, $j=1,...,Q$ on  the hypergradient estimation variance, and inspires an adaptive and more efficient choice  by setting $|\gB_{Q-j}|$ to   decay  exponentially with  $j$ from $0$ to $Q-1$. By doing so, we  achieve an improved  complexity. 

\begin{algorithm}[t]
	\caption{Stochastic bilevel optimizer (stocBiO)}  
	\label{alg:main}
	\begin{algorithmic}[1]
		\STATE {\bfseries Input:} $K,D,Q$, stepsizes $\alpha$ and $\beta$, initializations $x_0$ and $y_0$.
		\FOR{$k=0,1,2,...,K$}
		\STATE{Set $y_k^0 = y_{k-1}^{D} \mbox{ if }\; k> 0$ and $y_0$ otherwise 
		}
		\FOR{$t=1,....,D$}
		\STATE{Draw a sample batch $\gS_{t-1}$}  
		\vspace{0.05cm}
		\STATE{Update $y_k^t = y_k^{t-1}-\alpha \nabla_y G(x_k,y_k^{t-1}; \gS_{t-1}) $}
		\ENDFOR
                 \STATE{Draw sample batches $\gD_F,\gD_H$ and $\gD_G$ }
                 \STATE{Compute gradient {\small$v_0=\nabla_y F(x_k,y_k^D;\gD_F)$}}
                  \STATE{Construct estimate $v_Q$ via \Cref{alg:hessianEst} given $v_0$}
                  \STATE{Compute {\small$\nabla_x \nabla_y G(x_k,y_k^D;\gD_G) v_Q $}}

                  \STATE{ Compute gradient estimate $  \widehat \nabla \Phi(x_k) $ via~\cref{estG}
                    }
                 \STATE{Update $x_{k+1}=x_k- \beta \widehat \nabla \Phi(x_k) $}
		\ENDFOR
	\end{algorithmic}
	\end{algorithm}	
	\begin{algorithm}[t]
	\caption{Construct $v_Q$ given $v_0$  }  
	\label{alg:hessianEst}
	\begin{algorithmic}[1]
		\STATE {\bfseries Input:} Integer $Q$,  samples $\gD_H= \{\gB_j\}_{j=1}^{Q}$ and constant $\eta$.
		\FOR{$j=1,2,...,Q$}
		\STATE{Sample $\gB_j$ and compute $G_j(y)=y-\eta \nabla_y G(x,y; \gB_j)$
		}
		\ENDFOR
		\STATE{Set $r_Q = v_0$}
		\FOR{$i=Q,...,1$}
		\STATE{$r_{i-1}=\partial \big( G_i(y)r_i\big)/\partial y=r_i-\eta \nabla_y^2G(x,y;\gB_i)r_i$ via automatic differentiation}
		\ENDFOR
		\STATE{Return $v_Q=\eta \sum_{i=0}^Qr_i$
		}
	\end{algorithmic}
	\end{algorithm}

\section{Definitions and Assumptions}
Let $z=(x,y)$ denote all parameters. For simplicity, suppose sample sets $\gS_t$  for all $t=0,...,D-1$, $\gD_G$ and $\gD_F$ have the sizes of  $S$, $D_g$ and $D_f$, respectively.  In this thesis, we focus on the following types of loss functions.
\begin{assum}\label{assum:stoc_geo}
For any $\zeta$, the lower-level function $G(x,y;\zeta)$ is $\mu$-strongly-convex w.r.t.~$y$ and the total objective function $\Phi(x)=f(x,y^*(x))$ is nonconvex w.r.t.~$x$.  
\vspace{-0.2cm} 
\end{assum}
Since $\Phi(x)$ is nonconvex, algorithms are expected to find an $\epsilon$-accurate stationary point defined as follows. 
\begin{definition}
We say $\bar x$ is an $\epsilon$-accurate stationary point for the objective function $\Phi(x)$ in~\cref{objective} if $\mathbb{E}\|\nabla \Phi(\bar x)\|^2\leq \epsilon$, where $\bar x$ is the output of  an algorithm.
\end{definition}
In order to compare the performance of different bilevel algorithms, we adopt the following metrics of  complexity. 

\begin{definition}\label{stoc_com_measure}
For a function $F(x,y;\xi)$  and a vector $v$, let $\mbox{\normalfont Gc}(F,\epsilon)$ be the number of the partial gradient $\nabla_x F(x,y;\xi)$ or $\nabla_y F(x,y;\xi)$, and let $\mbox{\normalfont JV}(G,\epsilon)$ and  $\mbox{\normalfont HV}(G,\epsilon)$ be the number of Jacobian-vector products $\nabla_x\nabla_y G (x,y;\zeta)v$ 
  and  Hessian-vector products $\nabla_y^2G(x,y;\zeta) v$, respectively.  
 \end{definition} 
We  take the following standard assumptions on the loss functions in~\cref{objective}, which have been widely adopted in bilevel optimization~\cite{ghadimi2018approximation,ji2020convergence}.
\begin{assum}\label{ass:lip_stoc}
The loss function $F(z;\xi)$ and $G(z;\zeta)$ satisfy
\begin{list}{$\bullet$}{\topsep=0.2ex \leftmargin=0.2in \rightmargin=0.in \itemsep =0.01in}
\item $F(z;\xi)$ is $M$-Lipschitz, i.e.,  for any $z,z^\prime$ and $\xi$, $|F(z;\xi)-F(z^\prime;\xi)|\leq M\|z-z^\prime\|.$
\item $\nabla F(z;\xi)$ and $\nabla G(z;\zeta)$ are $L$-Lipschitz, i.e., for any $z,z^\prime,\xi,\zeta$, 
\vspace{-0.1cm}
\begin{align*}
\|\nabla F(z;\xi)-\nabla F(z^\prime;\xi)\|\leq& L\|z-z^\prime\|,\;\|\nabla G(z;\zeta)-\nabla G(z^\prime;\zeta)\|\leq L\|z-z^\prime\|.
\end{align*}
\end{list}
\vspace{-0.4cm}
\end{assum}
The following assumption imposes the Lipschitz conditions on such high-order derivatives, as also made in~\cite{ghadimi2018approximation}.
\begin{assum}\label{high_lip_stoc}
Suppose the derivatives $\nabla_x\nabla_y G(z;\zeta)$ and $\nabla_y^2 G(z;\zeta)$ are $\tau$- and $\rho$- Lipschitz, i.e.,
\begin{list}{$\bullet$}{\topsep=0.2ex \leftmargin=0.2in \rightmargin=0.in \itemsep =0.02in}
\item For any $z,z^\prime,\zeta$, $\|\nabla_x\nabla_y G(z;\zeta)-\nabla_x\nabla_y G(z^\prime;\zeta)\| \leq \tau \|z-z^\prime\|$.
\item For any $z,z^\prime,\zeta$, $\|\nabla_y^2 G(z;\zeta)-\nabla_y^2G(z^\prime;\zeta)\|\leq \rho \|z-z^\prime\|$.
\end{list} 
\end{assum}
As typically adopted in the analysis for stochastic optimization, we make the following bounded-variance assumption for the lower-level stochastic function $G(z;\zeta)$. 
\begin{assum} \label{ass:bound} 
Gradient $\nabla G(z;\zeta)$ has a bounded variance, i.e., $\mathbb{E}_\xi \|\nabla G(z;\zeta)-\nabla g(z)\|^2 \leq \sigma^2$ for some constant $\sigma>0$.
\end{assum}
\section{Convergence for Stochastic Bilevel Optimization}\label{main:result}
We first  characterize the bias and variance of a key component $v_Q$ in~\cref{ours:est}. 

\begin{proposition}\label{prop:hessian}
Let Assumptions~\ref{assum:stoc_geo},~\ref{ass:lip_stoc} and \ref{high_lip_stoc} hold. Let  $\eta\leq \frac{1}{L}$ and choose  $|\gB_{Q+1-j}|=BQ(1-\eta\mu)^{j-1}$ for $j=1,...,Q$, where $B\geq \frac{1}{Q(1-\eta\mu)^{Q-1}}$. 
The  bias satisfies   
\begin{align}{\label{bias}}
\big\|\mathbb{E} v_Q- [\nabla_y^2 g(x_k,y^D_k&)]^{-1}\nabla_y f(x_k,y^D_k)\big\| \leq  \mu^{-1}(1-\eta \mu)^{Q+1}M.
\end{align}
Furthermore, the estimation variance is given by 
\begin{align}\label{hessian_variance}
\mathbb{E}\|v_Q-[\nabla_y^2 &g(x_k,y^D_k)]^{-1}\nabla_y f(x_k,y^D_k)\|^2 \nonumber
\\&\leq\frac{4\eta^2  L^2M^2}{\mu^2} \frac{1}{B}+\frac{4(1-\eta \mu)^{2Q+2}M^2}{\mu^2}+ \frac{2M^2}{\mu^2D_f}.
\end{align}
\end{proposition}
\vspace{-0.2cm}
\Cref{prop:hessian} shows that if we choose $Q$, $B$ and $D_f$ at the order level of $\mathcal{O}(\log \frac{1}{\epsilon}) $, $\mathcal{O}(1/\epsilon)$ and $\mathcal{O}(1/\epsilon)$,  the bias and variance are smaller than $\mathcal{O}(\epsilon)$, and the required number of samples is
$\sum_{j=1}^Q BQ(1-\eta\mu)^{j-1} = \mathcal{O}\left(\epsilon^{-1}\log \frac{1}{\epsilon}\right)$.
Note that  the chosen batch size $|\gB_{Q+1-j}|$ exponentially decays w.r.t.~the index $j$. In comparison, the uniform choice of all $|\gB_j|$ would yield a worse complexity of $ \mathcal{O}\big( \epsilon^{-1}(\log\frac{1}{\epsilon})^2\big)$.

We next analyze stocBiO when $\Phi(x)$ is nonconvex.
\begin{theorem}\label{th:nonconvex}
Suppose Assumptions~\ref{assum:stoc_geo},~\ref{ass:lip_stoc}, \ref{high_lip_stoc} and~\ref{ass:bound} hold. Define 
$L_\Phi = L + \frac{2L^2+\tau M^2}{\mu} + \frac{\rho L M+L^3+\tau M L}{\mu^2} + \frac{\rho L^2 M}{\mu^3}$, and choose $\beta=\frac{1}{4L_\Phi}, \eta<\frac{1}{L}$, and 
 $D\geq\Theta(\kappa\log \kappa)$, where the detailed form of $D$ can be found in \Cref{mianshisimida}.  
We have
\begin{align}\label{eq:main_nonconvex}
\frac{1}{K}\sum_{k=0}^{K-1}\mathbb{E}\|\nabla&\Phi(x_k)\|^2 \leq  \mathcal{O}\Big( \frac{L_\Phi}{K}+  \kappa^2(1-\eta \mu)^{2Q}+\frac{\kappa^5\sigma^2}{S}+ \frac{\kappa^2}{D_g} +\frac{\kappa^2}{D_f}+ \frac{\kappa^2}{B}\Big).
\end{align}
In order to achieve an $\epsilon$-accurate stationary point, the complexities satisfy 
\begin{list}{$\bullet$}{\topsep=0.4ex \leftmargin=0.15in \rightmargin=0.in \itemsep =0.01in}
\item Gradient: {\small$\mbox{\normalfont Gc}(F,\epsilon)=\mathcal{O}(\kappa^5\epsilon^{-2}), \mbox{\normalfont Gc}(G,\epsilon)=\mathcal{O}(\kappa^9\epsilon^{-2}).$}
\item Jacobian-, Hessian-vector complexities: {\small$\mbox{\normalfont JV}(G,\epsilon)=\mathcal{O}(\kappa^5\epsilon^{-2}), \mbox{\normalfont HV}(G,\epsilon)=\mathcal{\widetilde O}(\kappa^6\epsilon^{-2}).$}
\end{list}
\end{theorem}
\Cref{th:nonconvex} shows that stocBiO converges sublinearly with the convergence error decaying exponentially w.r.t.~$Q$ and sublinearly w.r.t.~the batch sizes $S,D_g,D_f$ for gradient estimation and $B$ for Hessian inverse estimation. In addition, it can be seen that the number $D$ of the inner-loop steps is 
at  a  constant level, rather than a typical choice of $\Theta(\log(\frac{1}{\epsilon}))$.

As shown in~\Cref{tab:stochastic}, the gradient complexities of our proposed algorithm in terms of  $F$  and $G$  improve those of BSA in \cite{ghadimi2018approximation} by an order of $\kappa$ and $\epsilon^{-1}$, respectively. In addition, the Jacobian-vector product complexity $\mbox{\normalfont JV}(G,\epsilon)$ of our algorithm improves that of BSA by the order of $\kappa$. 
In terms of the accuracy $\epsilon$, our gradient,  Jacobian- and Hessian-vector product complexities  improve those of TTSA in \cite{hong2020two} all by an order of $\epsilon^{-0.5}$.  

\section{Applications to Hyperparameter Optimization}

The goal of hyperparameter optimization~\cite{franceschi2018bilevel,feurer2019hyperparameter} is to search for representation or regularization parameters  $\lambda$ to minimize the validation error evaluated over the learner's parameters $w^*$,  
 where $w^*$ is the minimizer of the inner-loop regularized  training error. Mathematically, the objective function is given by 
\begin{align}\label{obj:hyper_opt}
&\min_\lambda \gL_{\gD_{\text{val}}}(\lambda) = \frac{1}{|\gD_{\text{val}}|}\sum_{\xi\in \gD_{\text{val}}} \gL(w^*; \xi) \nonumber
\\& \;\mbox{s.t.} \; w^*= \argmin_{w} \underbrace{\frac{1}{|\gD_{\text{tr}}|}\sum_{\xi\in \gD_{\text{tr}}} \big(\gL(w,\lambda;\xi)  + \gR_{w,\lambda}\big)}_{\gL_{\gD_{\text{tr}}}(w,\lambda)},
\end{align}
where $\gD_{\text{val}}$ and $\gD_{\text{tr}}$ are validation and training data,  $\gL$ is the loss, and $\gR_{w,\lambda}$ is a regularizer. In practice,  the lower-level function $\gL_{\gD_{\text{tr}}}(w,\lambda)$ is often strongly-convex w.r.t.~$w$. For example, for the data hyper-cleaning application proposed by~\cite{franceschi2018bilevel,shaban2019truncated}, the predictor is modeled by a linear classifier, and  the 
loss function $\gL(w;\xi)  $ is convex w.r.t.~$w$ and $\gR_{w,\lambda}$ is a strongly-convex regularizer, e.g., $L^2$ regularization.  
The sample sizes of  $\gD_{\text{val}}$ and $\gD_{\text{tr}}$ are often large, and stochastic algorithms are preferred for achieving better efficiency. 
As a result, the above hyperparameter optimization falls into the stochastic bilevel optimization we study in~\cref{objective}, and we can apply the proposed stocBiO. Furthermore, \Cref{th:nonconvex} establishes its performance guarantee. 
\subsection*{Experiments}

We compare our proposed {\bf stocBiO} with the following baseline bilevel optimization algorithms. 
\begin{list}{$\bullet$}{\topsep=0.ex \leftmargin=0.12in \rightmargin=0.in \itemsep =0.0in}
\item {\bf BSA}~\cite{ghadimi2018approximation}: implicit gradient based stochastic optimizer via single-sample sampling.
\item {\bf TTSA}~\cite{hong2020two}: two-time-scale stochastic optimizer via single-sample data sampling. 
\item {\bf HOAG}~\cite{pedregosa2016hyperparameter}: a hyperparameter optimization algorithm with approximate gradient. We use the implementation in the repository~ {\small\url{https://github.com/fabianp/hoag}}.  
\item {\bf reverse}~\cite{franceschi2017forward}: an iterative differentiation based method that approximates the hypergradient via backpropagation. We use its implementation in {\small \url{https://github.com/prolearner/hypertorch}}.  
 \item {\bf AID-FP}~\cite{grazzi2020iteration}: AID with the  fixed-point method. We use its implementation in {\small \url{https://github.com/prolearner/hypertorch}} 
\item  {\bf AID-CG}~\cite{grazzi2020iteration}: AID with the conjugate gradient method. We use its implementation in {\small \url{https://github.com/prolearner/hypertorch}}.  
 \end{list}
We demonstrate the effectiveness of the proposed stocBiO algorithm on two experiments: data hyper-cleaning and logistic regression.  
\begin{figure*}[ht]
	\centering  
	\subfigure[Test loss and test accuracy v.s. running time]{\label{figure:lr}\includegraphics[width=60mm]{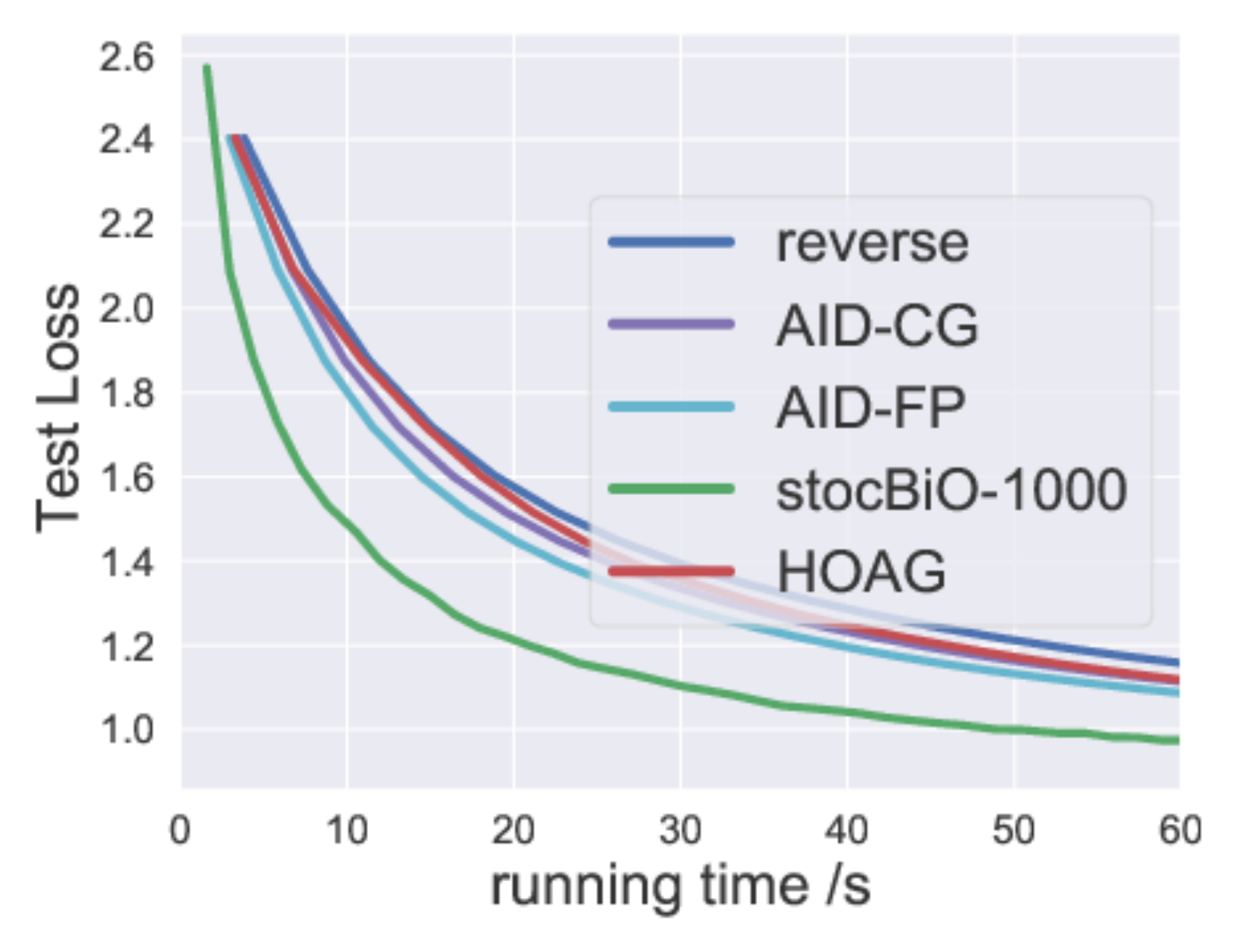}\includegraphics[width=60mm]{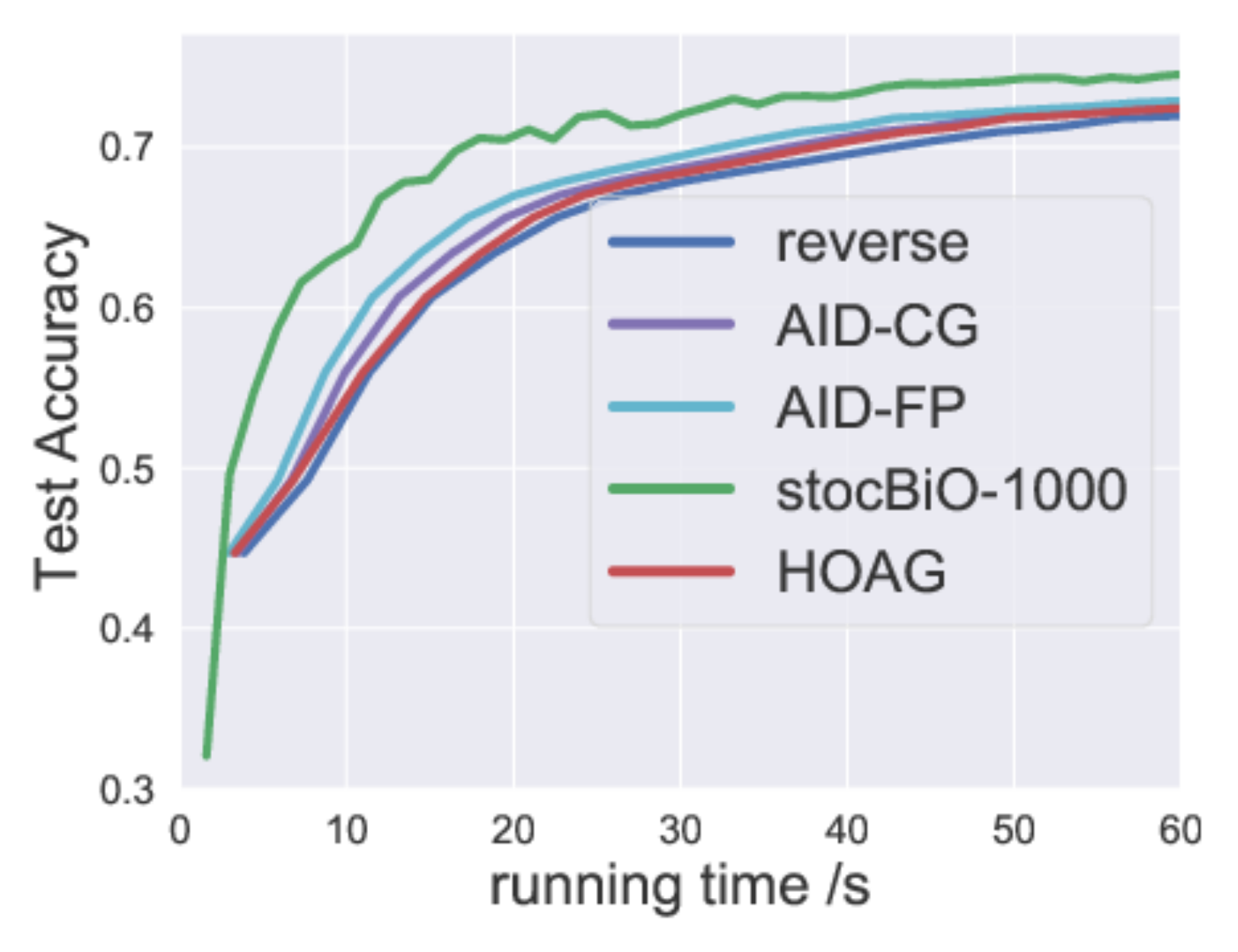}}    
	\subfigure[Convergence rate with different batch sizes]{\label{figure:batch}\includegraphics[width=60mm]{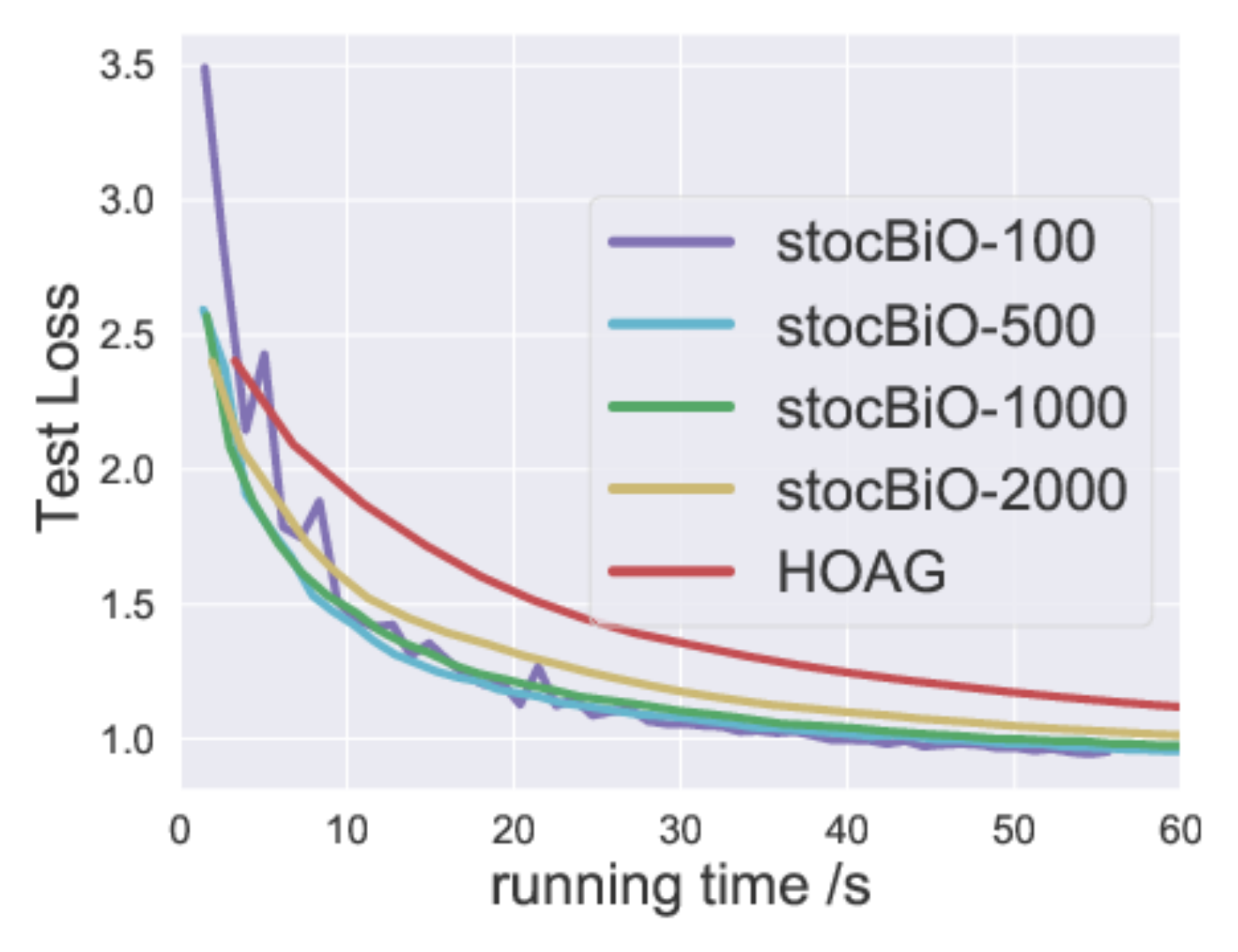}\includegraphics[width=60mm]{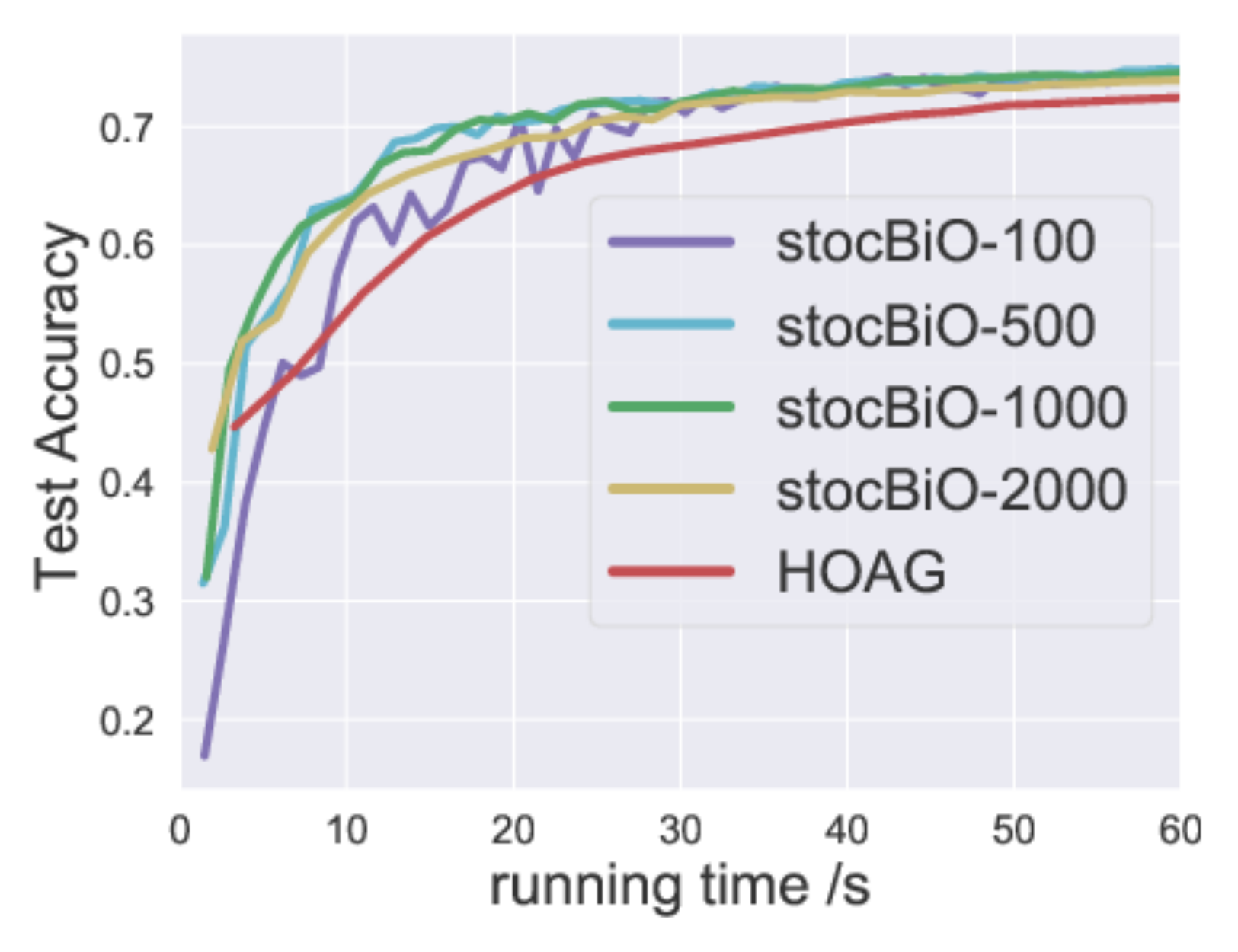}}  
	\vspace{-0.2cm}
	\caption{Comparison of various stochastic bilevel algorithms on logistic regression on 20 Newsgroup dataset.}\label{figure:newfigure}
	  \vspace{-0.3cm}
\end{figure*}

\vspace{0.1cm}
\noindent{\bf Logistic Regression on 20 Newsgroup:} 
We compare the performance of our algorithm {\bf stocBiO} with the existing baseline algorithms {\bf reverse, AID-FP, AID-CG and HOAG }over a logistic regression problem on $20$ Newsgroup dataset~\cite{grazzi2020iteration}. The objective function of such a problem is given by 
 \begin{align*}
&\min_\lambda E(\lambda,w^*) = \frac{1}{|\gD_{\text{val}}|}\sum_{(x_i,y_i)\in \gD_{\text{val}}} L(x_iw^*, y_i) \nonumber
\\& \;\mbox{s.t.} \quad w^* = \argmin_{w\in\mathbb{R}^{p\times c}}  \Big(\frac{1}{|\gD_{\text{tr}}|}\sum_{(x_i,y_i)\in \gD_{\text{tr}}}L(x_iw, y_i)  + \frac{1}{cp} \sum_{i=1}^c\sum_{j=1}^p \exp(\lambda_j)w_{ij}^2\Big),
\end{align*}
where $L$ is the cross-entropy loss, $c=20$ is the number of topics, and $p=101631$ is the feature dimension. Following \cite{grazzi2020iteration}, we use SGD as the optimizer for the outer-loop update for all algorithms. For reverse, AID-FP, AID-CG, we use the suggested and well-tuned hyperparameter setting in their implementations~\url{https://github.com/prolearner/hypertorch} on this application. In specific, they choose the inner- and outer-loop stepsizes as $100$, the number of inner loops as $10$, the number of CG steps as $10$.  For HOAG, we use the same parameters as reverse, AID-FP, AID-CG. For stocBiO, we use the same parameters as reverse, AID-FP, AID-CG, and choose $\eta=0.5,Q=10$. We use stocBiO-$B$ as shorthand of stocBiO with a batch size of $B$.

  As shown in \Cref{figure:lr}, the proposed stocBiO achieves the fastest convergence rate as well as the best test accuracy among all comparison algorithms. This demonstrates the practical advantage of our proposed algorithm stocBiO. Note that we do not include BSA and TTSA in the comparison, because they converge too slowly with a large variance, and are much worse than the other competing algorithms. In addition, we investigate the impact of the batch size on the performance of our stocBiO in \Cref{figure:batch}. It can be seen that stocBiO outperforms HOAG under the batch sizes of $100,500,1000,2000$. This shows that the performance of stocBiO is not very sensitive to the batch size, and hence the tuning of the batch size is easy to handle in practice. 

  \begin{figure*}[ht]
  \vspace{-2mm}
	\centering  
	\subfigure[Corruption rate $p=0.1$]{\label{fig2:cbikeveksa}\includegraphics[width=60mm]{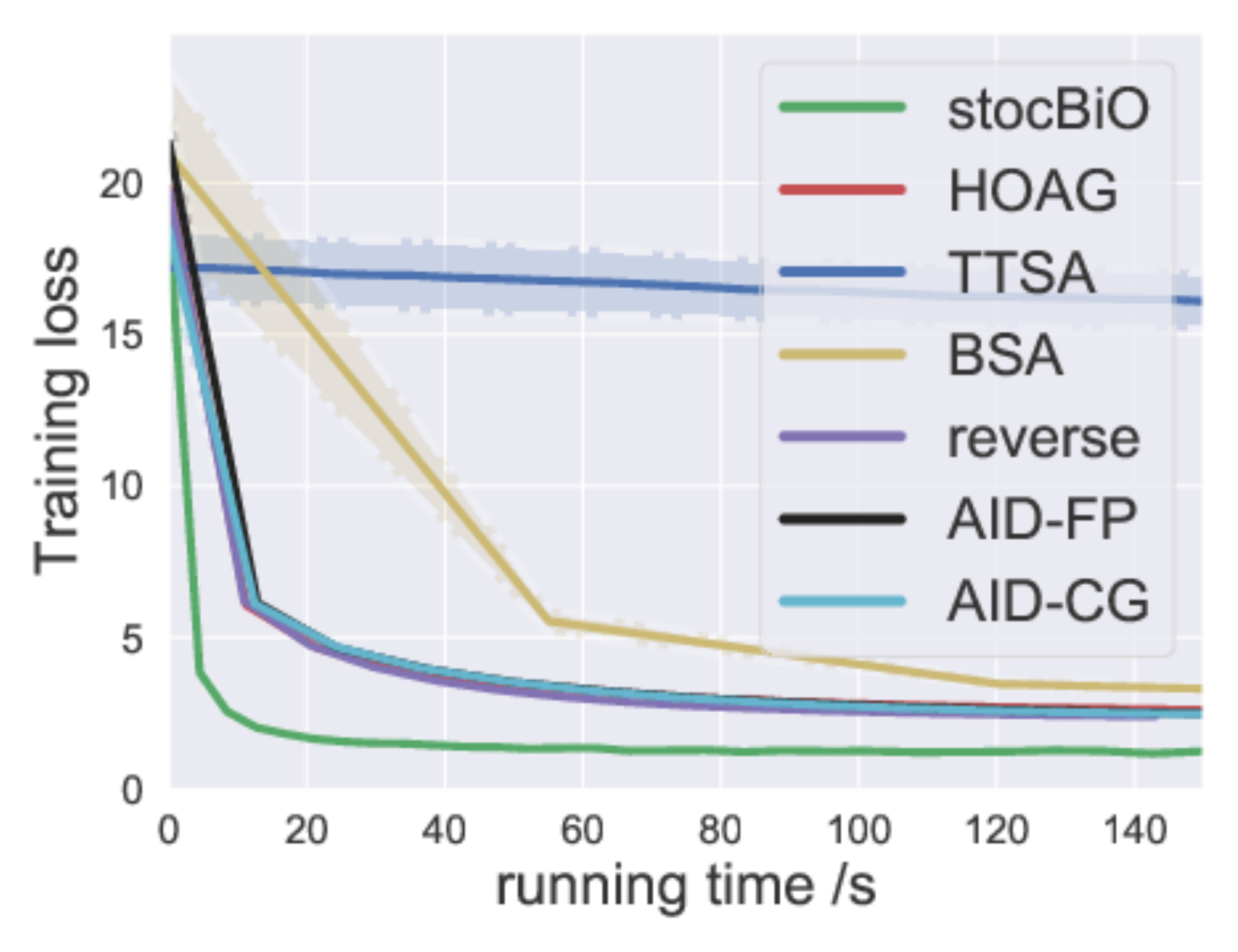}\includegraphics[width=60mm]{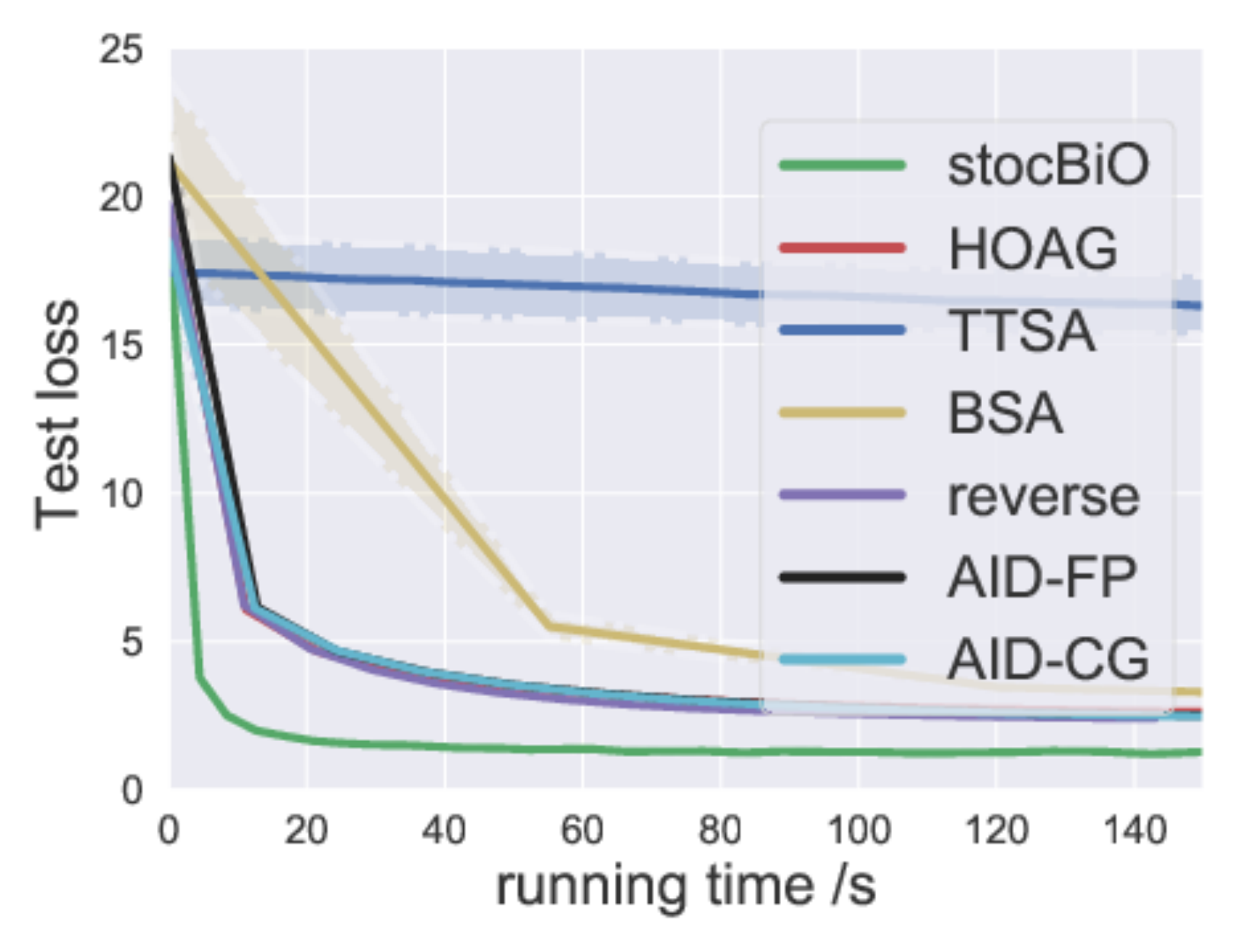}}    
	\subfigure[Corruption rate $p=0.4$]{\label{fig2:padascqsaca}\includegraphics[width=60mm]{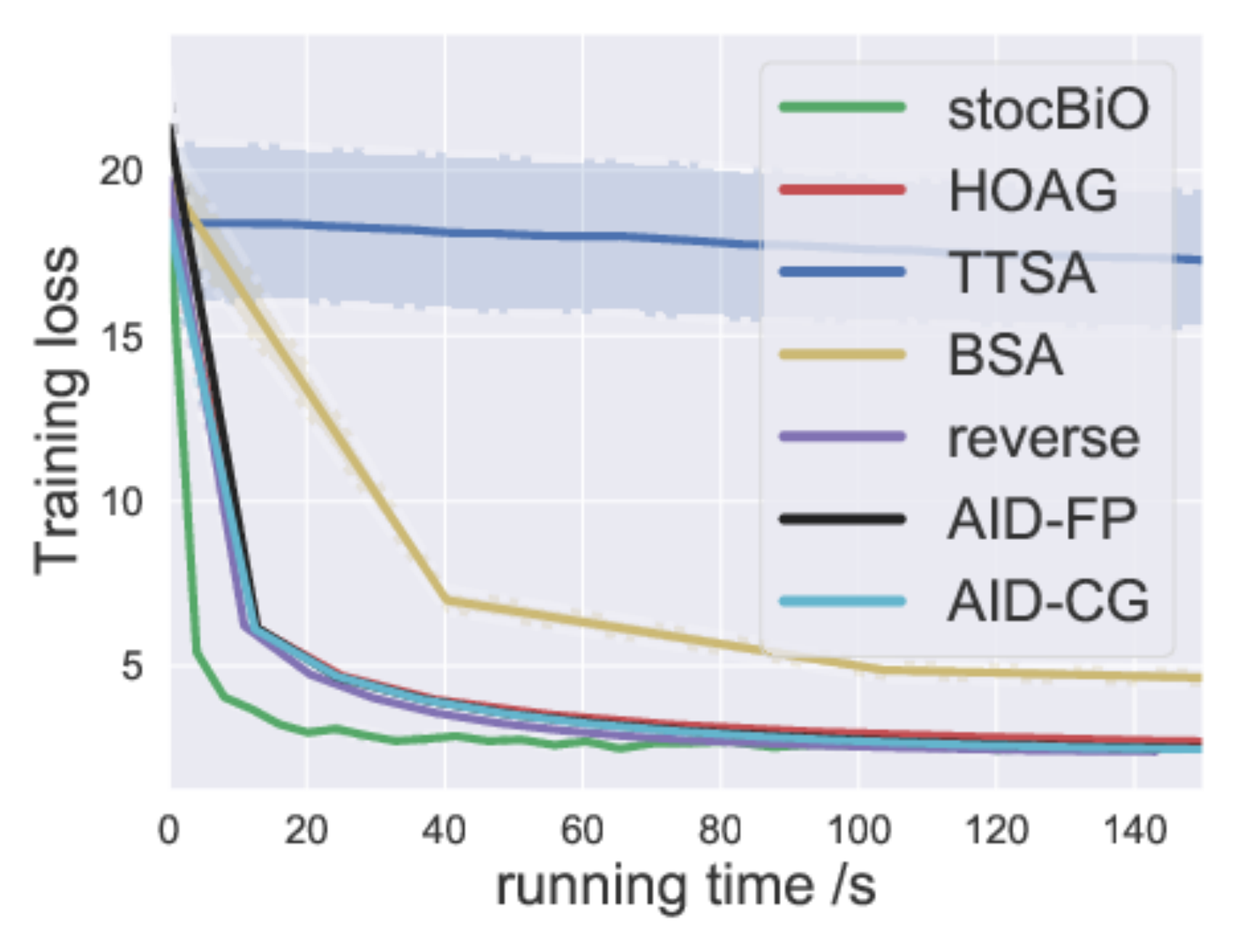}\includegraphics[width=60mm]{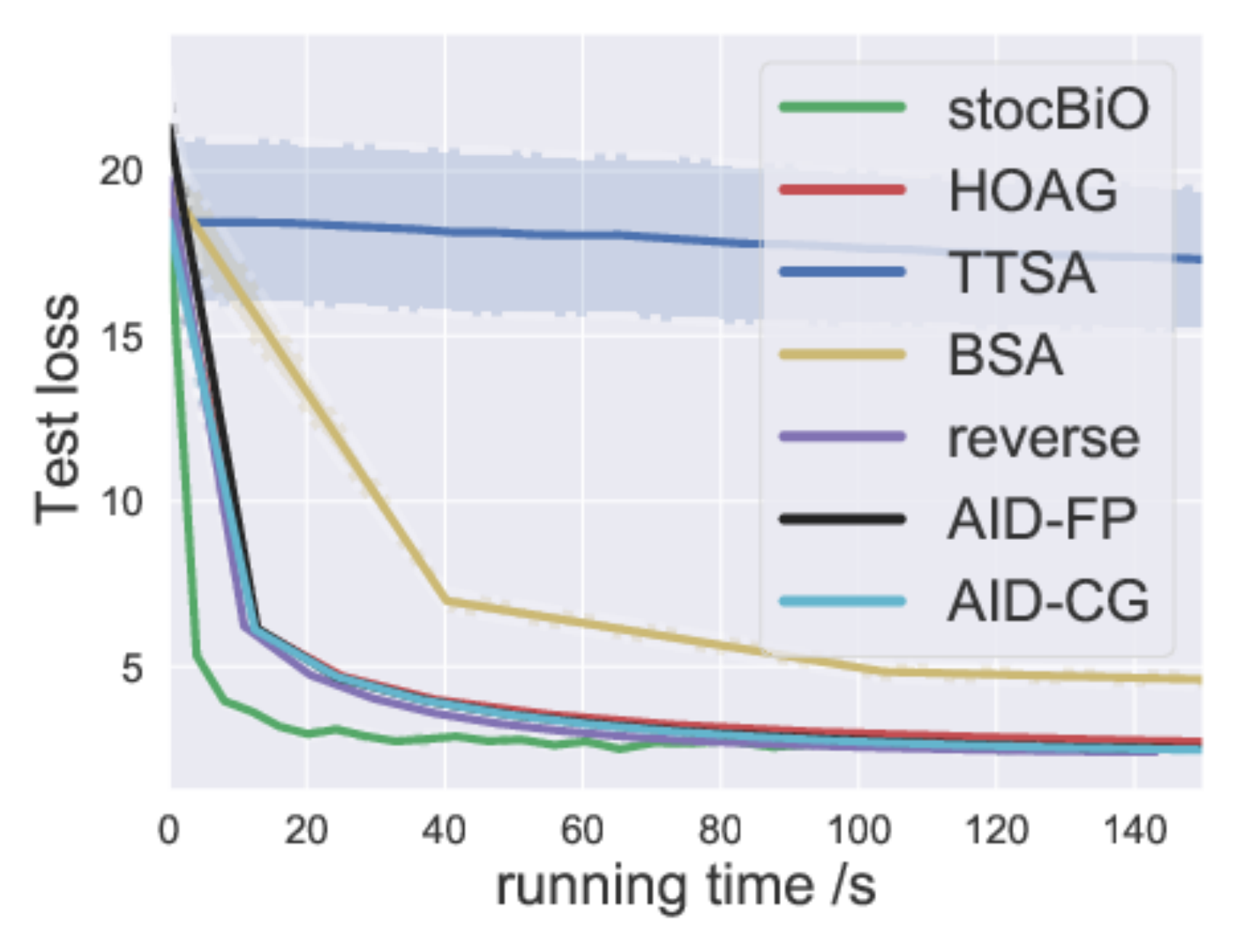}}  
	\vspace{-0.2cm}
	\caption{Comparison of various stochastic bilevel algorithms on hyperparameter optimization at different corruption rates.   For each corruption rate $p$, left plot: training loss v.s. running time; right plot: test loss v.s. running time.}\label{fig:hyper}
	  \vspace{-0.3cm}
\end{figure*}

  \begin{figure}[ht]
  \vspace{-2mm}
	\centering  
	\subfigure{\includegraphics[width=60mm]{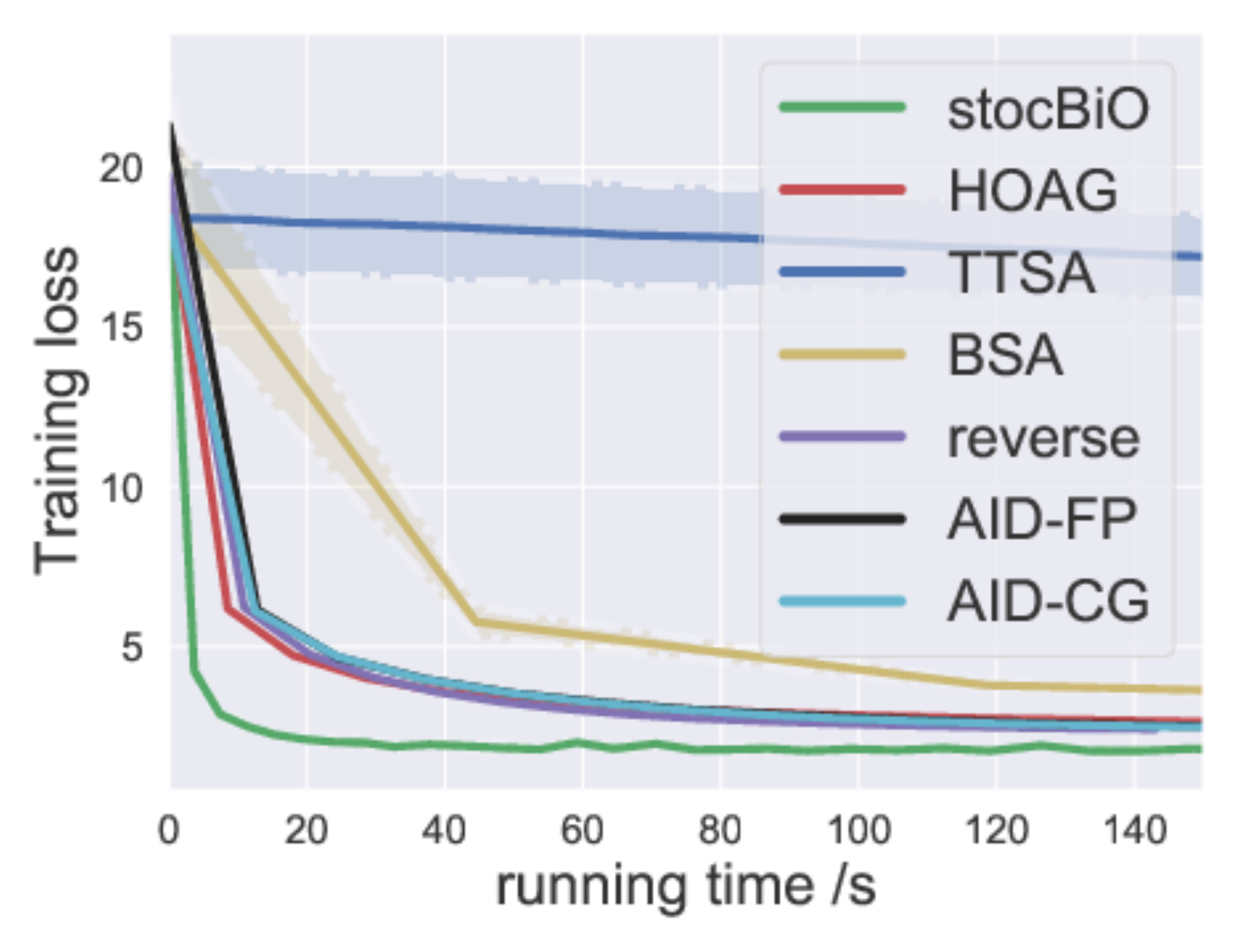}}
	\subfigure{\includegraphics[width=60mm]{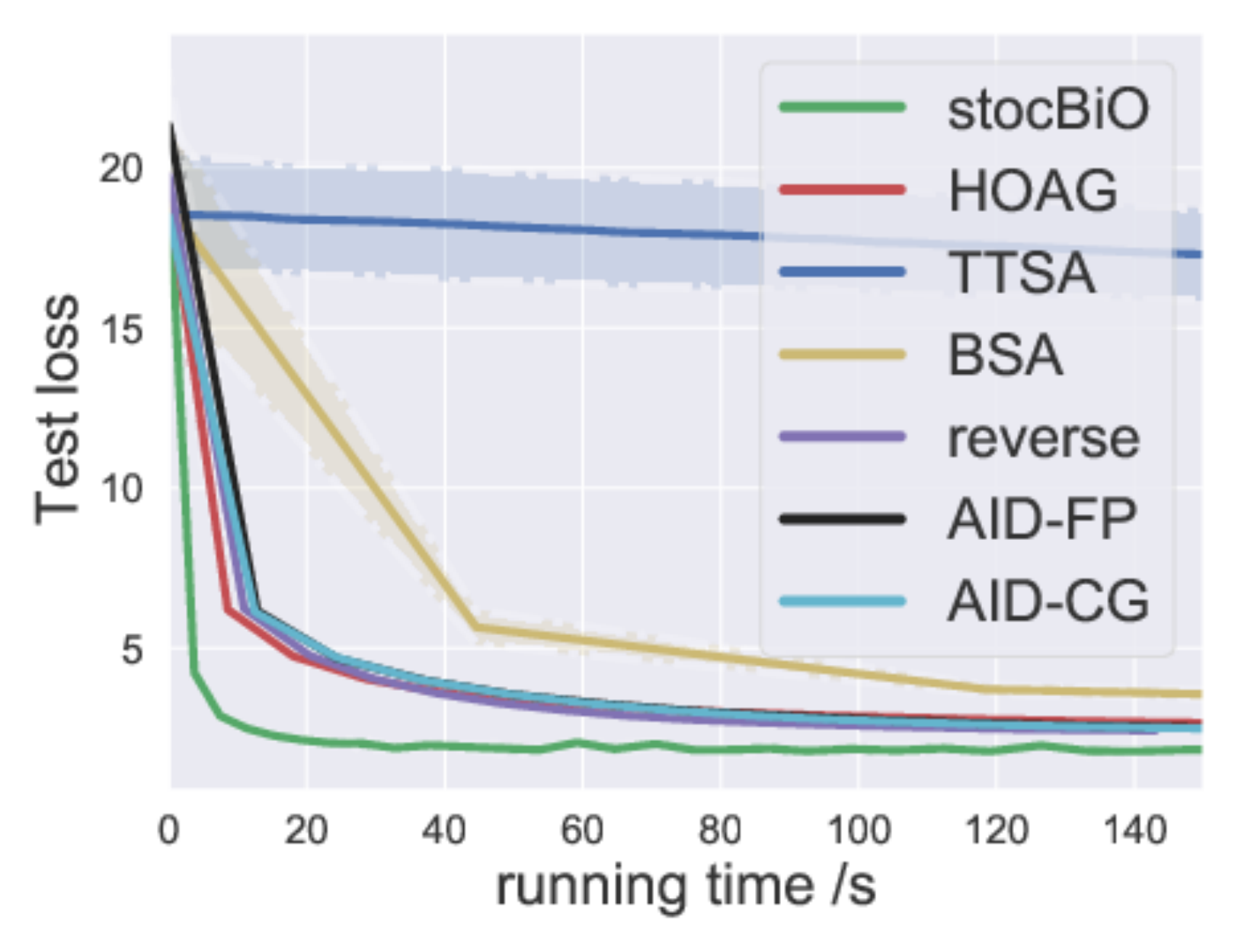}} 
	\vspace{-0.2cm}
	\caption{Convergence of algorithms at corruption rate $p=0.2$.   
	}\label{fig:hyper_appen}
	  \vspace{-0.4cm}
\end{figure}

\vspace{0.1cm}
\noindent{\bf Data Hyper-Cleaning on MNIST.} 
We compare the performance of our proposed algorithm stocBiO with other baseline algorithms BSA, TTSA, HOAG on a hyperparameter optimization problem: data hyper-cleaning~\cite{shaban2019truncated} on a dataset derived from MNIST~\cite{lecun1998gradient}, which consists of 20000 images for training, 5000 images for validation, and 10000 images for testing.  
Data hyper-cleaning is to train a classifier in a corrupted setting where each label of training data is replaced by a random class number with a probability $p$ (i.e., the corruption rate). The objective function is given by 
\begin{align*}
&\min_\lambda E(\lambda,w^*) = \frac{1}{|\gD_{\text{val}}|}\sum_{(x_i,y_i)\in \gD_{\text{val}}} L(w^*x_i, y_i) \nonumber
\\& \;\mbox{s.t.} \quad w^* = \argmin_{w} \gL(w,\lambda):= \frac{1}{|\gD_{\text{tr}}|}\sum_{(x_i,y_i)\in \gD_{\text{tr}}}\sigma(\lambda_i)L(wx_i, y_i)  + C_r \|w\|^2,
\end{align*}
where $L$ is the cross-entropy loss, $\sigma(\cdot)$ is the sigmoid function, $C_r$ is a regularization parameter. Following~\cite{shaban2019truncated}, we choose $C_r=0.001$. 
 All results are averaged over 10 trials with
different random seeds. We adopt Adam~\cite{kingma2014adam} as the optimizer for the outer-loop update for all algorithms. For stochastic algorithms, we set the batch size as $50$ for stocBiO, and $1$ for BSA and TTSA because they use the single-sample data sampling. For all algorithms, we use a grid search to choose the inner-loop stepsize from $\{0.01,0.1,1,10\}$, the  outer-loop stepsize from $\{10^{i},i=-4,-3,-2,-1,0,1,2,3,4\}$, and the number $D$ of inner-loop steps from $\{1,10,50,100,200,1000\}$, where values that achieve the lowest loss after a fixed running time are selected. For  stocBiO, BSA, and TTSA, we choose $\eta$ from $\{0.5\times 2^i, i=-3,-2,-1,0,1,2,3\}$, and $Q$ from $\{3\times 2^i, i=0,1,2,3\}$. 

It can be seen from Figures \ref{fig:hyper} and \ref{fig:hyper_appen} that our proposed stocBiO algorithm achieves the fastest convergence rate among all competing algorithms in terms of both the training loss and the test loss. It is also observed that such an improvement is more significant when the corruption rate $p$ is smaller.  We note that the stochastic algorithm TTSA converges very slowly with a large variance. This is because TTSA  updates the costly outer loop more frequently than other algorithms,  and has a larger variance due to the single-sample data sampling. As a comparison, our stocBiO has a much smaller variance for hypergradient estimation as well as a much faster convergence rate.  This validates our theoretical results in~\Cref{th:nonconvex}.

\section{Summary of Contributions}
In this chapter, we propose a faster stochastic optimization algorithm named stocBiO, and we show that its computational complexity outperforms the best known results orderwisely. Our results also provide the theoretical guarantee for stocBiO in hyperparameter optimization. Our experiments demonstrate the superior performance of the proposed stocBiO algorithm.  
We anticipate that the proposed algorithms will be useful for other applications such as reinforcement learning and Stackelberg game. 

\chapter{Convergence Theory for Model-Agnostic Meta-Learning}\label{chp: maml}
%
In this chapter, we study the convergence of the multi-step MAML algorithm. We consider two types of objective functions that are commonly used in practice: (a) {\bf resampling case}~\cite{finn2017model,fallah2020convergence}, where loss functions take the form  in expectation and new data are sampled as the algorithm runs; and (b) {\bf finite-sum case}~\cite{antoniou2019train}, where loss functions take the finite-sum form with given samples. The resampling case occurs often in reinforcement learning where data are continuously sampled as the algorithm iterates, whereas the finite-sum case typically occurs in classification problems where the datasets are already sampled in advance. In~\Cref{append: maml}, we provide examples for these two types of problems and all technical proofs for the results in this chapter.

\section{Resampling Case for Multi-Step MAML}
Suppose a set $\mathcal{T} = \{\mathcal{T}_i, i\in \mathcal{I}\}$ of tasks are available for learning and tasks are sampled based on a probability distribution $p(\mathcal{T})$ over the task set. Assume that each task $\mathcal{T}_i$  is associated with a loss  $l_i(w): \mathbb{R}^d \rightarrow \mathbb{R}$ parameterized by $w$.

The goal of multi-step MAML is to find a good initial parameter $w^*$ such that after observing a new task, a few gradient descend steps starting from such a point $w^*$ can efficiently approach the optimizer (or a stationary point) of the corresponding loss function. Towards this end, multi-step MAML consists of two nested stages, where the inner stage consists of  {\em multiple} steps of (stochastic) gradient descent for each individual tasks, and the outer stage updates the meta parameter over all the sampled tasks. More specifically, at each inner stage, each $\mathcal{T}_i$ initializes at the meta parameter, i.e., $\widetilde w^i_0 := w$, and runs $N$ {\em gradient descent} steps as 
\begin{align}\label{gd_w}
&\widetilde w^i_{j+1} = \widetilde w^i_j - \alpha \nabla l_i( \widetilde w^i_j), \quad j = 0,1,..., N-1 .
\end{align}
Thus, the loss of task $\mathcal{T}_i$ after the $N$-step inner stage iteration is given by $l_i(\widetilde w^i_N)$, where $\widetilde w^i_N$ depends on the meta parameter $w$ through the iteration updates in \cref{gd_w}, and can hence be written as $\widetilde w^i_N(w)$.  We further define $\mathcal{L}_i(w):=l_i(\widetilde w^i_N(w))$, and hence the overall meta objective is given by
\begin{align}\label{objectiveMAML}
\min_{w\in\mathbb{R}^d} \mathcal{L}(w):= \mathbb{E}_{i\sim p(\mathcal{T})}[ \mathcal{L}_i(w)]  :=  \mathbb{E}_{i\sim p(\mathcal{T})} [l_i(\widetilde w^i_N(w))].
\end{align}
Then the outer stage of meta update is a gradient decent step to optimize the above objective function. Using the chain rule, we provide a simplified form (see~\Cref{simplifeid} for its derivations) of  gradient $\nabla \mathcal{L}_i(w)$ by   
	\begin{align}\label{nablaF}
	\nabla \mathcal{L}_i(w) = \bigg[ \prod_{j=0}^{N-1}(I-\alpha \nabla^2 l_i(\widetilde w^i_{j}))\bigg]\nabla l_i(\widetilde w^i _{N}),
	\end{align}
	where $\widetilde w^i_{0} = w$ for all tasks. 
	Hence, the {\em full gradient descent} step of the outer stage for~\cref{objectiveMAML} can be written as  
\begin{align}\label{true_meta_gd}
w_{k+1}  = w_k - \beta_k \mathbb{E}_{i\sim p(\mathcal{T})}\bigg[\prod_{j=0}^{N-1}(I-\alpha \nabla^2 l_i(\widetilde w^i_{k,j}))\bigg]\nabla l_i(\widetilde w^i _{k,N}),
\end{align}
where the index $k$ is added to $\widetilde w^i_{j}$ in \cref{nablaF} to denote that these parameters are at the $k^{th}$ iteration of the meta parameter $w$.

\begin{algorithm}[t]
	\caption{Multi-step MAML in the resampling case} 
	\label{alg:online}
	\begin{algorithmic}[1]
		\STATE {\bfseries Input:}  Initial parameter $w_0$, inner stepsize $\alpha>0$	
		\FOR{$k=1,...,K$}
		\STATE{Sample $B_k\subset \mathcal{I}$ of i.i.d. tasks by distribution $p(\mathcal{T})$}
		\FOR{all tasks $\mathcal{T}_i$ in $B_k$}
		\FOR{$j = 0, 1,...,N-1$}
		\STATE{Sample a training set $S^i_{k,j}$ 
			\\ Update { $w^i_{k, j+1} = w^i_{k, j} - \alpha \nabla l_i(w^i_{k,j}; S^i_{k,j})$}
		}
		\ENDFOR
		\ENDFOR
		\STATE{Sample $T^i_k$ and $D_{k,j}^i$ and compute {\small$\widehat G_i(w_k)$} through~\cref{eq:metagrad_est}.}
 	\STATE{update  
$		w_{k+1}= w_k - \beta_k\frac{\sum_{i\in B_k}\widehat G_i(w_k) }{|B_k|}$.
			}
		\ENDFOR
	\end{algorithmic}
	\end{algorithm}

The inner- and outer-stage updates of MAML given in \cref{gd_w} and \cref{true_meta_gd} involve the gradient $\nabla l_i(\cdot)$ 
 and the Hessian $\nabla^2 l_i(\cdot)$ of the loss function $l_i(\cdot)$, which takes the form of the expectation over the distribution of data samples as given by 
\begin{align}\label{fiw}
l_i(\cdot) = \mathbb{E}_{\tau} l_i(\cdot\,; \tau),
\end{align} 
where $\tau$ represents the data sample. In practice, these two quantities based on the population loss function are estimated by samples. In specific, each task $\mathcal{T}_i$ samples a batch $\Omega$ of data under the current parameter $w$, and uses $\nabla l_i(\cdot\,; \Omega):= \frac{\sum_{\tau \in \Omega} \nabla l_i(\cdot\,; \tau)}{|\Omega|} $  and $\nabla^2 l_i(\cdot\,; \Omega):= \frac{\sum_{\tau \in \Omega} \nabla^2 l_i(\cdot\,; \tau)}{|\Omega|} $ as {\em unbiased} estimates of  the gradient $\nabla l_i(\cdot)$ and the Hessian $\nabla^2 l_i(\cdot)$, respectively. 

\vspace{0.01cm}

For practical multi-step MAML as shown in Algorithm~\ref{alg:online}, at the $k^{th}$ outer stage, we sample a set $ B_k$ of tasks. Then,  at the inner stage, each task $\mathcal{T}_i\in B_k$ samples a training set {\small$S_{k,j}^i$} for each iteration $j$ in the inner stage, uses  {\small$\nabla l_i(w^i_{k,j};S^i_{k,j})$} as an estimate of {\small$\nabla l_i(\widetilde w^i_{k,j})$} in \cref{gd_w}, and runs a SGD update as
\begin{align}\label{es:up}
w^i_{k, j+1} = w^i_{k, j} - \alpha \nabla l_i(w^i_{k,j};S^i_{k,j}), \quad  j=0,..,N-1,
\end{align} 
where the initialization parameter $w^i_{k,0}=w_k$ for all $i \in B_k$. 

At the $k^{th}$ outer stage, we draw a  batch {\small$T^i_k$} and {\small$D_{k,j}^i$} of data samples  independent from each other and both independent from {\small$S^i_{k,j}$} and
use {\small$\nabla l_i(w_{k,N}^i;T^i_k)$} and  {\small$ \nabla^2 l_i(w_{k,j}^i; D_{k,j}^i)$} to estimate {\small$\nabla l_i(\widetilde w^i _{k,N})$} and {\small$\nabla^2 l_i(\widetilde w^i_{k,j})$} in~\cref{true_meta_gd}, respectively.
Then, the meta parameter $w_{k+1}$ at the outer stage is updated by a SGD step as shown in line $10$ of Algorithm~\ref{alg:online}, 
where the estimated gradient $\widehat G_i(w_k)$  has a form of 
\begin{align}\label{eq:metagrad_est}
\widehat G_i(w_k) =\prod_{j=0}^{N-1}\big(I - \alpha \nabla^2 l_i\big(w_{k,j}^i;D_{k,j}^i\big)\big)\nabla l_i(w_{k,N}^i; T^i_k).
\end{align}
For simplicity, we suppose the sizes of { $S_{k,j}^i$, $D_{k,j}^i$} and  { $T_k^i$}  are $S$, $D$ and $T$.

\section{Finite-Sum Case for Multi-Step MAML}\label{app:finitesum}
In the finite-sum case, each task $\mathcal{T}_i$ is {\em pre-assigned} with a support/training sample set $S_i$ and a query/test sample set $T_i$. Differently from the resampling case, these sample sets are fixed and no additional fresh data are sampled as the algorithm runs. The goal here is to learn an initial parameter $w$ such that for each task $i$, after $N$ {\em gradient descent} steps on data from $S_i$ starting from this $w$, we can find a parameter $w_N$ that performs well on the test data set $T_i$. Thus, each task $\mathcal{T}_i$ is associated with two fixed loss functions { $l_{S_i}(w):= \frac{1}{|S_i|}\sum_{\tau \in S_i} l_i(w; \tau)$} and { $l_{T_i}(w):= \frac{1}{|T_i|}\sum_{\tau \in T_i} l_i(w; \tau)$} with a finite-sum structure, where $l_i(w; \tau)$ is the loss on a single sample point $\tau$ and a parameter $w$.  Then, the meta objective function takes the form of 
\begin{align}\label{objective2}
\min_{w\in\mathbb{R}^d} \mathcal{L}(w):= \mathbb{E}_{i\sim p(\mathcal{T})} [\mathcal{L}_i(w)] =  \mathbb{E}_{i\sim p(\mathcal{T})} [l_{T_i}(\widetilde w^i_N)], 
\end{align}
where $\widetilde w^i_N$ is obtained by 
\begin{align}\label{innerfinite}
&\widetilde w^i_{j+1} = \widetilde w^i_j - \alpha \nabla l_{S_i}(\widetilde w^i_j),  \quad j = 0, 1,..., N-1 \, \text{ with }\, \widetilde w^i_0 := w.
\end{align} 

We want to emphasize that  $S_i$ and $T_i$ are both training datasets (they together form into {\bf meta-training datasets}), and \cref{objective2} is the  meta-training loss, i.e., the empirical loss for estimating the test time expected loss. \cref{objective2} does not involve anything correlated with test error. During the test period, MAML is evaluated over {\bf meta-test datasets} that are separate from meta-training datasets $S_i$ and $T_i$.

Similarly to the resampling case, we define  the expected losses  $l_S(w)=\mathbb{E}_{i} l_{S_i}(w)$  and $l_T(w)=\mathbb{E}_{i} l_{T_i}(w)$, and the meta gradient step of the outer stage for \cref{objective2} can be written as 
\begin{align}\label{fulll_updd}
w_{k+1}  = w_k - \beta_k \mathbb{E}_{i\sim p(\mathcal{T})}\prod_{j=0}^{N-1}(I-\alpha \nabla^2 l_{S_i}(\widetilde w^i_{k,j}))\nabla l_{T_i}(\widetilde w^i _{k,N}),
\end{align}
where the index $k$ is added to $\widetilde w^i_{j}$ in \cref{innerfinite} to denote that these parameters are at the $k^{th}$ iteration of the meta parameter $w$.
 
\begin{algorithm}[t]
	\caption{Multi-step MAML in the finite-sum case} 
	\label{alg:offline}
	\begin{algorithmic}[1]
		\STATE {\bfseries Input:}  Initial parameter $w_0$, inner stepsize $\alpha>0$	
		\FOR{$k=1,...,K$}
		\STATE{Sample  $B_k\subset \mathcal{I}$ of i.i.d. tasks  by distribution $p(\mathcal{T})$}
		\FOR{all tasks $\mathcal{T}_i$ in $ B_k$}
		\FOR{$j = 0, 1,...,N-1$}
		\STATE{
			Update { $w^i_{k, j+1} = w^i_{k, j} - \alpha \nabla l_{S_i}\big(w^i_{k,j}\big)$}
		}
		\ENDFOR
		\ENDFOR
		\STATE{
			Update $w_{k+1}= w_k -\frac{\beta_k}{|B_k|} \sum_{i\in B_k}\widehat G_i(w_k)$
			}
		\ENDFOR
	\end{algorithmic}
\end{algorithm}

As shown in Algorithm~\ref{alg:offline}, MAML in the finite-sum case has a nested structure similar to that in the resampling case except that it does not sample fresh data at each iteration. 
In the inner stage, MAML performs a sequence of {\em full gradient descent steps} (instead of stochastic gradient steps as in the resampling case) for each task $i \in B_k$ as given by 
\begin{align}\label{giniteoo}
w^i_{k, j+1} = w^i_{k, j} - \alpha \nabla l_{S_i}\big(w^i_{k,j}\big), \text{ for } j=0,....,N-1
\end{align}  
where $w_{k,0}^i = w_k$ for all $i\in B_k$. As a result, the parameter $w_{k,j}$ (which denotes the parameter due to the full gradient update) in the update step~\cref{giniteoo} is equal to $\widetilde w_{k,j}$ in \cref{fulll_updd} for all $j=0,...,N$. 


At the outer-stage iteration, the meta optimization of MAML performs a SGD step as shown in line 9 of Algorithm~\ref{alg:offline},  where $ \widehat G_i(w_k) $ is given by 
\begin{align}\label{offline:obj}
 \widehat G_i(w_k) &=\prod_{j=0}^{N-1}(I - \alpha \nabla^2 l_{S_i}(w_{k,j}^i))\nabla l_{T_i}(w_{k,N}^i).
\end{align}

Compared with the resampling case, the biggest difference for analyzing Algorithm~\ref{alg:offline} in the finite-sum case is that the losses $ l_{S_i}(\cdot)$ and $ l_{T_i}(\cdot)$ used in the inner and outer stages respectively are different from each other,  whereas in the resampling case, they both are equal to $l_i(\cdot)$ which takes the expectation over the corresponding samples. 
Thus, the convergence analysis for the finite-sum case  requires to develop different techniques.  For simplicity, we assume that the sizes of all $B_k$ are $B$.

\section{Convergence of Multi-Step MAML in Resampling Case} \label{theory:online}
In this section, we first make some basic assumptions for the meta loss functions. 
\subsection*{Basic Assumptions}
We adopt the following standard assumptions 
\cite{fallah2020convergence,rajeswaran2019meta}. Let $\|\cdot\|$ denote the $\ell_2$-norm or spectrum norm for a vector or matrix, respectively. 
\begin{assum}\label{assum:smooth}
	The loss $l_i(\cdot)$ of task $\mathcal{T}_i$ given by~\cref{fiw} satisfies
	\begin{enumerate}
		\item The loss $l_i(\cdot)$  is bounded below, i.e., { $ \inf_{w\in\mathbb{R}^d} l_i(w) > -\infty$}.  \label{item1}		

		\item $\nabla l_i(\cdot)$ is $L_i$-Lipschitz, i.e., for any $w,u\in\mathbb{R}^d$,     \label{item2}
$		\|\nabla l_i(w)-\nabla l_i(u)\| \leq L_i\|w-u\|$.
		\item  $\nabla^2 l_i(\cdot)$ is  $\rho_i$-Lipschitz, i.e., for any $w,u\in\mathbb{R}^d$,
$		\|\nabla^2 l_i(w)-\nabla^2 l_i(u)\| \leq \rho_i\|w-u\|$.
	\end{enumerate}
\end{assum}

By the definition of the objective function $\mathcal{L}(\cdot)$ in~\cref{objectiveMAML},  item~\ref{item1} of  Assumption~\ref{assum:smooth} implies that $\mathcal{L}(\cdot)$ is bounded below. In addition, 
item~\ref{item2}  implies that {\small$\|\nabla^2 l_i(w)\|\leq L_i$} for any $w\in\mathbb{R}^d$. 
For notational convenience,   we take $L=\max_i L_i$ and $\rho = \max_{i} \rho_i$. 
The following assumptions impose the  bounded-variance conditions on $\nabla l_i(w)$, $\nabla l_i(w; \tau)$ and $\nabla^2 l_i(w; \tau)$.
\begin{assum}\label{a2}
	The stochastic gradient $\nabla l_i(\cdot)$ (with $i$ uniformly randomly chosen from set $\mathcal{I}$) has bounded variance, i.e., there exists a constant $\sigma>0$ such that, for any $w\in\mathbb{R}^d$,
	\begin{align*}
	\mathbb{E}_{i}\|\nabla l_i(w) - \nabla l(w)\|^2 \leq \sigma^2, 
	\end{align*}
	where the expected loss function $l(w):=\mathbb{E}_{i}l_i(w)$.
\end{assum}

\begin{assum}\label{a3}
	For any $w\in\mathbb{R}^d$ and $i\in\mathcal{I}$, there exist positive constants $\sigma_g, \sigma_H>0$ such that  
	\begin{align*}
	&\mathbb{E}_{\tau}\|\nabla l_i(w; \tau)- \nabla l_i(w)\|^2 \leq \sigma_g^2\;\text{ and }\;
\mathbb{E}_{\tau}\|\nabla^2 l_i(w;\tau)- \nabla^2 l_i(w)\|^2 \leq \sigma_H^2.
	\end{align*}
\end{assum}

Note that the above assumptions are  made only on individual loss functions $l_i(\cdot)$ rather than on the total loss $\mathcal{L}(\cdot)$, because some  conditions do not hold for $\mathcal{L}(\cdot)$, as shown later. 

\subsection*{Challenges of Analyzing Multi-Step MAML}

Several new challenges arise when we analyze the convergence of {\em multi}-step MAML (with $N\ge 2$) compared to the one-step case (with $N=1$).

First, each iteration of  the meta parameter affects the overall objective function via a nested structure of $N$-step SGD optimization paths over all tasks. 
Hence, our analysis of the convergence of such a meta parameter 
needs to characterize the nested structure and the recursive updates. 

Second, the meta gradient estimator  {\small$\widehat G_i(w_k)$}  given in \cref{eq:metagrad_est} involves {\small$\nabla^2 l_i(w_{k,j}^i;D_{k,j}^i)$} for $ j=1,...,N-1$, which are all {\em biased} estimators of {\small$\nabla^2 l_i(\widetilde w^i_{k,j})$} in terms of the randomness over $D_{k,j}^i$.
This is because $ w^i_{k,j}$ is  a stochastic estimator of $\widetilde w^i_{k,j}$ obtained via  random training sets $S_{k,t}^i, t=0,...,j-1$ along an $N$-step SGD optimization path in the inner stage. In fact, such a bias error occurs only for multi-step MAML with $N\geq 2$ (which equals zero for $N=1$), and requires additional efforts to handle. 

Third, both the Hessian term {\small$\nabla^2 l_i(w_{k,j}^i; D_{k,j}^i)$} for $j=2,...,N-1$ and the gradient term {\small$\nabla l_i(w_{k,N}^i; T^i_k)$} in the meta gradient estimator {\small$\widehat G_i(w_k)$} given in \cref{eq:metagrad_est} depend on the sample sets  $S_{k,i}^i$ used for inner stage iteration to obtain $w_{k,N}^i$, and hence they are statistically {\em correlated} even conditioned on $w_k$. Such complication also occurs only for multi-step MAML with $N\geq 2$ and requires new treatment (the two terms are independent for $N=1$).



\vspace{0.2cm}
\noindent {\bf Solutions to address the above challenges.} The first challenge is mainly caused by the recursive structure of the meta gradient $\nabla \mathcal{L}(w)$ in \cref{true_meta_gd} and the meta gradient estimator {\small$\widehat G_i(w_k)$}  given in \cref{eq:metagrad_est}. For example, when analyzing the smoothness of the meta gradient $\nabla \mathcal{L}(w)$, we  need to characterize the gap $\Delta_p$ between two quantities $\prod_{j=0}^{N-1}(I-\alpha \nabla^2 l_i(\widetilde w^i_{j}))$ and $\prod_{j=0}^{N-1}(I-\alpha \nabla^2 l_i(\widetilde u^i_{j}))$, where $w_j^i$ and $u_j^i$ are the $j^{th}$ iterates of two different inner-loop updating paths. Then, using the error decomposition strategy that $\|f_1f_2-f_1^\prime f_2^\prime\|\leq\|f_1-f_1^\prime\|\|f_2\| + \|f_1^\prime\|\|f_2-f_2^\prime\|$, we can decompose the error $\Delta_p$ into $N$ parts, where each one corresponds to the distance $\|w^i_{j}-u^i_{j}\|$. The remaining step is to bound the distances $\|w^i_{j}-u^i_{j}\|, j=0,...,N-1$ by finding the relationship between $\|w^i_{j+1}-u^i_{j+1}\|$ and $\|w^i_{j}-u^i_{j}\|$
based on the inner-loop gradient descent updates. 
\vspace{0.01cm}

To address the second and third challenges, we first use the strategy we propose in the first challenge to decompose the error into $N$ components with each one taking the form of $\|w_{k,j}^i-\widetilde w_{k,j}^i\|$, where $w_{k,j}^i$ and $\widetilde w_{k,j}^i$ are the $j^{th}$ stochastic gradient step and true gradient step
of the inner loop at iteration $k$.  The remaining step is to upper-bound the first- and second-moment distances between $w_{k,j}^i$ and $\widetilde w_{k,j}^i$ for all { $j= 0,..., N$}  by finding the relationship between $\|w_{k,j+1}^i-\widetilde w_{k,j+1}^i\|$ and $\|w_{k,j}^i-\widetilde w_{k,j}^i\|$
based on the inner-loop {\em stochastic} gradient updates.

\subsection*{Properties of Meta Gradient}
Differently from the conventional  gradient whose corresponding loss is evaluated directly at the current parameter $w$, the meta  gradient has a more complicated nested structure with respect to $w$, because its loss is evaluated at the final output of the inner optimization stage, which is $N$-step SGD updates.
As a result, analyzing the meta gradient is very different and more challenging compared to analyzing the conventional  gradient. In this subsection, we establish some important properties of the meta gradient which are useful for characterizing the convergence of multi-step MAML.  


Recall that  $\nabla \mathcal{L}(w) = \mathbb{E}_{i\sim p(\mathcal{T})} [\nabla \mathcal{L}_i(w)]$ with  $\nabla \mathcal{L}_i(w)$ given by~\cref{nablaF}. The following proposition characterizes the Lipschitz property of the gradient $\nabla \mathcal{L}(\cdot)$. 
\begin{proposition}\label{th:lipshiz}
	 Suppose   Assumptions~\ref{assum:smooth},~\ref{a2} and~\ref{a3} hold. For $\forall w,u\in\mathbb{R}^d$, we have 
	\begin{align*}
	\|\nabla \mathcal{L}(w) - \nabla \mathcal{L}(u)\|  \leq \big( (1+\alpha L) ^{2N}L + C_\mathcal{L} \mathbb{E}_{i}\|\nabla l_i(w)\|\big) \|w-u\|,
	\end{align*}
	where $C_\mathcal{L}$ is a positive constant  given by{
	\begin{align}\label{clcl}
	C_\mathcal{L}= \big(  (1+\alpha L) ^{N-1}\alpha \rho + \frac{\rho}{L}  (1+\alpha L) ^N ( (1+\alpha L) ^{N-1} -1) \big)  (1+\alpha L) ^N.
	\end{align}}
\end{proposition}   
The proof of Proposition~\ref{th:lipshiz} handles the aforementioned first challenge. More specifically,  
we bound the differences between $\widetilde w_j^i$ and $\widetilde u_j^i$ along two separate paths $(\widetilde w_j^i, j =0,....,N)$ and $(\widetilde u_j^i, j =0,....,N )$, and then connect these differences to the distance $\|w-u\|$. Proposition~\ref{th:lipshiz} shows that the objective $\mathcal{L}(\cdot)$ has a gradient-Lipschitz parameter $$L_w =  (1+\alpha L) ^{2N}L + C_\mathcal{L}\mathbb{E}_{i}\|\nabla l_i(w)\|,$$ 
which can  be unbounded since  $\nabla l_i(w)$ may be unbounded. 
Similarly to~\cite{fallah2020convergence}, we  use 
\begin{align}\label{hatlw}
\widehat L_{w_k} =   (1+\alpha L) ^{2N}L + \frac{ C_\mathcal{L}\sum_{i\in B_k^\prime}\|\nabla l_i(w_k; D_{L_k}^i)\|}{|B_k^\prime|}
\end{align}
to estimate $L_{w_k}$ at the meta parameter $w_k$, where we  {\em independently} sample the data sets $B_k^\prime$ and $D_{L_k}^i$. As will be shown in Theorem~\ref{th:mainonline}, we  set the meta stepsize $\beta_k$ to be inversely proportional to { $\widehat L_{w_k} $} to handle the  possibly unboundedness.


We next characterize several estimation properties of  the meta gradient estimator  $\widehat G_i(w_k)$ in \cref{eq:metagrad_est}. 
Here, we address the second and third challenges.  We first quantify how far  the stochastic gradient iterate $w_{k,j}^i$ is away from the true gradient iterate $\widetilde w_{k,j}^i$, and then provide upper bounds on the first- and second-moment distances between $w_{k,j}^i$ and $\widetilde w_{k,j}^i$ for all { $j= 0,..., N$} as below. 
\begin{proposition}\label{le:distance}
	Suppose that  Assumptions~\ref{assum:smooth},~\ref{a2} and~\ref{a3} hold. Then, for any  $j=0,..., N$ and $i \in B_k$, we have
\begin{itemize}

	\item {\bf First-moment :} $\mathbb{E}(\|w_{k,j}^i - \widetilde w_{k,j}^i \| \, | w_k) \leq \big( (1+\alpha L) ^j -1 \big) \frac{\sigma_g }{L \sqrt{S}}$.

\item  {\bf Second-moment:} $\mathbb{E}(\|w_{k,j}^i - \widetilde w_{k,j}^i \|^2\, | w_k) \leq \big((1+\alpha L +2\alpha^2 L^2)^j -1 \big) \frac{\alpha \sigma_g ^2}{(1+\alpha L)L   S}$.
	\end{itemize}
\end{proposition}
Proposition~\ref{le:distance} shows that we can effectively upper-bound the point-wise distance between two paths  by choosing $\alpha$ and $S$ properly. Using Proposition~\ref{le:distance}, we provide an upper bound on the first-moment estimation error of meta gradient estimator {$\widehat G_i(w_k)$}. 
\begin{proposition}\label{th:first-est} 
	Suppose   Assumptions~\ref{assum:smooth},~\ref{a2} and~\ref{a3} hold, and define constants 
	\begin{align}\label{ppolp}
  C_{{\text{\normalfont err}}_1} =  (1+\alpha L)^{2N} \sigma_g, \;\;C_{{\text{\normalfont err}}_2}  = \frac{ (1+\alpha L) ^{4N}\rho \sigma_g}{\big (2-  (1+\alpha L)^{2N}\big )L^2} .
	\end{align}
	Let $e_k := \mathbb{E}[\widehat G_i(w_k)] - \nabla \mathcal{L}(w_k) $
	 be the estimation error. If  the inner stepsize $\alpha < (2^{\frac{1}{2N}} - 1)/L$, then conditioning on $w_k$, we have 
	\begin{align}\label{es:original}
	\|e_k\| \leq \frac{C_{{\text{\normalfont err}}_1} }{\sqrt{S}} + \frac{C_{{\text{\normalfont err}}_2}  }{\sqrt{S}} (\|\nabla \mathcal{L}(w_k)\| + \sigma).
	\end{align} 
\end{proposition}
Note that 
the estimation error for the multi-step case shown in Proposition~\ref{th:first-est} involves a term $\mathcal{O}\big(\frac{\|\nabla \mathcal{L}(w_k)\|}{\sqrt{S}}\big)$, which cannot be avoided due to the Hessian approximation error caused by the randomness over the inner-loop samples sets $S_{k,j}^i$. Somewhat interestingly, our later analysis shows that this term does not affect the final convergence rate if we choose the size $S$ properly. 
The following proposition provides an upper-bound on the second moment  of the meta gradient estimator $\widehat G_i(w_k)$.
\begin{proposition}\label{th:second} 
	Suppose that Assumptions \ref{assum:smooth}, \ref{a2} and \ref{a3} hold. Define  constants 
 \begin{align}\label{para:seq12}
	C_{\text{\normalfont squ}_1} &= 3\Big(  \frac{\alpha^2 \sigma_H^2}{D} + (1+\alpha L)^2  \Big)^N \sigma_g^2, 	\;\;
	C_{\text{\normalfont squ}_3} =\frac{2C_{\text{\normalfont squ}_1}  (1+\alpha L)^{2N}}{(2- (1+\alpha L)^{2N})^2\sigma_g^2}, \nonumber
	\\       C_{\text{\normalfont squ}_2}  &= C_{\text{\normalfont squ}_1} \big((1+2\alpha L +2\alpha^2L^2)^N-1\big)	\alpha L  (1+\alpha L)^{-1}.
	\end{align}
	If  the inner stepsize $\alpha < (2^{\frac{1}{2N}} - 1)/L$, then conditioning on $w_k$, we have 
	\begin{align}\label{bigmans}
	\mathbb{E}\|\widehat G_i(w_k)\|^2 \leq& \frac{	C_{\text{\normalfont squ}_1}}{T} + \frac{	C_{\text{\normalfont squ}_2} }{S} +	C_{\text{\normalfont squ}_3} \left(  \|\nabla \mathcal{L}(w_k)\|^2 +\sigma^2   \right).
	\end{align}
\end{proposition}
By choosing set sizes $D,T,S$ and the inner stepsize $\alpha$ properly, the factor $C_{\text{\normalfont squ}_3}$ in the second-moment error bound in \cref{bigmans} can be made at a constant level and the first two error terms $\frac{C_{\text{\normalfont squ}_1}}{T} $ and $ \frac{C_{\text{\normalfont squ}_2} }{S}$ can be made sufficiently small so that the variance of the meta gradient estimator can be well controlled in the convergence analysis, as shown later. 
\subsection*{Main Convergence Results}
By using the established properties of the meta gradient, we provide the convergence rate for multi-step MAML of Algorithm~\ref{alg:online} in the following theorem. 
\begin{theorem}\label{th:mainonline}  
	Suppose that  Assumptions~\ref{assum:smooth},~\ref{a2} and~\ref{a3} hold. 
	Set the meta stepsize $\beta_k = \frac{1}{C_\beta \widehat L_{w_k}} $, where  $C_\beta>0$ is a positive constant and $\widehat L_{w_k}$ is the approximated smoothness parameter given by~\cref{hatlw}. For $\widehat L_{w_k}$ in~\cref{hatlw}, we choose $|B_k^\prime| > \frac{4C^2_{\mathcal{L}}\sigma^2}{3(1+\alpha L)^{4N}L^2}$ and $|D_{L_k}^i| > \frac{64\sigma^2_g C_\mathcal{L}^2}{(1+\alpha L)^{4N}L^2}$ for all $i \in B_k^\prime$, where  $C_\mathcal{L}$ is given by \cref{clcl}. We define 
	\begin{align}\label{para:com}
	\chi =&  \frac{(2- (1+\alpha L)^{2N}) (1+\alpha L)^{2N}L}{C_{\mathcal{L}}} + \sigma  \nonumber
	\\\xi =& \frac{6}{C_\beta L}\big( \frac{1}{5} +\frac{2}{C_\beta}   \big)\big( C^2_{{\text{\normalfont err}}_1} +C^2_{{\text{\normalfont err}}_2}\sigma^2\big), \; \phi = \frac{2}{C_\beta^2 L} \Big(  \frac{C_{\text{\normalfont squ}_1}}{T}  +  \frac{C_{\text{\normalfont squ}_2}}{S} + C_{\text{\normalfont squ}_3} \sigma^2 \Big) \nonumber
	\\\theta = &  \frac{2\big(2- (1+\alpha L)^{2N}\big)}{C_\beta C_{\mathcal{L}}}\Big(   \frac{1}{5} - \big(\frac{3}{5} + \frac{6}{C_\beta}\big)\frac{C^2_{{\text{\normalfont err}}_2} }{S} - \frac{C_{\text{\normalfont squ}_3}}{C_\beta B} - \frac{2}{C_\beta}\Big)  
	\end{align}
    where $C_{{\text{\normalfont err}}_1},C_{{\text{\normalfont err}}_2}$ are given in \cref{ppolp} and $C_{\text{\normalfont squ}_1},C_{\text{\normalfont squ}_2},C_{\text{\normalfont squ}_3} $ are given in \cref{para:seq12}. 
	Choose the inner stepsize $\alpha < (2^{\frac{1}{2N}} - 1)/L$, and choose $C_\beta, S$ and $B$ such that $\theta >0$. 
	Then, Algorithm~\ref{alg:online} finds a solution $w_{\zeta}$ such that 
	\begin{align}\label{c:result}
	\mathbb{E}\|\nabla \mathcal{L}(w_\zeta) \|  \leq &\frac{\Delta}{\theta }\frac{1}{K} +    \frac{\xi}{\theta}\frac{1}{S} +  \frac{\phi }{\theta}\frac{1}{B}  + \sqrt{\frac{\chi}{2} }\sqrt{\frac{\Delta}{\theta }\frac{1}{K}
	+    \frac{\xi}{\theta}\frac{1}{S} +  \frac{\phi }{\theta}\frac{1}{B}},
	\end{align}
	where $\Delta = \mathcal{L}(w_0) - \mathcal{L}^*$ with $\mathcal{L}^*=\inf_{w\in\mathbb{R}^d} \mathcal{L}(w)$.
\end{theorem}
Note that for $\chi$ in Theorem~\ref{th:mainonline}, we replace the notation $C_{l}$ by $(1+\alpha L)^{2N} - 1$ based on its definition.  
The proof of Theorem~\ref{th:mainonline} (see~\Cref{prop:meta_grad} for details) consists of four main steps: step $1$ of bounding an iterative meta update by the meta-gradient smoothness established by Proposition~\ref{th:lipshiz}; step $2$ of characterizing first-moment estimation error of the meta-gradient estimator $\widehat G_i(w_k)$ by Proposition~\ref{th:first-est}; step $3$ of characterizing second-moment estimation error of the meta-gradient estimator $\widehat G_i(w_k)$ by Proposition~\ref{th:second}; and step $4$ of combining steps 1-3, and telescoping to yield the convergence. 

In Theorem~\ref{th:mainonline}, the convergence rate given by  \cref{c:result}  mainly contains three parts:
the first term $\frac{\Delta}{\theta }\frac{1}{K}$ indicates that the meta parameter converges sublinearly with the number $K$ of meta iterations, 
 the second term $ \frac{\xi}{\theta}\frac{1}{S}$ captures the estimation error of $\nabla l_i(w^i_{k,j};S^i_{k,j})$ for approximating the full gradient $\nabla l_i(w^i_{k,j})$ which can be made sufficiently small by choosing a large sample size $S$,   
 and the third term $\frac{\phi }{\theta}\frac{1}{B}$ captures the estimation error and variance of the stochastic meta gradient, 
 which can be made small by choosing large $B,T$ and $D$
  (note that $\phi$ is proportional to both $\frac{1}{T}$ and $\frac{1}{D}$).

It is worthwhile mentioning that our results here focus on our resampling case, where fresh data are resampled as the algorithm runs. This resampling case often happens in bandit or reinforcement learning settings, where batch sizes $S,B,D,T$ can be chosen to be large and the resulting convergence errors will be small. However, for the cases where $S,B,D,T$ are small, our results in Theorem~\ref{th:mainonline} will contain large convergence errors. It is possible to use some techniques such as variance reduction to reduce or even remove such errors. However, this is not the focus of this thesis, and require future efforts to address.

Our analysis reveals several insights for the convergence of multi-step MAML as follows.
(a) To guarantee convergence, we require $\alpha L< 2^{\frac{1}{2N}} - 1$ (e.g., $\alpha=\Theta(\frac{1}{NL})$). Hence, if the number $N$ of inner gradient steps is large and $L$ is not small (e.g., for some RL problems), we need to choose a small inner stepsize  $\alpha$ so that  the last output of the inner stage has a {\em strong dependence} on the initialization (i.e., meta parameter). This is also explained in  \cite{rajeswaran2019meta}, where they add a regularizer $\lambda\|w^\prime-w\|^2$ to make sure the inner-loop output $w^\prime$ has a close connection to the initialization $w$. (b) For problems with small Hessians such as many classification/regression problems~\cite{finn2017model}, $L$ (which is an upper bound on the spectral norm of Hessian matrices) is small, and hence we can choose a larger $\alpha$. This explains the empirical findings in~\cite{finn2017model,antoniou2019train}, where their experiments tend to set a larger stepsize for the regression problems with smaller Hessians. 

We next specify the selection of parameters  to simplify the convergence result in Theorem~\ref{th:mainonline} and derive the  complexity of Algorithm~\ref{alg:online} for finding an $\epsilon$-accurate  stationary point. 

\begin{corollary}\label{co:online}
	Under the setting of Theorem~\ref{th:mainonline}, choose $\alpha = \frac{1}{8NL}, C_\beta = 100$ and let batch sizes $S\geq \frac{15\rho^2\sigma_g^2}{L^4}$ and $D\geq \sigma_H^2 L^2$. Then we have 
	\begin{align*}
	\mathbb{E}\|\nabla \mathcal{L}(w_\zeta) \|  \leq &\mathcal{O} \Big(  \frac{1}{K} + \frac{\sigma_g^2(\sigma^2+1)}{S} + \frac{\sigma_g^2 +\sigma^2}{B} +\frac{\sigma^2_g}{TB} 
	\\&+ \sqrt{\sigma +1}\sqrt{\frac{1}{K} + \frac{\sigma_g^2(\sigma^2+1)}{S} + \frac{\sigma_g^2 +\sigma^2}{B}+\frac{\sigma^2_g}{TB}}\Big).
	\end{align*}
To achieve $\mathbb{E}\|\nabla \mathcal{L}(w_\zeta) \|<\epsilon$, Algorithm~\ref{alg:online} requires at most  $\mathcal{O}\big(\frac{1}{\epsilon^2}\big)$ iterations, and $\mathcal{O}(\frac{N}{\epsilon^4}+\frac{1}{\epsilon^{2}})$ gradient computations and $\mathcal{O}\big(\frac{N}{\epsilon^{2}}\big)$ Hessian computations per meta iteration. 
\end{corollary}
\vspace{-0.15cm}


Differently from the conventional SGD that requires a gradient complexity of { $\mathcal{O}(\frac{1}{\epsilon^{4}})$}, MAML requires a higher gradient complexity by a factor of { $\mathcal{O}(\frac{1}{\epsilon^{2}})$}, which is unavoidable because MAML requires  { $\mathcal{O}(\frac{1}{\epsilon^{2}})$}  tasks to achieve an $\epsilon$-accurate meta point, whereas SGD runs only over one task.  

 Corollary~\ref{co:online} shows that 
given a properly chosen inner stepsize, e.g., $\alpha = \Theta(\frac{1}{NL})$, MAML is guaranteed to converge 
 with both the gradient and the Hessian computation complexities growing only {\em linearly} with $N$. These results explain some empirical findings for MAML training in~\cite{rajeswaran2019meta}.  
The above results can also be obtained by using a larger stepsize such as  $\alpha = \Theta(c^{\frac{1}{N}}-1)/L> \Theta\big(\frac{1}{NL}\big)$with a certain constant $c>1$. 

\section{Convergence of Multi-Step MAML in Finite-Sum Case}\label{theory:offline}
In this section, we provide several properties of the meta gradient for the finite-sum case, and then analyze the convergence and complexity of Algorithm~\ref{alg:offline}.  Differently from the resampling case, we develop novel techniques  to handle the difference between two losses over the training and  test sets (i.e., inner- and outer-loop losses) in the analysis, whereas these two losses are the same for the resampling case. 
\subsection*{Basic Assumptions} 
We state several standard assumptions  for the analysis in the finite-sum case.  
\begin{assum}\label{assum:smoothoff}
	For each task $\mathcal{T}_i$, the loss functions $l_{S_i}(\cdot)$ and  $l_{T_i}(\cdot)$
	in~\cref{objective2} satisfy
	\begin{enumerate}
		\item Loss functions $l_{S_i}(\cdot)$ and $l_{T_i}(\cdot)$ are bounded below.
		\item Gradients $\nabla l_{S_i}(\cdot)$ and  $\nabla l_{T_i}(\cdot)$ are $L$-Lipschitz continuous, i.e.,  for any $w,u\in\mathbb{R}^d$
		\begin{align*}
		\|\nabla l_{S_i}(w)-\nabla l_{S_i}(u)\| \leq L\|w-u\| \text{ and }
		\|\nabla l_{T_i}(w)-\nabla l_{T_i}(u)\| \leq L\|w-u\|.
		\end{align*}
		\item  Hessians $\nabla^2 l_{S_i}(\cdot)$ and  $\nabla^2 l_{T_i}(\cdot)$ are  $\rho$-Lipschitz continuous, i.e., for any $w,u\in\mathbb{R}^d$
		\begin{align*}
		\|\nabla^2l_{S_i}(w)-\nabla^2l_{S_i}(u)\| \leq \rho\|w-u\| \text{ and }	\|\nabla^2l_{T_i}(w)-\nabla^2l_{T_i}(u)\| \leq \rho\|w-u\|.
		\end{align*}
	\end{enumerate}
\end{assum}
The following assumption provides  two conditions   $\nabla l_{S_i}(\cdot)$ and  $\nabla l_{T_i}(\cdot)$. 
\begin{assum}\label{assum:vaoff}
	For all $w\in\mathbb{R}^d$, gradients $\nabla l_{S_i}(w)$ and  $\nabla l_{T_i}(w)$ satisfy
	\begin{enumerate}
		\item $\nabla l_{T_i}(\cdot)$ has a bounded variance, i.e., there exists a constant $\sigma>0$ such that 
		\begin{align*}
		\mathbb{E}_{i}\|\nabla  l_{T_i}(w) - \nabla  l_{T}(w)\|^2 \leq \sigma^2,
		\end{align*}
		where $\nabla l_{T}(\cdot) =\mathbb{E}_{i} \left [ \nabla l_{T_i}(\cdot)\right]$. 
		\item For each $i\in \mathcal{I}$, there exists a constant $b_i>0$ such that $\|\nabla l_{S_i}(w)-\nabla l_{T_i}(w)\| \leq b_i.$
	\end{enumerate}
\end{assum}
Instead of  imposing a  bounded variance condition on the stochastic gradient $\nabla l_{S_i}(w)$, we alternatively assume the difference $\|\nabla l_{S_i}(w)-\nabla  l_{T_i}(w)\|$ to be upper-bounded by  a constant, which is more reasonable because sample sets $S_i$ and $T_i$ are often sampled from the same distribution and share certain statistical similarity. We note that the second condition also implies $\|\nabla l_{S_i}(w)\| \leq \| \nabla l_{T_i}(w)\|+ b_i$, which  is weaker than the bounded gradient assumption made in papers such as~\cite{finn2019online}. 
It is worthwhile mentioning that the second condition can be relaxed to $\|\nabla l_{S_i}(w)\| \leq c_i\| \nabla l_{T_i}(w)\|+ b_i$ for a constant $c_i>0$. Without the loss of generality, we consider $c_i=1$ for simplicity.

\subsection*{Properties of Meta Gradient}
We develop several important properties of the meta gradient. 
The following proposition characterizes a Lipschitz property of the gradient  of the objective function   $$\nabla \mathcal{L}(w) =  \mathbb{E}_{i\sim p(\mathcal{T})} \prod_{j=0}^{N-1}(I - \alpha \nabla^2 l_{S_i}(\widetilde w_{j}^i))\nabla l_{T_i}(\widetilde w_{N}^i),$$ where the weights $ \widetilde w_{j}^i, i\in \mathcal{I}, j=0,..., N$ are given by the steps in~\cref{innerfinite}. 
\begin{proposition}\label{finite:lip} 
	Suppose Assumptions~\ref{assum:smoothoff},~\ref{assum:vaoff} hold.  Then, for any $w,u\in\mathbb{R}^d$, we have 
	\begin{align*}
	\|\nabla \mathcal{L}(w) - \nabla \mathcal{L}(u)\| \leq L_{w}\|w-u\|,\; L_{w} = (1+\alpha L)^{2N}L + C_bb + C_{\mathcal{L}} \mathbb{E}_{i}\|\nabla l_{T_i}(w)\|
	\end{align*}
	where $b=\mathbb{E}_{i} [b_i]$ and $C_b, C_{\mathcal{L}} >0$ are constants given by  
	\begin{small}
	\begin{align}\label{cl1ss}
	C_b= \big( \alpha \rho + \frac{\rho}{L} (1+\alpha L)^{N-1}  \big)(1+\alpha L)^{2N}, \;C_{\mathcal{L}} &= \big( \alpha \rho + \frac{\rho}{L}  (1+\alpha L)^{N-1} \big) (1+\alpha L)^{2N}.
	\end{align}
	\end{small}
\end{proposition}
Proposition~\ref{finite:lip} shows that $\nabla \mathcal{L}(w)$ has a Lipschitz parameter $L_{w}$. Similarly to~\cref{hatlw}, we use the following construction 
\begin{align}\label{hlwkoff}
\hat L_{w_k} =(1+\alpha L)^{2N}L + C_b b +  \frac{C_\mathcal{L}}{|B_k^\prime|}\sum_{i \in B_k^\prime}\|\nabla l_{T_i}(w_k)\|,
\end{align}
at the $k^{th}$ outer-stage iteration to approximate $L_{w_k}$, where $B_k^\prime \subset \mathcal{I}$ is  chosen independently from $B_k$. 
It can be verified that the gradient estimator $\widehat G_i(w_k)$ given in~\cref{offline:obj} is an {\em unbiased} estimate of  $\nabla \mathcal{L}(w_k)$. Thus, our next step is to upper-bound the second moment of $\widehat G_i(w_k)$.
\begin{proposition}\label{finite:seconderr} 
	Suppose Assumptions~\ref{assum:smoothoff} and~\ref{assum:vaoff} are hold, and define constants 	
\begin{align}\label{wocaopp}
A_{\text{\normalfont squ}_1} &= \frac{4(1+\alpha L)^{4N}}{(2-(1+\alpha L)^{2N})^2}, \nonumber
\\A_{\text{\normalfont squ}_2} &= \frac{4(1+\alpha L)^{8N}}{(2-(1+\alpha L)^{2N})^2}(\sigma+b)^2 + 2(1+\alpha)^{4N}(\sigma^2+\widetilde b),
	\end{align}
where  $\widetilde b =\mathbb{E}_{i\sim p(\mathcal{T})}[b_i^2]$. Then, if $\alpha < (2^{\frac{1}{2N}} - 1)/L$, then conditioning on $w_k$, we have 
	\begin{align*}
	\mathbb{E}\|\widehat G_i(w_k)\|^2 \leq A_{\text{\normalfont squ}_1} \|\nabla \mathcal{L}(w_k)\|^2 + A_{\text{\normalfont squ}_2}.
	\end{align*}
\end{proposition}
Based on the above properties, we next characterize the convergence of  MAML. 
\subsection*{Main Convergence Results}
In this subsection, we provide the convergence and complexity  analysis for Algorithm~\ref{alg:offline} based on the properties established in the previous subsection. 
\begin{theorem}\label{mainth:offline} 
	Let Assumptions~\ref{assum:smoothoff} and~\ref{assum:vaoff} hold, and apply Algorithm~\ref{alg:offline} to solve the objective function~\cref{objective2}. 
	Choose the meta stepsize $\beta_k = \frac{1}{C_\beta \widehat L_{w_k}} $ with  $\widehat L_{w_k}$  given by \cref{hlwkoff}, where $C_\beta>0$ is a constant. For $\widehat L_{w_k}$ in \cref{hlwkoff}, we choose  the batch size $|B_k^\prime| $ such that  $|B_k^\prime| \geq \frac{2C^2_\mathcal{L}\sigma^2}{( C_b b + (1+\alpha L)^{2N} L)^2}$, where $C_b$ and $C_\mathcal{L}$ are given by \cref{cl1ss}.
	Define constants 
	\begin{small}
	\begin{align}\label{offline:constants}
	\xi = &\frac{2-(1+\alpha L)^{2N}}{C_\mathcal{L}} (1+\alpha L)^{2N}L+\frac{ \big(2-(1+\alpha L)^{2N}\big) C_b b }{C_\mathcal{L}} +(1+\alpha L)^{3N}b, \nonumber
	\\ 
	\theta = &	\frac{2-(1+\alpha L)^{2N}}{C_\mathcal{L}} \Big(  \frac{1}{C_\beta}  - \frac{1}{C_\beta^2}\Big( \frac{A_{\text{\normalfont squ}_1}}{B}+1\Big)\Big), \; \phi = \frac{A_{\text{\normalfont squ}_2}}{LC_\beta^2}
	\end{align}
	\end{small}
	\hspace{-0.15cm}where $C_b,C_{\mathcal{L}}, A_{\text{\normalfont squ}_1}$ and $A_{\text{\normalfont squ}_1}$ are given by \cref{cl1ss} and \cref{wocaopp}. Choose  $\alpha < (2^{\frac{1}{2N}} - 1)/L$, and choose $C_\beta$ and $B$ such that $\theta >0$. Then,  Algorithm~\ref{alg:offline} attains a solution $w_{\zeta}$  such that 
	\begin{align}\label{iopnn}
	\mathbb{E}\|\nabla \mathcal{L}(w_\zeta)\| \leq \frac{\Delta}{2\theta K} +\frac{\phi}{2\theta B} + \sqrt{ \xi \Big(\frac{\Delta}{\theta K} +\frac{\phi}{\theta B}\Big) + \Big(\frac{\Delta}{2\theta K} +\frac{\phi}{2\theta B}\Big)^2 }.
	\end{align}
\end{theorem}
The parameters $\theta, \phi$ and $\xi$ in Theorem~\ref{mainth:offline} take complicate forms. 
The following corollary specifies the parameters $C_\beta, \alpha$ in Theorem~\ref{mainth:offline} and 
provides a simplified result for Algorithm~\ref{alg:offline}.
\begin{corollary}\label{co:mainoffline}
	Under the same setting of Theorem~\ref{mainth:offline}, choose  $\alpha = \frac{1}{8NL}$ and $ C_\beta = 80$. Then, we have
	\begin{align*}
	\mathbb{E}\|\nabla \mathcal{L}(w_\zeta)\|  \leq  \mathcal{O}\Big(  \frac{1}{K} +\frac{\sigma^2}{B} +\sqrt{\frac{1}{K} +\frac{\sigma^2}{B} }  \Big).
	\end{align*}
In addition, suppose the  batch size $B$ further satisfies $B\geq C_B\sigma^2\epsilon^{-2}$, where $C_B$ is a sufficiently large constant. Then, to achieve an $\epsilon$-approximate stationary point,  Algorithm~\ref{alg:offline} requires at most $K=\mathcal{O}(\epsilon^{-2})$ iterations, and a total number $\mathcal{O}\big((T+NS)\epsilon^{-2}\big)$ of gradient computations and a number $\mathcal{O}\big(NS\epsilon^{-2}\big)$ of Hessian computations per iteration, where $T$ and $S$ correspond to the sample sizes of the pre-assigned sets $T_i,i\in\mathcal{I}$ and $S_i,i\in\mathcal{I}$.
\end{corollary}

\section{Summary of Contributions}
In this chapter, we provide a new theoretical framework for analyzing the convergence of multi-step MAML algorithm for both the resampling and finite-sum cases. Our analysis covers most applications including reinforcement learning and supervised learning of interest. 
Our analysis reveals that 
a properly chosen inner stepsize is crucial for guaranteeing MAML to converge with the complexity increasing only linearly with $N$ (the number of the inner-stage gradient updates). 
Moreover, for problems with small  Hessians, the inner stepsize can be set larger while maintaining the convergence. We expect that our analysis framework can be applied to understand the convergence of MAML in other scenarios such as various RL problems and Hessian-free MAML algorithms.

\chapter{Meta-Learning with Adaptation on Partial Parameters}\label{chp:anil}
In this chapter, we first present the problem formulation and the algrithom description for ANIL, and then provide the convergence rate and complexity analysis for ANIL under different loss geometries.  All technical proofs for the results in this chapter are provided in \Cref{appendix:anil}. For ease of presentation, we introduce the following notations for this chapter. For a function {\small $L(w,\phi)$} and a realization {\small $( w^\prime, \phi^\prime)$}, we define {\small $\nabla_{w} L (w^\prime,\phi^\prime) =\frac{\partial L (w,\phi)}{\partial w}\big |_{(w^\prime, \phi^\prime)} $, $\nabla^2_{w} L( w^\prime, \phi^\prime)=\frac{\partial^2 L(w,\phi)}{\partial w^2}\big |_{(w^\prime, \phi^\prime)}$,$\nabla_\phi\nabla_w L(w^\prime, \phi^\prime) = \frac{\partial^2 L(w,\phi)}{\partial \phi \partial w}\big |_{( w^\prime, \phi^\prime)}$}. The same notations hold for $\phi$.

\section{Problem Formulation}\label{problem}
Let $\mathcal{T}=(\mathcal{T}_i, i\in \mathcal{I})$ be a set of tasks available for meta-learning, where tasks are sampled for use by a distribution of $p_\mathcal{T}$. Each task $\mathcal{T}_i$ contains a training sample set $\mathcal{S}_i$ and a test set $\mathcal{D}_i$. Suppose that meta-learning divides all model parameters into mutually-exclusive sets $(w,\phi)$ as described below.
\begin{list}{$\bullet$}{\topsep=0.8ex \leftmargin=0.5in \rightmargin=0.1in \itemsep =0.05in}
\item $w$ includes task-specific parameters, and meta-learning trains a good initialization of $w$. 

\item $\phi$ includes common parameters shared by all tasks, and meta-learning trains $\phi$ for direct reuse. 
\end{list}
For example, in training neural networks, $w$ often represents the parameters of some partial layers,  and $\phi$ represents the parameters of the remaining inner layers.  The goal of meta-learning here is to jointly learn $w$ as a good initialization parameter and $\phi$ as a reuse parameter, such that $(w_N,\phi)$ performs well on a sampled individual task $\mathcal{T}$, where $w_N$ is the $N$-step gradient descent update of $w$. 
To this end, ANIL solves the following optimization problem with the objective function given by
\begin{align}\label{eq:obj} 
\hspace{-0.3cm}\text{(Meta objective function):}\quad & 
\min_{ w,\phi} L^{meta} ( w, \phi) := \mathbb{E}_{i\sim p_{\mathcal{T}}} L_{\mathcal{D}_i}( w^i_N( w,\phi), \phi), 
\end{align}
where the loss function $L_{\mathcal{D}_i}(w^i_N, \phi):=\sum_{\xi \in \mathcal{D}_i } \ell (w^i_N,\phi; \xi)$ takes the finite-sum form over the test dataset $\mathcal{D}_i$, and the parameter $w^i_N$  for task $i$ is obtained via an inner-loop $N$-step gradient descent update of $w^i_0=w$ (aiming to minimize the task $i$'s loss function $L_{\mathcal{S}_i}(w,\phi)$ over $w$) as given by
\begin{align}\label{inner:gd}
\text{(Inner-loop updates):} \quad w_{m+1}^i =  w_{m}^i - \alpha \nabla_{w} L_{\mathcal{S}_i} (w_{m}^i,\phi),\, m=0,...,N-1. 
\end{align}
Here, $w^i_N(w,\phi)$ explicitly indicates the dependence of $w^i_N$ on $\phi$ and the initialization $w$ via the iterative updates in~\cref{inner:gd}. To draw connection, the problem here reduces to the  MAML~\cite{finn2017model} framework if $w$ includes all training parameters and $\phi$ is empty, i.e., no parameters are reused directly.

\section{ANIL Algorithm}\label{Algorithm}
ANIL~\cite{raghu2019rapid} (as described in \Cref{alg:anil}) solves the problem in \cref{eq:obj} via two nested optimization loops, i.e., inner loop for task-specific adaptation and outer loop for updating meta-initialization and reuse parameters. At the $k$-th outer loop, ANIL samples a batch $\mathcal{B}_k$ of identical and independently distributed (i.i.d.) tasks based on $p_{\mathcal{T}}$. Then, each task in $\mathcal{B}_k$ runs an inner loop of $N$ steps of gradient descent with a stepsize $\alpha$ as in lines $5$-$7$ in~\Cref{alg:anil}, where $w_{k,0}^i =  w_k$ for all tasks $\mathcal{T}_i\in\mathcal{B}_{k}$.

After obtaining the inner-loop output $w^i_{k,N}$ for all tasks, ANIL computes two partial gradients $\frac{\partial L_{\mathcal{D}_i}(w^i_{k,N},\,\phi_k)}{\partial { w_k}}$ and $\frac{\partial L_{\mathcal{D}_i}(w^i_{k,N},\,\phi_k)}{\partial {\phi_k}}$ 
respectively by back-propagation, and updates $w_k$ and $\phi_k$ by stochastic gradient descent as in line $10$ in~\Cref{alg:anil}.
Note that $\phi_k$ and $w_k$ are treated to be mutually-independent during the differentiation process. Due to the nested dependence of $w_{k,N}^i$ on $\phi_k$ and $w_k$, the two partial gradients involve complicated second-order derivatives. Their explicit forms are provided in the following proposition. 
\begin{proposition}\label{le:gd_form}
The partial meta gradients take the following explicit form: 
\begin{align*}
{1)} \quad \frac{\partial L_{\mathcal{D}_i}( w^i_{k,N}, \phi_k)}{\partial w_k} =& \prod_{m=0}^{N-1}(I - \alpha \nabla_w^2L_{\mathcal{S}_i}(w_{k,m}^i,\phi_k)) \nabla_{w} L_{\mathcal{D}_i} (w_{k,N}^i,\phi_k). \nonumber
\\ {2)}\quad \frac{\partial L_{\mathcal{D}_i}( w^i_{k,N}, \phi_k)}{\partial \phi_k} =&
\nabla_\phi L_{\mathcal{D}_i}(w_{k,N}^i,\phi_k) \nonumber
\\ -\alpha \sum_{m=0}^{N-1}\nabla_\phi\nabla_w &L_{\mathcal{S}_i}(w_{k,m}^i,\phi_k) \prod_{j=m+1}^{N-1}(I-\alpha\nabla_w^2L_{\mathcal{S}_i}(w_{k,j}^i,\phi_k))\nabla_w L_{\mathcal{D}_i}(w_{k,N}^i,\phi_k).
\end{align*}
\end{proposition}

 \begin{algorithm}[t]
	\caption{ANIL Algorithm} 
	\label{alg:anil}
	\begin{algorithmic}[1]
		\STATE {\bfseries Input:} Distribution over tasks $p_\mathcal{T}$, inner stepsize $\alpha$, outer stepsize $\beta_w,\beta_\phi$, initialization $ w_0, \phi_0$ 
		\WHILE{not converged}
		\STATE{Sample a mini-batch of i.i.d. tasks $\mathcal{B}_k= \{\mathcal{T}_i\}_{i=1}^B$  based on  the distribution $p_\mathcal{T}$}
		\FOR{each task $\mathcal{T}_i$ in $\mathcal{B}_k$}
		\FOR{$m=0,1,...,N-1$}
		\STATE{Update $w_{k,m+1}^i = w_{k,m}^i - \alpha \nabla_w L_{\mathcal{S}_i}( w_{k,m}^i,\phi_k)$}
		\ENDFOR
		\STATE{Compute gradients $\frac{\partial L_{\mathcal{D}_i}(w^i_{k,N},\phi_k)}{\partial { w_k}},\frac{\partial L_{\mathcal{D}_i}(w^i_{k,N},\phi_k)}{\partial {\phi_k}}$ by back-propagation}
		\ENDFOR
                 \STATE{Update parameters $ w_k$ and $\phi_k$ by mini-batch SGD:
                 \begin{align*}
                 w_{k+1}=  w_{k} - \frac{\beta_w}{B}\sum_{i\in\mathcal{B}_k} \frac{\partial L_{\mathcal{D}_i}( w^i_{k,N},\phi_k)}{\partial {w_k}}, \quad \phi_{k+1}= \phi_{k} - \frac{\beta_\phi}{B}\sum_{i\in\mathcal{B}_k} \frac{\partial L_{\mathcal{D}_i}( w^i_{k,N},\phi_k)}{\partial {\phi_k}}
                 \end{align*}
                 \vspace{-0.2cm}
                 }
                                  \STATE{Update $k \leftarrow k+1$}
		\ENDWHILE
	\end{algorithmic}
	\end{algorithm}
\section{Technical Assumptions and Definitions}\label{se:pre} 
We let $z=(w,\phi)\in\mathbb{R}^{n}$ denote all parameters. For simplicity, suppose $\mathcal{S}_i$ and $\mathcal{D}_i$ for all $i\in\mathcal{I}$ have sizes of $S$ and $D$, respectively. In this paper, we consider the following types of loss functions.
\begin{list}{$\bullet$}{\topsep=0.7ex \leftmargin=0.32in \rightmargin=0.2in \itemsep =0.05in}

\item The outer-loop meta loss function in \cref{eq:obj} takes the finite-sum form as $L_{\mathcal{D}_i}(w^i_N, \phi):=\sum_{\xi \in \mathcal{D}_i } \ell (w^i_N,\phi; \xi)$. It is generally nonconvex in terms of both $w$ and $\phi$.

\item The inner-loop loss function $L_{\mathcal{S}_i}(w,\phi)$ with respect to $w$ has two cases: strongly-convexity and nonconvexity. The strongly-convex case occurs often when $w$ corresponds to parameters of the last {\em linear} layer of a neural network, so that the loss function of such a $w$ is naturally chosen to be a quadratic function or a logistic loss with a strongly convex regularizer~\cite{bertinetto2018meta,lee2019meta}. The nonconvex case can occur if $w$ represents parameters of more than one layers (e.g., last two layers~\cite{raghu2019rapid}). As we show later, such geometries affect the convergence rate significantly.
\vspace{0.1cm}
\end{list}
Since the objective function $L^{meta}(w,\phi)$ in~\cref{eq:obj} is generally nonconvex, we use the gradient norm as the convergence criterion, which is standard in nonconvex optimization.
\begin{definition}
We say that $(\bar{w},\bar{\phi})$ is an $\epsilon$-accurate solution for the meta optimization problem in~\cref{eq:obj} if 
{\small $\,\mathbb{E}\Big\| \frac{\partial L^{meta}(\bar{w},\bar{\phi})}{\partial \bar w}\Big\|^2 <\epsilon$ and $\mathbb{E}\Big\| \frac{\partial L^{meta}(\bar{w},\bar{\phi})}{\partial \bar \phi}\Big\|^2 <\epsilon$}.
\end{definition}
We further take the following standard assumptions on the {\em individual} loss function for each task, which have been commonly adopted in conventional minimization problems~\cite{ghadimi2013stochastic,wang2018spiderboost,ji2019improved} and min-max optimization~\cite{lin2020near} as well as the MAML-type optimization~\cite{finn2019online,ji2020multi}. 
\begin{assum}\label{assm:smooth}
The loss function $L_{\mathcal{S}_i}(z)$ and $L_{\mathcal{D}_i}(z)$ for each task $\mathcal{T}_i$ satisfy:
\begin{list}{$\bullet$}{\topsep=0.7ex \leftmargin=0.3in \rightmargin=0.5in \itemsep =0.05in}
\vspace{0.2cm}
\item $L_{\mathcal{S}_i}(z)$ and $L_{\mathcal{D}_i}(z)$ are $L$-smooth, i.e., for any $z,z^\prime \in \mathbb{R}^{n}$, $$\|\nabla L_{\mathcal{S}_i}(z) - \nabla L_{\mathcal{S}_i}(z^\prime) \| \leq L\|z-z^\prime\|, \|\nabla L_{\mathcal{D}_i}(z) - \nabla L_{\mathcal{D}_i}(z^\prime) \| \leq L\|z-z^\prime\|.$$
\item $L_{\mathcal{D}_i}(z)$ is $M$-Lipschitz, i.e., for any $z,z^\prime\in \mathbb{R}^{n}$,  $\| L_{\mathcal{D}_i}(z) -  L_{\mathcal{D}_i}(z^\prime) \| \leq M\|z-z^\prime\|$.
\end{list}
\end{assum}

Note that we {\em do not} impose the function Lipschitz assumption (i.e., item 2 in~\Cref{assm:smooth}) on the inner-loop loss function $L_{S_i}(z)$.
As shown in~\Cref{le:gd_form},  the partial meta gradients involve two types of high-order derivatives {\small$\nabla_w^2 L_{\mathcal{S}_i}(\cdot,\cdot)$} and {\small$\nabla_\phi\nabla_w L_{\mathcal{S}_i}(\cdot,\cdot)$}. 
The following assumption imposes a Lipschitz condition for these two high-order derivatives, which has been widely adopted in optimization problems that involve two sets of parameters, e.g, bi-level programming~\cite{ghadimi2018approximation}.
\begin{assum}\label{assm:second} $\nabla_w^2 L_{\mathcal{S}_i}(z)$ and $\nabla_\phi\nabla_w L_{\mathcal{S}_i}(z)$ are $\rho$-Lipschitz and $\tau$-Lipschitz, i.e., 
\begin{list}{$\bullet$}{\topsep=0.7ex \leftmargin=0.3in \rightmargin=0.5in \itemsep =0.05in}
\vspace{0.2cm}
\item For any $z,z^\prime\in\mathbb{R}^{n}$, $\|\nabla^2_w L_{\mathcal{S}_i}(z) - \nabla^2_w L_{\mathcal{S}_i}(z^\prime)\| \leq \rho \|z - z^\prime\|$.
\item For any $z,z^\prime\in\mathbb{R}^{n}$, $\|\nabla_\phi \nabla_w L_{\mathcal{S}_i}(z) - \nabla_\phi\nabla_w L_{\mathcal{S}_i}(z^\prime)\| \leq \tau\|z-z^\prime\|$.
\end{list}
\end{assum}


\section{Convergence of ANIL with Strongly-Convex Inner Loop}\label{se:strong-convex}
We first analyze the convergence rate of ANIL for the case where the inner-loop loss function $L_{\mathcal{S}_i}(\cdot,\phi)$ satisfies the following strongly-convex condition. 
\begin{definition}
$L_{\mathcal{S}_i}(w,\phi)$ is $\mu$-strongly convex with respect to $w$ if for any $w,w^\prime,\phi$,
 \begin{align*}
 L_{\mathcal{S}_i}(w^\prime,\phi) \geq L_{\mathcal{S}_i}(w,\phi) + \big\langle w^\prime -w , \nabla_w L_{\mathcal{S}_i}(w,\phi) \big\rangle + \frac{\mu}{2} \|w-w^\prime\|^2.
 \end{align*} 
 \end{definition}

Based on  \Cref{le:gd_form}, we characterize the smoothness property of $L^{meta}(w,\phi)$ in~\cref{eq:obj} as below. 
\begin{proposition}\label{le:strong-convex}
Suppose Assumptions~\ref{assm:smooth} and \ref{assm:second} hold
and choose the inner stepsize $\alpha=\frac{\mu}{L^2}$. Then, for any two points $(w_1,\phi_1), (w_2,\phi_2)\in\mathbb{R}^n$, we have
\begin{align*}
1)  \quad \Big\| &\frac{\partial L^{meta}( w,\phi)}{\partial w} \Big |_{(w_1,\phi_1)} - \frac{\partial L^{meta}( w,\phi)}{\partial w} \Big |_{(w_2,\phi_2)} \Big\| 
\\&\leq \text{\normalfont poly}(L,M,\rho)\frac{L}{\mu}(1-\alpha \mu)^{N}\|w_1-w_2\|
\\&\quad\quad+\text{\normalfont poly}(L,M,\rho)\left(\frac{L}{\mu}+1\right)N(1-\alpha \mu)^{N}\|\phi_1-\phi_2\|,
\\ 2)\quad  \Big\| &\frac{\partial L^{meta}( w,\phi)}{\partial \phi} \Big |_{(w_1,\phi_1)} -  \frac{\partial L^{meta}( w,\phi)}{\partial \phi} \Big |_{(w_2,\phi_2)} \Big\|
\\&\leq \text{\normalfont poly}(L,M,\tau,\rho) \frac{L}{\mu}(1-\alpha\mu)^{\frac{N}{2}}\|w_1-w_2\|+\text{\normalfont poly}(L,M,\rho) \frac{L^3}{\mu^3}\|\phi_1-\phi_2\|,
\end{align*}
where $\tau,\rho,L$ and $M$ are given in Assumptions~\ref{assm:smooth} and \ref{assm:second}, and $\text{\normalfont poly}(\cdot)$ denotes the polynomial function of the parameters with the explicit forms 
given in~\Cref{append:str}.
\end{proposition}


 \Cref{le:strong-convex} indicates that increasing the number $N$ of inner-loop gradient descent steps yields much {\em smaller} smoothness parameters for the meta objective function $L^{meta}(w,\phi)$. From an optimization perspective, this allows a larger stepsize chosen for the outer-loop meta optimization, and hence yields a faster convergence rate, as characterized in the following convergence theorem.
\begin{theorem}\label{th:strong-convex} 
Suppose Assumptions~\ref{assm:smooth} and \ref{assm:second} hold, and apply \Cref{alg:anil} to solve the meta  problem~\cref{eq:obj} with stepsizes $\beta_w={\small\text{\normalfont poly}(\rho,\tau,L,M) \mu^2(1-\frac{\mu^2}{L^2})^{-\frac{N}{2}}}$ and $\beta_\phi={\small\text{\normalfont poly}(\rho,\tau,L,M)\mu^{3}}$. 
Then, ANIL  finds a point $(w,\phi)\in\big\{(w_k,\phi_k),k=0,...,K-1\big\}$ such that 
\begin{align*}
\text{\normalfont (Rate w.r.t. $w$)}\;\mathbb{E}\left\| \frac{\partial L^{meta}(w,\phi)}{\partial w}\right\|^2  \leq& \mathcal{O}\Bigg( \frac{  \frac{1}{\mu^2}\left(1-\frac{\mu^2}{L^2}\right)^{\frac{N}{2}}}{K}      +\frac{\frac{1}{\mu} \left(1-\frac{\mu^2}{L^2}\right)^{\frac{N}{2}}}{B}     \Bigg), \nonumber
\\\text{\normalfont (Rate w.r.t. $\phi$)}\; \mathbb{E}\left\| \frac{\partial L^{meta}(w,\phi)}{\partial \phi}\right\|^2  \leq & \mathcal{O}\bigg(\frac{ \frac{1}{\mu^2}\left(1-\frac{\mu^2}{L^2}\right)^{\frac{N}{2}}+\frac{1}{\mu^3}}{K} +\frac{\frac{1}{\mu}\left(1-\frac{\mu^2}{L^2}\right)^{\frac{3N}{2}}+\frac{1}{\mu^2}}{B}\bigg).
\end{align*}
To achieve an $\epsilon$-accurate point, ANIL requires at most {\small $\mathcal{O}\big(\frac{N}{\mu^{4}}\big(1-\frac{\mu^2}{L^2}\big)^{N/2}+\frac{N}{\mu^{5}}\big)\epsilon^{-2}$} 
 gradient evaluations in $w$,  {\small $\mathcal{O}\big(\frac{1}{\mu^{4}}\big(1-\frac{\mu^2}{L^2}\big)^{N/2}+\frac{1}{\mu^{5}}\big)\epsilon^{-2}$} gradient evaluations in $\phi$,  and {\small $\mathcal{O}\big(\frac{N}{\mu^{4}}\big(1-\frac{\mu^2}{L^2}\big)^{N/2}+\frac{N}{\mu^{5}}\big)\epsilon^{-2}$} second-order derivative evaluations of {\small$\nabla_w^2 L_{\mathcal{S}_i}(\cdot,\cdot)$} and {\small$\nabla_\phi\nabla_w L_{\mathcal{S}_i}(\cdot,\cdot)$}.
\end{theorem}
\Cref{th:strong-convex} shows that ANIL converges sublinearly with the number $K$ of outer-loop meta iterations, and the convergence error decays sublinearly with the number $B$ of sampled tasks, which are consistent with the nonconvex nature of the meta objective function. The convergence rate is further significantly affected by the number $N$ of the inner-loop steps. Specifically, with respect to $w$, ANIL converges exponentially fast as $N$ increases due to the strong convexity of the inner-loop loss. With respect to $\phi$, the convergence rate depends on two components: an exponential decay term with $N$ and an $N$-independent term. As a result, the overall convergence of meta optimization becomes faster as $N$ increases, and then saturates for large enough $N$ as the second component starts to dominate. 
This is demonstrated by our experiments in \Cref{exp:anilsahjb}.

\Cref{th:strong-convex} further indicates that ANIL attains an $\epsilon$-accurate stationary point with the gradient and second-order evaluations at the order of $\mathcal{O}(\epsilon^{-2})$ due to nonconvexity of the meta objective function. The computational cost is further significantly affected by inner-loop steps. Specifically, the gradient and second-order derivative evaluations contain two terms: an exponential decay term with $N$ and a linear growth term with $N$. As a result, the computational cost of ANIL initially decreases because the exponential reduction dominates the linear growth. But when $N$ is large enough, the exponential decay saturates and the linear growth dominates, and hence the overall computational cost of ANIL gets higher as $N$ further increases. 
 This suggests to take a moderate but not too large $N$ in practice to achieve an optimized performance, which we also demonstrate in our experiments in \Cref{exp:anilsahjb}.

\section{Convergence of ANIL with Nonconvex Inner Loop}\label{sec:nonconvex}
In this section, we study the case, in which the inner-loop loss function $L_{\mathcal{S}_i}(\cdot,\phi)$ is  nonconvex. 
The following proposition characterizes the smoothness of $L^{meta}(w,\phi)$ in~\cref{eq:obj}.

\begin{proposition}\label{le:smooth_nonconvex}
Suppose Assumptions~\ref{assm:smooth} and \ref{assm:second} hold, and choose the inner-loop stepsize {\small$\alpha <\mathcal{O}(\frac{1}{N})$}.  
Then, for any two points $(w_1,\phi_1)$, $(w_2,\phi_2)\in\mathbb{R}^n$, we have 
\begin{align*}
1) \Big\| &\frac{\partial L^{meta}( w,\phi)}{\partial w} \Big |_{(w_1,\phi_1)} -  \frac{\partial L^{meta}( w,\phi)}{\partial w} \Big |_{(w_2,\phi_2)} \Big\|
\\&\quad\quad\quad\quad \leq\text{\normalfont poly}(M,\rho,\alpha,L) N(\|w_1-w_2\| + \|\phi_1-\phi_2\|),
\\ 2) \Big\| &\frac{\partial L^{meta}( w,\phi)}{\partial \phi} \Big |_{(w_1,\phi_1)} -  \frac{\partial L^{meta}( w,\phi)}{\partial \phi} \Big |_{(w_2,\phi_2)} \Big\|
\\&\quad\quad\quad\quad\leq \text{\normalfont poly}(M,\rho,\tau,\alpha,L) N(\|w_1-w_2\| + \|\phi_1-\phi_2\|),
\end{align*}
where $\tau,\rho,L$ and $M$ are given by Assumptions~\ref{assm:smooth} and \ref{assm:second}, and $\text{\normalfont poly}(\cdot)$ denotes the polynomial function of the parameters with the explicit forms of the smoothness parameters given in~\Cref{appen:smooth_nonconvex}.

\end{proposition}
\Cref{le:smooth_nonconvex} indicates that the meta objective function $L^{meta}(w,\phi)$ is smooth with respect to both $w$ and $\phi$ with their smoothness parameters increasing linearly with $N$. Hence, $N$ should be chosen to be small so that the outer-loop meta optimization can take reasonably large stepsize to run fast. Such a property is in sharp contrast to the strongly-convex case in which the corresponding smoothness parameters decrease with $N$.


 The following theorem provides the convergence rate of ANIL under the nonconvex inner-loop loss.
\begin{theorem}\label{th:nonconvexcaonidaxeas} 
Under the setting of~\Cref{le:smooth_nonconvex}, and apply \Cref{alg:anil} to solve the meta optimization problem in~\cref{eq:obj} with the stepsizes {\small $\beta_w = \beta_\phi=\text{\normalfont poly}(\rho,\tau,M,\alpha,L) N^{-1}$}. 
Then, ANIL  finds a point $(w,\phi)\in\{(w_k,\phi_k),k=0,...,K-1\}$ such that 
\begin{align*}
\mathbb{E}\left\| \frac{\partial L^{meta}(w,\phi)}{\partial w}\right\|^2  \leq& \,\mathcal{O}\bigg(  \frac{N}{K} + \frac{N}{B}  \bigg), \qquad
\mathbb{E}\left\| \frac{\partial L^{meta}(w,\phi)}{\partial \phi}\right\|^2  \leq \,\mathcal{O}\bigg(  \frac{N}{K} + \frac{N}{B}  \bigg).
\end{align*}
To achieve an $\epsilon$-accurate point, ANIL requires at most {\small $\mathcal{O}(N^2\epsilon^{-2})$} gradient evaluations in $w$, {\small $\mathcal{O}(N\epsilon^{-2})$} gradient evaluations in $\phi$,  and {\small $\mathcal{O}(N^2\epsilon^{-2})$} second-order derivative evaluations.
\end{theorem}
\Cref{th:nonconvexcaonidaxeas} shows that ANIL converges sublinearly with $K$, the convergence error decays sublinearly with $B$, and the computational complexity scales at the order of $\mathcal{O}(\epsilon^{-2})$. But the nonconvexity of the inner loop affects the convergence very differently. Specifically, increasing the number $N$ of the inner-loop gradient descent steps yields slower convergence and higher computational complexity. This suggests to choose a relatively small $N$ for an efficient optimization process, which is demonstrated in our experiments in \Cref{exp:anilsahjb}
\section{Comparison of Different Geometries and Algorithms}

In this section, we first compare the performance for ANIL under strongly convex and nonconvex inner-loop loss functions, and then compare the performance between ANIL and MAML.
\begin{table*}[th] 
\vspace{-0.2cm}
	\centering 
	\small
	\caption{Comparison of different geometries on convergence rate and complexity of ANIL. GC: gradient complexity. SOC: second-order complexity. }
	\vspace{0.2cm}
	\begin{tabular}{lccc} \toprule
		{Geometries} &Convergence rate  & GC & SOC  \\   \midrule
		Strongly convex & $\mathcal{O}\Big(\frac{ (1-\xi)^{\frac{N}{2}}+c_k}{K} +\frac{(1-\xi)^{\frac{3N}{2}}+c_b}{B}\Big){\color{red}^\sharp}$ &    $\mathcal{O}\Big(\frac{N((1-\xi)^{\frac{N}{2}}+c_\epsilon)}{\epsilon^2}\Big){\color{red}^\S}$ &$\mathcal{O}\Big(\frac{N((1-\xi)^{\frac{N}{2}}+c_\epsilon)}{\epsilon^2}\Big)$\\   
		Nonconvex &$\mathcal{O}\Big(  \frac{N}{K} + \frac{N}{B}  \Big)$ &    $\mathcal{O}\big(\frac{N^2}{\epsilon^{2}}\big)$ & $\mathcal{O}\big(\frac{N^2}{\epsilon^{2}}\big)$  \\  
		\bottomrule	
		\multicolumn{4}{l}{%
  \begin{minipage}{11cm}%
  \vspace{0.1cm}
    \footnotesize Each order term in the table summarizes the dominant components of both $w$ and $\phi$. \\${\color{red}^\sharp}:$ $\xi=\frac{\mu^2}{L^2}<1$, $c_k,c_b$ are constants.  ${\color{red}^\S}:$ $c_\epsilon$ is constant.
  \end{minipage}%
}\\
	\end{tabular} 
	\label{tas11anil}
	\vspace{-0.2cm}
\end{table*}

{\noindent \bf Comparison for ANIL between strongly convex and nonconvex inner-loop  geometries:} 
Our results in \Cref{se:strong-convex,sec:nonconvex} have showed that the inner-loop geometry can significantly affect the convergence rate and the computational complexity of ANIL. The detailed comparison is provided in~\Cref{tas11anil}. It can be seen that increasing $N$ yields a faster convergence rate  for the strongly-convex inner loop, but a slower convergence rate for the nonconvex inner loop. \Cref{tas11anil} also indicates that increasing $N$ first reduces and then increases the computational complexity for the strongly-convex inner loop, but constantly increases the complexity for the nonconvex inner loop. 


We next provide an intuitive explanation for such different behaviors under these two geometries. For the nonconvex inner loop, $N$ gradient descent iterations starting from two different initializations  likely reach two points that are far away from each other due to the nonconvex landscape so that  the meta objective function can have a large smoothness parameter. Consequently, the stepsize should be small to avoid divergence, which yields slow convergence.
However, for the strongly-convex inner loop, also consider two $N$-step inner-loop gradient descent paths. Due to the strong convexity, they both approach to the same unique optimal point, and hence their corresponding values of the meta objective function are guaranteed to be close to each other as $N$ increases. Thus, increasing $N$ reduces the smoothness parameter, and allows a faster convergence rate. 


\vspace{0.2cm}
{\noindent \bf Comparison between ANIL and MAML:} \cite{raghu2019rapid} empirically showed that ANIL significantly speeds up MAML due to the fact that only a very small subset of parameters go through the inner-loop update. The complexity results in~\Cref{th:strong-convex} and \Cref{th:nonconvexcaonidaxeas} provide theoretical characterization of such an acceleration. To formally compare the performance between ANIL and MAML, let $n_w$ and $n_\phi$ be the dimensions of $w$ and $\phi$, respectively. 
The detailed comparison is provided in~\Cref{table:maml_anil}. 

 

\begin{table*}[th] 
	\centering 
	\caption{Comparison of the computational complexities of ANIL and MAML.}
	\vspace{0.1cm}
	\begin{tabular}{lcc} \toprule
		{Algorithms} & \# of gradient entries ${\color{red}^\sharp }$  & \# of second-order entries ${\color{red}^\S }$  		 \\   \midrule
		MAML~\cite[Theorem 2]{ji2020multi} & $\mathcal{O}\Big(\frac{(Nn_w+Nn_\phi)(1+\kappa L)^N}{\epsilon^{2}}\Big){\color{red}^\aleph }$&    $\mathcal{O}\Big(\frac{(n_w+n_\phi)^2N(1+\kappa L)^N}{\epsilon^{2}}\Big)$	 \\  
		ANIL (Strongly convex)& $\mathcal{O}\Big(\frac{(N n_w+n_\phi)((1-\xi )^{\frac{N}{2}}+c_\epsilon)}{\epsilon^{2}}\Big){\color{red}^\flat }$ &    $\mathcal{O}\Big(\frac{(n^2_w+n_w n_\phi) N ((1-\xi )^{\frac{N}{2}}+ c_\epsilon )}{\epsilon^{2}}\Big)$ \\   
		ANIL (Nonconvex) &$\mathcal{O}\Big(\frac{(N n_w+n_\phi)N}{\epsilon^{2}}\Big)$ &   $\mathcal{O}\Big(\frac{(n^2_w+n_wn_\phi)N^2}{\epsilon^{2}}\Big)$ \\  
		\bottomrule	
		\multicolumn{3}{l}{%
  \begin{minipage}{13cm}%
  \vspace{0.1cm}
    \footnotesize ${\color{red}^\sharp }$: number of evaluations with respect to each dimension of gradient. ${\color{red}^\S}$: number of evaluations with respect to each entry of second-order derivatives.
    \\ ${\color{red}^\aleph }$: $\kappa$ is the inner-loop stepsize used in MAML.
    ${\color{red}^\flat}:$ $\xi=\frac{\mu^2}{L^2}<1$ and  $c_\epsilon$ is a constant.
      \end{minipage}%
}\\
	\end{tabular} 
	\label{table:maml_anil}
	\vspace{-0.4cm}
\end{table*} 

For ANIL with the strongly-convex inner loop, \Cref{table:maml_anil} shows that ANIL requires fewer  gradient and second-order entry evaluations than MAML by a factor of {\small$\mathcal{O}\big(\frac{N n_w+N n_\phi}{N n_w+n_\phi}\big(1+\kappa L\big)^N\big)$} and {\small$\mathcal{O}\big(\frac{n_w+n_\phi}{n_w}\big(1+\kappa L\big)^N\big)$}, respectively. Such improvements are significant because  $n_\phi$ is often much larger than $n_w$.

For nonconvex inner loop, we set $\kappa\leq 1/N$ for MAML~\cite[Corollary 2]{ji2020multi} to be consistent with our analysis for ANIL in \Cref{th:nonconvexcaonidaxeas}. Then,~\Cref{table:maml_anil} indicates that 
ANIL requires fewer gradient and second-order entry computations than MAML by a factor of {\small$\mathcal{O}\big(\frac{Nn_w+Nn_\phi}{Nn_w+n_\phi}\big)$} and {\small$\mathcal{O}\big(\frac{n_w+n_\phi}{n_w}\big)$}.

%
%

\section{Experiments}\label{exp:anilsahjb}
In this section, we validate our theory on the ANIL algorithm over two benchmarks for few-shot multiclass classification, i.e., FC100~\cite{oreshkin2018tadam} and miniImageNet~\cite{vinyals2016matching}. The experimental implementation and the model architectures are adapted from the existing repository~\cite{learn2learn2019} for ANIL. 
We consider a 5-way 5-shot task on both the FC100 and miniImageNet datasets. 
We relegate the introduction of datasets, model architectures and hyper-parameter settings to~\Cref{appen:exp}.
Our experiments aim to explore how the different geometry (i.e., strong convexity and nonconvexity) of the inner loop affects the convergence performance of ANIL. 

\subsection*{ANIL with Strongly-Convex Inner-Loop Loss}

We first validate the convergence results of ANIL under the {\em strongly-convex} inner-loop loss function $L_{\mathcal{S}_i}(\cdot,\phi)$, as we establish in~\Cref{se:strong-convex}. 
Here, we let $w$ be parameters of {\em the last layer} of CNN and $\phi$ be parameters of the remaining inner layers. As in~\cite{bertinetto2018meta,lee2019meta}, the inner-loop loss function adopts $L^2$ regularization on $w$ with a hyper-parameter $\lambda>0$, and hence is {\em strongly convex}. 

  \begin{figure*}[h]
	\centering    
	\subfigure[dataset: FC100 ]{\label{fig1:a}\includegraphics[width=60mm]{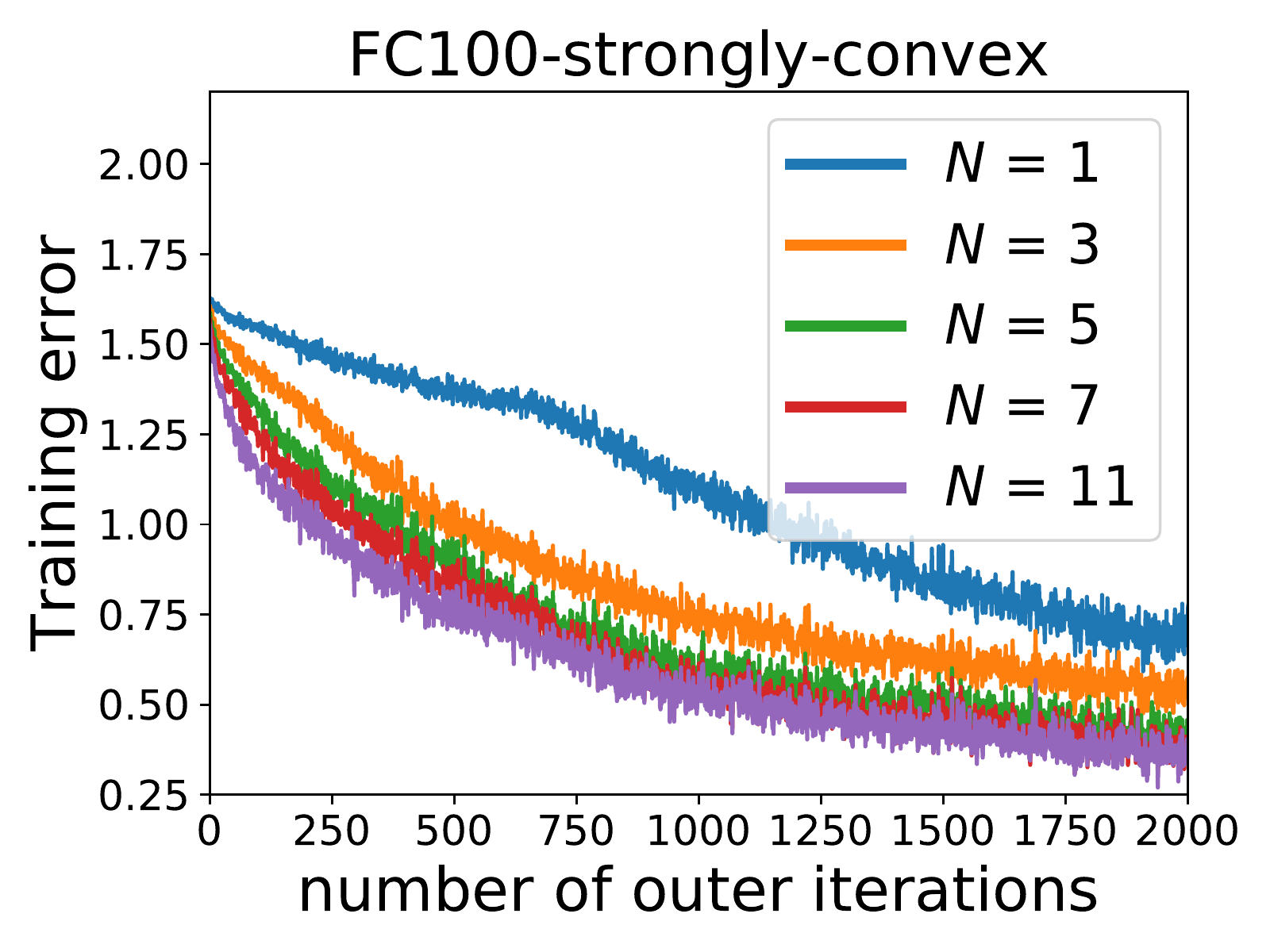}\includegraphics[width=60mm]{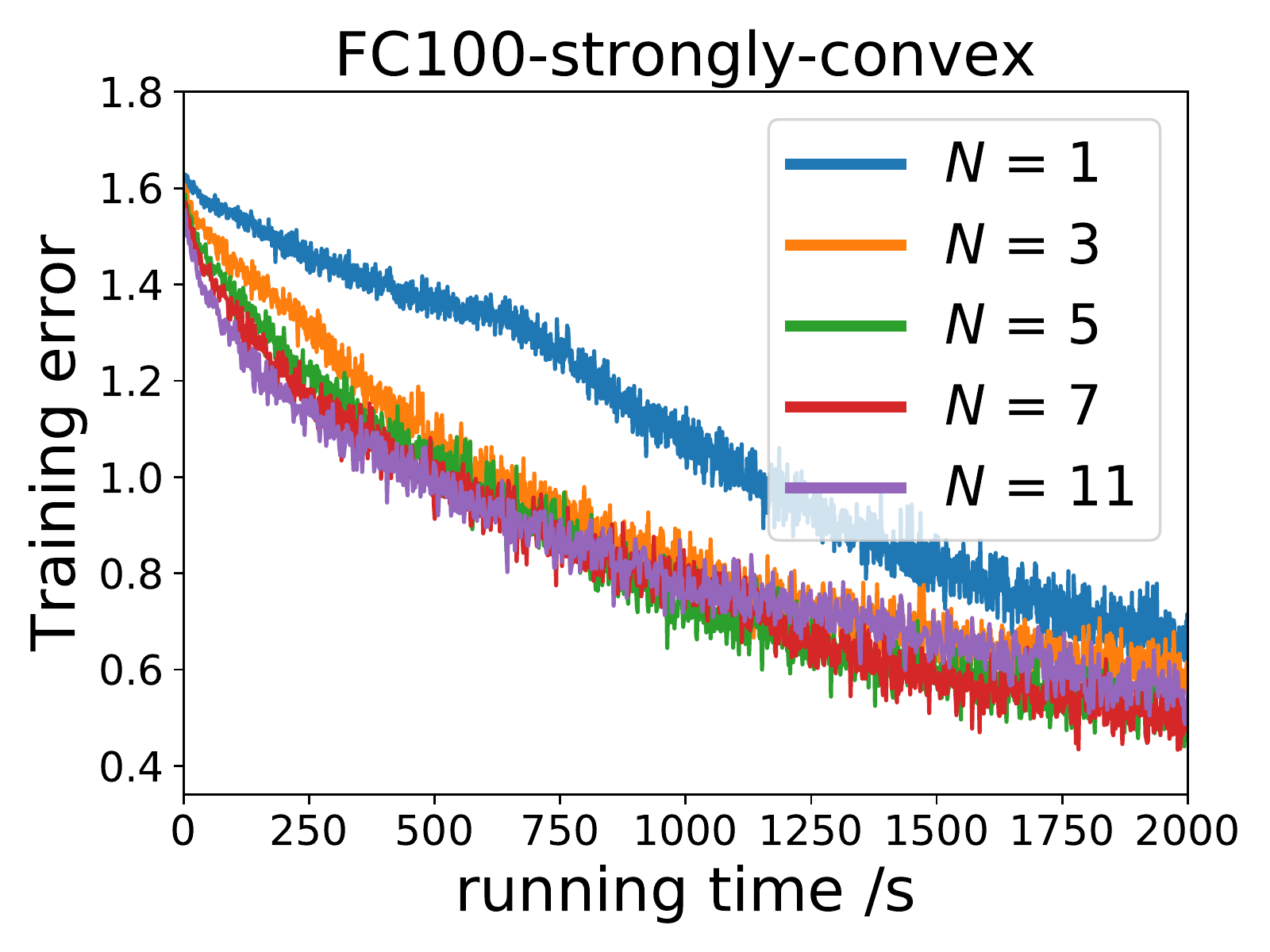}} 
	\subfigure[dataset: miniImageNet ]{\label{fig1:b}\includegraphics[width=60mm]{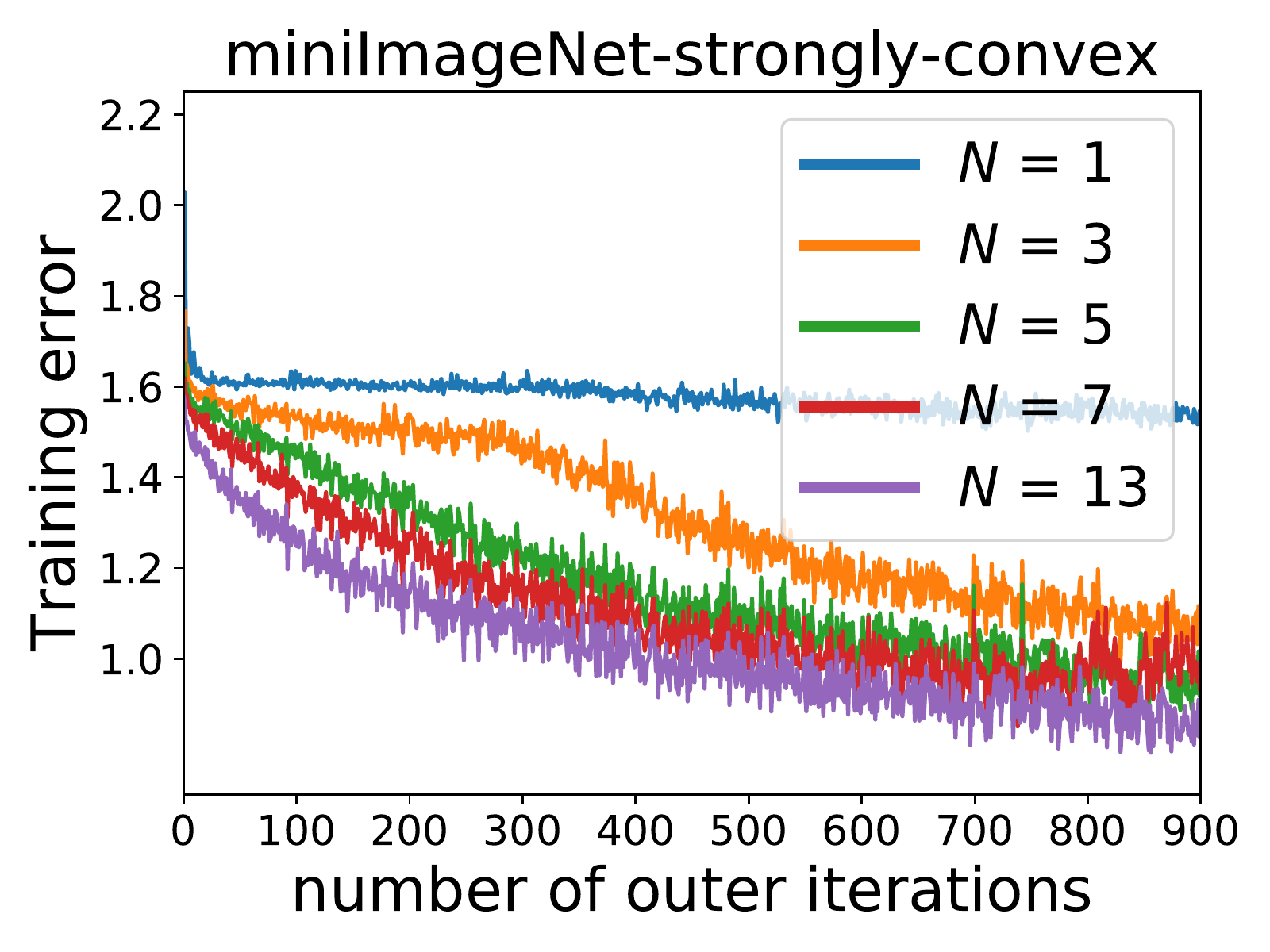}\includegraphics[width=60mm]{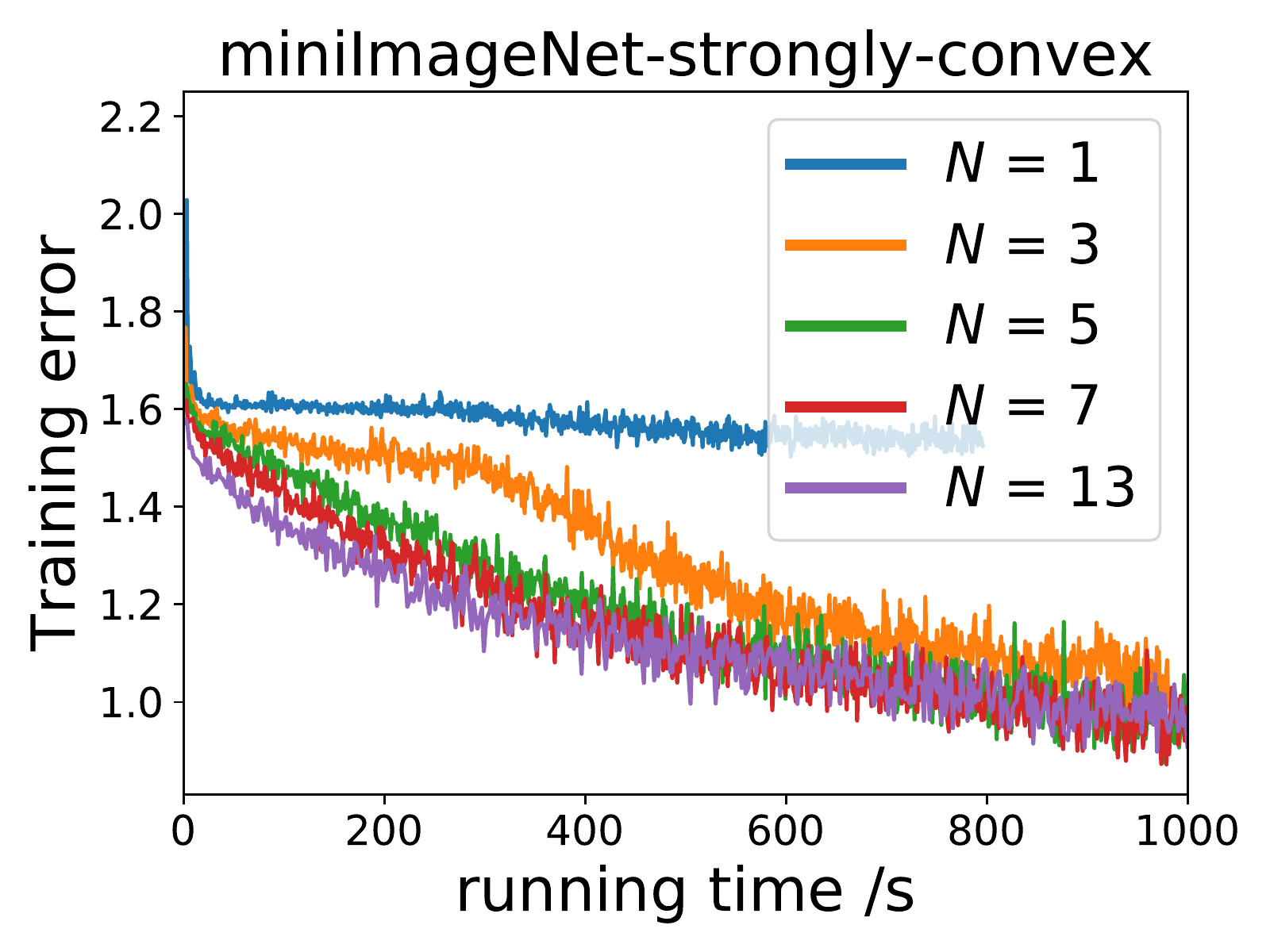}}  
	\caption{Convergence of ANIL with strongly-convex inner-loop loss function.  For each dataset, left plot: training loss v.s. number of total meta iterations; right plot: training loss v.s. running time.}\label{fig:strfc100}
\end{figure*}

For the FC100 dataset, the left plot of~\Cref{fig1:a} shows that the convergence rate in terms of the number of meta outer-loop iterations becomes faster as the inner-loop steps $N$ increases, but nearly saturates at $N=7$ (i.e., there is not much improvement for $N\geq 7$).
This is consistent with \Cref{th:strong-convex}, in which the gradient convergence bound first  
 decays exponentially with $N$, and then the bound in $\phi$ dominates and saturates to a constant. Furthermore, the right plot of~\Cref{fig1:a} shows that the running-time convergence first becomes faster  as $N$ increases up to $N\leq 7$, and then starts to slow down as $N$ further increases. 
This is also captured by~\Cref{th:strong-convex} as follows. 
The computational cost of ANIL initially decreases because the exponential reduction dominates the linear growth in the gradient and second-order derivative evaluations. But when $N$ becomes large enough, the linear growth dominates, and hence the overall computational cost of ANIL gets higher as $N$ further increases. 
Similar nature of convergence behavior is also observed over the miniImageNet dataset as shown in~\Cref{fig1:b}.  Thus, our experiment suggests that for the strongly-convex inner-loop loss, choosing a relatively large $N$ (e.g., $N=7$) 
achieves a good balance between the convergence rate (as well as the convergence error) and the computational complexity.  


\subsection*{ANIL with Nonconvex Inner-Loop Loss}

We next validate the convergence results of ANIL under the {\em nonconvex} inner-loop loss function $L_{\mathcal{S}_i}(\cdot,\phi)$, as we establish in~\Cref{sec:nonconvex}. 
Here, we let $w$ be the parameters of {\em the last two layers with ReLU activation} of CNN (and hence the inner-loop loss is nonconvex with respect to $w$) and $\phi$ be the remaining parameters of the inner layers. 

  \begin{figure*}[h]
	\centering    
	\subfigure[dataset: FC100 ]{\label{fig2:acansicasa}\includegraphics[width=60mm]{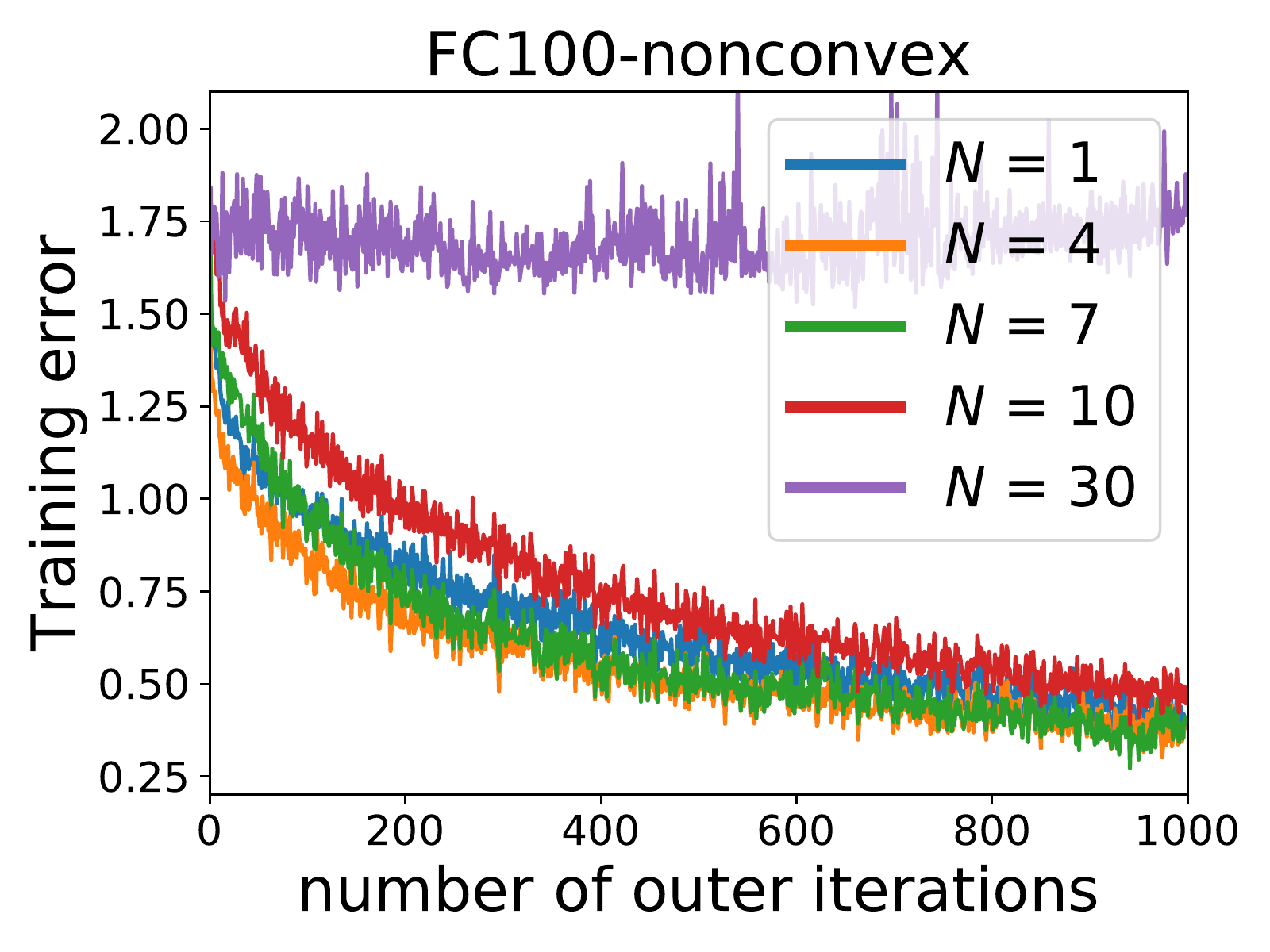}\includegraphics[width=60mm]{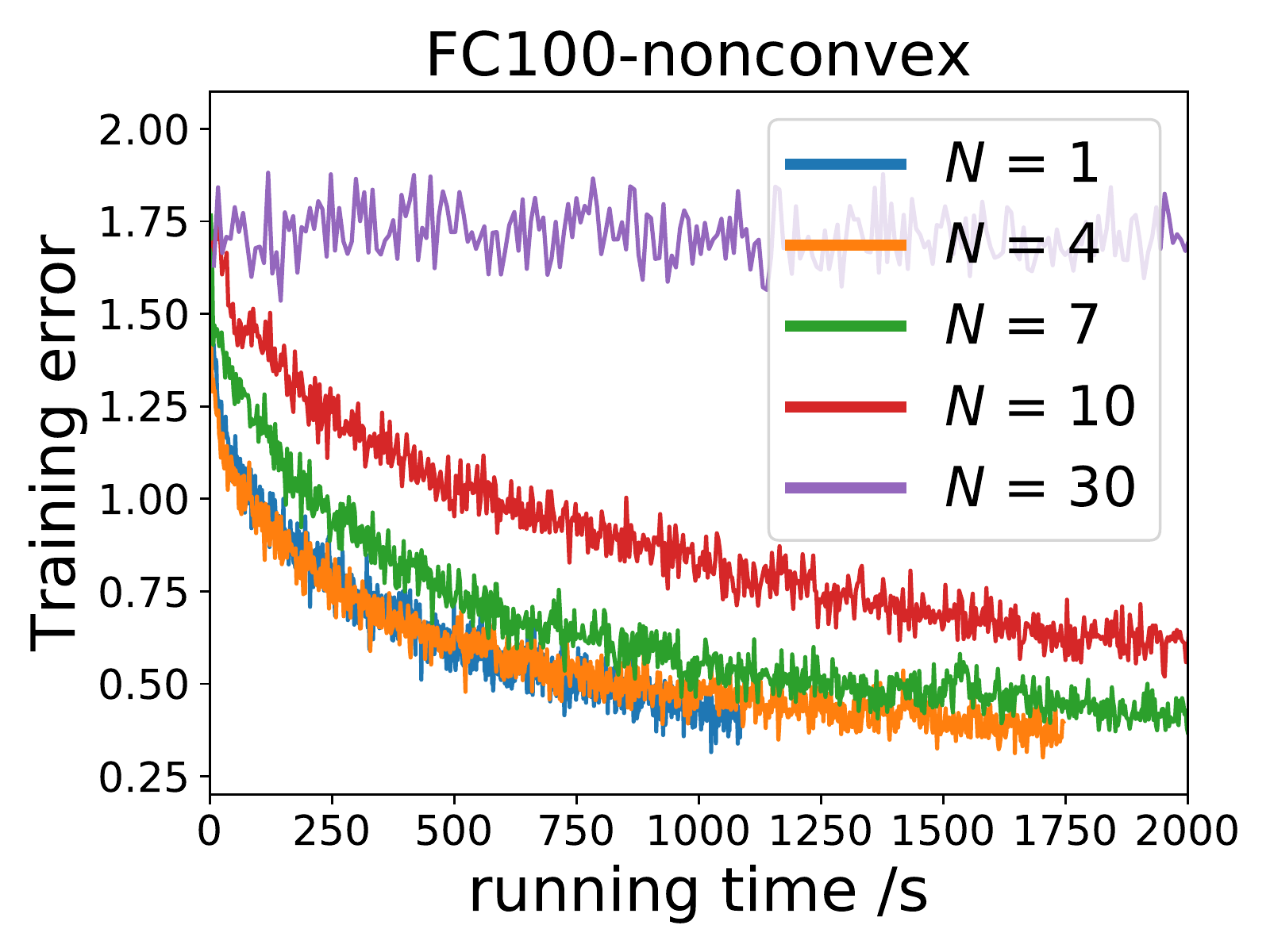}} 
	\subfigure[dataset: miniImageNet ]{\label{fig2:sacasd1asca}\includegraphics[width=60mm]{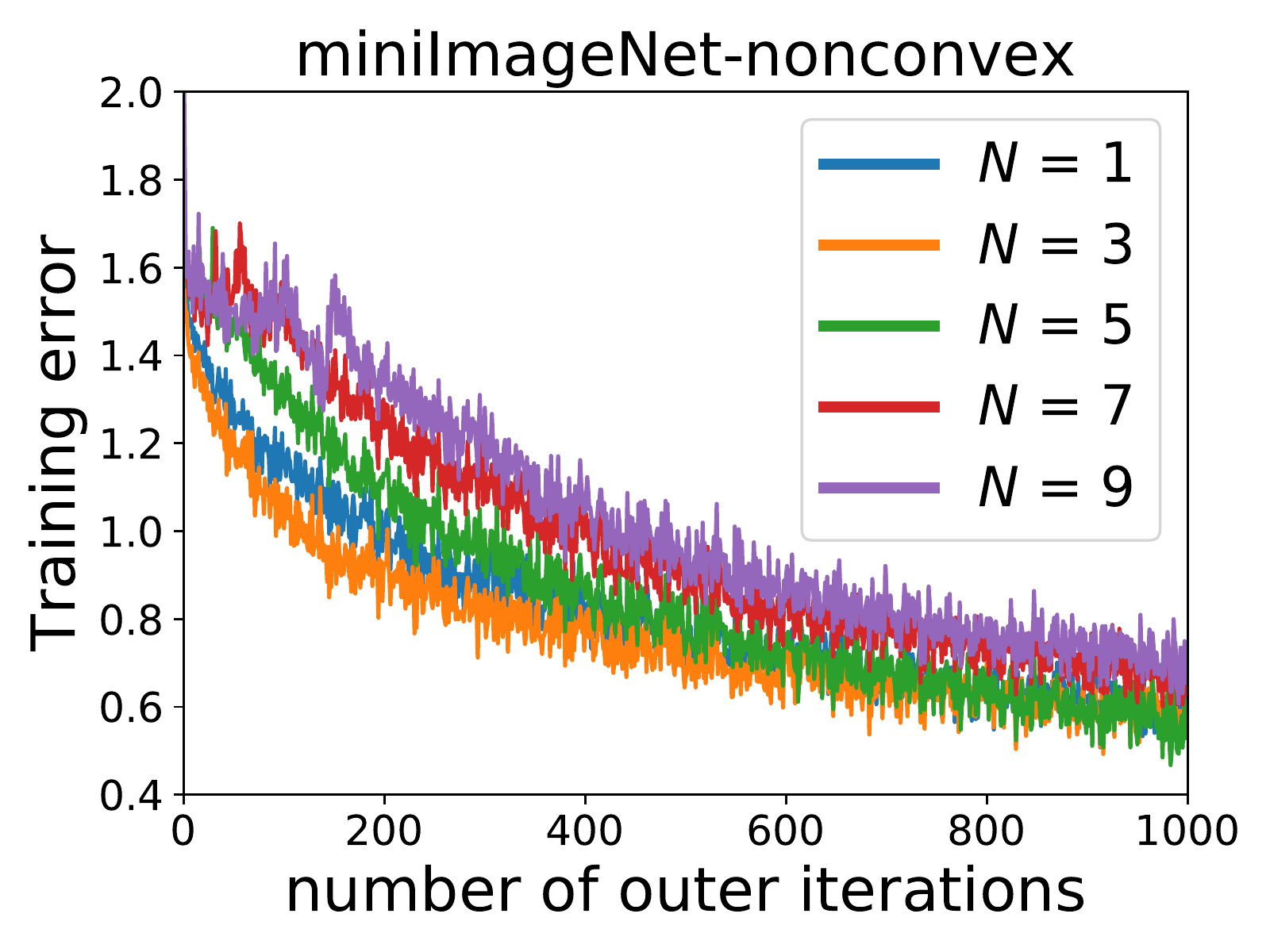}\includegraphics[width=60mm]{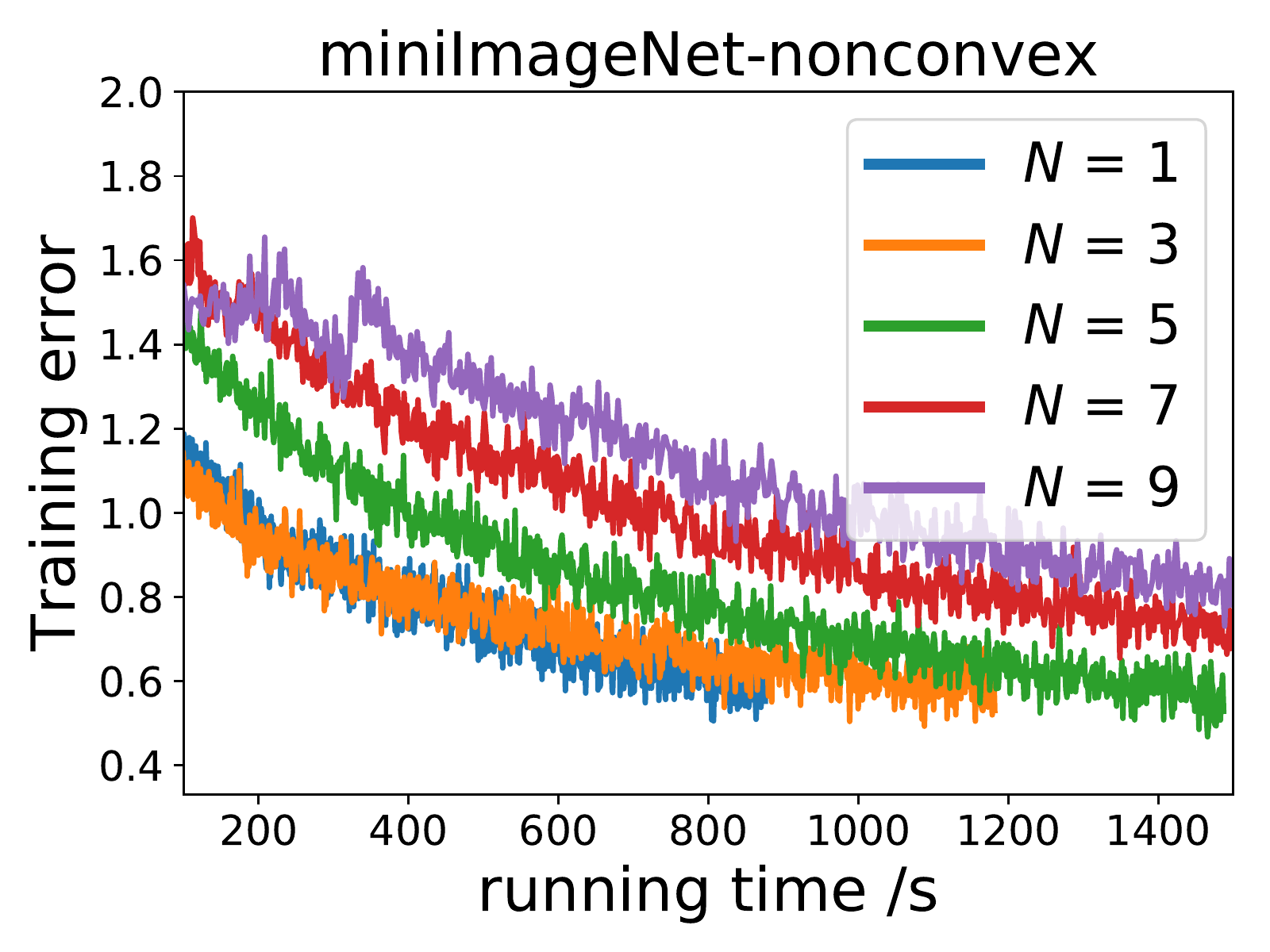}}  
	\caption{Convergence of ANIL with nonconvex inner-loop loss function.  For each dataset, left plot: training loss v.s. number of total meta iterations; right plot: training loss v.s. running time.}\label{figure:resultxiarenbaba}
\end{figure*}

\Cref{figure:resultxiarenbaba} provides the experimental results over the datasets FC100 and miniImageNet. 
For both datasets, the running-time convergence (right plot for each dataset) 
becomes {\em slower} as $N$ increases, where $N=1$ is fastest, and the algorithm even diverges for $N=30$ over the FC100 dataset. The plots are consist with \Cref{th:nonconvexcaonidaxeas}, in which the computational complexity increases as $N$ becomes large. Note that $N=1$ is not the fastest in the left plot for each dataset because the influence of $N$ is more prominent in terms of the running time than the number of outer-loop iterations (which is likely offset by other constant-level parameters for small $N$).
Thus, the optimization perspective here suggests that $N$ should be chosen as small as possible for computational efficiency, which in practice should be jointly considered with other aspects such as generalization for determining $N$.

\section{Summary of Contributions}
In this chapter, we provide theoretical convergence guarantee for the ANIL algorithm under strongly-convex and nonconvex inner-loop loss functions, respectively. Our analysis reveals different performance behaviors of ANIL under the two geometries by characterizing the impact of inner-loop adaptation steps on the overall convergence rate. Our results further provide guidelines for the hyper-parameter selections for ANIL under different inner-loop loss geometries. 

\chapter{Future Work and Other Ph.D. Studies}\label{end.ch}

In this chapter, we first propose several interesting research directions for the future study, and then briefly talk about some of the author's other research works. 

\section{Future Work}
In this section, we provide several potential research directions for future studies. 

\subsection*{Bilevel Optimization beyond Inner Strong Convexity}
Existing convergence rate analysis relies on the assumption that the inner-level function $g(x,\cdot)$ is strongly convex to ensure that 1) the total objective function $\Phi(x)$ is smooth, 2) the convergence rate for the inner-level problem is easy to characterize and 3) the Hessian in the hypergradient is invertible. However, this may sometimes restrict the application of the developed theory in areas where the loss $g(x,\cdot)$ contains multiple solutions, e.g., when $g(x,\cdot)$ is convex or satisfies the Polyak-{\L}ojasiewicz  (PL) inequality. For such cases, some crucial properties of the hypergradient in bilevel optimization do not hold any more. For example, the explicit form of the hypergradient via implicit gradient theorem may not hold because the outer-level objective function $\Phi(x)=f(x,y^*(x))$ is not necessarily differentiable. This means that the convergence metric for conventional smooth bilevel optimization cannot be directly adopted here, and new convergence criterions and analysis frameworks need to be developed. For example, for the nonconvex-convex setting, one possible solution is to measure the convergence in terms of an alternative notion of stationarity~\cite{davis2019stochastic} based on the Moreau envelope, and show that at least one subgradient has $\epsilon$-level magnitude. 

\subsection*{Lower Bound for Nonconvex Bilevel Optimization}
This thesis provides lower bounds for the convex-strongly-convex  and strongly-convex-strongly-convex bilevel optimization. The lower bounds for nonconvex-convex-strongly bilevel optimization problems still remain unexplored. Compared to the minimization optimization, constructing the worst-case instances for nonconvex bilevel optimization can be even harder due to the nested structure of the objective function. For example, \cite{carmon2019lower} provided lower bounds for first-order minimization optimization via constructing weakly convex worst-case objective functions. However, directly using such constructed worst-case instances in bilevel optimization is not applicable because they do not satisfy the nested structure as in bilevel optimization. Then, one possible solution is to add the worst-case instance functions we construct in \Cref{chp_lower_bilevel} with a nonconvex regularizer similarly \cite{carmon2019lower}. However, this requires future efforts to address. 

\subsection*{Optimal Bilevel Optimization Algorithms}
In \cite{ji2021lower}, we show that 
 for the strongly-convex-strongly-convex setting, our proposed AccBiO achieves the optimal complexity for the quadratic case with $\kappa_y\leq\mathcal{O}(1)$, where $\kappa_y$ is the condition number of the inner-level loss function. For the general case, there is a gap of $\mathcal{O}(\kappa_y^{-0.5})$. For the convex-strongly-convex setting, AccBiO is optimal for the quadratic case with $ \kappa_y\leq \mathcal{O}(1)$, and there is a gap of $\mathcal{O}(\kappa_y^{-0.5})$ for the general case. Such a gap is mainly due to the large smoothness parameter of the overall objective function. We note that a similar issue occurs for minimax optimization, which has been addressed by \cite{lin2020near} using an accelerated proximal point for inner-level problem and based on Sion's minimax theorem $\min_x\max_y f(x,y) =\max_y\min_x f(x,y)$. 
 However, as mentioned before, this scheme may not work here because the roles of variables $x$ and $y$ are unchangeable for bilevel optimization, i.e., Sion's theorem does work here. Then, another possibility is to develop a single-loop bilevel optimization by regarding $x$ and $y$ as a concatenated vector $z=(x,y)$, and then directly applying accelerated gradient methods to $z$. However, this still requires great efforts to the asymmetric between the outer and inner variables $x$ and $y$. 


\subsection*{Application of Our Lower Bound Analysis}
We note that some of our analysis can be applied to other problem domains such as minimax optimization. For example, our lower-bounding technique for \Cref{main:convex} can be extended to {\bf convex-concave} or {\bf convex-strongly-concave minimax} optimization, where the objective function $f(x,y)$ satisfies the general smoothness property as in \cref{def:first} with the general smoothness parameters $L_x,L_{xy}, L_y\geq 0$. The resulting lower bound will be different from that in \cite{ouyang2019lower}, which considered a special case with $L_y=0$ and the convergence is measured in terms of the suboptimality gap $\mathcal{O}(\Phi(x)-\Phi(x^*))$ rather than the gradient norm $\|\nabla\Phi(x)\|$ considered in this paper. Thus, such an extension will serve as a new contribution to lower complexity bounds for minimax optimization.

\section{Other Ph.D. Research}
To provide a neat version of thesis with closely correlated topics, this thesis does not include all of the author's works. We briefly talk about some representatives of the author's other research works~\cite{ji2018minimax,ji2021understanding,ji2020learning,wang2019spiderboost,ji2019improved,ji2020history,xu2021will,quan2018lru,tan2018resource,guan2020robust,zhang2020improving,zhang2020boosting,zhou2020proximal} as follows. 

 \vspace{0.2cm}
\noindent{\bf 1) Fundamental Limits of Generative adversarial networks (GANs)~(reference~\cite{ji2018minimax}):} 
This work developed a new theoretical framework to characterize the generalization error of GAN training from an information theoretic viewpoint. 
We first established a better convergence rate of the empirical estimator than the existing one, which captures much more refined dependence on the neural network parameters. Second, by Le Cam's method with various new technical developments, we further provided the first known lower bound on the minimax estimation error.
Combining the two steps then establishes that the GANs' framework is statistically optimal, which provides a theoretical foundation for the success of GAN training in practice. 

 \vspace{0.2cm}
\noindent{\bf 2) Generalization of GANs~(reference~\cite{ji2021understanding}):}
This work investigates the estimation and generalization errors of GAN training. On the statistical side, we develop an upper bound as well as a minimax lower bound on the estimation error  for training GANs. The upper  bound incorporates the roles of both the discriminator and the generator of GANs, and matches the minimax lower bound in terms of the sample size and the norm of the parameter matrices of neural networks under ReLU activation. On the algorithmic side,   
we develop a generalization error bound for the stochastic gradient method (SGM) in training GANs. Such a bound justifies the generalization ability of the GAN training via SGM after multiple passes over the data and reflects the interplay between the discriminator and the generator. 
Our results imply that the training of the generator requires more samples than the training of the discriminator. The experiments validate our theoretical results.


\vspace{0.2cm}
\noindent{\bf 3) Enhanced Matrix Completion via Pairwise Penalties (reference~\cite{ji2020learning}):} Low-rank matrix completion (MC) has achieved great success in many real-world data applications including movie recommendation and image restoration.  To fully empower pairwise learning for matrix completion, 
we propose a general optimization framework that allows a rich class of (non-)convex pairwise penalty functions, and develop a new and efficient algorithm with a theoretical convergence guarantee. 
The proposed framework shows superior performance in various applications including movie recommendation and data subgrouping. 

\vspace{0.2cm}
\noindent{\bf 4) Asymptotic Miss Ratio of LRU Caching with Consistent Hashing (reference~\cite{ji2018asymptotic}):} 
To efficiently scale data caching infrastructure to support emerging big data applications, many caching systems rely on consistent
hashing to group
a large number of servers to form a cooperative cluster.
These servers are organized together according to a random hash function.
They jointly provide a unified but distributed hash table to serve swift and voluminous
data item requests. 
In this work, we derive the asymptotic miss ratio of data item requests on a LRU cluster
with consistent hashing.
We show that these individual cache spaces on different servers  
can be effectively viewed as if
they could be pooled together 
to form a single virtual LRU cache space parametrized by an appropriate cache size.  This equivalence can be established rigorously
under the condition that the cache sizes of
the individual servers are large. For typical data caching systems this condition is common.  
Our theoretical framework provides a convenient abstraction that can directly apply the results from the simpler
single LRU cache to the more complex LRU cluster with consistent hashing.  

 \vspace{0.2cm}
\noindent{\bf 5) Variance Reduced Zeroth-Order Optimization~(reference~\cite{ji2019improved}):}
This work addresses several open issues in zeroth-order optimization. First, all existing SVRG-type zeroth-order algorithms suffer from worse function query complexities than either zeroth-order gradient descent (ZO-GD) or stochastic gradient descent (ZO-SGD). In this work, we propose a new algorithm ZO-SVRG-Coord-Rand and develop a new analysis for an existing ZO-SVRG-Coord algorithm proposed in~\cite{liu2018zeroth}, and show that both ZO-SVRG-Coord-Rand and ZO-SVRG-Coord (under our new analysis) outperform other exiting SVRG-type zeroth-order methods as well as ZO-GD and ZO-SGD. Second, the existing SPIDER-type algorithm SPIDER-SZO \cite{fang2018spider} has superior theoretical performance, but suffers from the generation of a large number of Gaussian random variables as well as a $\sqrt{\epsilon}$-level stepsize in practice. In this work, we develop a new algorithm ZO-SPIDER-Coord, which is free from Gaussian variable generation and allows a large constant stepsize while maintaining the same convergence rate and query complexity.

 \vspace{0.2cm}
\noindent{\bf 6) History-Gradient Aided Batch Size Adaptation~(reference~\cite{ji2020history}):}
Variance-reduced algorithms, although achieve great theoretical performance, can run slowly in practice due to the periodic gradient estimation with a large batch of data. Batch-size adaptation thus arises as a promising approach to accelerate such algorithms. However, existing schemes either apply prescribed batch-size adaption rule or exploit the information along optimization path via additional backtracking and condition verification steps. In this paper, we propose a novel scheme, which eliminates backtracking line search but still exploits the information along optimization path by adapting the batch size via history stochastic gradients. We further theoretically show that such a scheme substantially reduces the overall complexity for popular variance-reduced algorithms SVRG and SARAH/SPIDER for both conventional nonconvex optimization and reinforcement learning problems. To this end, we develop a new convergence analysis framework to handle the dependence of the batch size on history stochastic gradients. Extensive experiments validate the effectiveness of the proposed batch-size adaptation scheme.

\appendix

\chapter{Experimental Details and Proof of \Cref{chp_deter_bilevel}}\label{sec: append_deter_bilevel}

\section{Experimental Details}\label{appen:meta_learning}
\subsection*{Datasets and Model Architectures}
FC100~\cite{oreshkin2018tadam} is a dataset derived from CIFAR-100~\cite{krizhevsky2009learning}, and contains $100$ classes with each class consisting of $600$ images of size $32\time 32$. Following~\cite{oreshkin2018tadam}, these $100$ classes are split  into $60$ classes for meta-training, $20$ classes for meta-validation, and $20$ classes for meta-testing.   For all comparison algorithms, we use a $4$-layer convolutional neural networks (CNN) with four convolutional blocks, in which each convolutional block contains a $3\times 3$ convolution ($\text{padding}=1$, $\text{stride}=2$), batch normalization, ReLU activation, and $2\times 2$
max pooling. Each convolutional layer has $64$ filters. 

The miniImageNet dataset~\cite{vinyals2016matching} is generated from ImageNet~\cite{russakovsky2015imagenet}, and consists of $100$ classes with each class containing $600$ images of size $84\times 84$. Following the repository~\cite{learn2learn2019}, we partition these classes into $64$ classes for meta-training, $16$ classes for meta-validation, and $20$ classes for meta-testing.
Following the repository~\cite{learn2learn2019}, we use a four-layer CNN with four convolutional blocks, where each block sequentially consists of  a $3\times 3$ convolution, batch normalization, ReLU activation, and $2\times 2$
max pooling. Each convolutional layer has $32$ filters. 
\subsection*{Implementations and Hyperparameter Settings}
We adopt the existing implementations in the repository~\cite{learn2learn2019} for ANIL and MAML. 
 For all algorithms, we adopt Adam~\cite{kingma2014adam} as the optimizer for the outer-loop update. 
 
\vspace{0.2cm}
\noindent {\bf Parameter selection for the experiments in~\Cref{fig1:abilevel}:} For ANIL and MAML, we adopt the suggested hyperparameter selection in the repository~\cite{learn2learn2019}. In specific, for ANIL, we choose the inner-loop stepsize as $0.1$, the outer-loop (meta) stepsize as $0.002$, the task sampling size as $32$, and the number of inner-loop steps as $5$. For MAML, we choose the inner-loop stepsize as $0.5$, the outer-loop stepsize as $0.003$, the task sampling sizeas $32$, and the number of inner-loop steps as $3$. 
For ITD-BiO, AID-BiO-constant and AID-BiO-increasing, we use a grid search to choose the inner-loop stepsize from $\{0.01,0.1,1,10\}$, the task sampling size from $\{32,128,256\}$, and  the  outer-loop stepsize from $\{10^{i},i=-3,-2,-1,0,1,2,3\}$, where values that achieve the lowest loss after a fixed running time are selected.  
 For ITD-BiO and AID-BiO-constant, we choose the  number of inner-loop steps from $\{5,10,15,20,50\}$, and for AID-BiO-increasing, we choose the number of inner-loop steps as $\lceil c{(k+1)}^{1/4}\rceil$ as adopted by the analysis in \cite{ghadimi2018approximation}, where we choose $c$ from $\{0.5,2,5,10,50\}$.
For both AID-BiO-constant and AID-BiO-increasing, we choose the number $N$ of CG steps for solving the linear system from $\{5,10,15\}$.

\vspace{0.2cm}
\noindent 
{\bf Parameter selection for the experiments in~\Cref{fig1:bbilevelssc}:}  For ANIL and MAML, we adopt the suggested hyperparameter selection in the repository~\cite{learn2learn2019}. Specifically, for ANIL, we choose the inner-loop stepsize as $0.1$, the outer-loop (meta) stepsize as $0.001$, the task sampling size as $32$ and the number of inner-loop steps as $10$. For MAML, we choose the inner-loop stepsize as $0.5$, the outer-loop stepsize as $0.001$,  the  task samling size as $32$, and the number of inner-loop steps as $3$. For ITD-BiO, AID-BiO-constant and AID-BiO-increasing, we adopt the same procedure as in the experiments in~\Cref{fig1:abilevel}. 


\vspace{0.2cm}
\noindent 
{\bf Parameter selection for the experiments in~\Cref{figure:resultlg}:} 
For the experiments in~\Cref{fig1:ci}, we choose the inner-loop stepsize as $0.05$, the outer-loop (meta) stepsize as $0.002$, the  mini-batch size as $32$, and the number $T$ of inner-loop steps as $10$ for both ANIL and  ITD-BiO. For the experiments in~\Cref{fig1:di}, we choose the inner-loop stepsize as $0.1$, the outer-loop (meta) stepsize as $0.001$, the  mini-batch size as $32$, and the number $T$ of inner-loop steps as $20$ for both ANIL and  ITD-BiO.

%

\section{Supporting Lemmas}
In this section, we provide some auxiliary lemmas used for proving the main convergence results. 

Recall $\Phi(x)=f(x,y^*(x))$ in~\cref{objective_deter}. Then, we use 
the following lemma to establish the Lipschitz properties of $\nabla \Phi(x)$, which is adapted from Lemma 2.2 in~\cite{ghadimi2018approximation}.
\begin{lemma}\label{le:lipphi}
Suppose Assumptions~\ref{assum:geo},~\ref{ass:lip} and \ref{high_lip} hold. Then, we have, for any $x,x^\prime\in\mathbb{R}^p$,  
\begin{align*}
\|\nabla \Phi(x)- \nabla \Phi(x^\prime)\| \leq L_\Phi \|x-x^\prime\|,
\end{align*}
where the constant $L_\Phi$ is given by
\begin{align}
L_\Phi = L + \frac{2L^2+\tau M^2}{\mu} + \frac{\rho L M+L^3+\tau M L}{\mu^2} + \frac{\rho L^2 M}{\mu^3}.
\end{align}
\end{lemma}

\section{Proof of~\Cref{prop:grad}}
Using the chain rule over the gradient  $\nabla \Phi(x_k)=\frac{\partial f(x_k,y^*(x_k))}{\partial x_k}$, we have
\begin{align}\label{eq:maoxian}
\nabla \Phi(x_k)= \nabla_x f(x_k,y^*(x_k)) + \frac{\partial y^*(x_k)}{\partial x_k}\nabla_y f(x_k,y^*(x_k)).
\end{align}
Based on the optimality of $y^*(x_k)$, we have $\nabla_yg(x_k,y^*(x_k)) = 0$, which, using the implicit differentiation w.r.t. $x_k$, yields
\begin{align}\label{ineq:sscsa}  
\nabla_x\nabla_yg(x_k,y^*(x_k)) +\frac{\partial y^*(x_k)}{\partial x_k}\nabla_y^2g(x_k,y^*(x_k))= 0.
\end{align}
Let $v_k^*$ be the solution of the linear system $ \nabla_y^2g(x_k,y^*(x_k))v=\nabla_y f(x_k,y^*(x_k))$. Then, multiplying $v_k^*$ at the both sides of \cref{ineq:sscsa}, yields
\begin{align*}
-\nabla_x\nabla_yg(x_k,y^*(x_k)) v_k^* = \frac{\partial y^*(x_k)}{\partial x_k}\nabla_y^2g(x_k,y^*(x_k)) v_k^*=  \frac{\partial y^*(x_k)}{\partial x_k} \nabla_y f(x_k,y^*(x_k)),
\end{align*}
which, in conjunction with~\cref{eq:maoxian}, completes the proof.  

\section{Proof of~\Cref{deter:gdform} }
Based on the iterative update of  line $5$ in~\Cref{alg:main_deter}, we have $y_k^D = y_k^{0}-\alpha \sum_{t=0}^{D-1}\nabla_y g(x_k,y_k^{t})$, which, combined with the fact that  $\nabla_y g(x_k,y_k^{t})$ is differentiable w.r.t. $x_k$, indicates that the inner output $y_k^T$ is differentiable w.r.t. $x_k$. Then, based on the chain rule, 
we have 
\begin{align}\label{grad:est}
\frac{\partial f(x_k,y^D_k)}{\partial x_k}= \nabla_x f(x_k,y_k^D) + \frac{\partial y_k^D}{\partial x_k}\nabla_y f(x_k,y_k^D).
\end{align}
Using the iterative updates that $y_k^t = y_k^{t-1}-\alpha \nabla_y g(x_k,y_k^{t-1}) $ for $t=1,...,D$, we have 
\begin{align*}
\frac{\partial y_k^t}{\partial x_k} =& \frac{\partial y_k^{t-1}}{\partial x_k}-\alpha \nabla_x\nabla_y g(x_k,y_k^{t-1})-\alpha\frac{\partial y_k^{t-1}}{\partial x_k} \nabla^2_y g(x_k,y_k^{t-1}) 
\\= &\frac{\partial y_k^{t-1}}{\partial x_k}(I-\alpha  \nabla^2_y g(x_k,y_k^{t-1}))-\alpha \nabla_x\nabla_y g(x_k,y_k^{t-1}).
\end{align*}
Telescoping the above equality over $t$ from $1$ to $D$ yields
\begin{align}\label{gd:formss}
\frac{\partial y_k^D}{\partial x_k} =&\frac{\partial y_k^0}{\partial x_k} \prod_{t=0}^{D-1}(I-\alpha  \nabla^2_y g(x_k,y_k^{t}))-\alpha\sum_{t=0}^{D-1}\nabla_x\nabla_y g(x_k,y_k^{t})\prod_{j=t+1}^{D-1}(I-\alpha  \nabla^2_y g(x_k,y_k^{j})) \nonumber
\\\overset{(i)}=&-\alpha\sum_{t=0}^{D-1}\nabla_x\nabla_y g(x_k,y_k^{t})\prod_{j=t+1}^{D-1}(I-\alpha  \nabla^2_y g(x_k,y_k^{j})). 
\end{align}
where $(i)$ follows from the fact that  $\frac{\partial y_k^0}{\partial x_k}=0$. 
Combining~\cref{grad:est} and~\cref{gd:formss} finishes the proof.

\section{Proof of \Cref{th:aidthem}}\label{appen:aid-bio}
For notation simplification, we define the following quantities. 
{\small\begin{align}\label{eq:notaionssscas}
\Gamma =&3L^2+\frac{3\tau^2 M^2}{\mu^2} + 6L^2\big(1+\sqrt{\kappa}\big)^2\big(\kappa +\frac{\rho M}{\mu^2}\big)^2, \delta_{D,N}=\Gamma (1-\alpha \mu)^D  + 6L^2 \kappa \big( \frac{\sqrt{\kappa}-1}{\sqrt{\kappa}+1} \big)^{2N}
\nonumber
\\\Omega =&8\Big(\beta\kappa^2+\frac{2\beta ML}{\mu^2}+\frac{2\beta LM\kappa}{\mu^2}\Big)^2,\; \Delta_0 = \|y_0-y^*(x_{0})\|^2 + \|v_{0}^*-v_0\|^2.
\end{align}}
\hspace{-0.15cm}We first provide some supporting lemmas. The following lemma 
characterizes the Hypergradient estimation error $\|\widehat \nabla \Phi(x_k)- \nabla \Phi(x_k)\|$, where $\widehat \nabla \Phi(x_k)$ is given by \cref{hyper-aid} via implicit differentiation. 
\begin{lemma}\label{le:aidhy}
Suppose Assumptions~\ref{assum:geo},~\ref{ass:lip} and \ref{high_lip} hold.  
Then, we have 
\begin{align}
\|\widehat \nabla \Phi(x_k)- \nabla \Phi(x_k)\|^2 \leq &\Gamma (1-\alpha \mu)^D \|y^*(x_k)-y_k^0\|^2 + 6L^2 \kappa \Big( \frac{\sqrt{\kappa}-1}{\sqrt{\kappa}+1} \Big)^{2N}\|v_k^*-v_k^0\|^2. \nonumber
\end{align}
where $\Gamma$ is given by \cref{eq:notaionssscas}. 
\end{lemma}
\begin{proof}[\bf Proof of \Cref{le:aidhy}]
Based on the form of $\nabla\Phi(x_k)$ given by~\Cref{prop:grad},  we have 
\begin{align*}
\|&\widehat \nabla \Phi(x_k)- \nabla \Phi(x_k)\|^2 \leq 3\|\nabla_x f(x_k,y^*(x_k))-\nabla_x f(x_k,y_k^D)\|^2   \nonumber
\\&+3\|\nabla_x \nabla_y g(x_k,y_k^D)\|^2\|v_k^*-v_k^N\|^2+ 3\|\nabla_x \nabla_y g(x_k,y^*(x_k))-\nabla_x \nabla_y g(x_k,y_k^D) \|^2 \|v_k^*\|^2,
\end{align*}
which, in conjunction with Assumptions~\ref{assum:geo},~\ref{ass:lip} and \ref{high_lip}, yields
\begin{align}\label{eq:midpos}
\|\widehat \nabla \Phi(x_k)&- \nabla \Phi(x_k)\|^2  \nonumber
\\\leq &3L^2\|y^*(x_k)-y_k^D\|^2 + 3L^2\|v_k^*-v_k^N\|^2+3\tau^2\|v_k^*\|^2\|y_k^D-y^*(x_k)\|^2\nonumber
\\\overset{(i)}\leq&  3L^2\|y^*(x_k)-y_k^D\|^2 + 3L^2\|v_k^*-v_k^N\|^2+\frac{3\tau^2 M^2}{\mu^2}\|y_k^D-y^*(x_k)\|^2.
\end{align}
where $(i)$ follows from the fact that $\|v_k^*\|\leq\|(\nabla_y^2g(x_k,y^*(x_k)))^{-1}\|\|\nabla_y f(x_k,y^*(x_k))\|\leq \frac{M}{\mu}$.   
For notation simplification, let $\widehat v_k=(\nabla_y^2g(x_k,y^D_k))^{-1}\nabla_y f(x_k,y^D_k)$.  We next  upper-bound $\|v_k^*-v_k^N\|$ in \cref{eq:midpos}. Based on the convergence result of CG for the quadratic programing, e.g., eq. (17) in~\cite{grazzi2020iteration}, we have 
 $\|v_k^N-\widehat v_k\| \leq \sqrt{\kappa}\Big( \frac{\sqrt{\kappa}-1}{\sqrt{\kappa}+1} \Big)^N\|v_k^0-\widehat v_k\|.$
Based on this inequality, we further have 
 \begin{align}\label{eq:letknca}
 \|v_k^*-v_k^N\| \leq &\|v_k^*-\widehat v_k\| + \|v_k^N-\widehat v_k\| \leq  \|v_k^*-\widehat v_k\|  +  \sqrt{\kappa}\Big( \frac{\sqrt{\kappa}-1}{\sqrt{\kappa}+1} \Big)^N\|v_k^0-\widehat v_k\| \nonumber
 \\\leq& \Big(1+\sqrt{\kappa}\Big( \frac{\sqrt{\kappa}-1}{\sqrt{\kappa}+1} \Big)^N\Big) \|v_k^*-\widehat v_k\| + \sqrt{\kappa}\Big( \frac{\sqrt{\kappa}-1}{\sqrt{\kappa}+1} \Big)^N\|v_k^*-v_k^0\|. 
 \end{align}
Next, based on the definitions of $v_k^*$ and $ \widehat v_k$, we have 
\begin{align}\label{eq:omgsaca}
\|v_k^*-\widehat v_k\|=& \|(\nabla_y^2g(x_k,y^D_k))^{-1}\nabla_y f(x_k,y^D_k) -(\nabla_y^2g(x_k,y^*(x_k))^{-1}\nabla_y f(x_k,y^*(x_k))\| \nonumber
\\\leq & \Big(\kappa +\frac{\rho M}{\mu^2} \Big)\|y^D_k-y^*(x_k)\|. 
\end{align}
Combining \cref{eq:midpos},~\cref{eq:letknca}, \cref{eq:omgsaca} yields
\begin{align*}
\|\widehat \nabla \Phi(x_k)- &\nabla \Phi(x_k)\|^2 
\\\leq &\Big(3L^2+\frac{3\tau^2 M^2}{\mu^2}\Big)\|y^*(x_k)-y_k^D\|^2 + 6L^2 \kappa \Big( \frac{\sqrt{\kappa}-1}{\sqrt{\kappa}+1} \Big)^{2N}\|v_k^*-v_k^0\|^2 \nonumber
\\&+ 6L^2\Big(1+\sqrt{\kappa}\Big( \frac{\sqrt{\kappa}-1}{\sqrt{\kappa}+1} \Big)^N\Big)^2\Big(\kappa +\frac{\rho M}{\mu^2} \Big)^2\|y^D_k-y^*(x_k)\|^2, 
\end{align*}
which, in conjunction with $\|y_k^{D} -y^*(x_k)\| \leq (1-\alpha\mu)^{\frac{D}{2}} \|y^0_k-y^*(x_k)\|$ and the notations in \cref{eq:notaionssscas}, finishes the proof. 
\end{proof}

\begin{lemma}\label{le:bibibiss}
Suppose Assumptions~\ref{assum:geo},~\ref{ass:lip} and \ref{high_lip} hold. Choose
\begin{small}
\begin{align}\label{eq:findbogas}
D\geq& \log{(36 \kappa (\kappa +\frac{\rho M}{\mu^2} )^2+16(\kappa^2+\frac{4LM\kappa}{\mu^2})^2\beta^2\Gamma)}/\log\frac{1}{1-\alpha}=\Theta(\kappa)  \nonumber
\\ N\geq& \frac{1}{2}\log(8\kappa+48(\kappa^2+\frac{2ML}{\mu^2}+\frac{2LM\kappa}{\mu^2})^2\beta^2L^2 \kappa ) /\log \frac{\sqrt{\kappa}+1}{\sqrt{\kappa}-1} = \Theta(\sqrt{\kappa}),
 \end{align}
 \end{small}
\hspace{-0.12cm}where $\Gamma$ is given by \cref{eq:notaionssscas}.   Then, we have 
\begin{align}
\|y^0_k-y^*(x_k)\|^2 + &\|v_k^*-v_k^0\|^2   \leq \Big(\frac{1}{2}\Big)^k  \Delta_0+\Omega\sum_{j=0}^{k-1}\Big(\frac{1}{2}\Big)^{k-1-j}\|\nabla \Phi(x_{j})\|^2,
\end{align}
where $\Omega$ and $\Delta_0$ are given by \cref{eq:notaionssscas}. 
\end{lemma}
\begin{proof}[\bf Proof of \Cref{le:bibibiss}]
Recall that $y^0_k=y^D_{k-1}$. Then, we have 
\begin{align}\label{wocaoleis}
\|y^0_k-y^*(x_k)&\|^2 \nonumber
\\\leq &2\|y^D_{k-1}-y^*(x_{k-1})\|^2 + 2\|y^*(x_k)-y^*(x_{k-1})\|^2 \nonumber
\\\overset{(i)}\leq& 2(1-\alpha\mu)^{D} \|y_{k-1}^0-y^*(x_{k-1})\|^2 + 2\kappa^2\beta^2\|\widehat \nabla \Phi(x_{k-1})\|^2 \nonumber
\\\leq&2(1-\alpha\mu)^{D} \|y_{k-1}^0-y^*(x_{k-1})\|^2 + 4\kappa^2\beta^2\| \nabla \Phi(x_{k-1})-\widehat \nabla \Phi(x_{k-1})\|^2  \nonumber
\\&+ 4\kappa^2\beta^2 \| \nabla \Phi(x_{k-1})\|^2 \nonumber
\\\overset{(ii)}\leq&\big(2(1-\alpha\mu)^{D}+ 4\kappa^2\beta^2\Gamma (1-\alpha \mu)^D\big)\|y^*(x_{k-1})-y_{k-1}^0\|^2\nonumber
\\ &+24\kappa^4L^2\beta^2  \Big( \frac{\sqrt{\kappa}-1}{\sqrt{\kappa}+1} \Big)^{2N} \|v_{k-1}^*-v_{k-1}^0\|^2+ 4\kappa^2\beta^2 \| \nabla \Phi(x_{k-1})\|^2,
\end{align}
where $(i)$ follows from Lemma 2.2 in \cite{ghadimi2018approximation} and $(ii)$ follows from \Cref{le:aidhy}. In addition, 
\begin{align}\label{eq:kdasdaca}
\|v_k^*-v_k^0\|^2 =& \|v_k^*-v_{k-1}^N\|^2 \leq 2\|v_{k-1}^*-v_{k-1}^N\|^2+2\|v_k^*-v_{k-1}^*\|^2 \nonumber
\\\overset{(i)}\leq &4 \Big(1+\sqrt{\kappa}\Big)^2\Big(\kappa +\frac{\rho M}{\mu^2} \Big)^2(1-\alpha\mu)^D\|y_{k-1}^0-y^*(x_{k-1})\|^2 \nonumber
\\&+4\kappa \Big( \frac{\sqrt{\kappa}-1}{\sqrt{\kappa}+1} \Big)^{2N}\|v_{k-1}^*-v_{k-1}^0\|^2 + 2\|v_k^*-v_{k-1}^*\|^2, 
\end{align}
where $(i)$ follows from \cref{eq:letknca}. Combining \cref{eq:kdasdaca} with $\|v_k^*-v_{k-1}^*\|\leq(\kappa^2+\frac{2ML}{\mu^2}+\frac{2LM\kappa}{\mu^2})\|x_k-x_{k-1}\|$, we have 
\begin{align}\label{kopcasa}
\|v_k^*-v_k^0\|^2\overset{(i)}\leq&\Big(16 \kappa \Big(\kappa +\frac{\rho M}{\mu^2} \Big)^2+4\Big(\kappa^2+\frac{4LM\kappa}{\mu^2}\Big)^2\beta^2\Gamma \Big)(1-\alpha\mu)^D\|y_{k-1}^0-y^*(x_{k-1})\|^2 \nonumber
\\&+\Big(4\kappa+48\Big(\kappa^2+\frac{2ML}{\mu^2}+\frac{2LM\kappa}{\mu^2}\Big)^2\beta^2L^2 \kappa \Big) \Big( \frac{\sqrt{\kappa}-1}{\sqrt{\kappa}+1} \Big)^{2N}\|v_{k-1}^*-v_{k-1}^0\|^2 \nonumber
\\&+ 4\Big(\kappa^2+\frac{2ML}{\mu^2}+\frac{2LM\kappa}{\mu^2}\Big)^2\beta^2\|\nabla \Phi(x_{k-1})\|^2, 
\end{align}
where $(i)$ follows from \Cref{le:aidhy}. Combining \cref{wocaoleis} and  \cref{kopcasa} yields
\begin{align*}
\|y^0_k-y^*(x_k)&\|^2 + \|v_k^*-v_k^0\|^2  \nonumber
\\\leq &\Big(18 \kappa \Big(\kappa +\frac{\rho M}{\mu^2} \Big)^2+8\Big(\kappa^2+\frac{4LM\kappa}{\mu^2}\Big)^2\beta^2\Gamma \Big)(1-\alpha\mu)^D\|y_{k-1}^0-y^*(x_{k-1})\|^2 \nonumber
\\&+\Big(4\kappa+24\Big(\kappa^2+\frac{2ML}{\mu^2}+\frac{2LM\kappa}{\mu^2}\Big)^2\beta^2L^2 \kappa \Big) \Big( \frac{\sqrt{\kappa}-1}{\sqrt{\kappa}+1} \Big)^{2N}\|v_{k-1}^*-v_{k-1}^0\|^2 \nonumber
\\&+ 8\Big(\kappa^2+\frac{2ML}{\mu^2}+\frac{2LM\kappa}{\mu^2}\Big)^2\beta^2\|\nabla \Phi(x_{k-1})\|^2,
\end{align*}
which, in conjunction with \cref{eq:findbogas},   yields
\begin{align}\label{eq:televk00}
\|y^0_k-y^*(x_k)\|^2 +& \|v_k^*-v_k^0\|^2 \nonumber
\\\leq &\frac{1}{2} (\|y^0_{k-1}-y^*(x_{k-1})\|^2 + \|v_{k-1}^*-v_{k-1}^0\|^2) \nonumber
\\&+8\Big(\beta\kappa^2+\frac{2\beta ML}{\mu^2}+\frac{2\beta LM\kappa}{\mu^2}\Big)^2\|\nabla \Phi(x_{k-1})\|^2.
\end{align}
Telescoping \cref{eq:televk00} over $k$ and using the notations in~\cref{eq:notaionssscas} finish the proof. 
\end{proof}
\begin{lemma}\label{le:gamma1}
Under the same setting as in \Cref{le:bibibiss}, we have 
\begin{align*}
\|\widehat \nabla \Phi(x_k)- \nabla \Phi(x_k)\|^2 \leq &\delta_{D,N}\Big(\frac{1}{2}\Big)^k  \Delta_0+ \delta_{D,N}\Omega\sum_{j=0}^{k-1}\Big(\frac{1}{2}\Big)^{k-1-j}\|\nabla \Phi(x_{j})\|^2. \end{align*}
where $\delta_{T,N}$, $\Omega$ and $\Delta_0$ are given by \cref{eq:notaionssscas}.
\end{lemma}
\begin{proof}[\bf Proof of \Cref{le:gamma1}] 
Based on \Cref{le:aidhy}, \cref{eq:notaionssscas} and using $ab+cd\leq (a+c)(b+d)$ for any positive $a,b,c,d$, we have
\begin{align*}
\|\widehat \nabla \Phi(x_k)- \nabla \Phi(x_k)\|^2 \leq &\delta_{D,N}(\|y^*(x_k)-y_k^0\|^2+\|v_k^*-v_k^0\|^2),
\end{align*}
which, in conjunction with \Cref{le:bibibiss}, finishes the proof. 
\end{proof}

We now provide the proof for \Cref{th:aidthem}. 
Based on the smoothness of the function $\Phi(x)$ established in~\Cref{le:lipphi}, we have 
\begin{align}\label{eq:intimidern_pre}
\Phi(&x_{k+1}) \leq  \Phi(x_k)  + \langle \nabla \Phi(x_k), x_{k+1}-x_k\rangle + \frac{L_\Phi}{2} \|x_{k+1}-x_k\|^2 \nonumber
\\\leq& \Phi(x_k)  - \beta \langle \nabla \Phi(x_k),\widehat \nabla \Phi(x_k)- \nabla \Phi(x_k)\rangle -\beta\| \nabla \Phi(x_k)\|^2 + \beta^2 L_\Phi \|\nabla\Phi(x_k)\|^2\nonumber
\\&+\beta^2 L_\Phi\|\nabla\Phi(x_k)-\widehat \nabla\Phi(x_k)\|^2\nonumber
\\\leq&\Phi(x_k) -\Big(\frac{\beta}{2}-\beta^2 L_\Phi \Big)\| \nabla \Phi(x_k)\|^2 +\Big(\frac{\beta}{2}+\beta^2 L_\Phi\Big)\|\nabla\Phi(x_k)-\widehat \nabla\Phi(x_k)\|^2,
\end{align}
which, combined with \Cref{le:gamma1}, yields
\begin{align}\label{eq:teletop}
\Phi(x_{k+1}) \leq &\Phi(x_k) -\Big(\frac{\beta}{2}-\beta^2 L_\Phi \Big)\| \nabla \Phi(x_k)\|^2 + \Big(\frac{\beta}{2}+\beta^2 L_\Phi\Big)
\delta_{D,N}\Big(\frac{1}{2}\Big)^k  \Delta_0 \nonumber
\\&+ \Big(\frac{\beta}{2}+\beta^2 L_\Phi\Big)\delta_{D,N}\Omega\sum_{j=0}^{k-1}\Big(\frac{1}{2}\Big)^{k-1-j}\|\nabla \Phi(x_{j})\|^2.
\end{align}
Telescoping \cref{eq:teletop} over k from $0$ to $K-1$ yields
\begin{align*}
\Big(\frac{\beta}{2}-\beta^2 L_\Phi \Big) \sum_{k=0}^{K-1}\| &\nabla \Phi(x_k)\|^2 \leq \Phi(x_0) - \inf_x\Phi(x) + \Big(\frac{\beta}{2}+\beta^2 L_\Phi\Big) \delta_{D,N} \Delta_0\nonumber
\\&+ \Big(\frac{\beta}{2}+\beta^2 L_\Phi\Big)\delta_{D,N}\Omega\sum_{k=1}^{K-1}\sum_{j=0}^{k-1}\Big(\frac{1}{2}\Big)^{k-1-j}\|\nabla \Phi(x_{j})\|^2,
\end{align*}
which, by  the fact that {\small $\sum_{k=1}^{K-1}\sum_{j=0}^{k-1}\Big(\frac{1}{2}\Big)^{k-1-j}\|\nabla \Phi(x_{j})\|^2 \leq \sum_{k=0}^{K-1}\frac{1}{2^k}\sum_{k=0}^{K-1}\|\nabla \Phi(x_{k})\|^2\leq 2\sum_{k=0}^{K-1}\|\nabla \Phi(x_{k})\|^2$}, yields
\begin{align}\label{jkunisas}
\Big(\frac{\beta}{2}-\beta^2 L_\Phi -\big(\beta\Omega+2\Omega\beta^2 &L_\Phi\big)\delta_{D,N}\Big) \sum_{k=0}^{K-1}\| \nabla \Phi(x_k)\|^2  \nonumber
\\&\leq \Phi(x_0) - \inf_x\Phi(x) + \Big(\frac{\beta}{2}+\beta^2 L_\Phi\Big) \delta_{D,N} \Delta_0. 
\end{align}
Choose $N$ and $D$ such that 
\begin{align}\label{NTsatis}
 \big(\Omega+2\Omega\beta L_\Phi\big)\delta_{D,N} \leq \frac{1}{4}, \quad \delta_{D,N}\leq 1.
\end{align}
Note that based on the definition of $\delta_{D,N}$ in~\cref{eq:notaionssscas}, it suffices to choose $D\geq\Theta(\kappa)$ and $N\geq \Theta(\sqrt{\kappa})$ to satisfy \cref{NTsatis}. Then, substituting \cref{NTsatis} into \cref{jkunisas} yields 
\begin{align*}
\Big(\frac{\beta}{4}-\beta^2 L_\Phi \Big) \sum_{k=0}^{K-1}\| \nabla \Phi(x_k)\|^2 \leq \Phi(x_0) - \inf_x\Phi(x) + \Big(\frac{\beta}{2}+\beta^2 L_\Phi\Big)\Delta_0,
\end{align*}
which, in conjunction with $\beta\leq \frac{1}{8L_\Phi}$, yields
\begin{align}\label{eq:woele}
\frac{1}{K}\sum_{k=0}^{K-1}\| \nabla \Phi(x_k)\|^2 \leq \frac{64L_\Phi (\Phi(x_0) - \inf_x\Phi(x))+5\Delta_0}{K}.  
\end{align}
In order to achieve an $\epsilon$-accurate stationary point, we obtain from~\cref{eq:woele} that 
AID-BiO requires at most the total number $K=\mathcal{O}(\kappa^3\epsilon^{-1})$ of outer iterations. 
Then, based on \cref{hyper-aid}, we have the following complexity results.
\begin{itemize}
\item Gradient complexity: $$\mbox{\normalfont Gc}(f,\epsilon)=2K=\mathcal{O}(\kappa^3\epsilon^{-1}), \mbox{\normalfont Gc}(g,\epsilon)=KD=\mathcal{O}\big(\kappa^4\epsilon^{-1}\big).$$
\item Jacobian- and Hessian-vector product complexities: $$ \mbox{\normalfont JV}(g,\epsilon)=K=\mathcal{O}\left(\kappa^3\epsilon^{-1}\right), \mbox{\normalfont HV}(g,\epsilon)=KN=\mathcal{O}\left(\kappa^{3.5}\epsilon^{-1}\right).$$
\end{itemize}
Then, the proof is complete. 

\section{Proof of~\Cref{th:determin}}\label{append:itd-bio}
We first characterize an important estimation property of the outer-loop gradient estimator $\frac{\partial f(x_k,y^D_k)}{\partial x_k}$ in ITD-BiO for approximating the true gradient $\nabla \Phi(x_k)$ based on \Cref{deter:gdform}.

\begin{lemma}\label{prop:partialG}  
Suppose Assumptions~\ref{assum:geo},~\ref{ass:lip} and \ref{high_lip} hold. Choose  $\alpha\leq \frac{1}{L}$. Then, we have
\begin{small}
\begin{align*}
\Big\|\frac{\partial f(x_k,y^D_k)}{\partial x_k}-\nabla\Phi(x_k)\Big\| \leq&\big( \frac{L(L+\mu)(1-\alpha\mu)^{\frac{D}{2}}}{\mu} +\frac{2M\left(  \tau\mu+ L\rho \right)}{\mu^2}(1-\alpha\mu)^{\frac{D-1}{2}} \big)\|y^0_k-y^*(x_k)\|  \nonumber
\\&+ \frac{LM(1-\alpha\mu)^D}{\mu}.
\end{align*}
\end{small}
\end{lemma}
\vspace{-0.4cm}
\Cref{prop:partialG} shows that the gradient estimation error  $\big\|\frac{\partial f(x_k,y^D_k)}{\partial x_k}-\nabla\Phi(x_k)\big\|$ decays exponentially w.r.t. the number $D$ of the inner-loop steps. 
We note that \cite{grazzi2020iteration} proved a similar result via a fixed point based approach. As a comparison, our proof of \Cref{prop:partialG} directly characterizes the rate of the sequence  $\big(\frac{\partial y^t_k}{\partial x_k},t=0,...,D\big)$ converging to $\frac{\partial y^*(x_k)}{\partial x_k}$ via the differentiation over all corresponding points along the inner-loop GD path as well as the optimality of the point $y^*(x_k)$. 
\begin{proof}[\bf Proof of \Cref{prop:partialG}]
Based on $\nabla \Phi(x_k)= \nabla_x f(x_k,y^*(x_k)) + \frac{\partial y^*(x_k)}{\partial x_k}\nabla_y f(x_k,y^*(x_k))$ and~\cref{grad:est} , and using the triangle inequality, we have 
 \begin{align}\label{eq:woaijingii}
\Big\|&\frac{\partial f(x_k,y^D_k)}{\partial x_k} -\nabla\Phi(x_k)\Big\| \nonumber
\\&=\| \nabla_x f(x_k,y_k^D)-\nabla_x f(x_k,y^*(x_k))\| + \left\|\frac{\partial y_k^D}{\partial x_k}-\frac{\partial y^*(x_k)}{\partial x_k}\right\|\|\nabla_y f(x_k,y_k^D)\| \nonumber
\\&\quad+\Big\|\frac{\partial y^*(x_k)}{\partial x_k}\Big\|\big\|\nabla_y f(x_k,y_k^D)-\nabla_y f(x_k,y^*(x_k))\big\|\nonumber
\\&\overset{(i)}\leq L\|y_k^D-y^*(x_k)\| + M \left\|\frac{\partial y_k^D}{\partial x_k}-\frac{\partial y^*(x_k)}{\partial x_k}\right\| + L\Big\|\frac{\partial y^*(x_k)}{\partial x_k}\Big\|\|y_k^D-y^*(x_k)\|, 
\end{align}
where $(i)$ follows from Assumption~\ref{ass:lip}. Our next step is to upper-bound $\left\|\frac{\partial y_k^D}{\partial x_k}-\frac{\partial y^*(x_k)}{\partial x_k}\right\| $ in \cref{eq:woaijingii}.

Based on the updates $y_k^t = y_k^{t-1}-\alpha \nabla_y g(x_k,y_k^{t-1}) $ for $t=1,...,D$  in ITD-BiO and using the chain rule, we have 
\begin{align}\label{eq:1ss}
\frac{\partial y_k^t}{\partial x_k} = \frac{\partial y_k^{t-1}}{\partial x_k} - \alpha \left( \nabla_x\nabla_y g(x_k,y_k^{t-1}) +\frac{\partial y_k^{t-1}}{\partial x_k}\nabla_y^2 g(x_k,y_k^{t-1})\right).
\end{align}
Based on the optimality of $y^*(x_k)$, we have $\nabla_y g(x_k,y^*(x_k))=0$, which, in conjunction with the implicit differentiation theorem, yields
\begin{align}\label{eq:2sss}
\nabla_x\nabla_y g(x_k,y^*(x_k)) + \frac{\partial y^*(x_k)}{\partial x_k}\nabla_y^2 g(x_k,y^*(x_k))=0.
\end{align}
Substituting \cref{eq:2sss} into~\cref{eq:1ss} yields
\begin{align}\label{eq:zhaoposdo}
\frac{\partial y_k^t}{\partial x_k} -\frac{\partial y^*(x_k)}{\partial x_k} =& \frac{\partial y_k^{t-1}}{\partial x_k} -\frac{\partial y^*(x_k)}{\partial x_k}- \alpha \left( \nabla_x\nabla_y g(x_k,y_k^{t-1}) +\frac{\partial y_k^{t-1}}{\partial x_k}\nabla_y^2 g(x_k,y_k^{t-1})\right) \nonumber
\\&+\alpha\left( \nabla_x\nabla_y g(x_k,y^*(x_k)) + \frac{\partial y^*(x_k)}{\partial x_k}\nabla_y^2 g(x_k,y^*(x_k)) \right) \nonumber
\\ = &\frac{\partial y_k^{t-1}}{\partial x_k} -\frac{\partial y^*(x_k)}{\partial x_k}- \alpha \left( \nabla_x\nabla_y g(x_k,y_k^{t-1}) - \nabla_x\nabla_y g(x_k,y^*(x_k))\right) \nonumber
\\&-\alpha\left(\frac{\partial y_k^{t-1}}{\partial x_k}- \frac{\partial y^*(x_k)}{\partial x_k}\right)\nabla_y^2 g(x_k,y_k^{t-1}) \nonumber
\\&+\alpha\frac{\partial y^*(x_k)}{\partial x_k}\left(\nabla_y^2 g(x_k,y^*(x_k))-\nabla_y^2 g(x_k,y_k^{t-1})   \right). 
\end{align}
Combining \cref{eq:2sss} and Assumption~\ref{ass:lip} yields
\begin{align}\label{ggpdas}
\left\| \frac{\partial y^*(x_k)}{\partial x_k}\right\| =\left\|\nabla_x\nabla_y g(x_k,y^*(x_k))\left[\nabla_y^2 g(x_k,y^*(x_k))\right]^{-1}\right\|\leq\frac{L}{\mu}.
\end{align}
Then, combining~\cref{eq:zhaoposdo} and~\cref{ggpdas} 
yields 
\begin{align}\label{eq:ykdef}
\Big\|\frac{\partial y_k^t}{\partial x_k} -\frac{\partial y^*(x_k)}{\partial x_k} \Big\| \overset{(i)}\leq &\Big\| I-\alpha \nabla_y^2 g(x_k,y_k^{t-1})  \Big\| \Big\| \frac{\partial y_k^{t-1}}{\partial x_k} -\frac{\partial y^*(x_k)}{\partial x_k}\Big\|  \nonumber
\\&+\alpha\left( \tau+ \frac{L\rho}{\mu} \right)\|y_k^{t-1} -y^*(x_k)\| \nonumber
\\\overset{(ii)}\leq &(1-\alpha\mu) \Big\| \frac{\partial y_k^{t-1}}{\partial x_k} -\frac{\partial y^*(x_k)}{\partial x_k}\Big\|  +\alpha\left( \tau+ \frac{L\rho}{\mu} \right)\|y_k^{t-1} -y^*(x_k)\|,
\end{align}
where $(i)$ follows from Assumption \ref{high_lip} and $(ii)$ follows from the strong-convexity of $g(x,\cdot)$. Based on the strong-convexity of the lower-level function $g(x,\cdot)$, we have 
\begin{align}\label{eq:mindsa}
\|y_k^{t-1} -y^*(x_k)\| \leq (1-\alpha\mu)^{\frac{t-1}{2}} \|y^0_k-y^*(x_k)\|.
\end{align}
Substituting~\cref{eq:mindsa} into~\cref{eq:ykdef} and telecopting~\cref{eq:ykdef} over $t$ from $1$ to $D$, we have
\begin{align} \label{mamadewen}
\Big\|\frac{\partial y_k^D}{\partial x_k} -&\frac{\partial y^*(x_k)}{\partial x_k} \Big\| \leq (1-\alpha\mu)^{D}\Big\|\frac{\partial y_k^0}{\partial x_k} -\frac{\partial y^*(x_k)}{\partial x_k} \Big\|  \nonumber
\\&\hspace{2cm}+\alpha\left( \tau+ \frac{L\rho}{\mu} \right)\sum_{t=0}^{D-1} (1-\alpha\mu)^{D-1-t}(1-\alpha\mu)^{\frac{t}{2}} \|y^0_k-y^*(x_k)\| \nonumber
\\=&(1-\alpha\mu)^{D}\Big\|\frac{\partial y_k^0}{\partial x_k} -\frac{\partial y^*(x_k)}{\partial x_k} \Big\| 
+ \frac{2\left(  \tau\mu+ L\rho \right)}{\mu^2}(1-\alpha\mu)^{\frac{D-1}{2}}\|y^0_k-y^*(x_k)\| \nonumber
\\\leq &\frac{L(1-\alpha\mu)^D}{\mu} + \frac{2\left(  \tau\mu+ L\rho \right)}{\mu^2}(1-\alpha\mu)^{\frac{D-1}{2}}\|y^0_k-y^*(x_k)\|,
\end{align}
where the last inequality follows from $\frac{\partial y_k^0}{\partial x_k} =0$ and \cref{ggpdas}. Then, combining \cref{eq:woaijingii}, \cref{ggpdas}, \cref{eq:mindsa} and \cref{mamadewen} completes the proof.
\end{proof}
Based on the characterization on the estimation error of the gradient estimate $\frac{\partial f(x_k,y^D_k)}{\partial x_k}$ in \Cref{prop:partialG}, we now prove~\Cref{th:determin}. 

Recall the notation that $\widehat \nabla\Phi(x_k) =\frac{\partial f(x_k,y^D_k)}{\partial x_k}$.
Using an approach similar to \cref{eq:intimidern_pre}, we have 
\begin{align}\label{eq:intimidern} 
\Phi(x_{k+1}) \leq\Phi(x_k) -&\Big(\frac{\beta}{2}-\beta^2 L_\Phi \Big)\| \nabla \Phi(x_k)\|^2 \nonumber
\\& +\Big(\frac{\beta}{2}+\beta^2 L_\Phi\Big)\|\nabla\Phi(x_k)-\widehat \nabla\Phi(x_k)\|^2,
\end{align}
which, in conjunction with \Cref{prop:partialG} and using $\|y^0_k-y^*(x_k)\|^2\leq\Delta$, yields
\begin{align}\label{eq:pladwa}
\Phi(x_{k+1}) \leq & \Phi(x_k) -\Big(\frac{\beta}{2}-\beta^2 L_\Phi \Big)\| \nabla \Phi(x_k)\|^2 \nonumber
\\ &+3\Delta\Big(\frac{\beta}{2}+\beta^2 L_\Phi\Big)\Big( \frac{L^2(L+\mu)^2}{\mu^2} (1-\alpha\mu)^{D} +\frac{4M^2\left(  \tau\mu+ L\rho \right)^2}{\mu^4}(1-\alpha\mu)^{D-1} \Big) \nonumber
\\&+3\Big(\frac{\beta}{2}+\beta^2 L_\Phi\Big)\frac{L^2M^2(1-\alpha\mu)^{2D}}{\mu^2}.
\end{align}
Telescoping~\cref{eq:pladwa} over $k$ from $0$ to $K-1$ yields
\begin{align}\label{eq:xiangle}
\frac{1}{K}&\sum_{k=0}^{K-1}\Big(\frac{1}{2}-\beta L_\Phi \Big)\| \nabla \Phi(x_k)\|^2 \leq \frac{ \Phi(x_0)-\inf_x\Phi(x)}{\beta K} +3\Big(\frac{1}{2}+\beta L_\Phi\Big)\frac{L^2M^2(1-\alpha\mu)^{2D}}{\mu^2} \nonumber
\\+& 3\Delta\Big(\frac{1}{2}+\beta L_\Phi\Big)\Big( \frac{L^2(L+\mu)^2}{\mu^2} (1-\alpha\mu)^{D} +\frac{4M^2\left(  \tau\mu+ L\rho \right)^2}{\mu^4}(1-\alpha\mu)^{D-1} \Big). 
\end{align}
Substituting {\small$D=\log\Big(\max\big\{\frac{3LM}{\mu},9\Delta L^2(1+\frac{L}{\mu})^2,\frac{36\Delta M^2(\tau\mu+L\rho)^2}{(1-\alpha\mu)\mu^4}  \big\}\frac{9}{2\epsilon}\Big)/\log\frac{1}{1-\alpha\mu}=\Theta(\kappa\log\frac{1}{\epsilon})$} and $\beta=\frac{1}{4L_\Phi}$ in~\cref{eq:xiangle} yields
\begin{align}\label{holiugen}
\frac{1}{K}\sum_{k=0}^{K-1}\| \nabla \Phi(x_k)\|^2 \leq \frac{16 L_\Phi (\Phi(x_0)-\inf_x\Phi(x))}{K} + \frac{2\epsilon}{3}.
\end{align}
In order to achieve an $\epsilon$-accurate stationary point, we obtain from~\cref{holiugen} that 
ITD-BiO requires at most the total number $K=\mathcal{O}(\kappa^3\epsilon^{-1})$ of outer iterations. 
Then, based on the gradient form by \Cref{deter:gdform}, we have the following complexities.
\begin{itemize}
\item Gradient complexity: $$\mbox{\normalfont Gc}(f,\epsilon)=2K=\mathcal{O}(\kappa^3\epsilon^{-1}), \mbox{\normalfont Gc}(g,\epsilon)=KD=\mathcal{O}\left(\kappa^4\epsilon^{-1}\log\frac{1}{\epsilon}\right).$$
\item Jacobian- and Hessian-vector product complexities: $$ \mbox{\normalfont JV}(g,\epsilon)=KD=\mathcal{O}\left(\kappa^4\epsilon^{-1}\log\frac{1}{\epsilon}\right), \mbox{\normalfont HV}(g,\epsilon)=KD=\mathcal{O}\left(\kappa^4\epsilon^{-1}\log\frac{1}{\epsilon}\right).$$
\end{itemize}
Then, the proof is complete. 

\section{Proof of \Cref{th:meta_learning}}
To prove \Cref{th:meta_learning}, we first establish the following lemma to  characterize the estimation variance $\mathbb{E}_\gB\big\|\frac{\partial \gL_{\gD} (\phi_k,\widetilde w^D_k;\gB)}{\partial \phi_k} - \frac{\partial \gL_{\gD} (\phi_k,\widetilde w^D_k)}{\partial \phi_k} \big\|^2$, where $\widetilde w_k^{D}$ is the output of $D$ inner-loop steps of gradient descent at the $k^{th}$ outer loop. \begin{lemma}\label{le:bvarinace}
Suppose Assumptions~\ref{ass:lip} and \ref{high_lip} are satisfied and suppose each task loss $\gL_{\gS_i}(\phi,w_i)$ is $\mu$-strongly-convex w.r.t. $w_i$.  Then, we have
\begin{align*}
\mathbb{E}_\gB\Big\|\frac{\partial \gL_{\gD} (\phi_k, \widetilde w^D_k;\gB)}{\partial \phi_k} - \frac{\partial \gL_{\gD} (\phi_k, \widetilde w^D_k)}{\partial \phi_k} \Big\|^2 \leq \Big(1+\frac{L}{\mu}\Big)^2\frac{M^2}{|\gB|}.
\end{align*}
\end{lemma}
\begin{proof}
Let $\widetilde w_k^{D}=(w_{1,k}^D,...,w_{m,k}^D)$ be the output of $D$ inner-loop steps of gradient descent at the $k^{th}$ outer loop. 
Using \Cref{deter:gdform}, we have, for task $\gT_i$,
\begin{align}\label{maoxianssa}
\Big\|\leq & \|\nabla_\phi  \gL_{\gD_i}(\phi_k,w_{i,k}^D)\|\frac{\partial \gL_{\gD_i}(\phi_k,w^D_{i,k})}{\partial \phi_k} \Big\|\leq  \|\nabla_\phi  \gL_{\gD_i}(\phi_k,w_{i,k}^D)\| \nonumber
\\&+ \Big\|\alpha\sum_{t=0}^{D-1}\nabla_\phi\nabla_{w_i}  \gL_{\gS_i}(\phi_k,w_{i,k}^{t})\prod_{j=t+1}^{D-1}(I-\alpha  \nabla^2_{w_i}  \gL_{\gS_i}(\phi_k,w_{i,k}^{j}))\nabla_{w_i}  \gL_{\gD_i}(\phi_k,w_{i.k}^D)\Big\| \nonumber
\\\overset{(i)}\leq& M + \alpha LM\sum_{t=0}^{D-1} (1-\alpha\mu)^{D-t-1} = M+\frac{LM}{\mu},
\end{align} 
where $(i)$ follows from Assumptions~\ref{ass:lip} and strong-convexity of $ \gL_{\gS_i}(\phi,\cdot)$. Then, using the definition of $ \gL_{\gD} (\phi,\widetilde w;\gB) = \frac{1}{|\gB|}\sum_{i\in\gB}\gL_{\gD_i}(\phi,w_i)$, we have 
\begin{align}
\mathbb{E}_\gB\Big\|\frac{\partial \gL_{\gD} (\phi_k, \widetilde w^D_k;\gB)}{\partial \phi_k} - \frac{\partial \gL_{\gD} (\phi_k, \widetilde w^D_k)}{\partial \phi_k} \Big\|^2 =& \frac{1}{|\gB|}\mathbb{E}_i\Big\|\frac{\partial \gL_{\gD_i}(\phi_k,w^D_{i,k})}{\partial \phi_k}  - \frac{\partial \gL_{\gD} (\phi_k, \widetilde w^D_k)}{\partial \phi_k} \Big\|^2 \nonumber
\\\overset{(i)}\leq&\frac{1}{|\gB|} \mathbb{E}_i \Big\|\frac{\partial \gL_{\gD_i}(\phi_k,w^D_{i,k})}{\partial \phi_k} \Big\|^2 \nonumber
\\\overset{(ii)}\leq& \Big(1+\frac{L}{\mu}\Big)^2\frac{M^2}{|\gB|}.
\end{align}
where $(i)$ follows from $\mathbb{E}_i \frac{\partial \gL_{\gD_i}(\phi_k,w^D_{i,k})}{\partial \phi_k}=\frac{\partial \gL_{\gD} (\phi_k, \widetilde w^D_k)}{\partial \phi_k}  $ and $(ii)$ follows from~\cref{maoxianssa}. Then, the proof is complete.
\end{proof}

\begin{proof}[\bf Proof of~\Cref{th:meta_learning}] 
Recall $\Phi(\phi):=\gL_{\gD} (\phi,\widetilde w^{*}(\phi))$ be the objective function, and let $\widehat \nabla\Phi(\phi_k) = \frac{\partial \gL_{\gD} (\phi_k, \widetilde w^D_k)}{\partial \phi_k} $.
Using an approach similar to \cref{eq:intimidern}, we have 
\begin{align}\label{eq:starteq}
\Phi(\phi_{k+1}) &\leq \Phi(\phi_k)  + \langle \nabla \Phi(\phi_k), \phi_{k+1}-\phi_k\rangle + \frac{L_\Phi}{2} \|\phi_{k+1}-\phi_k\|^2 \nonumber
\\\leq& \Phi(\phi_k)  - \beta \Big\langle \nabla \Phi(\phi_k), \frac{\partial \gL_{\gD} (\phi_k, \widetilde w^D_k;\gB)}{\partial \phi_k}\Big\rangle + \frac{\beta^2L_\Phi}{2} \Big\| \frac{\partial \gL_{\gD} (\phi_k, \widetilde w^D_k;\gB)}{\partial \phi_k} \Big\|^2.
\end{align}
Taking the expectation of~\cref{eq:starteq}  yields
{\small
\begin{align}\label{eq:uijks}
\mathbb{E}\Phi(\phi_{k+1}) \overset{(i)}\leq& \mathbb{E}\Phi(\phi_k)  - \beta \mathbb{E}\big\langle \nabla \Phi(\phi_k),\widehat \nabla\Phi(\phi_k)\big\rangle +\frac{\beta^2L_\Phi}{2}\mathbb{E} \|\widehat \nabla\Phi(\phi_k)\|^2\nonumber
\\&+ \frac{\beta^2L_\Phi}{2}\mathbb{E} \Big\| \widehat \nabla\Phi(\phi_k)-\frac{\partial \gL_{\gD} (\phi_k, \widetilde w^D_k;\gB)}{\partial \phi_k} \Big\|^2  \nonumber
\\\overset{(ii)}\leq &\mathbb{E}\Phi(\phi_k)  - \beta \mathbb{E}\big\langle \nabla \Phi(\phi_k),\widehat \nabla\Phi(\phi_k)\big\rangle +\frac{\beta^2L_\Phi}{2}\mathbb{E} \|\widehat \nabla\Phi(\phi_k)\|^2 +\frac{\beta^2L_\Phi}{2}  \Big(1+\frac{L}{\mu}\Big)^2\frac{M^2}{|\gB|}\nonumber
\\\leq &\mathbb{E}\Phi(\phi_k) -\Big(\frac{\beta}{2}-\beta^2 L_\Phi \Big)\mathbb{E}\| \nabla \Phi(\phi_k)\|^2 +\Big(\frac{\beta}{2}+\beta^2 L_\Phi\Big)\mathbb{E}\|\nabla\Phi(\phi_k)-\widehat \nabla\Phi(\phi_k)\|^2 \nonumber
\\&+\frac{\beta^2L_\Phi}{2}  \Big(1+\frac{L}{\mu}\Big)^2\frac{M^2}{|\gB|},
\end{align}}
\hspace{-0.15cm}where $(i)$ follows from $\mathbb{E}_{\gB}\gL_{\gD} (\phi_k, \widetilde w^D_k;\gB)=\gL_{\gD} (\phi_k, \widetilde w^D_k)$ and $(ii)$ follows from~\Cref{le:bvarinace}. Using  \Cref{prop:partialG} in \cref{eq:uijks} and rearranging the terms, we have 
\begin{align*}
\frac{1}{K}\sum_{k=0}^{K-1}&\Big(\frac{1}{2}-\beta L_\Phi \Big)\mathbb{E}\| \nabla \Phi(\phi_k)\|^2   \nonumber
\\\leq& \frac{ \Phi(\phi_0)-\inf_\phi\Phi(\phi)}{\beta K} +3\Big(\frac{1}{2}+\beta L_\Phi\Big)\frac{L^2M^2(1-\alpha\mu)^{2D}}{\mu^2}+\frac{\beta L_\Phi}{2}  \Big(1+\frac{L}{\mu}\Big)^2\frac{M^2}{|\gB|} \nonumber
\\&+ 3\Delta\Big(\frac{1}{2}+\beta L_\Phi\Big)\Big( \frac{L^2(L+\mu)^2}{\mu^2} (1-\alpha\mu)^{D} +\frac{4M^2\left(  \tau\mu+ L\rho \right)^2}{\mu^4}(1-\alpha\mu)^{D-1} \Big),
\end{align*}
where $\Delta=\max_{k}\|\widetilde w^0_k-\widetilde w^*(\phi_k)\|^2<\infty$. Choose the same parameters $\beta,D$ as in~\Cref{th:determin}. Then, we have
\begin{align*}
\frac{1}{K}\sum_{k=0}^{K-1}\mathbb{E}\| \nabla \Phi(\phi_k)\|^2 \leq \frac{16 L_\Phi (\Phi(\phi_0)-\inf_\phi\Phi(\phi))}{K} + \frac{2\epsilon}{3}+ \Big(1+\frac{L}{\mu}\Big)^2\frac{M^2}{8|\gB|}.
\end{align*}
Then, the proof is complete. 
\end{proof}

\chapter{Proof of \Cref{chp_acc_bilevel}}\label{appendix:acc_bilevel}

\section{Proof of \Cref{upper_srsr_withnoB}}\label{proof:upss} 
To simplify the notations, we define several quantities as below. 
{\small\begin{align}\label{pf:ntationscs}
\mathcal{M}_k =& \|y^*(x^*)\|+ \frac{\widetilde L_{xy}}{\mu_y}\|x_k-x^*\|, \;\;\mathcal{N}_k = \|\nabla_y f( x^*,y^*(x^*))\|+ \Big(L_{xy}+\frac{L_y\widetilde L_{xy}}{\mu_y}\Big)\|x_k-x^*\| \nonumber
\\\mathcal{M}_* =& \|y^*(x^*)\|+ \frac{3\widetilde L_{xy}}{\mu_y}\sqrt{\frac{2}{\mu_x}(\Phi(0) -\Phi(x^*))+ \|x^*\|^2+\frac{\epsilon}{\mu_x}}
 \nonumber
\\\mathcal{N}_* =&\|\nabla_y f( x^*,y^*(x^*))\|+ 3\Big(L_{xy}+\frac{L_y\widetilde L_{xy}}{\mu_y}\Big)\sqrt{\frac{2}{\mu_x}(\Phi(0) -\Phi(x^*))+ \|x^*\|^2+\frac{\epsilon}{\mu_x}}, 
\end{align}}
\hspace{-0.12cm}where $\mathcal{M}_k,\mathcal{N}_k$ changes with the optimality gap $\|x_k-x^*\|$ at the $k^{th}$ iteration and $\mathcal{M}_*,\mathcal{N}_*$ are two positive constants depending on the information of the objective function at the optimal point $x^*$.
We first establish the following lemma to 
upper-bound the hypergradient estimation error $\|\nabla\Phi( x_k)-G_k\|$. 
\begin{lemma}\label{le:hgestr}
Let $G_k$ be the hypergradient estimator used in \Cref{alg:bioNoBG} at iteration $k$. Then, we have 
\begin{align}\label{eq:hgesterr}
\|G_k-\nabla \Phi( x_k)\| \leq &\sqrt{\frac{\widetilde L_y +\mu_y}{\mu_y}} \Big(L_y +\frac{2\widetilde L_{xy}L_y}{\mu_y} +\Big(\frac{\rho_{xy}}{\mu_y}+\frac{\widetilde L_{xy}\rho_{yy}}{\mu_y^2}\Big)\mathcal{N}_k\Big) \mathcal{M}_k \exp\Big(-\frac{N}{2\sqrt{\kappa_y}}\Big) \nonumber
\\&+\frac{\widetilde L_{xy}}{\mu_y}\Big(\frac{\sqrt{\kappa_y}-1}{\sqrt{\kappa_y}+1}\Big)^M\mathcal{N}_k,
\end{align}
where the quantities $\mathcal{M}_k$ and $\mathcal{N}_k$ are defined in \cref{pf:ntationscs}. 
\end{lemma}
\Cref{le:hgestr} shows that the estimation error $\|\nabla\Phi( x_k)-G_k\|$ is bounded given that the optimality gap $\|x_k-x^*\|$ is bounded. We will show in the proof of \Cref{upper_srsr_withnoB} that $\|x_k-x^*\|$ is bounded as the algorithm runs due to the strongly-convex geometry of the objective function $\Phi(x)$. In addition, it can be seen that this error decays exponentially with respect to the number $N$ of inner-level steps and the number $M$ of steps of heavy-ball method for solving the linear system in  \Cref{alg:bioNoBG}. 
Then, to prove the convergence of \Cref{alg:bioNoBG}, we set $N,M=c\sqrt{\kappa_y}\log (\kappa_y)$ in the proof of \Cref{upper_srsr_withnoB}, where $c$ is a constant independent of $\kappa_y$.

\begin{proof}
Recall from line $7$ of \Cref{alg:bioNoBG} that 
\begin{align}\label{laomuzhu}
G_k:= \nabla_x f(x_k,y_k^N) -\nabla_x \nabla_y g( x_k,y_k^N)v_k^M,
\end{align}
where $v_k^M$ is the output of $M$-steps of heavy-ball method for solving $$\min_v Q(v):=\frac{1}{2}v^T\nabla_y^2 g(x_k,y_k^N) v - v^T
\nabla_y f( x_k,y^N_k).$$
Recall the smoothness parameter $\widetilde L_y$ of $g(x,\cdot)$ defined in Assumption~\ref{fg:smooth}. 
Then, based on the convergence result of heavy-ball method in~\cite{badithela2019analysis} with stepsizes $\lambda=\frac{4}{(\sqrt{\widetilde L_y}+\sqrt{\mu_y})^2}$ and $\theta=\max\big\{\big(1-\sqrt{\lambda\mu_y}\big)^2,\big(1-\sqrt{\lambda\widetilde L_y}\big)^2\big\}$ and noting that $v_k^0=v_k^1=0$, we have 
\begin{align}\label{gg:worimass}
\|v_k^M - \nabla_y^2 &g(x_k,y_k^N)^{-1}\nabla_y f( x_k,y^N_k) \| \nonumber
\\\leq  &\Big(\frac{\sqrt{\kappa_y}-1}{\sqrt{\kappa_y}+1}\Big)^M \Big\| \big(\nabla_y^2 g(x_k,y_k^N)\big)^{-1}\nabla_y f(x_k,y^N_k)\Big\| \nonumber
\\\leq &\frac{L_y}{\mu_y}\Big(\frac{\sqrt{\kappa_y}-1}{\sqrt{\kappa_y}+1}\Big)^M \|y^*(x_k)-y_k^N\|  + \frac{\|\nabla_y f( x_k,y^*(x_k))\|}{\mu_y}\Big(\frac{\sqrt{\kappa_y}-1}{\sqrt{\kappa_y}+1}\Big)^M \nonumber
\\\overset{(i)}\leq &\frac{L_y}{\mu_y} \|y^*(x_k)-y_k^N\|  + \frac{\|\nabla_y f( x_k,y^*(x_k))\|}{\mu_y}\Big(\frac{\sqrt{\kappa_y}-1}{\sqrt{\kappa_y}+1}\Big)^M
\end{align}
where $y^*(x_k)=\argmin_{y\in\mathbb{R}^q} g( x_k,y)$ and $(i)$ follows from $\frac{\sqrt{\kappa_y}-1}{\sqrt{\kappa_y}+1}\leq 1$. 
 Then, based on the forms of $G_k$ and $\nabla\Phi(x)$ in \cref{laomuzhu} and \cref{hyperG}, and using Assumptions~\ref{fg:smooth} and~\ref{g:hessiansJaco}, we have 
 {\small
\begin{align}\label{jingyikeai}
\|G_k&-\nabla \Phi(x_k)\|\nonumber
\\\overset{(i)}\leq & \| \nabla_x f( x_k,y_k^N) -\nabla_x f( x_k,y^*(x_k))\| + \widetilde L_{xy}\|v_k^M- \nabla_y^2 g(x_k,y^*(x_k))^{-1}\nabla_y f(x_k,y^*(x_k)) \|  \nonumber
\\&+\frac{\|\nabla_y f( x_k,y^*(x_k)) \|}{\mu_y}  \|\nabla_x \nabla_y g( x_k,y_k^N)-\nabla_x \nabla_y g( x_k,y^*(x_k))\| \nonumber
\\\leq & L_y \|y^*(x_k)-y_k^N\| + \widetilde L_{xy} \|v_k^M- \nabla_y^2 g(x_k,y_k^N)^{-1}\nabla_y f( x_k,y^N_k) \| \nonumber
\\&+ \widetilde L_{xy}\big\|\nabla_y^2 g(x_k,y_k^N)^{-1}\nabla_y f( x_k,y^N_k)-\nabla_y^2 g(x_k,y^*(x_k))^{-1}\nabla_y f(x_k,y^*(x_k)) \big\|\nonumber
\\&+\frac{\rho_{xy}}{\mu_y} \|y_k^N-y^*(x_k)\| \|\nabla_y f( x_k,y^*(x_k)) \|\nonumber
\\\leq & \Big(L_y +\frac{\widetilde L_{xy}L_y}{\mu_y} +\frac{\rho_{xy}}{\mu_y}\|\nabla_y f( x_k,y^*(x_k)) \|\Big)\|y_k^N-y^*(x_k)\|  \nonumber
\\&+ \frac{\widetilde L_{xy}\rho_{yy}\|y_k^N-y^*(x_k)\|}{\mu_y^2}\|\nabla_y f( x_k,y^*(x_k)) \| + \widetilde L_{xy} \|v_k^M- \nabla_y^2 g(x_k,y_k^N)^{-1}\nabla_y f( x_k,y^N_k) \|  \nonumber
\\\overset{(ii)}\leq&\Big(L_y +\frac{2\widetilde L_{xy}L_y}{\mu_y} +\Big(\frac{\rho_{xy}}{\mu_y}+\frac{\widetilde L_{xy}\rho_{yy}}{\mu_y^2}\Big)\|\nabla_y f( x_k,y^*(x_k)) \|\Big)\|y_k^N-y^*(x_k)\|  \nonumber
\\&+\frac{\widetilde L_{xy}}{\mu_y}\left(\frac{\sqrt{\kappa_y}-1}{\sqrt{\kappa_y}+1}\right)^M\|\nabla_y f( x_k,y^*(x_k))\|,
\end{align}}
\hspace{-0.14cm}where $(i)$ follows from Assumption~\ref{fg:smooth} that $\|\nabla_x\nabla_y g(\cdot,\cdot)\|\leq \widetilde L_{xy}$ and $\|(\nabla_y^2 g(\cdot,\cdot))^{-1}\|\leq \frac{1}{\mu_y}$ and $(ii)$ follows from \cref{gg:worimass}. Note that $y_k^N$ is obtained using $N$ steps of AGD for minimizing the inner-level loss function $g(x_k,\cdot)$ and recall $y^*(x_k)=\argmin_{y\in\mathbb{R}^q} g( x_k,y)$. Then, based on the analysis in \cite{nesterov2003introductory} for AGD, we have 
\begin{align}\label{eq:ideazhiqian}
\|y_k^N-y^*(x_k)\|\leq &\sqrt{ \frac{\widetilde L_y +\mu_y}{\mu_y} }\|y_k^0-y^*(x_k)\| \exp\Big(-\frac{N}{2\sqrt{\kappa_y}}\Big) \nonumber
\\\leq & \sqrt{\frac{\widetilde L_y +\mu_y}{\mu_y}} \Big(\|y^*(x^*)\| + \frac{\widetilde L_{xy}}{\mu_y}\|x_k-x^*\|\Big) \exp\Big(-\frac{N}{2\sqrt{\kappa_y}}\Big), 
\end{align}
where $x^*=\argmin_{x\in\mathbb{R}^p}\Phi(x)$. 
Moreover, based on Lemma 2.2 in~\cite{ghadimi2018approximation}, we have $\|y^*(x_1)-y^*(x_2)\|\leq \frac{\widetilde L_{xy}}{\mu_y} \|x_1-x_2\|$ for any $x_1,x_2\in\mathbb{R}^p$, and hence 
\begin{align}\label{maoxinadaoxx}
\|\nabla_y &f( x_k,y^*(x_k))\| \leq \|\nabla_y f( x^*,y^*(x^*))\| + \Big( L_{xy}+\frac{L_y\widetilde L_{xy}}{\mu_y}\Big)\|x_k-x^*\|.
\end{align}
Substituting \cref{eq:ideazhiqian}  and \cref{maoxinadaoxx} into \cref{jingyikeai}, and using the definition of  $\mathcal{M}_k$ and $\mathcal{N}_k$ in \cref{pf:ntationscs}, we finish the proof. 
\end{proof}
We then establish the following lemma to characterize the smoothness parameter of the objective function $\Phi(x)$ around the iterate $x_k$. 
Recall from \cref{hyperG} that $\nabla \Phi(x)$ is given by 
\begin{align}\label{hgforms}
\nabla \Phi(x) =  \nabla_x f(x,y^*(x)) -\nabla_x \nabla_y g(x,y^*(x)) [\nabla_y^2 g(x,y^*(x)) ]^{-1}\nabla_y f(x,y^*(x)),
\end{align}
where $y^*(x)=\argmin_{y}g(x,\cdot)$ be the minimizer of the inner-level function $g(x,\cdot)$. 
\begin{lemma}\label{smoothness_Phis}
Consider the hypergradient $\nabla \Phi(x)$ given by \cref{hgforms}. For any $x\in\mathbb{R}^p$, we have
\begin{align}\label{maoxizaosscsa}
\|\nabla &\Phi(x)  - \nabla \Phi(x_k)\| \nonumber
\\&\leq \Big(\underbrace{ L_x+\frac{2L_{xy}\widetilde L_{xy}}{\mu_y} + \frac{L_y\widetilde L^2_{xy}}{\mu_y^2}+\Big(\frac{\widetilde L_{xy} \rho_{yy}}{\mu_y^2} +  \frac{\rho_{xy}}{\mu_y}\Big) \Big( 1+\frac{\widetilde L_{xy}}{\mu_y}  \Big)\mathcal{N}_k}_{L_{\Phi_k}}\Big)\|x-x_k\|, \end{align}
where $\mathcal{N}_k$ is defined in \cref{pf:ntationscs}. Furthermore, \cref{maoxizaosscsa} implies that, for any $x\in\mathbb{R}^p$,  
\begin{align}
\Phi(x)\leq \Phi(x_k)+\langle \nabla \Phi(x_k),x-x_k\rangle +  \frac{L_{\Phi_k}}{2}\|x-x_k\|^2.
\end{align}
\end{lemma}
\Cref{smoothness_Phis} shows that $\nabla\Phi(x)$ is Lipschitz continuous around the iterate $x_k$, i.e., smooth, where the smoothness parameter $L_{\Phi_k}$ contains a term proportional to $\|x_k-x^*\|$. 
We will show in the proof of \Cref{upper_srsr_withnoB} that optimality distance $\|x_k-x^*\|$ is bounded as the algorithm runs, and hence the smoothness parameter $L_{\Phi_k}$ is bounded by $\mathcal{O}(\frac{1}{\mu_y^3})$ during the entire process. 
\begin{proof}
Based on the form of $\nabla \Phi(x)$ in \cref{hgforms}, we have 
{\footnotesize
\begin{align*}
\|\nabla &\Phi(x)  - \nabla \Phi(x_k)\|  
\\\leq& \| \nabla_x f(x,y^*(x)) - \nabla_x f(x_k,y^*(x_k)) \|+\frac{\widetilde L_{xy}}{\mu_y} \|\nabla_y f(x,y^*(x))-\nabla_y f(x_k,y^*(x_k))\| 
\\&+ \underbrace{\|\nabla_x \nabla_y g(x,y^*(x)) \nabla_y^2 g(x,y^*(x))^{-1}-\nabla_x \nabla_y g(x_k,y^*(x_k)) \nabla_y^2 g(x_k,y^*(x_k))^{-1}\|}_{P}\|\nabla_y f(x_k,y^*(x_k))\|,
\end{align*}}
\hspace{-0.13cm}which, in conjunction with the inequality
\begin{align*}
P \leq &\frac{\widetilde L_{xy} \rho_{yy}}{\mu_y^2} (\|x-x_k\| + \|y^*(x) -y^*(x_k)\|) + \frac{\rho_{xy}}{\mu_y}(\|x-x_k\| + \|y^*(x) -y^*(x_k)\|) 
\\\overset{(i)}\leq  & \Big(\frac{\widetilde L_{xy} \rho_{yy}}{\mu_y^2} +  \frac{\rho_{xy}}{\mu_y}\Big) \Big( 1+\frac{\widetilde L_{xy}}{\mu_y}  \Big) \|x-x_k\|,
\end{align*}
and using Assumption~\ref{fg:smooth}, yields 
\begin{align}\label{eq:particsasq}
\|\nabla &\Phi(x)  - \nabla \Phi(x_k)\|  \nonumber
\\\leq& \Big( L_x+\frac{2L_{xy}\widetilde L_{xy}}{\mu_y} + \frac{L_y\widetilde L^2_{xy}}{\mu_y^2}\Big)\|x-x_k\|  \nonumber
\\&\hspace{2cm}+ \Big(\frac{\widetilde L_{xy} \rho_{yy}}{\mu_y^2} +  \frac{\rho_{xy}}{\mu_y}\Big) \Big( 1+\frac{\widetilde L_{xy}}{\mu_y}  \Big)\|\nabla_y f(x_k,y^*(x_k))\| \|x-x_k\|,
\end{align}
where $(i)$ follows from the $\frac{\widetilde L_{xy}}{\mu_y}$-smoothness of $y^*(\cdot)$. 
Substituting \cref{maoxinadaoxx} into \cref{eq:particsasq} and using the definition of $\mathcal{N}_k$ in \cref{pf:ntationscs}, we have
\begin{align}\label{ijj:ntidesaca}
\|&\nabla \Phi(x)  - \nabla \Phi(x_k)\| \nonumber
\\&\leq \Big(\underbrace{ L_x+\frac{2L_{xy}\widetilde L_{xy}}{\mu_y} + \frac{L_y\widetilde L^2_{xy}}{\mu_y^2}+\Big(\frac{\widetilde L_{xy} \rho_{yy}}{\mu_y^2} +  \frac{\rho_{xy}}{\mu_y}\Big) \Big( 1+\frac{\widetilde L_{xy}}{\mu_y}  \Big)\mathcal{N}_k}_{L_{\Phi_k}}\Big)\|x-x_k\|. 
\end{align}
Based on \cref{ijj:ntidesaca}, we further obtain
\begin{align*}
|\Phi(x)-\Phi(x_k)&-\langle \nabla \Phi(x_k),x-x_k\rangle| 
\\=& \Big|\int_0^1 \langle \nabla\Phi(x_k+t(x-x_k)),x-x_k \rangle d t-\langle \nabla \Phi(x_k),x-x_k\rangle\Big| \nonumber
\\\leq& \Big| \int_0^1 \langle \nabla\Phi(x_k+t(x-x_k))-\nabla \Phi(x_k),x-x_k \rangle d t \Big| \nonumber
\\\leq & \Big| \int_0^1 \| \nabla\Phi(x_k+t(x-x_k))-\nabla \Phi(x_k)\| \|x-x_k\| d t \Big| \nonumber
\\\overset{(i)}\leq &\Big| \int_0^1 L_{\Phi_k}\|(x-x_k)\|^2 td t \Big| = \frac{L_{\Phi_k}}{2}\|x-x_k\|^2.
\end{align*}
Then, the proof is now complete. 
\end{proof}
Based on \Cref{le:hgestr} and \Cref{smoothness_Phis}, we are ready to prove \Cref{upper_srsr_withnoB}. 
\begin{proof}[{\bf Proof of \Cref{upper_srsr_withnoB}}]
 \Cref{alg:bioNoBG} conduct the following updates 
\begin{align}\label{updateRule}
z_{k+1}=&x_k -\frac{1}{L_\Phi} G_k, \nonumber
\\x_{k+1}=&\Big(1+\frac{\sqrt{\kappa_x}-1}{\sqrt{\kappa_x}+1}\Big)z_{k+1} - \frac{\sqrt{\kappa_x}-1}{\sqrt{\kappa_x}+1} z_k,
\end{align}
where the smoothness parameter $L_{\Phi}$ takes the form of 
{\small\begin{align}\label{wodishenamxd}
 L_{\Phi} =& L_x+\frac{2L_{xy}\widetilde L_{xy}}{\mu_y} + \frac{L_y\widetilde L^2_{xy}}{\mu_y^2}  +\Big(\frac{\widetilde L_{xy} \rho_{yy}}{\mu_y^2} +  \frac{\rho_{xy}}{\mu_y}\Big) \Big( 1+\frac{\widetilde L_{xy}}{\mu_y}  \Big)\|\nabla_y f( x^*,y^*(x^*))\|\nonumber
 \\&+3\Big(\frac{\widetilde L_{xy} \rho_{yy}}{\mu_y^2} +  \frac{\rho_{xy}}{\mu_y}\Big) \Big( 1+\frac{\widetilde L_{xy}}{\mu_y}  \Big)\Big(L_{xy}+\frac{L_y\widetilde L_{xy}}{\mu_y}\Big)\sqrt{\frac{2}{\mu_x}(\Phi(0) -\Phi(x^*))+ \|x^*\|^2+\frac{\epsilon}{\mu_x}} \nonumber
 \\= &\Theta\Big(\frac{1}{\mu_y^2}+\Big(\frac{ \rho_{yy}}{\mu_y^3} +  \frac{\rho_{xy}}{\mu_y^2}\Big)\Big(\|\nabla_y f( x^*,y^*(x^*))\|+\frac{\|x^*\|}{\mu_y}+\frac{\sqrt{\Phi(0)-\Phi(x^*)}}{\sqrt{\mu_x}\mu_y}\Big)  \Big)
\end{align}}
\hspace{-0.15cm}and  
$\kappa_x = \frac{L_\Phi}{\mu_x}$ is the condition number of the objective function $\Phi(x)$. 

The remaining proof is based on the modification of the results in Section 2.2.5 of \cite{nesterov2018lectures}. The key differences here are that  we need to prove the boundedness of the iterates  as the algorithm runs, and carefully handle the hypergradient estimation error in the convergence analysis for accelerated gradient methods. In specific, we first need to construct the estimate sequences as follows. 
\begin{align}\label{ssses_seq}
S_0(x) =& \Phi(x_0) +\frac{\mu_x}{2} \|x-x_0\|^2 \nonumber
\\S_{k+1}(x) = & \Big(1 -\frac{1}{\sqrt{\kappa_x}} \Big)S_{k}(x)  \nonumber
\\&+\frac{1}{\sqrt{\kappa_x}} \Big( \Phi(x_k) +\langle G_k,x-x_k\rangle + \frac{\mu_x}{2}\|x - x_k\|^2 +  \frac{\epsilon}{4}\Big).
\end{align}
Note that $\nabla^2S_0(x) = \mu_x I $  and $\nabla^2 S_{k+1}(x) = \big(1 -\frac{1}{\sqrt{\kappa_x}} \big)\nabla^2 S_{k}(x)+\frac{ \mu_x}{\sqrt{\kappa_x}}I$. Then, by induction, it can be verified that $\nabla^2 S_{k}(x)=\mu_x I$ for all $k=0,...,K$. This implies that $S_k(x)$ can be written as $S_k(x) = S_k^* + \frac{\mu_x}{2}\|x-v_k\|^2$, where $v_k = \argmin_{x\in\mathbb{R}^p}S_k(x)$. Next, we show by induction that
\begin{align}
&1.\quad \|z_k-x^*\|\leq \sqrt{\frac{2}{\mu_x}(\Phi(0) -\Phi(x^*))+ \|x^*\|^2+\frac{\epsilon}{\mu_x}} \text{ for all } k=0,...,K. \label{pocasca}
\\&2.\quad S_k^*\geq  \Phi(z_k) \text{ for all } k=0,...,K. \label{dasimaomaoss}
\end{align} 
Combining the first item~\cref{pocasca} above with the updates~\cref{updateRule} also implies the boundedness of sequence $x_k,k=0,...,K$ by noting that 
\begin{align}\label{boundednessofx_k}
\|x_{k}-x^*\| \leq &\Big(1+\frac{\sqrt{\kappa_x}-1}{\sqrt{\kappa_x}+1}\Big)\|z_{k} -x^*\|+ \frac{\sqrt{\kappa_x}-1}{\sqrt{\kappa_x}+1} \|z_{k-1}-x^*\| \nonumber
\\ \leq& 3\sqrt{\frac{2}{\mu_x}(\Phi(0) -\Phi(x^*))+ \|x^*\|^2+\frac{\epsilon}{\mu_x}}. 
\end{align}
Next, we prove the above two items \cref{pocasca} and \cref{dasimaomaoss} by induction. First, it can be verified that they hold for $k=0$ by noting that $\|z_0-x^*\|=\|x^*\|$ and $S_0^* = \Phi(x_0)$. Then, we suppose they hold for all $k=0,...,k^\prime$ and prove the $k^\prime+1$ case. 

Based on \Cref{smoothness_Phis}, we have, for all $k=0,...,k^\prime$,  
\begin{align}\label{dotahaonanss}
\Phi(z_{k+1})\leq& \Phi(x_k)+\langle \nabla \Phi(x_k),z_{k+1}-x_k\rangle +  \frac{L_{\Phi_k}}{2}\|z_{k+1}-x_k\|^2 \nonumber
\\\overset{(i)}=& \Phi(x_k)-\frac{1}{L_\Phi}\langle \nabla \Phi(x_k),G_k\rangle +  \frac{L_{\Phi_k}}{2L_\Phi^2}\|G_k\|^2,  
\end{align}
 where $(i)$ follows from the updates in~\cref{updateRule}. Note that for $k=0,...,k^\prime$, it is seen from \cref{boundednessofx_k} that the optimality gap  $\|x_k-x^*\|\leq 3\sqrt{\frac{2}{\mu_x}(\Phi(0) -\Phi(x^*))+ \|x^*\|^2+\frac{\epsilon}{\mu_x}}$, which, combined with the definition of $L_{\Phi_k}$ in \cref{maoxizaosscsa}, 
 yields $L_{\Phi_k} \leq L_\Phi$ for all $k=0,...,k^\prime$, where $L_\Phi$ is given by \cref{wodishenamxd}. Then, we obtain from \cref{dotahaonanss} that for all $k=0,...,k^\prime$, 
 \begin{align}\label{mxdhuaqingduo}
 \Phi(z_{k+1})\leq& \Phi(x_k)-\frac{1}{L_\Phi}\langle \nabla \Phi(x_k),G_k\rangle +  \frac{1}{2L_\Phi}\|G_k\|^2 \nonumber
 \\=& \Phi(x_k)-\frac{1}{L_\Phi}\| \nabla \Phi(x_k)\|^2 - \frac{1}{L_\Phi} \langle \nabla \Phi(x_k),G_k- \nabla \Phi(x_k)\rangle +  \frac{1}{2L_\Phi}\|G_k\|^2 \nonumber
 \\=& \Phi(x_k)-\frac{1}{L_\Phi}\| \nabla \Phi(x_k)\|^2 + \frac{1}{2L_\Phi} \|\nabla \Phi(x_k)\|^2 +  \frac{1}{2L_\Phi} \|G_k-\nabla \Phi(x_k)\|^2 \nonumber
 \\=& \Phi(x_k)-\frac{1}{2L_\Phi}\| \nabla \Phi(x_k)\|^2 +  \frac{1}{2L_\Phi} \|G_k-\nabla \Phi(x_k)\|^2,
 \end{align}
 which, in conjunction with the strong convexity of $\Phi(\cdot)$,  yields
 \begin{align}\label{wodialayacs}
 \Phi(z_{k+1})\leq &\Big( 1- \frac{1}{\sqrt{\kappa_x}}\Big) \Phi(z_k) +\Big( 1- \frac{1}{\sqrt{\kappa_x}}\Big) \langle \nabla\Phi(x_k),x_k-z_k\rangle + \frac{1}{\sqrt{\kappa_x}}\Phi(x_k) \nonumber
 \\&-\frac{1}{2L_\Phi}\| \nabla \Phi(x_k)\|^2 +  \frac{1}{2L_\Phi} \|G_k-\nabla \Phi(x_k)\|^2 \nonumber
 \\\overset{(i)}\leq &\Big( 1- \frac{1}{\sqrt{\kappa_x}}\Big) S_k^* +\Big( 1- \frac{1}{\sqrt{\kappa_x}}\Big) \langle \nabla\Phi(x_k),x_k-z_k\rangle + \frac{1}{\sqrt{\kappa_x}}\Phi(x_k) \nonumber
 \\&-\frac{1}{2L_\Phi}\| \nabla \Phi(x_k)\|^2 +  \frac{1}{2L_\Phi} \|G_k-\nabla \Phi(x_k)\|^2,
 \end{align}
 where $(i)$ follows from $S_k^*\geq  \Phi(z_k)$ for $k=0,...,k^\prime$. Next, based on the definition of $S_k(x)$ in \cref{ssses_seq} and taking derivative w.r.t.~$x$ on both sides of  \cref{ssses_seq}, we have 
 \begin{align}\label{gg_smidasssca}
 \nabla S_{k+1}(x) &\overset{(i)}= \Big( 1- \frac{1}{\sqrt{\kappa_x}} \Big)\nabla S_k(x) + \frac{1}{\sqrt{\kappa_x}} G_k +  \frac{\mu_x}{\sqrt{\kappa_x}}(x-x_k) \nonumber
 \\&=\mu_x\Big( 1- \frac{1}{\sqrt{\kappa_x}} \Big)(x-v_k)+\frac{1}{\sqrt{\kappa_x}} G_k +  \frac{\mu_x}{\sqrt{\kappa_x}}(x-x_k), 
 \end{align}
 where $(i)$ follows from $S_k(x) = S_k^* + \frac{\mu_x}{2}\|x-v_k\|^2$. Noting that $\nabla S_{k+1}(v_{k+1})= 0$, we obtain from  \cref{gg_smidasssca} that 
 \begin{align*}
 \mu_x\Big( 1- \frac{1}{\sqrt{\kappa_x}} \Big)(v_{k+1}-v_k)+\frac{1}{\sqrt{\kappa_x}} G_k +  \frac{\mu_x}{\sqrt{\kappa_x}}(v_{k+1}-x_k) = 0,
 \end{align*}
 which yields 
 \begin{align}\label{mamamiyasscas}
 v_{k+1} = \Big(1-\frac{1}{\sqrt{\kappa_x}} \Big)v_k + \frac{1}{\sqrt{\kappa_x}} x_k - \frac{1}{\mu_x\sqrt{\kappa_x}} G_k.
 \end{align}
 Based on \cref{ssses_seq} and using $S_k(x) = S_k^* + \frac{\mu_x}{2}\|x-v_k\|^2$,  we have 
 \begin{align*}
S_{k+1}^*  + \frac{\mu_x}{2}\|x_k-v_{k+1}\|^2 = \Big(1-\frac{1}{\sqrt{\kappa_x}} \Big)\Big(S_{k}^*  + \frac{\mu_x}{2}\|x_k-v_{k}\|^2 \Big) + \frac{1}{\sqrt{\kappa_x}}\Phi(x_k) + \frac{\epsilon}{4\sqrt{\kappa_x}},
 \end{align*}
 which, in conjunction with \cref{mamamiyasscas}, yields
 \begin{align}\label{xingyunwos}
 S_{k+1}^*  =&  \Big(1-\frac{1}{\sqrt{\kappa_x}} \Big) S_{k}^*  +  \Big(1-\frac{1}{\sqrt{\kappa_x}} \Big)  \frac{\mu_x}{2}\|x_k-v_k\|^2 + \frac{1}{\sqrt{\kappa_x}}\Phi(x_k) + \frac{\epsilon}{4\sqrt{\kappa_x}} \nonumber
 \\&- \Big(1-\frac{1}{\sqrt{\kappa_x}} \Big)^2\frac{\mu_x}{2}\|x_k-v_k\|^2 - \frac{1}{2\mu_x\kappa_x}\|G_k\|^2 +  \Big(1-\frac{1}{\sqrt{\kappa_x}} \Big)\frac{1}{\sqrt{\kappa_x}}\langle v_k-x_k,G_k \rangle \nonumber
 \\ = & \Big(1-\frac{1}{\sqrt{\kappa_x}} \Big) S_{k}^*  + \Big(1-\frac{1}{\sqrt{\kappa_x}} \Big) \frac{1}{\sqrt{\kappa_x}} \frac{\mu_x}{2}\|x_k-v_k\|^2 +\frac{1}{\sqrt{\kappa_x}}\Phi(x_k) + \frac{\epsilon}{4\sqrt{\kappa_x}}\nonumber
 \\&- \frac{1}{2\mu_x\kappa_x}\|G_k\|^2 +  \Big(1-\frac{1}{\sqrt{\kappa_x}} \Big)\frac{1}{\sqrt{\kappa_x}}\langle v_k-x_k,G_k \rangle.
 \end{align}  
 Based on the definition of $\kappa_x$, we simplify \cref{xingyunwos} to 
 \begin{align}\label{yifenyimiaos}
  S_{k+1}^* \geq & \Big(1-\frac{1}{\sqrt{\kappa_x}} \Big) S_{k}^* +\frac{1}{\sqrt{\kappa_x}}\Phi(x_k) + \frac{\epsilon}{4\sqrt{\kappa_x}}- \frac{1}{2L_\Phi}\|G_k\|^2  \nonumber
  \\&+  \Big(1-\frac{1}{\sqrt{\kappa_x}} \Big)\frac{1}{\sqrt{\kappa_x}}\langle v_k-x_k,G_k \rangle.
 \end{align}
 Next, we prove $v_k-x_k = \sqrt{\kappa_x}(x_k-z_k)$ by induction. First note that this equality holds for $k=0$ based on the fact that $v_0-x_0=\sqrt{\kappa_x}(x_0-z_0)=0$. Then, suppose that it holds for the $k$ case, and for the $k+1$ case, we obtain from \cref{mamamiyasscas} that 
 \begin{align*}
  v_{k+1}& -x_{k+1}
  \\= &\Big(1-\frac{1}{\sqrt{\kappa_x}} \Big)v_k + \frac{1}{\sqrt{\kappa_x}} x_k-x_{k+1} - \frac{1}{\mu_x\sqrt{\kappa_x}} G_k \nonumber
  \\ \overset{(i)}=& \Big(1-\frac{1}{\sqrt{\kappa_x}} \Big)\Big(1+\sqrt{\kappa_x} \Big) x_k -\Big(1-\frac{1}{\sqrt{\kappa_x}} \Big)\sqrt{\kappa_x}z_k +\frac{1}{\sqrt{\kappa_x}} x_k -x_{k+1} -\frac{1}{\mu_x\sqrt{\kappa_x}} G_k \nonumber
  \\=& \sqrt{\kappa_x} \Big(x_k-\frac{1}{L_\Phi}G_k\Big) -(\sqrt{\kappa_x}-1) z_k -x_{k+1} \nonumber
  \\\overset{(ii)}=&\sqrt{\kappa_x} (x_{k+1}-z_{k+1}),
 \end{align*}
 where $(i)$ follows from $v_k-x_k = \sqrt{\kappa_x}(x_k-z_k)$ and $(ii)$ follows from the updating step in \cref{updateRule}. Then, by induction, we have $v_k-x_k = \sqrt{\kappa_x}(x_k-z_k)$ holds for all $k=0,...,K$. Combining this equality with \cref{yifenyimiaos}, we have
 \begin{align}\label{wocayoudiandaca}
 S_{k+1}^* \geq & \big(1-\frac{1}{\sqrt{\kappa_x}} \big) S_{k}^* +\frac{1}{\sqrt{\kappa_x}}\Phi(x_k) + \frac{\epsilon}{4\sqrt{\kappa_x}}- \frac{1}{2L_\Phi}\|G_k\|^2 +  \big(1-\frac{1}{\sqrt{\kappa_x}} \big)\langle x_k-z_k,G_k \rangle \nonumber
 \\=& \Big(1-\frac{1}{\sqrt{\kappa_x}} \Big) S_{k}^* +\frac{1}{\sqrt{\kappa_x}}\Phi(x_k) + \frac{\epsilon}{4\sqrt{\kappa_x}}- \frac{1}{2L_\Phi}\|\nabla\Phi(x_k)\|^2  \nonumber
 \\&  +\Big(1-\frac{1}{\sqrt{\kappa_x}} \Big)\langle x_k-z_k,\nabla \Phi(x_k) \rangle +  \Big(1-\frac{1}{\sqrt{\kappa_x}} \Big)\langle x_k-z_k,G_k-\nabla \Phi(x_k)\rangle \nonumber
 \\& - \frac{1}{2L_\Phi}\|G_k-\nabla\Phi(x_k)\|^2 - \frac{1}{L_\Phi}\langle G_k-\nabla \Phi(x_k), \nabla\Phi(x_k)\rangle \nonumber
\\\overset{(i)}\geq& \big(1-\frac{1}{\sqrt{\kappa_x}} \big) S_{k}^* +\frac{1}{\sqrt{\kappa_x}}\Phi(x_k) - \frac{1}{2L_\Phi}\|\nabla\Phi(x_k)\|^2 +\big(1-\frac{1}{\sqrt{\kappa_x}} \big)\langle x_k-z_k,\nabla \Phi(x_k) \rangle  \nonumber
 \\&+\frac{\epsilon}{4\sqrt{\kappa_x}} -  \Big(1-\frac{1}{\sqrt{\kappa_x}} \Big)\| x_k-z_k\|\|G_k-\nabla \Phi(x_k)\| - \frac{1}{2L_\Phi}\|G_k-\nabla\Phi(x_k)\|^2 \nonumber
 \\&-\| G_k-\nabla \Phi(x_k)\|\| x_k-x^*\|
 \end{align}
 where $(i)$ follows from the smoothness of $\Phi(\cdot)$. Based on $$\|z_k-x^*\|\leq\sqrt{\frac{2}{\mu_x}(\Phi(0) -\Phi(x^*))+ \|x^*\|^2+\frac{\epsilon}{\mu_x}}$$ and $\|x_k-x^*\|<3\sqrt{\frac{2}{\mu_x}(\Phi(0) -\Phi(x^*))+ \|x^*\|^2+\frac{\epsilon}{\mu_x}}$ for $k=0,...,k^\prime$, we obtain from \cref{wocayoudiandaca} that 
 \begin{align}\label{wocaolalascs}
 S_{k+1}^* \geq & \big(1-\frac{1}{\sqrt{\kappa_x}} \big) S_{k}^* +\frac{1}{\sqrt{\kappa_x}}\Phi(x_k) - \frac{1}{2L_\Phi}\|\nabla\Phi(x_k)\|^2 +\big(1-\frac{1}{\sqrt{\kappa_x}} \big)\langle x_k-z_k,\nabla \Phi(x_k) \rangle  \nonumber
 \\& + \frac{\epsilon}{4\sqrt{\kappa_x}}-  \Big(7-\frac{4}{\sqrt{\kappa_x}} \Big)\sqrt{\frac{2}{\mu_x}(\Phi(0) -\Phi(x^*))+ \|x^*\|^2+\frac{\epsilon}{\mu_x}}\|G_k-\nabla \Phi(x_k)\| \nonumber
 \\& - \frac{1}{2L_\Phi}\|G_k-\nabla\Phi(x_k)\|^2.
 \end{align}
Next, we upper-bound the hypergradient estimation error $\|G_k-\nabla\Phi(x_k)\|$ in \cref{wocaolalascs}.  Based on \Cref{le:hgestr}, we have 
\begin{align*}
 \|G_k-&\nabla \Phi( x_k)\|  \nonumber 
\\\leq &\sqrt{\frac{\widetilde L_y +\mu_y}{\mu_y}} \Big(L_y +\frac{2\widetilde L_{xy}L_y}{\mu_y} +\Big(\frac{\rho_{xy}}{\mu_y}+\frac{\widetilde L_{xy}\rho_{yy}}{\mu_y^2}\Big)\mathcal{N}_k\Big) \mathcal{M}_k \exp\Big(-\frac{N}{2\sqrt{\kappa_y}}\Big) \nonumber
\\&+\frac{\widetilde L_{xy}}{\mu_y}\Big(\frac{\sqrt{\kappa_y}-1}{\sqrt{\kappa_y}+1}\Big)^M\mathcal{N}_k, \nonumber
\end{align*}
which, combined with the definitions of $\mathcal{M}_k,\mathcal{N}_k$ in \cref{pf:ntationscs} and $\|x_{k}-x^*\|\leq 3\sqrt{\frac{2}{\mu_x}(\Phi(0) -\Phi(x^*))+ \|x^*\|^2+\frac{\epsilon}{\mu_x}}$ for $k=0,...,k^\prime$, yields 
\begin{align}
 \|G_k-\nabla \Phi( x_k)\| \leq &\sqrt{\frac{\widetilde L_y +\mu_y}{\mu_y}} \Big(L_y +\frac{2\widetilde L_{xy}L_y}{\mu_y} +\Big(\frac{\rho_{xy}}{\mu_y}+\frac{\widetilde L_{xy}\rho_{yy}}{\mu_y^2}\Big)\mathcal{N}_*\Big) \mathcal{M}_* \exp\frac{-N}{2\sqrt{\kappa_y}} \nonumber
\\&+\frac{\widetilde L_{xy}}{\mu_y}\Big(\frac{\sqrt{\kappa_y}-1}{\sqrt{\kappa_y}+1}\Big)^M\mathcal{N}_*,
\end{align}
where constants $\mathcal{M}_*,\mathcal{N}_*$ are defined in \cref{pf:ntationscs}. Choose 
{\small
\begin{align}\label{wangdehuas}
N&=\Theta\Big(\sqrt{\kappa_y}\log \Big(\frac{\mathcal{M}_*(\mathcal{N}_*+\mu_y)}{\mu_x^{0.25}\mu_y^{2.5}\sqrt{\epsilon L_\Phi}}+\frac{\mathcal{M}_*(\mathcal{N}_*+\mu_y)\sqrt{L_\Phi}(\Phi(0) -\Phi(x^*)+\mu_x^{0.5}\|x^*\|+\epsilon)}{\mu_x\mu_y^{2.5}\epsilon}\Big)\Big), \nonumber
\\M&=\Theta\Big(\sqrt{\kappa_y}\log \Big(\frac{\mathcal{N}_*}{\mu_x^{0.25}\mu_y\sqrt{\epsilon L_\Phi}}+\frac{\mathcal{N}_*\sqrt{L_\Phi}(\Phi(0) -\Phi(x^*)+\mu_x^{0.5}\|x^*\|+\epsilon)}{\mu_x\mu_y\epsilon}\Big)\Big).
\end{align} }
\hspace{-0.12cm}In other words, $M$ and $N$ scale linearly with $\sqrt{\kappa_x}$ and depend only logarithmically on other constants such as $\mu_x,\mu_y,\|x^*\|,\|y^*(x^*)\|$ and so on.

 Then, we have $\|G_k-\nabla \Phi( x_k)\| \leq \frac{\sqrt{\epsilon L_\Phi}}{2\sqrt{2}\kappa_x^{1/4}}$ and $$\big(7-\frac{4}{\sqrt{\kappa_x}} \big)\sqrt{\frac{2}{\mu_x}(\Phi(0) -\Phi(x^*))+ \|x^*\|^2+\frac{\epsilon}{\mu_x}}\|G_k-\nabla \Phi( x_k)\| \leq \frac{\epsilon}{8\sqrt{\kappa_x}}.$$   Substituting these two inequalities into \cref{wocaolalascs} yields, for any $k=0,...,k^\prime$,
\begin{align}\label{jigyiscasesc}
 S_{k+1}^* \geq & \Big(1-\frac{1}{\sqrt{\kappa_x}} \Big) S_{k}^* +\frac{1}{\sqrt{\kappa_x}}\Phi(x_k) - \frac{1}{2L_\Phi}\|\nabla\Phi(x_k)\|^2 \nonumber
 \\&+\Big(1-\frac{1}{\sqrt{\kappa_x}} \Big)\langle x_k-z_k,\nabla \Phi(x_k) \rangle  + \frac{\epsilon}{16\sqrt{\kappa_x}}.
\end{align}
Then, using $\|G_k-\nabla \Phi( x_k)\| \leq \frac{\sqrt{\epsilon L_\Phi}}{2\sqrt{2}\kappa_x^{1/4}}$ in \cref{wodialayacs}, we have, for any $k=0,...,k^\prime$
\begin{align*}
\Phi(z_{k+1})\leq &\Big( 1- \frac{1}{\sqrt{\kappa_x}}\Big) S_k^* +\Big( 1- \frac{1}{\sqrt{\kappa_x}}\Big) \langle \nabla\Phi(x_k),x_k-z_k\rangle + \frac{1}{\sqrt{\kappa_x}}\Phi(x_k) \nonumber
 \\&-\frac{1}{2L_\Phi}\| \nabla \Phi(x_k)\|^2 + \frac{\epsilon}{16\sqrt{\kappa_x}},
\end{align*}
which, in conjunction with \cref{jigyiscasesc}, yields $S_{k^\prime+1}^*\geq \Phi(z_{k^\prime+1})$.  Then, by induction, we finish the proof of the second item \cref{dasimaomaoss}. To prove the first item \cref{pocasca}, letting $x=x^*$ in \cref{ssses_seq} yields, for $x=0,...,k^\prime$, 
\begin{align}\label{dalaoqiuings}
S_{k+1}&(x^*) \nonumber
\\= & \Big(1 -\frac{1}{\sqrt{\kappa_x}} \Big)S_{k}(x^*)  +\frac{1}{\sqrt{\kappa_x}} \big( \Phi(x_k) +\langle \nabla\Phi(x_k),x^*-x_k\rangle + \frac{\mu_x}{2}\|x^* - x_k\|^2 + \frac{ \epsilon}{4}\big) \nonumber
\\&+\frac{1}{\sqrt{\kappa_x}}\langle G_k-\nabla\Phi(x_k),x^*-x_k\rangle \nonumber
\\\leq&  \Big(1 -\frac{1}{\sqrt{\kappa_x}} \Big)S_{k}(x^*) +\frac{1}{\sqrt{\kappa_x}} \Phi(x^*) +\frac{\epsilon}{4\sqrt{\kappa_x}}  + \frac{1}{\sqrt{\kappa_x}} \|x_k-x^*\|\|G_k-\nabla\Phi(x_k)\|\nonumber
\\\overset{(i)}\leq & \Big(1 -\frac{1}{\sqrt{\kappa_x}} \Big)S_{k}(x^*) +\frac{1}{\sqrt{\kappa_x}} \Phi(x^*) +\frac{\epsilon}{2\sqrt{\kappa_x}},  
\end{align} 
where $(i)$ follows from the inequality $\|x_k-x^*\|\|G_k-\nabla\Phi(x_k)\|\leq  \frac{\epsilon}{8\sqrt{\kappa_x}}/(7-\frac{4}{\sqrt{\kappa_x}})<\frac{\epsilon}{24\sqrt{\kappa_x}}<\frac{\epsilon}{4}$. Subtracting both sides of \cref{dalaoqiuings} by $\Phi(x^*)$ yields, for all $k=0,...,k^\prime$,
\begin{align}\label{aihenqingchous}
S_{k+1}(x^*) - \Phi(x^*) \leq \Big(1 -\frac{1}{\sqrt{\kappa_x}} \Big)(S_{k}(x^*)-\Phi(x^*)) +\frac{\epsilon}{2\sqrt{\kappa_x}}.
\end{align}
Telescoping \cref{aihenqingchous} over $k$ from $0$ to $k^\prime$ and using $S_0(x^*) = \Phi(0)+\frac{\mu_x}{2} \|x^*\|^2$, we have 
\begin{align*}
S_{k^\prime+1}(x^*) - \Phi(x^*) &\leq  \Big(1 -\frac{1}{\sqrt{\kappa_x}} \Big)^{k^\prime+1}( \Phi(0) -\Phi(x^*)+\frac{\mu_x}{2} \|x^*\|^2)+\frac{\epsilon}{2} 
\\&\leq \Phi(0) -\Phi(x^*)+\frac{\mu_x}{2} \|x^*\|^2+\frac{\epsilon}{2},
\end{align*} 
which, in conjunction with $S_{k^\prime+1}(x^*)\geq S_{k^\prime+1}^*\geq \Phi(z_{k^\prime+1})$ and $\Phi(z_{k^\prime+1})-\Phi(x^*)\geq \frac{\mu_x}{2}\|z_{k^\prime+1}-x^*\|^2$, yields
\begin{align*}
\|z^{k^\prime+1}-x^*\|\leq \sqrt{\frac{2}{\mu_x}\Phi(0) -\Phi(x^*)+\|x^*\|^2+\frac{\epsilon}{\mu_x}}.
\end{align*}
Then, by induction, we finish the proof of the first item \cref{pocasca}.  Therefore, based on  \cref{pocasca} and \cref{dasimaomaoss} and using an approach similar to \cref{aihenqingchous}, we have
\begin{align}\label{heiyeibaizhoussc}
\Phi(z_K)- \Phi(x^*)&\leq S_{K}(x^*) - \Phi(x^*) \nonumber
\\&\leq \Big(1 -\frac{1}{\sqrt{\kappa_x}} \Big)^{K}(\Phi(0) -\Phi(x^*)+\frac{\mu_x}{2} \|x^*\|^2) +\frac{\epsilon}{2}.
\end{align}
Then, in order to achieve $\Phi(z_K)- \Phi(x^*)\leq S_{K}(x^*) - \Phi(x^*) \leq \epsilon$, it requires at most 
{\footnotesize
\begin{align}\label{anqilababascs}
K \leq& \mathcal{O}\Big( \sqrt{\frac{L_\Phi}{\mu_x}}\log\Big(\frac{\Phi(0) -\Phi(x^*)+\frac{\mu_x}{2} \|x^*\|^2}{\epsilon}\Big)\Big) \nonumber
\\\leq& \mathcal{\widetilde O}\Big(\frac{1}{\mu_x^{0.5}\mu_y}+\Big(\frac{ \sqrt{\rho_{yy}}}{\mu_x^{0.5}\mu_y^{1.5}} +  \frac{\sqrt{\rho_{xy}}}{\mu_x^{0.5}\mu_y}\Big)\sqrt{\|\nabla_y f( x^*,y^*(x^*))\|+\frac{\|x^*\|}{\mu_y}+\frac{\sqrt{\Phi(0)-\Phi(x^*)}}{\sqrt{\mu_x}\mu_y}}\Big). 
\end{align}}
\hspace{-0.13cm}Based on the choice of $M=N=\Theta(\sqrt{\kappa_y})$,  the complexity of \Cref{alg:bioNoBG} is given by 
{\footnotesize
\begin{align*}
\mathcal{C}_{\text{\normalfont sub}}(\mathcal{A},\epsilon) &\leq \mathcal{O}(n_J+n_H + n_G)\leq \mathcal{O}(K+KM+KN)   \nonumber
\\\leq& \mathcal{\widetilde O}\Big(\frac{\widetilde L_y^{0.5}}{\mu_x^{0.5}\mu_y^{1.5}}+\Big(\frac{ (\rho_{yy}\widetilde L_y)^{0.5}}{\mu_x^{0.5}\mu_y^{2}} +  \frac{(\rho_{xy}\widetilde L_y)^{0.5}}{\mu_x^{0.5}\mu_y^{1.5}}\Big)\sqrt{\|\nabla_y f( x^*,y^*(x^*))\|+\frac{\|x^*\|}{\mu_y}+\frac{\sqrt{\Phi(0)-\Phi(x^*)}}{\sqrt{\mu_x}\mu_y}}\Big),
\end{align*}}
\hspace{-0.13cm}which finish the proof. 
\end{proof}

\section{Proof of \Cref{coro:quadaticSr}}\label{proof:coUsrwB}
The proof follows a procedure similar to that for \Cref{upper_srsr_withnoB} except that the smoothness parameter of $\Phi(\cdot)$ at iterate $x_k$  and the hypergradient estimation error $\|G_k-\nabla\Phi(x_k)\|$ are different. In specific, for the quadratic inner problem, we have that there exist constant matrices $H,J$ such that 
$\nabla_y^2 g (x,y) 	\equiv H, \nabla_x\nabla_y g(x,y) \equiv J, \forall  x\in\mathbb{R}^p, y\in \mathbb{R}^q$. Then, based on the form of $\nabla \Phi(x)$ in \cref{hgforms}, we have 
\begin{align*}
\|\nabla&\Phi(x_1) - \nabla\Phi(x_2)\| \nonumber
\\\leq& \|\nabla_x f(x_1,y^*(x_1)) -\nabla_x f(x_2,y^*(x_2))  \| \nonumber
\\&+ \|JH^{-1}\nabla_y f(x_1,y^*(x_1)) -JH^{-1}\nabla_y f(x_2,y^*(x_2))\| \nonumber
\\  \leq & L_x\|x_1-x_2\| + L_{xy}\|y^*(x_1)-y^*(x_2)\| + \frac{\widetilde L_{xy}}{\mu_y}(L_{xy}\|x_1-x_2\|+L_y\|y^*(x_1)-y^*(x_2)\|)
\end{align*}
which, in conjunction with $\|y^*(x_1)-y^*(x_2)\|\leq \frac{\widetilde L_{xy}}{\mu_y} \|x_1-x_2\|$, yields 
\begin{align}\label{co:qiguqiaopwc}
\|\nabla\Phi(x_1) - \nabla\Phi(x_2)\| \leq \Big(\underbrace{ L_x + \frac{2\widetilde L_{xy}L_{xy}}{\mu_y} +\frac{L_y\widetilde L_{xy}^2}{\mu_y^2} }_{L_\Phi} \Big) \|x_1-x_2\|.
\end{align}
Note that \cref{co:qiguqiaopwc} shows that the objective function $\Phi(\cdot)$ is globally smooth, i.e., the smoothness parameter is bounded at all $x\in\mathbb{R}^p$. This is different from the proof in \Cref{upper_srsr_withnoB}, where the smoothness parameter is unbound at all $x\in\mathbb{R}^p$, but can be bounded at all iterates $x_k,k=0,...,K$ along the optimization path of the algorithm. Therefore, the proof for this quadratic special case is simpler. 

We next upper-bound the hypergradient estimation error $\|G_k-\nabla \Phi(x_k)\|$. Using an approach similar to \cref{jingyikeai}, we have 
\begin{align}\label{co:huxiuwancong}
\|G_k&-\nabla \Phi(x_k)\| \nonumber
\\\leq & L_y \|y^*(x_k)-y_k^N\| + \widetilde L_{xy} \|v_k^M- H^{-1}\nabla_y f( x_k,y^N_k) \| \nonumber
\\&+ \widetilde L_{xy}\big\|H^{-1}\nabla_y f( x_k,y^N_k)-H^{-1}\nabla_y f(x_k,y^*(x_k)) \big\|\nonumber
\\\leq & \Big(L_y +\frac{\widetilde L_{xy}L_y}{\mu_y} \Big)\|y_k^N-y^*(x_k)\| + \widetilde L_{xy} \|v_k^M- H^{-1}\nabla_y f( x_k,y^N_k) \|  \nonumber
\\\leq& \Big(L_y +\frac{\widetilde L_{xy}L_y}{\mu_y} \Big)\|y_k^N-y^*(x_k)\|+\frac{\widetilde L_{xy}}{\mu_y}\left(\frac{\sqrt{\kappa_y}-1}{\sqrt{\kappa_y}+1}\right)^M\|\nabla_y f( x_k,y^*(x_k))\| \nonumber
\\\leq &\sqrt{\frac{\widetilde L_y +\mu_y}{\mu_y}} \Big(L_y +\frac{\widetilde L_{xy}L_y}{\mu_y}\Big) \mathcal{M}_* \exp\Big(-\frac{N}{2\sqrt{\kappa_y}}\Big) +\frac{\widetilde L_{xy}}{\mu_y}\Big(\frac{\sqrt{\kappa_y}-1}{\sqrt{\kappa_y}+1}\Big)^M\mathcal{N}_*,
\end{align}
where $\mathcal{M}_*$ and $\mathcal{N}_*$ are given by \cref{pf:ntationscs}. 
Based on \cref{co:qiguqiaopwc} and \cref{co:huxiuwancong}, we choose
\begin{itemize}
\item $N=\Theta(\sqrt{\kappa_y}\log (\frac{\mathcal{M}_*}{\mu_x^{0.25}\mu_y^{1.5}\sqrt{\epsilon L_\Phi}}+\frac{\mathcal{M}_*\sqrt{L_\Phi}(\Phi(0) -\Phi(x^*)+\mu_x^{0.5}\|x^*\|+\epsilon)}{\mu_x\mu_y^{1.5}\epsilon}))$
\item $M=\Theta(\sqrt{\kappa_y}\log (\frac{\mathcal{N}_*}{\mu_x^{0.25}\mu_y\sqrt{\epsilon L_\Phi}}+\frac{\mathcal{N}_*\sqrt{L_\Phi}(\Phi(0) -\Phi(x^*)+\mu_x^{0.5}\|x^*\|+\epsilon)}{\mu_x\mu_y\epsilon})).$
\end{itemize} 
Then, using an approach similar to \cref{heiyeibaizhoussc} with $\rho_{xy}=\rho_{yy}=0$, we have 
\begin{align}
\Phi(z_K)- \Phi(x^*) \leq \Big(1 -\sqrt{\frac{\mu_x}{L_\Phi}} \Big)^{K}(\Phi(0) -\Phi(x^*)+\frac{\mu_x}{2} \|x^*\|^2) +\frac{\epsilon}{2},
\end{align}
where $L_\Phi$ is given in \cref{co:qiguqiaopwc}. Then, in order to achieve $\Phi(z_K)- \Phi(x^*) \leq \epsilon$, it requires at most 
\begin{align*}
\mathcal{C}_{\text{\normalfont sub}}(\mathcal{A},\epsilon) \leq &\mathcal{O}(n_J+n_H + n_G)\leq \mathcal{O}(K+KM+KN)   \nonumber
\\\leq& \mathcal{O}\Big(\sqrt{\frac{\widetilde L_y}{\mu_x\mu_y^3}}\log\, {\small \text{poly}(\mu_x,\mu_y,\|x^*\|,\Phi(0)-\Phi(x^*),\|\nabla_y f(x^*,y^*(x^*))\|)}\Big),
\end{align*}
which finishes the proof. 


\section{Proof of \Cref{th:upper_csc1sc}}

Recall that $\widetilde \Phi(\cdot)=\widetilde f(x,y^*(x))$ with  $\widetilde f(x,y) = f(x,y) +\frac{\epsilon}{2R} \|x\|^2$. Then, we have $\widetilde\Phi(x) = \Phi(x) +\frac{\epsilon}{2R} \|x\|^2$ is strongly-convex with parameter $\mu_x=\frac{\epsilon}{R}$. Note that the smoothness parameters of $\widetilde f(x,y)$ are the same as those of $f(x,y)$ except that $L_x$ in \cref{def:first} becomes $L_x+ \frac{\epsilon}{R}$ for $\widetilde f(x,y)$. 
Let $x^*\in\argmin_{x\in\mathbb{R}^p}\Phi(x)$ be one minimizer of the original objective function $\Phi(\cdot)$ and let $\widetilde x^*=\argmin_{x\in\mathbb{R}^p}\widetilde\Phi(x)$ be the minimizer of the regularized object function $\widetilde \Phi(\cdot)$.  We next characterize some useful inequalities between $x^*$ and $\widetilde x^*$. Based on the definition of $x^*$ and $\widetilde x^*$, we have $\nabla \widetilde \Phi(\widetilde x^*) = 0$ and $\nabla\widetilde\Phi( x^*) = \nabla\Phi(x^*) + \frac{\epsilon}{R} x^* = \frac{\epsilon}{R} x^*  $, which, combined with the strong convexity of $\widetilde \Phi(\cdot)$,  implies that $\frac{\epsilon}{R}\|x^*-\widetilde x^*\|\leq\|\nabla \widetilde \Phi(\widetilde x^*) -\nabla\widetilde\Phi( x^*) \|=  \frac{\epsilon}{R} \|x^*\|$ and hence $\|\widetilde x^*\|\leq2\|x^*\|$. Similarly, we have 
\begin{align}\label{ggsmidacposcs1}
\|y^*(\widetilde x^*)\|&\leq \|y^*(x^*)\|+\frac{3\widetilde L_{xy}}{\mu_y}\|x^*\|, \nonumber
\\\|\nabla_y \widetilde f(\widetilde x^*,y^*(\widetilde x^*))\|&\leq \|\nabla_y f(\widetilde x^*,y^*(\widetilde x^*))\| + \frac{\epsilon}{R}\|\widetilde x^*\| \nonumber
\\&\leq\|\nabla_y f(x^*,y^*(x^*))\| + \Big( 3L_{xy}+\frac{3L_y\widetilde L_{xy}}{\mu_y}+\frac{2\epsilon}{R}\Big)\|x^*\|  \nonumber
\\\widetilde \Phi(0) -\widetilde \Phi(\widetilde x^*) &= \Phi(0)-\Phi(\widetilde x^*) - \frac{\epsilon}{2R}\|\widetilde x^*\|^2 \overset{(i)}\leq \Phi(0) -\Phi(x^*),
\end{align}
where $(i)$ follows from the definition of $x^*\in\argmin_x\Phi(x)$.  

Let $L_{\widetilde \Phi}$ be one smoothness parameter of function $\widetilde \Phi(\cdot)$, which takes the same form as $L_\Phi$ in \cref{wodishenamxd} except that $L_x,f,x^*$ and $\Phi$ become $L_x+\frac{\epsilon}{R},\widetilde f,\widetilde x^*$ and $\widetilde \Phi$ in \cref{wodishenamxd}, respectively. Similarly to \cref{wangdehuas}, we choose  
\begin{align}\label{michiganletgo}
N=&\Theta(\sqrt{\kappa_y}\log (\text{poly}(\epsilon,\mu_x,\mu_y,\|\widetilde x^*\|,\|y^*(\widetilde x^*)\|,\|\nabla_y \widetilde f(\widetilde x^*,y^*(\widetilde x^*))\|,\widetilde \Phi(0) -\widetilde \Phi(\widetilde x^*)))), \nonumber
\\M=&\Theta(\sqrt{\kappa_y}\log (\text{poly}(\epsilon,\mu_x,\mu_y,\|\widetilde x^*\|,\|y^*(\widetilde x^*)\|,\|\nabla_y \widetilde f(\widetilde x^*,y^*(\widetilde x^*))\|,\widetilde \Phi(0) -\widetilde \Phi(\widetilde x^*) ))).
\end{align} 

We first prove the case when the convergence is measured in term of the suboptimality gap. Note that in this case we choose $R=B^2$. Using an approach similar to \cref{heiyeibaizhoussc} in the proof of \Cref{upper_srsr_withnoB} with $\epsilon$ and $\mu_x$ being replaced by $\epsilon/2$ and  $\frac{\epsilon}{B^2}$, respectively, we have 
\begin{align*}
\widetilde\Phi(z_K)- \widetilde\Phi(\widetilde x^*) \leq \Big(1 -\sqrt{\frac{\epsilon}{B^2L_{\widetilde \Phi}}} \Big)^{K}(\widetilde\Phi(0) -\widetilde\Phi(\widetilde x^*)+\frac{\epsilon}{2B^2} \|\widetilde x^*\|^2) +\frac{\epsilon}{4}.
\end{align*}
which, combined with $\widetilde\Phi(z_K)\geq\Phi(z_K)$ and $\widetilde\Phi(\widetilde x^*)\leq\widetilde\Phi(x^*)=\Phi(x^*)+\frac{\epsilon}{2B^2}\|x^*\|^2$, yields
{\small\begin{align}\label{ca:youdiannaoketongs}
\Phi(z_K) - \Phi(x^*)\leq \Big(1 -\sqrt{\frac{\epsilon}{B^2L_{\widetilde \Phi}}} \Big)^{K}(\widetilde\Phi(0) -\widetilde\Phi(\widetilde x^*)+\frac{\epsilon}{2B^2} \|\widetilde x^*\|^2) +\frac{\epsilon}{4}+\frac{\epsilon}{2B^2}\|x^*\|^2.
\end{align}}
\hspace{-0.13cm}Recall that $\|x^*\|=B$. Similarly to \cref{anqilababascs}, we choose
{\footnotesize
\begin{align}\label{Kwocayoudianda}
K =& \Theta\Big( \sqrt{\frac{B^2L_{\widetilde \Phi}}{\epsilon}}\log\Big(\frac{\widetilde\Phi(0) -\widetilde\Phi(\widetilde x^*)+\frac{\epsilon}{2B^2} \|\widetilde x^*\|^2}{\epsilon}\Big)\Big) \nonumber
\\=&\widetilde \Theta \Big(\sqrt{\frac{B^2}{\epsilon\mu_y^2}}+\Big(\sqrt{\frac{B^2\rho_{yy}}{\epsilon\mu_y^{3}}} +  \sqrt{\frac{B^2\rho_{xy}}{\epsilon\mu_y^2}}\Big)\sqrt{\|\nabla_y \widetilde f(\widetilde x^*,y^*(\widetilde x^*))\|+\frac{\|\widetilde x^*\|}{\mu_y}+\frac{\sqrt{B^2(\widetilde\Phi(0)-\widetilde\Phi(\widetilde x^*))}}{\sqrt{\epsilon}\mu_y}}\Big).
\end{align}}
\hspace{-0.13cm}Then,  we obtain from \cref{ca:youdiannaoketongs} that $\Phi(z_K) - \Phi(x^*)\leq \epsilon$, and the complexity $\mathcal{C}_{\text{\normalfont sub}}(\mathcal{A},\epsilon)$ after substituting \cref{ggsmidacposcs1} into \cref{michiganletgo} and \cref{Kwocayoudianda} is given by 
\begin{align}
\mathcal{C}_{\text{\normalfont sub}}&(\mathcal{A},\epsilon) \leq \mathcal{O}(n_J+n_H + n_G)\leq \mathcal{O}(K+KM+KN)   \nonumber
\\\leq& \mathcal{O}\Big( \Big( \sqrt{\frac{B^2\widetilde L_y}{\epsilon\mu_y^3}}+\Big(\sqrt{\frac{B^2\rho_{yy}\widetilde L_y}{\epsilon\mu_y^{4}}} +  \sqrt{\frac{B^2\rho_{xy}\widetilde L_y}{\epsilon\mu_y^3}}\Big)\sqrt{\Delta^*_{\text{\normalfont\tiny CSC}}}\Big)\log\, {\small \text{poly}(\epsilon,\mu_x,\mu_y,\Delta^*_{\text{\normalfont\tiny CSC}})}\Big).
\end{align}

Next, we characterize the convergence rate and complexity under the gradient norm metric.  Note that in this case we choose $R=B$. 
Using eq. (9.14) in \cite{boyd2004convex}, we have $\|\nabla\widetilde \Phi (z_k)\|^2\leq 2L_{\widetilde \Phi}(\widetilde\Phi(z_K)- \widetilde\Phi(\widetilde x^*)) $, which, combined with  
$\|\nabla\widetilde \Phi (z_k)\|^2 \geq \frac{1}{2}\|\nabla\Phi (z_k)\|^2 - \frac{\epsilon^2}{B^2}\|z_k\|^2\geq  \frac{1}{2}\|\nabla\Phi (z_k)\|^2 -  \frac{\epsilon^2}{B^2}(2\|z_k-\widetilde x^*\|^2 + 2\|\widetilde x^*\|^2)$
yields
\begin{align}\label{haofantoutongsc}
\|\nabla \Phi (z_k)\|^2\leq& 4L_{\widetilde \Phi}(\widetilde\Phi(z_K)- \widetilde\Phi(\widetilde x^*)) +\frac{4\epsilon^2}{B^2}\|z_k-\widetilde x^*\|^2 +\frac{4\epsilon^2}{B^2} \|\widetilde x^*\|^2 \nonumber
\\\overset{(i)}\leq&4L_{\widetilde \Phi}(\widetilde\Phi(z_K)- \widetilde\Phi(\widetilde x^*)) + \frac{8\epsilon}{B}(\widetilde\Phi(z_K)- \widetilde\Phi(\widetilde x^*)) + \frac{16\epsilon^2}{B^2} \| x^*\|^2  \nonumber
\\=& \Big(4L_{\widetilde \Phi}+ \frac{8\epsilon}{B}\Big)(\widetilde\Phi(z_K)- \widetilde\Phi(\widetilde x^*)) + \frac{16\epsilon^2}{B^2} \| x^*\|^2,  
\end{align}
where $(i)$ follows from the strong convexity of $\widetilde \Phi(\cdot)$ and $\|\widetilde x^*\|\leq 2\|x^*\|$, and $L_{\widetilde \Phi}$ takes the same form as $L_\Phi$ in \cref{wodishenamxd} except that $L_x,f,x^*$ and $\Phi$ become $L_x+\frac{\epsilon}{B},\widetilde f,\widetilde x^*$ and $\widetilde \Phi$ in \cref{wodishenamxd}, respectively.  Then, using an approach similar to \cref{heiyeibaizhoussc} in the proof of \Cref{upper_srsr_withnoB} with $\epsilon$ and $\mu_x$ being replaced by $\epsilon^2/(4L_{\widetilde \Phi}+ \frac{8\epsilon}{B})$ and $\frac{\epsilon}{B}$, respectively, we have
\begin{align*}
\widetilde\Phi(z_K)- \widetilde\Phi(\widetilde x^*) \leq \Big(1 -\sqrt{\frac{\epsilon}{BL_{\widetilde \Phi}}} \Big)^{K}(\widetilde\Phi(0) -\widetilde\Phi(\widetilde x^*)+\frac{\epsilon}{2B} \|\widetilde x^*\|^2) +\frac{\epsilon^2}{2(4L_{\widetilde \Phi}+ \frac{8\epsilon}{B})},
\end{align*}
which, in conjunction with \cref{wangdehuas} and \cref{haofantoutongsc}, yields 
{\small\begin{align*}
\|\nabla \Phi (z_k)\|^2\leq \Big(1 -\sqrt{\frac{\epsilon}{BL_{\widetilde \Phi}}} \Big)^{K}\Big(\widetilde\Phi(0) -\widetilde\Phi(\widetilde x^*)+\frac{\epsilon}{2B} \|\widetilde x^*\|^2\Big) \Big(4L_{\widetilde \Phi}+ \frac{8\epsilon}{B}\Big)  + \frac{\epsilon^2}{2} + \frac{16\epsilon^2}{B^2} \| x^*\|^2.
\end{align*}}
\hspace{-0.15cm}Then, to achieve $\|\nabla \Phi (z_k)\|\leq 5\epsilon$, it suffices to  choose $M,N$ as in \cref{michiganletgo} by replacing $\epsilon$ with $\epsilon^2/(4L_{\widetilde \Phi}+ \frac{8\epsilon}{B})$, and choose
\begin{align*}
K =& \Theta\Big( \sqrt{\frac{BL_{\widetilde \Phi}}{\epsilon}}\log\Big(\frac{(\widetilde\Phi(0) -\widetilde\Phi(\widetilde x^*)+\frac{\epsilon}{2B} \|\widetilde x^*\|^2)(4L_{\widetilde \Phi}+ \frac{8\epsilon}{B})}{\epsilon}\Big)\Big), 
\end{align*}
which, in conjunction with \cref{ggsmidacposcs1}, yields
\begin{align*}
\mathcal{C}_{\text{\normalfont norm}}&(\mathcal{A},\epsilon) \leq \mathcal{O}(n_J+n_H + n_G)\leq \mathcal{O}(K+KM+KN)   \nonumber
\\\leq& \mathcal{O}\Big( \Big( \sqrt{\frac{B\widetilde L_y}{\epsilon\mu_y^3}}+\Big(\sqrt{\frac{B\rho_{yy}\widetilde L_y}{\epsilon\mu_y^{4}}} +  \sqrt{\frac{B\rho_{xy}\widetilde L_y}{\epsilon\mu_y^3}}\Big)\sqrt{\Delta^*_{\text{\normalfont\tiny CSC}}}\Big)\log\, {\small \text{poly}(\epsilon,\mu_x,\mu_y,\Delta^*_{\text{\normalfont\tiny CSC}})}\Big),
\end{align*}
which finishes the proof. 
\section{Proof of \Cref{coro:quadaticConv}}
Note that for this quadratic inner problem, the Jacobians $\nabla_x\nabla_y g(x,y)$ and Hessians $\nabla_y^2 g(x,y)$ are {\bf constant} matrices, which imply that the parameters $\rho_{xx}=\rho_{xy}=0$ in Assumption \ref{g:hessiansJaco}. Then, letting $\rho_{xx}=\rho_{xy}=0$ in the results of \Cref{th:upper_csc1sc} yields the proof.

\section{Proof of \Cref{upper_srsr}}\label{proof:upss_wb}
Based on the update in line 9 of \Cref{alg:bio}, we  have,  for any $x\in\mathbb{R}^p$
\begin{align}\label{e1:ggmida}
\langle \beta_k G_k,x_{k+1}-x\rangle = \tau_k\underbrace{\langle x-x_{k+1}, x_{k+1} -\widetilde x_k\rangle}_{P} + (1-\tau_k)\underbrace{\langle x-x_{k+1},x_{k+1}-x_k\rangle}_{Q}.
\end{align}
Note that $P$ in the above \cref{e1:ggmida} satisfies 
\begin{align*}
P &= \langle \widetilde x_k - x_{k+1}, x -\widetilde x_k \rangle + \|x-\widetilde x_k\|^2 - \|x-x_{k+1}\|^2
\\&=-P + \|x-\widetilde x_k\|^2 -\|\widetilde x_k-x_{k+1}\|^2 -\|x-x_{k+1}\|^2,
\end{align*}
which yields $P =\frac{1}{2}(\|x-\widetilde x_k\|^2 -\|\widetilde x_k-x_{k+1}\|^2 -\|x-x_{k+1}\|^2)$. Taking an approach similar to the derivation of $P$, we can obtain $Q=\frac{1}{2}(\|x- x_k\|^2 -\| x-x_{k+1}\|^2 -\|x_{k}-x_{k+1}\|^2)$. Then,  substituting  the forms of $P,Q$ to \cref{e1:ggmida} and using the choices of $\tau_k$ and $\beta_k$, we have 
\begin{align}\label{eq:gdangle}
\big\langle G_k, \frac{\sqrt{\alpha\mu_x}}{2}(&x_{k+1}-x)\big\rangle =  \frac{\sqrt{\alpha\mu_x}\mu_x}{8} (\|x-\widetilde x_k\|^2 -\|\widetilde x_k-x_{k+1}\|^2 -\|x-x_{k+1}\|^2)  \nonumber
\\+&\frac{2\mu_x-\sqrt{\alpha\mu_x}\mu_x}{8}(\|x- x_k\|^2 -\| x-x_{k+1}\|^2 -\|x_{k}-x_{k+1}\|^2).
\end{align}
By the update $z_{k+1}= \widetilde x_k -\alpha_k G_k $ and the choice of $\alpha_k=\alpha$, we have, for any $x^\prime\in\mathbb{R}^p$,
\begin{align}\label{eq:changeone}
\langle z_{k+1}-x^\prime, G_k \rangle  =& \frac{1}{\alpha}\langle x^\prime-z_{k+1}, z_{k+1}-\widetilde x_k \rangle \nonumber
\\ = & \frac{1}{2\alpha} (\|x^\prime-\widetilde x_k\| - \|x^\prime-z_{k+1}\|^2 - \|z_{k+1}-\widetilde x_k\|^2).
\end{align}
Let  $x^{\prime} = (1-\frac{\sqrt{\alpha\mu_x}}{2})z_k + \frac{\sqrt{\alpha \mu_x}}{2}$ and recall $\widetilde x_k = \eta_kx_k + (1-\eta_k)z_k$. Then, we have
\begin{align}\label{eq:xprimes}
\|x^\prime - \widetilde x_k\|^2 = & \Big\|\frac{\sqrt{\alpha\mu_x}}{2}(x_{k+1}-z_k)+ \frac{\sqrt{\alpha\mu_x}}{\sqrt{\alpha\mu_x}+2}(z_k-x_k)\Big\|^2  \nonumber
\\ =& \Big\| \frac{\sqrt{\alpha\mu_x}}{2}(x_{k+1}-x_k)+ \frac{\alpha\mu_x}{2(\sqrt{\alpha\mu_x}+2)}(z_k-x_k)  \Big\|^2 \nonumber
\\ \overset{(i)}= & \frac{\alpha\mu_x}{4} \Big\|(1-\frac{\sqrt{\alpha\mu_x}}{2})(x_{k+1}-x_k) + \frac{\sqrt{\alpha\mu_x}}{2} (x_{k+1}-\widetilde x_k)\Big\|^2 \nonumber
\\ \leq &  \frac{\alpha\mu_x}{4} \big(1-\frac{\sqrt{\alpha\mu_x}}{2}\big) \|x_{k+1}-x_k\|^2 +   \frac{\alpha\mu_x\sqrt{\alpha\mu_x}}{8}\|x_{k+1}-\widetilde x_k\|^2,
\end{align}
where $(i)$ follows from $\widetilde x_k - x_k=\frac{2}{2+\sqrt{\alpha\mu_x} }(z_k-x_k)$.  Then, substituting \cref{eq:xprimes} in \cref{eq:changeone} and adding \cref{eq:gdangle}, \cref{eq:changeone} and cancelling out several negative terms, we have
\begin{align}\label{eq:gkquick}
\big \langle G_k, &\frac{\sqrt{\alpha\mu_x}}{2}(z_{k+1}-x) + (1-\frac{\sqrt{\alpha\mu_x}}{2}) (z_{k+1}-z_k) \big\rangle \nonumber
\\\leq& \frac{\sqrt{\alpha\mu_x}\mu_x}{8}\|x-\widetilde x_k\|^2-\frac{1}{2\alpha} \|z_{k+1}-\widetilde x_k\|^2 - \frac{\mu_x\sqrt{\alpha\mu_x}}{16} \|x_{k+1}-\widetilde x_k\|^2\nonumber
\\&-\frac{\mu_x}{4} \|x-x_{k+1}\|^2 -\frac{2\mu_x -\sqrt{\alpha\mu_x}\mu_x}{16} \|x_k-x_{k+1}\|^2.
\end{align}
Next, we characterize the smoothness property of $\Phi(x)$. Using the form of $\nabla \Phi(x)$ in \cref{hyperG}, and based on Assumptions~\ref{fg:smooth},~\ref{g:hessiansJaco} and Assumption~\ref{bounded_f_assump} that $\|\nabla_yf(\cdot,\cdot)\|\leq U$, we have, for any $x_1,x_2\in\mathbb{R}^p$, 
\begin{align*}
\|\nabla\Phi(x_1) - \nabla\Phi(x_2)\|
\leq& \|\nabla_x f(x_1,y^*(x_1)) -\nabla_x f(x_2,y^*(x_2))  \| \nonumber
\\&+ \|\nabla_x \nabla_y g(x_1,y^*(x_1)) \nabla_y^2 g(x_1,y^*(x_1))^{-1}\nabla_y f(x_1,y^*(x_1)) \nonumber
\\&\hspace{0.5cm}-\nabla_x \nabla_y g(x_2,y^*(x_2)) \nabla_y^2 g(x_2,y^*(x_2))^{-1}\nabla_y f(x_2,y^*(x_2))\| \nonumber
\\  \leq & L_x\|x_1-x_2\| + L_{xy}\|y^*(x_1)-y^*(x_2)\| \nonumber
\\&+ \Big(\frac{U\rho_{xy}}{\mu_y}+\frac{\widetilde L_{xy}U\rho_{yy}}{\mu_y^2}\Big)(\|x_1-x_2\|+\|y^*(x_1)-y^*(x_2)\|)  \nonumber
\\&+ \frac{\widetilde L_{xy}}{\mu_y} (L_{xy}\|x_1-x_2\|
+L_y\|y^*(x_1)-y^*(x_2)\|),
\end{align*}
which, combined with Lemma 2.2 in~\cite{ghadimi2018approximation} that $\|y^*(x_1)-y^*(x_2)\|\leq \frac{\widetilde L_{xy}}{\mu_y} \|x_1-x_2\|$, yields 
\begin{align}\label{smooth_phi}
\|&\nabla\Phi(x_1) - \nabla\Phi(x_2)\| \nonumber
\\&\leq \Big (\underbrace{L_x+ \frac{2L_{xy}\widetilde L_{xy}}{\mu_y} +\Big(\frac{U\rho_{xy}}{\mu_y}+\frac{U\widetilde L_{xy}\rho_{yy}}{\mu_y^2}\Big)\Big(1+\frac{\widetilde L_{xy}}{\mu_y}\Big)+\frac{\widetilde L^2_{xy}L_y}{\mu^2_y}}_{L_\Phi}\Big) \|x_1-x_2\|.
\end{align}
Then, based on the above $L_\Phi$-smoothness of $\Phi(\cdot)$, we have 
\begin{align}\label{eq:phismooth}
\Phi(z_{k+1}) \leq & \Phi(\widetilde x_k) + \langle \nabla \Phi(\widetilde x_k), z_{k+1}-\widetilde x_k\rangle  + \frac{L_\Phi}{2} \|z_{k+1}-\widetilde x_k\|^2 \nonumber
\\=& \big(1-\frac{\sqrt{\alpha\mu_x}}{2}\big)(\Phi(\widetilde x_k) + \langle \nabla \Phi(\widetilde x_k), z_{k+1}-\widetilde x_k\rangle) \nonumber
\\&+ \frac{\sqrt{\alpha\mu_x}}{2}(\Phi(\widetilde x_k) + \langle \nabla \Phi(\widetilde x_k), z_{k+1}-\widetilde x_k\rangle)+\frac{L_\Phi}{2} \|z_{k+1}-\widetilde x_k\|^2.
\end{align}
Adding \cref{eq:gkquick} and~\cref{eq:phismooth} yields
{\small \begin{align*}
\Phi(z_{k+1})\leq &  \big(1-\frac{\sqrt{\alpha\mu_x}}{2}\big)(\Phi(\widetilde x_k) + \langle \nabla \Phi(\widetilde x_k), z_{k}-\widetilde x_k\rangle) + \frac{\sqrt{\alpha\mu_x}}{2}(\Phi(\widetilde x_k) + \langle \nabla \Phi(\widetilde x_k), x-\widetilde x_k\rangle)  \nonumber
\\&+\big \langle \nabla\Phi(\widetilde x_k)-G_k, \frac{\sqrt{\alpha\mu_x}}{2}(z_{k+1}-x) + (1-\frac{\sqrt{\alpha\mu_x}}{2}) (z_{k+1}-z_k) \big\rangle  \nonumber
\\&+ \frac{\sqrt{\alpha\mu_x}\mu_x}{8}\|x-\widetilde x_k\|^2-\frac{1}{2\alpha}(1-\alpha L_\Phi) \|z_{k+1}-\widetilde x_k\|^2 - \frac{\mu_x\sqrt{\alpha\mu_x}}{16} \|x_{k+1}-\widetilde x_k\|^2\nonumber
\\&-\frac{\mu_x}{4} \|x-x_{k+1}\|^2 -\frac{2\mu_x -\sqrt{\alpha\mu_x}\mu_x}{16} \|x_k-x_{k+1}\|^2,
\end{align*}}
\hspace{-0.13cm}which, in conjunction with the strong-convexity of $\Phi(\cdot)$, the fact that $\sqrt{\alpha \mu_x}\leq 1$ and $\alpha\leq \frac{1}{2L_\Phi}$, yields 
\begin{align}\label{eq:orginalsca}
\Phi(z_{k+1})\leq& \big(1-\frac{\sqrt{\alpha\mu_x}}{2}\big) \big(\Phi(z_k) -\frac{\mu_x}{2}\|z_k-\widetilde x_k\|^2\big) + \frac{\sqrt{\alpha\mu_x}}{2} \big(\Phi(x) -\frac{\mu_x}{2}\|x-\widetilde x_k\|^2 \big)   \nonumber
\\&+\big \langle \nabla\Phi(\widetilde x_k)-G_k, \frac{\sqrt{\alpha\mu_x}}{2}(z_{k+1}-x) + (1-\frac{\sqrt{\alpha\mu_x}}{2}) (z_{k+1}-z_k) \big\rangle \nonumber
\\&+\frac{\sqrt{\alpha\mu_x}\mu_x}{8}\|x-\widetilde x_k\|^2-\frac{1}{4\alpha} \|z_{k+1}-\widetilde x_k\|^2 - \frac{\mu_x\sqrt{\alpha\mu_x}}{16} \|x_{k+1}-\widetilde x_k\|^2.
\end{align}
Note that  we have the equality that 
\begin{align}\label{eq:xzk1}
\frac{\sqrt{\alpha\mu_x}}{2}(z_{k+1}-x) + &(1-\frac{\sqrt{\alpha\mu_x}}{2}) (z_{k+1}-z_k) \nonumber
\\& =  (z_{k+1}-\widetilde x_k) + \frac{\sqrt{\alpha\mu_x}}{2} (\widetilde x_k - x) + (1-\frac{\sqrt{\alpha\mu_x}}{2})(\widetilde x_k-z_k).
\end{align}
Then, using \cref{eq:xzk1} and the Cauchy-Schwarz inequality, we have
\begin{align}\label{eq:gkdistance}
\big \langle \nabla\Phi(\widetilde x_k)-G_k, &\frac{\sqrt{\alpha\mu_x}}{2}(z_{k+1}-x) + (1-\frac{\sqrt{\alpha\mu_x}}{2}) (z_{k+1}-z_k) \big\rangle \nonumber
\\\leq & \Big(2\alpha+\frac{1}{2\mu_x} +\frac{\sqrt{\alpha\mu_x}}{4\mu_x}\Big)\|\nabla\Phi(\widetilde x_k)-G_k\|^2 + \frac{1}{8\alpha}\|z_{k+1}-\widetilde x_k\|^2
 \nonumber
\\&  +   \frac{\sqrt{\alpha\mu_x}\mu_x}{8} \|\widetilde x_k - x\|+ (1-\frac{\sqrt{\alpha\mu_x}}{2})\frac{\mu_x}{2} \|z_k-\widetilde x_k\|^2.
\end{align}
Substituting \cref{eq:gkdistance} into \cref{eq:orginalsca} and cancelling out negative terms, we have 
\begin{align}\label{cocoinasca}
\Phi(z_{k+1})\leq & \big(1-\frac{\sqrt{\alpha\mu_x}}{2}\big) \Phi(z_k) + \frac{\sqrt{\alpha\mu_x}}{2} \Phi(x) -\frac{1}{8\alpha} \|z_{k+1}-\widetilde x_k\|^2 - \frac{\mu_x\sqrt{\alpha\mu_x}}{16} \|x_{k+1}-\widetilde x_k\|^2 \nonumber
\\&+ \Big(2\alpha+\frac{1}{2\mu_x} +\frac{\sqrt{\alpha\mu_x}}{4\mu_x}\Big)\|\nabla\Phi(\widetilde x_k)-G_k\|^2.
\end{align}
We next upper-bound the hypergradient estimation error $\|\nabla\Phi(\widetilde x_k)-G_k\|^2$. Recall  
\begin{align}
G_k:= \nabla_x f(\widetilde x_k,y_k^N) -\nabla_x \nabla_y g(\widetilde x_k,y_k^N)v_k^M,
\end{align}
where $v_k^M$ is the output of $M$-steps of heavy-ball method for solving $$\min_v Q(v):=\frac{1}{2}v^T\nabla_y^2 g(\widetilde x_k,y_k^N) v - v^T
\nabla_y f(\widetilde x_k,y^N_k)$$
Then, based on the convergence result of heavy-ball method in~\cite{badithela2019analysis} with stepsizes $\lambda=\frac{4}{(\sqrt{\widetilde L_y}+\sqrt{\mu_y})^2}$ and $\theta=\max\big\{\big(1-\sqrt{\lambda\mu_y}\big)^2,\big(1-\sqrt{\lambda\widetilde L_y}\big)^2\big\}$, we have 
\begin{align}\label{gg:opascas}
\|v_k^M - \nabla_y^2 g(\widetilde x_k,y_k^N)^{-1}\nabla_y f(\widetilde x_k,y^N_k) \| \leq  &\Big(\frac{\sqrt{\kappa_y}-1}{\sqrt{\kappa_y}+1}\Big)^M \Big\| \nabla_y^2 g(\widetilde x_k,y_k^N)^{-1}\nabla_y f(\widetilde x_k,y^N_k)\Big\| \nonumber
\\\overset{(i)}\leq& \frac{U}{\mu_y}\Big(\frac{\sqrt{\kappa_y}-1}{\sqrt{\kappa_y}+1}\Big)^M, 
\end{align}
where $(i)$ follows from Assumption~\ref{bounded_f_assump} that  $\|\nabla_y f(\cdot,\cdot)\|$ is bounded by $U$. Let $y_k^*=\argmin_{y} g(\widetilde x_k,y)$. Then, based on the form of $\nabla\Phi(x)$ in \cref{hyperG}, we have 
\begin{align}\label{jingyisang}
\|G_k-&\nabla \Phi(\widetilde x_k)\| \nonumber
\\\overset{(i)}\leq & \| \nabla_x f( \widetilde x_k,y_k^N) -\nabla_x f(\widetilde x_k,y^*_k)\| + \widetilde L_{xy}\|v_k^M- \nabla_y^2 g(\widetilde x_k,y^*_k)^{-1}\nabla_y f(\widetilde x_k,y^*_k) \|  \nonumber
\\&+\frac{\|\nabla_y f(\widetilde  x_k,y^*_k) \|}{\mu_y}  \|\nabla_x \nabla_y g(\widetilde  x_k,y_k^N)-\nabla_x \nabla_y g(\widetilde x_k,y^*_k)\| \nonumber
\\\leq & L_y \|y^*_k-y_k^N\| + \widetilde L_{xy} \|v_k^M- \nabla_y^2 g(\widetilde x_k,y_k^N)^{-1}\nabla_y f(\widetilde x_k,y^N_k) \| \nonumber
\\&+ \widetilde L_{xy}\big\|\nabla_y^2 g(\widetilde x_k,y_k^N)^{-1}\nabla_y f(\widetilde x_k,y^N_k)-\nabla_y^2 g(\widetilde x_k,y^*_k)^{-1}\nabla_y f(\widetilde x_k,y^*_k) \big\| \nonumber
\\&+\frac{U\rho_{xy}}{\mu_y} \|y_k^N-y^*_k\| \nonumber
\\\overset{(ii)}\leq&\Big(L_y +\frac{\widetilde L_{xy}L_y}{\mu_y} +\Big(\frac{\rho_{xy}}{\mu_y}+\frac{\widetilde L_{xy}\rho_{yy}}{\mu_y^2}\Big)U\Big)\|y_k^N-y^*_k\| + \frac{U\widetilde L_{xy}}{\mu_y}\Big(\frac{\sqrt{\kappa_y}-1}{\sqrt{\kappa_y}+1}\Big)^M,
\end{align}
where $(ii)$ follows from \cref{gg:opascas}. Note that $y_k^N$ is obtained using $N$ steps of AGD. Then, based on the analysis in \cite{nesterov2003introductory} for AGD, we have 
\begin{align}\label{eq:taun}
\|y_k^N-y_k^*\|^2\leq & \frac{\widetilde L_y +\mu_y}{\mu_y} \|y_k^0-y^*_k\|^2 \exp\big(-\frac{N}{\sqrt{\kappa_y}}\big) =  \frac{\widetilde L_y +\mu_y}{\mu_y} \|y_{k-1}^N-y^*_k\|^2 \exp\big(-\frac{N}{\sqrt{\kappa_y}}\big) \nonumber
\\\leq &\frac{2(\widetilde L_y +\mu_y)}{\mu_y} \exp\Big(-\frac{N}{\sqrt{\kappa_y}}\Big) (\|y_{k-1}^N-y_{k-1}^*\|^2 + \|y_{k-1}^*-y_k^*\|^2) \nonumber
\\\leq &\underbrace{\frac{2(\widetilde L_y +\mu_y)}{\mu_y} \exp\Big(-\frac{N}{\sqrt{\kappa_y}}\Big) }_{\tau_N}(\|y_{k-1}^N-y_{k-1}^*\|^2 + \kappa_y\|\widetilde x_k - \widetilde x_{k-1}\|^2),
\end{align}
which, in conjunction with $\widetilde x_k - \widetilde x_{k-1}= \eta_k (x_k-\widetilde x_{k-1}) + (1-\eta_k)(z_k - \widetilde x_{k-1})$, yields
\begin{align}\label{eq:telejjs}
\|y_k^N-y_k^*\|^2 \leq & \tau_N \|y_{k-1}^N-y_{k-1}^*\|^2+ \kappa_y\eta_k\tau_N\|x_k-\widetilde x_{k-1}\|^2 \nonumber
\\& +\kappa_y(1-\eta_k)\tau_N \|z_k-\widetilde x_{k-1}\|^2.
\end{align} 
Telescoping the above \cref{eq:telejjs} over $k$ yields
{\small\begin{align*}
\|y_k^N-y_k^*\|^2\leq \tau_N^k\| y_0^N-y_0^*  \|^2 + \sum_{i=0}^{k-1}\tau_N^{k-i}\kappa_y\eta_k\|x_{i+1}-\widetilde x_{i}\|^2  +\sum_{i=0}^{k-1}\tau_N^{k-i}\kappa_y(1-\eta_k)\|z_{i+1}-\widetilde x_{i}\|^2, 
\end{align*}}
\hspace{-0.13cm}which, in conjunction with \cref{cocoinasca} and \cref{jingyisang} and letting $x=x^*$, yields
{\small\begin{align}\label{tele:bigeq}
\Phi(z_{k+1})-\Phi(x^*)\leq & \big(1-\frac{\sqrt{\alpha\mu_x}}{2}\big) (\Phi(z_k) -\Phi(x^*)-\frac{1}{8\alpha} \|z_{k+1}-\widetilde x_k\|^2 - \frac{\mu_x\sqrt{\alpha\mu_x}}{16} \|x_{k+1}-\widetilde x_k\|^2 \nonumber
\\&+ \lambda\sum_{i=0}^{k-1}\tau_N^{k-i}\kappa_y\eta_k\|x_{i+1}-\widetilde x_{i}\|^2  +\lambda\sum_{i=0}^{k-1}\tau_N^{k-i}\kappa_y(1-\eta_k)\|z_{i+1}-\widetilde x_{i}\|^2 \nonumber
\\&+\Delta+ \lambda \tau_N^k \|y_0^*-y_0^N\|^2,
\end{align}}
\hspace{-0.13cm}where the  notations $\Delta$ and $\lambda$ are given by 
\begin{small}
\begin{align}\label{def:lambda}
\Delta =& \Big(4\alpha+\frac{1}{\mu_x} +\frac{\sqrt{\alpha\mu_x}}{2\mu_x}\Big)\frac{U^2\widetilde L^2_{xy}}{\mu^2_y}\Big(\frac{\sqrt{\kappa_y}-1}{\sqrt{\kappa_y}+1}\Big)^{2M} \nonumber
\\\lambda =& \Big(4\alpha+\frac{1}{\mu_x} +\frac{\sqrt{\alpha\mu_x}}{2\mu_x}\Big) \Big(L_y +\frac{\widetilde L_{xy}L_y}{\mu_y} +\Big(\frac{\rho_{xy}}{\mu_y}+\frac{\widetilde L_{xy}\rho_{yy}}{\mu_y^2}\Big)U\Big)^2.
\end{align}
\end{small}
\hspace{-0.12cm} Telescoping \cref{tele:bigeq} over $k$ from $0$ to $K-1$ and noting $0<\eta_k\leq 1$, we have 
\begin{small}
\begin{align}
\Phi(z_{K})-\Phi(x^*) \leq & \big(1-\frac{\sqrt{\alpha\mu_x}}{2}\big)^K(\Phi(z_{0})-\Phi(x^*)) - \frac{1}{8\alpha}\sum_{k=0}^{K-1} \big(1-\frac{\sqrt{\alpha\mu_x}}{2}\big)^{K-1-k}\|z_{k+1}-\widetilde x_k\|^2 \nonumber
\\&- \frac{\mu_x\sqrt{\alpha\mu_x}}{16}\sum_{k=0}^{K-1} \big(1-\frac{\sqrt{\alpha\mu_x}}{2}\big)^{K-1-k}\|x_{k+1}-\widetilde x_k\|^2 + \frac{2\Delta}{\sqrt{\alpha\mu_x}} \nonumber
\\&+\sum_{k=0}^{K-1} \big(1-\frac{\sqrt{\alpha\mu_x}}{2}\big)^{K-1-k}\lambda \tau_N^k \|y_0^*-y_0^N\|^2 \nonumber
\\&+ \lambda\sum_{k=0}^{K-1} \big(1-\frac{\sqrt{\alpha\mu_x}}{2}\big)^{K-1-k}\sum_{i=0}^{k-1}\tau_N^{k-i}\kappa_y\|x_{i+1}-\widetilde x_{i}\|^2  \nonumber
\\&+\lambda\sum_{k=0}^{K-1} \big(1-\frac{\sqrt{\alpha\mu_x}}{2}\big)^{K-1-k}\sum_{i=0}^{k-1}\tau_N^{k-i}\kappa_y\|z_{i+1}-\widetilde x_{i}\|^2, \nonumber
\end{align}
\end{small}
\hspace{-0.12cm}which, in conjunction with the fact that $k\leq K-1$, yields
{\small
\begin{align}\label{eq:zkcgsaca}
\Phi(z_{K})-&\Phi(x^*) \leq  \big(1-\frac{\sqrt{\alpha\mu_x}}{2}\big)^K(\Phi(z_{0})-\Phi(x^*)) - \frac{1}{8\alpha}\sum_{k=0}^{K-1} \big(1-\frac{\sqrt{\alpha\mu_x}}{2}\big)^{K-1-k}\|z_{k+1}-\widetilde x_k\|^2 \nonumber
\\&- \frac{\mu_x\sqrt{\alpha\mu_x}}{16}\sum_{k=0}^{K-1} \big(1-\frac{\sqrt{\alpha\mu_x}}{2}\big)^{K-1-k}\|x_{k+1}-\widetilde x_k\|^2 + \frac{2\Delta}{\sqrt{\alpha\mu_x}} \nonumber
\\&+\sum_{k=0}^{K-1} \big(1-\frac{\sqrt{\alpha\mu_x}}{2}\big)^{K-1-k}\lambda \tau_N^k \|y_0^*-y_0^N\|^2 \nonumber
\\&+ \frac{2\tau_N\lambda\kappa_y}{\sqrt{\alpha\mu_x}}\sum_{i=0}^{K-2}\tau_N^{K-2-i}\|x_{i+1}-\widetilde x_{i}\|^2  + \frac{2\tau_N\lambda\kappa_y}{\sqrt{\alpha\mu_x}}\sum_{i=0}^{K-2}\tau_N^{K-2-i}\|z_{i+1}-\widetilde x_{i}\|^2.
\end{align}}
\hspace{-0.13cm}Recall the definition of $\tau_N$ in \cref{eq:taun}. Then, choose $N$ such that 
\begin{align}\label{eq:nrequire}
\tau_N=\frac{2(\widetilde L_y +\mu_y)}{\mu_y} \exp\Big(-\frac{N}{\sqrt{\kappa_y}}\Big) \leq \min\Big\{ \frac{\sqrt{\mu_x}}{16\lambda\kappa_y\sqrt{\alpha}}, \frac{\alpha\mu_x^2}{32\lambda\kappa_y},\big(1-\frac{\sqrt{\alpha\mu_x}}{2}\big)^2\Big\}, 
\end{align}
which, in conjunction with \cref{eq:zkcgsaca}, yields  
\begin{align*}
\Phi(z_{K})-&\Phi(x^*) \leq  \big(1-\frac{\sqrt{\alpha\mu_x}}{2}\big)^K\Big(\Phi(z_{0})-\Phi(x^*)+\frac{2\lambda  \|y_0^*-y_0^N\|^2}{\sqrt{\alpha\mu_x}}\Big)+ \frac{2\Delta}{\sqrt{\alpha\mu_x}}. 
\end{align*}
Then, based on the definitions of $\lambda,\Delta$ in~\cref{def:lambda} and $L_\Phi$ in~\cref{smooth_phi}, to achieve $\Phi(z^K)-\Phi(x^*)\leq \epsilon$, we have 
\begin{align}\label{eq:com_km}
K&\leq \mathcal{O}\Big( \sqrt{\frac{1}{\mu_x\mu_y^3}}\log\frac{\mbox{\small poly}(\mu_x,\mu_y,U,\Phi(x_{0})-\Phi(x^*))}{\epsilon} \Big), \nonumber
\\M &\leq \mathcal{O} \Big(\sqrt{\frac{\widetilde L_y}{\mu_y}}\log \frac{\mbox{\small poly}(\mu_x,\mu_y,U)}{\epsilon}\Big).
\end{align}
In addition, it follows from \cref{eq:nrequire} that 
\begin{align}\label{eq:com_n}
N\leq \mathcal{O}\Big(\sqrt{\frac{\widetilde L_y}{\mu_y}}\log (\mbox{\small poly}(\mu_x,\mu_y,U))\Big).
\end{align} 
Based on \cref{eq:com_km} and \cref{eq:com_n}, the total complexity is given by 
\begin{align*}
\mathcal{C}_{\text{\normalfont sub}}(\mathcal{A},\epsilon) &\leq \mathcal{O}(n_J+n_H + n_G)\leq \mathcal{O}(K+KM+KN) \nonumber
\\&\leq \mathcal{O}\Big(\sqrt{\frac{\widetilde L_y}{\mu_x\mu_y^4}}\log\frac{\mbox{\small poly}(\mu_x,\mu_y,U,\Phi(x_{0})-\Phi(x^*))}{\epsilon}  \log \frac{\mbox{\small poly}(\mu_x,\mu_y,U)}{\epsilon} \Big), 
\end{align*}
which finishes the proof. 

\section{Proof of \Cref{convex_upper_BG}}
Let $\widetilde x^*$ be the minimizer of $\widetilde \Phi(\cdot)$. Then, applying the results in \Cref{upper_srsr} to $\widetilde\Phi(x)$ with $\mu_x = \frac{\epsilon}{B^2}$ and choosing $N=\Theta \big(\sqrt{\frac{\widetilde L_y}{\mu_y}}\log (\mbox{\small poly}(B,\epsilon,\mu_y,U))\big)$, we have 
\begin{align*}
\widetilde\Phi(z_{K})-\widetilde\Phi(\widetilde x^*) \leq &  \big(1-\frac{\sqrt{\epsilon}}{2\sqrt{2L_{\widetilde\Phi}}B}\big)^K\Big(\widetilde \Phi(z_{0})-\widetilde \Phi(\widetilde x^*)+\frac{2\sqrt{2L_{\widetilde\Phi}}B\widetilde\lambda  \|y_0^*-y_0^N\|^2}{\sqrt{\alpha\epsilon}}\Big)
\\&+ \frac{2\widetilde\Delta\sqrt{2L_{\widetilde\Phi}}B}{\sqrt{\epsilon}},
\end{align*}
where $\widetilde \Delta$ and $\widetilde \lambda$ takes the same forms as $\Delta$ and $\lambda$ in \cref{def:lambda} with $\mu_x$ being replaced by $\frac{\epsilon}{B^2}$. By choosing $M = \Theta \Big(\sqrt{\frac{\widetilde L_y}{\mu_y}}\log \frac{\mbox{\small poly}(B,\epsilon,\mu_y,U)}{\epsilon}\Big)$ in $\widetilde \Delta$,  we have $\frac{2\widetilde\Delta\sqrt{2L_{\widetilde\Phi}}B}{\sqrt{\epsilon}}\leq \frac{\epsilon}{4}$, and 
\begin{align*}
\widetilde\Phi(z_{K})-&\widetilde\Phi(\widetilde x^*) \leq  \big(1-\frac{\sqrt{\epsilon}}{2\sqrt{2L_{\widetilde\Phi}}B}\big)^K\Big(\widetilde \Phi(z_{0})-\widetilde \Phi(\widetilde x^*)+\frac{2\sqrt{2L_{\widetilde\Phi}}B\widetilde\lambda  \|y_0^*-y_0^N\|^2}{\sqrt{\alpha\epsilon}}\Big)+ \frac{\epsilon}{4},
\end{align*}
which, in conjunction with $\widetilde\Phi(z_K)\geq\Phi(z_K)$, $\widetilde\Phi(\widetilde x^*)\leq\widetilde\Phi(x^*)=\Phi(x^*)+\frac{\epsilon}{2B^2}\|x^*\|^2$ and $z_0=0$, yields 
\begin{align}\label{maiyigelunhuiba}
\Phi(z_{K})-\Phi( x^*) \leq &  \big(1-\frac{\sqrt{\epsilon}}{2\sqrt{2L_{\widetilde\Phi}}B}\big)^K\Big(\Phi(0)-\widetilde \Phi(\widetilde x^*)+\frac{2\sqrt{2L_{\widetilde\Phi}}B\widetilde\lambda  \|y_0^*-y_0^N\|^2}{\sqrt{\alpha\epsilon}}\Big) \nonumber
\\&+ \frac{\epsilon}{4} +\frac{\epsilon}{2B^2}\|x^*\|^2.
\end{align}
Based on \cref{ggsmidacposcs1}, we have $\Phi(0)-\widetilde \Phi(\widetilde x^*) \leq \Phi(0)-\Phi(x^*)$, which, combined with $\|x^*\|= B$ and $K=\Theta \Big(B \sqrt{\frac{1}{\epsilon\mu_y^3}}\log\frac{\mbox{\small poly}(\epsilon,\mu_y,B,U,\Phi(x_{0})-\Phi(x^*))}{\epsilon} \Big)$, yields  $\Phi(z_{K})-\Phi( x^*) \leq \epsilon$, and the total complexity satisfies 
\begin{align}
&\mathcal{C}_{\text{\normalfont sub}}(\mathcal{A},\epsilon) \leq \mathcal{O}(n_J+n_H + n_G)\leq \mathcal{O}(K+KM+KN) \nonumber
\\&\leq \mathcal{O}\Big(B\sqrt{\frac{\widetilde L_y}{\epsilon\mu_y^4}}\log\frac{\mbox{\small poly}(\epsilon,\mu_y,B,U,\Phi(x_{0})-\Phi(x^*))}{\epsilon}  \log \frac{\mbox{\small poly}(B,\epsilon,\mu_y,U)}{\epsilon} \Big), 
\end{align}
which finishes the proof. 

\chapter{Proof of \Cref{chp_lower_bilevel}}\label{appendix:lower_bilevel}

\section{Proof of \Cref{thm:low1}}\label{appen:thm1}
In this section, we provide a complete proof of \Cref{thm:low1} under the strongly-convex-strongly-convex geometry. Note that our construction sets the dimensions of variables $x$ and $y$ to be the same, i.e.,  $p=q=d$.   
 From our proof sketch, the main proofs are divided into four steps: 1) constructing the worst-case instance that belongs to the problem class $\mathcal{F}_{scsc}$ defined in \Cref{de:pc}; 2) characterizing the optimal point $x^*=\argmin_{x\in\mathbb{R}^d}\Phi(x)$; 3) characterizing the subspaces $\mathcal{H}_x^k,\mathcal{H}_y^k$; and 4) developing lower bounds on the convergence and complexity.  

\vspace{0.2cm}
\noindent {\bf Step 1: construct the worst-case instance that satisfies \Cref{de:pc}.}
\vspace{0.2cm}

In this step, we show that the constructed $f,g$ in~\cref{str_fg} satisfy Assumptions~\ref{fg:smooth} and~\ref{g:hessiansJaco}, and $\Phi(x)$ is $\mu_x$-strongly-convex. It can be seen from~\cref{str_fg} that $f,g$ satisfies \cref{def:first} \eqref{df:sec} and \eqref{def:three} in Assumptions~\ref{fg:smooth} and~\ref{g:hessiansJaco} with arbitrary constants $L_x,L_y,\widetilde L_y,\widetilde L_{xy}$ and $\rho_{xy}=\rho_{yy}=0$ but requires $L_{xy}\geq \frac{(L_x-\mu_x)(\widetilde L_y-\mu_y)}{2\widetilde L_{xy}}$ (which is still at a constant level) due to the introduction of the term $\frac{\alpha\beta}{\widetilde L_{xy}}x^TZ^3y$ in $f$. We note that such a term introduces necessary connection between $f$ and $g$, and yields a tighter lower bound, as pointed out in the remark at the end of the proof sketch of \Cref{thm:low1}. 

We next show that the overall objective function $\Phi(x)=f(x,y^*(x)) $ is $\mu_x$-strongly-convex. 
According to~\cref{str_fg}, we have $g(x,\cdot)$ is $\mu_y$-strongly-convex with a single minimizer 
$y^*(x) = (\beta Z^2+\mu_y I)^{-1} \big(\frac{\widetilde L_{xy}}{2}Zx-b\big)$,
and hence we obtain from \cref{objective_deter} that $\Phi(x)$ is given by 
\begin{align}\label{phi_x}
\Phi(x) = &\frac{1}{2} x^T(\alpha Z^2 +\mu_x I) x -\frac{\alpha\beta}{\widetilde L_{xy}}x^TZ^3(\beta Z^2+\mu_y I)^{-1} \Big(\frac{\widetilde L_{xy}}{2}Zx-b\Big) \nonumber
\\ &+\frac{\bar L_{xy}}{2}x^TZ(\beta Z^2+\mu_y I)^{-1} \Big(\frac{\widetilde L_{xy}}{2}Zx-b\Big) + \frac{\bar L_{xy}}{\widetilde L_{xy}} b^T (\beta Z^2+\mu_y I)^{-1} \Big(\frac{\widetilde L_{xy}}{2}Zx-b\Big) \nonumber
\\&+ \frac{L_y}{2} \Big(\frac{\widetilde L_{xy}}{2}Zx-b\Big)^T (\beta Z^2+\mu_y I)^{-1}(\beta Z^2+\mu_y I)^{-1} \Big(\frac{\widetilde L_{xy}}{2}Zx-b\Big).  
\end{align}
Note that $Z$ is symmetric and invertible, and hence the singular value decomposition of $Z$ can be written as $Z= U \,\text{Diag}\{\sigma_1,...,\sigma_d\}U^T$, where $\sigma_i>0, i=1,...,d$ and $U$ is an orthogonal matrix. Then, for any integers $i,j>0$, simple calculation yields
\begin{align}\label{iterchange}
Z^i (\beta Z^2+\mu_y I)^{-j} &= U \text{Diag}\bigg\{\frac{\sigma^i_1}{(\beta \sigma_1^2+\mu_y)^j},...,\frac{\sigma^i_d}{(\beta \sigma_d^2+\mu_y)^j}\bigg\}U^T \nonumber
\\&= (\beta Z^2+\mu_y I)^{-j} Z^i. 
\end{align}
Using the relationship in \cref{iterchange}, we have  
\begin{align*}
\frac{1}{2} x^T\alpha Z^2 x =& \frac{\alpha\beta}{2} x^T Z^4 (\beta Z^2+\mu_y I)^{-1} x + \frac{\alpha \mu_y}{2}x^TZ^2(\beta Z^2+\mu_y I)^{-1}x,
\end{align*}
which, in conjunction with  \cref{phi_x} and \cref{iterchange}, yields
\begin{align}
\Phi(x) = & \frac{1}{2} \mu_x \|x\|^2 + \frac{2\alpha\mu_y+\bar L_{xy}\widetilde L_{xy}}{4} x^T Z^2 (\beta Z^2+\mu_y I)^{-1} x - \frac{\bar L_{xy}}{\widetilde L_{xy}} b^T(\beta Z^2+\mu_y I)^{-1} b
 \nonumber
\\&+ \frac{L_y}{2} \Big(\frac{\widetilde L_{xy}}{2}Zx-b\Big)^T (\beta Z^2+\mu_y I)^{-2}\Big(\frac{\widetilde L_{xy}}{2}Zx-b\Big) \nonumber
\\&+ \frac{2\alpha\beta}{\widetilde L_{xy}^2}b^TZ^2 (\beta Z^2+\mu_y I)^{-1} b 
\end{align}
which is $\mu_x$-strongly-convex. 

\vspace{0.2cm}
\noindent {\bf Step 2: characterize $x^*=\argmin_{x\in\mathbb{R}^d}\Phi(\cdot).$}
\vspace{0.2cm}

Based on the form of $\Phi(\cdot)$,  we have 
{\small\begin{align}\label{eq:aboveg}
\nabla \Phi(x) =& (\beta Z^2 +\mu_y I)^2\mu_x x + \Big(\alpha \mu_y +\frac{\bar L_{xy}\widetilde L_{xy}}{2}\Big) (\beta Z^2 +\mu_y I) Z^2 x + \frac{L_y\widetilde L_{xy}}{2}\Big(\frac{\widetilde L_{xy}}{2}Z^2x-Zb\Big) \nonumber
\\=& \Big( \beta^2\mu_x +\alpha \beta\mu_y+\frac{ \beta \bar L_{xy}\widetilde L_{xy}}{2}\Big) Z^4x + (2\beta\mu_x\mu_y+\alpha\mu_y^2+\frac{\mu_y\bar L_{xy}\widetilde L_{xy}}{2} +\frac{L_y\widetilde L_{xy}^2}{4}) Z^2x  \nonumber
\\&+ \mu_x\mu_y^2 x  - \frac{L_y\widetilde L_{xy}}{2} Zb.
\end{align}}
\hspace{-0.15cm}By setting $\nabla\Phi(x^*) = 0$, we have 
\begin{align}\label{z_equations}
Z^4x^* + &\underbrace{\frac{2\beta\mu_x\mu_y+\alpha\mu_y^2+\frac{\mu_y\bar L_{xy}\widetilde L_{xy}}{2} +\frac{L_y\widetilde L_{xy}^2}{4}}{\beta^2\mu_x +\alpha \beta\mu_y+\frac{ \beta \bar L_{xy}\widetilde L_{xy}}{2}}}_{\lambda} Z^2x^*  \nonumber
\\&\hspace{1cm}+ \underbrace{\frac{ \mu_x\mu_y^2}{\beta^2\mu_x +\alpha \beta\mu_y+\frac{ \beta \bar L_{xy}\widetilde L_{xy}}{2}}}_{\tau} x^* = \underbrace{\frac{L_y\widetilde L_{xy}Zb}{2(\beta^2\mu_x +\alpha \beta\mu_y+\frac{ \beta \bar L_{xy}\widetilde L_{xy}}{2})}}_{\widetilde b},
\end{align}
where we define $\lambda,\tau,\widetilde b$ for notational convenience. The following lemma establish useful properties of $x^*$ under a specific selection of $\widetilde b$. 
\begin{lemma}\label{le:x_star}
Let $b$ is chosen such that $\widetilde b$  satisfies $\widetilde  b_1= (2+\lambda +\tau) r -(3+\lambda)r^2 + r^3, \widetilde b_2 = r-1$ and $\widetilde b_t=0, t=3,...,d$, where $0<r<1$ is a solution of equation 
\begin{align}\label{eq:soulc}
1 - (4+\lambda)r + (6+2\lambda +\tau) r^2 -(4+\lambda) r^3 +r^4=0. 
\end{align}
Let $\hat x$ be a vector with each coordinate $\hat x_i = r^i$. Then, we have
\begin{align}
\|\hat x - x^*\| \leq \frac{(7+\lambda)
}{\tau}r^d.
\end{align}
\end{lemma}
\begin{proof}
Note that the choice of $b$ is achievable because $Z$ is invertible with $Z^{-1}$, which is given by 
\begin{align*}
Z^{-1}= \begin{bmatrix}
 &   &  & 1\\
 &  &  1& 1  \\
  &\text{\reflectbox{$\ddots$}} &\text{\reflectbox{$\ddots
  $}} & \vdots \\
  1& 1  &1  &1 \\ 
\end{bmatrix}.
\end{align*}
Then, define a vector $\hat b$ with $\hat b_t=\widetilde b_t$ for $t=1,...,d-2$ and 
\begin{align}\label{eq:helpss}
\hat b_{d-1} =& r^{d-3} - (4+\lambda) r^{d-2} +(6+2\lambda+\tau)r^{d-1} - (4+\lambda) r^d \overset{\eqref{eq:soulc}}= -r^{d+1}\nonumber
\\ \hat b_d =& r^{d-2} - (4+\lambda) r^{d-1} +(5+2\lambda +\tau) r^d \overset{\eqref{eq:soulc}}= -r^d + (4+\lambda) r^{d+1}-r^{d+2}.
\end{align}
Then, it can be verified that $\hat x$ satisfies the following equations 
\begin{align*}
(2+\lambda +\tau)\hat x_1 - \hat (3+\lambda) x_2 +\hat x_3 &= \hat b_1 \nonumber
\\ -(3+\lambda) \hat x_1 +(6+2\lambda +\tau) \hat x_2 -(4+\lambda) \hat x_3 + \hat x_4 &= \hat b_2 \nonumber
\\ \hat x_t - (4+\lambda)\hat x_{t+1} + (6+2\lambda +\tau) \hat x_{t+2} -(4+\lambda) \hat x_{t+3}+\hat x_{t+4} &= \hat  b_{t+2}, \mbox{ for } 1\leq t\leq d-4  \nonumber
\\ \hat x_{d-3} - (4+\lambda)\hat x_{d-2} +(6+2\lambda +\tau) \hat x_{d-1} -(4+\lambda)\hat x_d &= \hat b_{d-1}\nonumber
\\\hat x_{d-2} -(4+\lambda) \hat x_{d-1} + (5+2\lambda+\tau) \hat x_d &= \hat b_d, 
\end{align*}
which, in conjunction with the forms of  $Z^2$ and $Z^4$ in \cref{matrices_coupling}, yields 
\begin{align*}
Z^4\hat x + \lambda Z^2 \hat x + \tau \hat x  = \hat b.
\end{align*}
Noting that $Z^4 x^* + \lambda Z^2 x^* + \tau x^*  = \widetilde b$, we have
\begin{align*}
\tau \|x^*-\hat x \|\leq \|(Z^4+\lambda Z^2 + \tau I)(x^*-\hat x)\| = \|\widetilde b -\hat b\| \overset{(i)}\leq (7+\lambda) r^d
\end{align*}
where $(i)$ follows from the definition of $\hat b $ in \cref{eq:helpss}. 
\end{proof}
\noindent {\bf Step 3: characterize subspaces $\mathcal{H}_x^K$ and $\mathcal{H}_y^K$.}
\vspace{0.2cm}

In this step, we characterize the forms of  the subspaces $\mathcal{H}_x^K$ and $\mathcal{H}_y^K$
for bilevel optimization algorithms considered in \Cref{alg_class}. Based on the constructions of $f,g$ in~\cref{str_fg}, we have
\begin{align*}
\nabla_x f(x,y) &= (\alpha Z^2+\mu_x I)x -\frac{\alpha\beta}{\widetilde L_{xy}} Z^3y +\frac{\bar L_{xy}}{2}Zy\nonumber
\\ \nabla_y f(x,y) &= -\frac{\alpha \beta}{\widetilde L_{xy}} Z^3x + \frac{\bar L_{xy}}{2} Zx +L_y y +\frac{\bar L_{xy}}{\widetilde L_{xy}} b - \frac{2\alpha\beta}{\widetilde L_{xy}^2} Z^2 b
\\\nabla_y g(x,y) &= (\beta Z^2 + \mu_y I) y - \frac{\widetilde L_{xy}}{2} Zx + b
\\\nabla_x\nabla_y g(x,y) &= -\frac{\widetilde L_{xy}}{2} Z, \; \nabla_y^2g(x,y) = \beta Z^2 +\mu_y I, 
\end{align*}
which, in conjunction with \cref{hxk} and \cref{x_span}, yields 
\begin{align}\label{hxy_steps}
\mathcal{H}_y^0 &= \mbox{Span}\{0\}, ...., \mathcal{H}_y^{s_0} = \mbox{Span}\{Z^{2(s_0-1)}b,...,Z^2b,b\} \nonumber
\\\mathcal{H}_x^0 &= .... \mathcal{H}_x^{s_0-1} =\mbox{Span}\{0\}, \mathcal{H}_x^{s_0} \subseteq\mbox{Span}\{Z^{2(T+s_0)}(Zb),....,Z^2(Zb),(Zb)\}. 
\end{align}
Repeating the same steps as in~\cref{hxy_steps}, it can be verified that 
\begin{align}\label{eq:intimedia}
H_{x}^{s_{Q-1}} \subseteq \mbox{Span}\{Z^{2(s_{Q-1} + QT+Q)}(Zb),...,Z^{2j}(Zb),...,Z^2(Zb),(Zb)\}. 
\end{align}
Recall from \cref{x_span} that $\mathcal{H}_x^K = \mathcal{H}_x^{s_{Q-1}}$ and $s_{Q-1}\leq K$. Then, we obtain from \cref{eq:intimedia} that $H_{x}^{K }$ satisfies 
\begin{align}\label{eq:H_xK}
H_{x}^{K }\subseteq \mbox{Span}\{Z^{2(K+QT+Q)}(Zb),....,Z^2(Zb),(Zb) \}.
\end{align}
\noindent {\bf Step 4: characterize convergence and complexity.}
\vspace{0.2cm}

Based on the results in Steps 1 and 2, we are now ready to provide a lower bound on the convergence rate and complexity of bilevel optimization algorithms.  Let $M = K+QT+Q+ 2$ and $x_0 = {\bf 0}$, and  
let  dimension $d$ satisfy 
\begin{align}\label{d_conditions}
d> \max\Big\{2M,M+1+\log_{r}\Big(\frac{\tau}{4(7+\lambda)}\Big)\Big\}.
\end{align}
Recall from \Cref{le:x_star} that $Zb$ has zeros at all coordinates $t=3,...,d$. Then, based on the form of subspaces $\mathcal{H}_x^K$ in \Cref{eq:H_xK} and using the zero-chain property in \Cref{zero_chain}, we have $x^K$ has zeros at coordinates $t=M+1,...,d$, and hence 
\begin{align}\label{x_kssca}
\|x^K-\hat x\| \geq \sqrt{\sum_{i=M+1}^d\|\hat x_i\|}  = r^{M}\sqrt{r^2+...+r^{2(d-M)}} \overset{(i)}\geq \frac{r^M}{\sqrt{2}}\|\hat x -x_0\|,
\end{align}
where $(i)$ follows from \cref{d_conditions}. Then, based on \cref{le:x_star} and~\cref{d_conditions}, we have 
\begin{align}\label{eq:hatxxs}
\|\hat x - x^*\| \leq \frac{7+\lambda}{\tau} < \frac{r^M}{2\sqrt{2}} r \overset{(i)}\leq \frac{r^M}{2\sqrt{2}} \|\hat x-x_0\|, 
\end{align}
where $(i)$ follows from the fact that $\|\hat x\| \geq r$. Combining \cref{x_kssca} and \cref{eq:hatxxs} further yields 
\begin{align}\label{eq:supp1}
\|x^K-x^*\| &\geq \|x^K-\hat x\| - \|\hat x - x^*\|  \nonumber
\\&\geq \frac{r^M}{\sqrt{2}} \|\hat x -x_0\| - \frac{r^M}{2\sqrt{2}}\|\hat x-x_0\| =\frac{r^M}{2\sqrt{2}}\|\hat x-x_0\|.
\end{align}
In addition,  note that 
\begin{align*}
\|x^* -\hat x\| \leq \frac{7+\lambda}{\tau} r^d  \overset{\eqref{d_conditions}}\leq \frac{1}{4} r \leq \frac{1}{4}\|\hat x\| \leq \frac{1}{4}\|\hat x - x^*\| +\frac{1}{4} \|x^*\|,
\end{align*}
which, in conjunction with $\|x_0-\hat x\|\geq \|x^*-x_0\| - \|x^*-\hat x\|$, yields
\begin{align}\label{eq:supp2}
\|x_0-\hat x\|\geq \frac{2}{3} \|x^*-x_0\|.
\end{align}
Combining \cref{eq:supp1} and \cref{eq:supp2} yields 
\begin{align}\label{x_lowbound}
\|x^K-x^*\| \geq \frac{\|x^*-x_0\|}{3\sqrt{2}} r^M.
\end{align}
 Then, since the objective function $\Phi(x)$ is $\mu_x$-strongly-convex, we have $\Phi(x^K)-\Phi(x^*)\geq \frac{\mu_x}{2}\|x^K-x^*\|^2$ and $\|x_0-x^*\|^2\geq \Omega(\mu_y^2)(\Phi(x_0)-\Phi(x^*))$, and hence \cref{x_lowbound} yields
\begin{align}
\Phi(x^K)-\Phi(x^*) \geq \Omega\Big( \frac{\mu_x\mu_y^2(\Phi(x_0)-\Phi(x^*))}{36} r^{2M}\Big). 
\end{align}
Recall that $r$ is the solution of equation $1 - (4+\lambda)r + (6+2\lambda +\tau) r^2 -(4+\lambda) r^3 +r^4=0$. Based on Lemma 4.2 in~\cite{zhang2019lower}, we have 
\begin{align}\label{eq:prange}
1-\frac{1}{\frac{1}{2}+\sqrt{\frac{\lambda}{2\tau}+\frac{1}{4}}} < r< 1. 
\end{align}
which, in conjunction with the definitions of $\lambda$ and $\tau$ in \cref{z_equations} and the fact $\bar L_{xy}\geq 0$, yields the first result \cref{result:first} in \Cref{thm:low1}. Then, in order to achieve an $\epsilon$-accurate solution, i.e., $\Phi(x^K)-\Phi(x^*) \leq \epsilon$, it requires
\begin{align}\label{eq: mbound}
M=K+QT+Q + 2 &\geq \frac{\log  \frac{\mu_x\mu_y^2(\Phi(x_0)-\Phi(x^*))}{\epsilon} }{2\log \frac{1}{r}}  \overset{(i)} \geq \Omega\Big(\sqrt{\frac{\lambda}{2\tau}}\log  \frac{\mu_x\mu_y^2(\Phi(x_0)-\Phi(x^*))}{\epsilon} \Big) \nonumber
\\&\geq \Omega\bigg(\sqrt{\frac{L_y\widetilde L_{xy}^2}{\mu_x\mu_y^2}}\log  \frac{\mu_x\mu_y^2(\Phi(x_0)-\Phi(x^*))}{\epsilon} \bigg),
\end{align}
where $(i)$ follows from \cref{eq:prange}. 
Recall that the complexity measure is given by $\mathcal{C}_{\text{\normalfont sub}}(\mathcal{A},\epsilon) \geq \Omega(n_J+n_H + n_G)$, where the numbers $n_J,n_H$ of Jacobian- and Hessian-vector products  are given by $n_J=Q$ and $n_H=QT$ and the number $n_G$ of gradient evaluation is given by $n_G=K$. Then, the total complexity $\mathcal{C}_{\text{\normalfont sub}}(\mathcal{A},\epsilon)\geq \Omega(Q+QT+K)$, which combined with \cref{eq: mbound}, implies 
\begin{align*}
\mathcal{C}_{\text{\normalfont sub}}(\mathcal{A},\epsilon) \geq \Omega\bigg(\sqrt{\frac{L_y\widetilde L_{xy}^2}{\mu_x\mu_y^2}}\log  \frac{\mu_x\mu_y^2(\Phi(x_0)-\Phi(x^*))}{\epsilon}  \bigg).
\end{align*}
Then, the proof is complete. 


\section{Proof of \Cref{main:convex}}\label{apep:mainconvexs}
In this section, we provide the proof for \Cref{main:convex} under the convex-strongly-convex geometry.  

 The proof is divided into the following steps: 
1) constructing the worst-case instance that belongs to the convex-strongly-convex problem class $\mathcal{F}_{csc}$ defined in \Cref{de:pc}; 2) characterizing $x^*\in \argmin_{x\in\mathbb{R}^d}\Phi(x)$; 3) developing the lower bound on gradient norm $\|\nabla\Phi(x)\|$ when last several coordinates of $x$ are zeros; 4) characterizing the subspaces $\mathcal{H}_x^k$ and $\mathcal{H}_x^k$; and 5) characterizing the convergence and complexity.

\vspace{0.2cm}
\noindent {\bf Step 1: construct the worst-case instance that satisfies \Cref{de:pc}.}
\vspace{0.2cm}

It can be verified that  the constructed $f,g$ in \cref{con_fg} satisfies \cref{def:first} \eqref{df:sec} and \eqref{def:three}  in Assumptions~\ref{fg:smooth} and~\ref{g:hessiansJaco}.  Then, similarly to the proof of \Cref{thm:low1}, we have $y^*(x) = (\beta Z^2 +\mu_y I)^{-1} (\frac{\widetilde L_{xy}}{2} Zx-b)$ and hence $\Phi(x) = f(x,y^*(x))$ takes the form of 
\begin{align*}
\Phi(x) =  \frac{L_x}{8} x^T Z^2 x+ \frac{L_y}{2} \Big (\frac{\widetilde L_{xy}}{2} Zx-b \Big)^T (\beta Z^2 +\mu_y I)^{-2}\Big(\frac{\widetilde L_{xy}}{2} Zx-b\Big),
\end{align*}
which can be verified to be convex. 

\vspace{0.2cm}
\noindent {\bf Step 2: characterize $x^*$.}
\vspace{0.2cm}

Note that 
the gradient $\nabla \Phi(x)$ is given by 
\begin{align}\label{conv_gdform}
\nabla \Phi(x) =  \frac{L_x}{4} Z^2 x + \frac{L_y\widetilde L_{xy}}{2} Z (\beta Z^2 +\mu_y I)^{-2}\Big(\frac{\widetilde L_{xy}}{2} Zx-b\Big). 
\end{align}
Then, setting $\nabla \Phi(x^*) = 0$ and using \cref{iterchange}, we have 
\begin{align}\label{x_star_convex}
\Big (\frac{L_x\beta^2}{4} Z^6 +  \frac{L_x\beta^2\beta\mu_y}{2} Z^4 +\Big (\frac{L_y\widetilde L_{xy}^2}{4}+\frac{L_x\mu_y^2}{4}\Big) Z^2 \Big) x^* =  \frac{L_y\widetilde L_{xy}}{2} Zb. 
\end{align}
Let $\widetilde b =\frac{L_y\widetilde L_{xy}}{2} Zb$, and we choose $b$ such that $\widetilde b_t = 0$ for $t=4,...,d$ and 
\begin{align}\label{de:bwidetilde}
\widetilde b_1 = &\frac{B}{\sqrt{d}} \Big( \frac{5}{4} L_x\beta^2+ L_x \beta\mu_y+ \frac{\widetilde L^2_{xy}L_y}{4}+\frac{L_x}{4} \mu_y^2\Big), \nonumber
\\ \widetilde b_2 =& \frac{B}{\sqrt{d}} (-L_x \beta^2 -\frac{L_x\beta }{2}\mu_y),\; \widetilde b_3 = \frac{B}{\sqrt{d}} \frac{L_x\beta^2}{4},
\end{align}
where the selection of $b$ is achievable because $Z$ is invertible with $Z^{-1}$ given by 
\begin{align*}
Z^{-1}= \begin{bmatrix}
 &   &  & -1\\
 &  &  -1& -1  \\
  &\text{\reflectbox{$\ddots$}} &\text{\reflectbox{$\ddots
  $}} & \vdots \\
  -1& -1  &-1  &-1 \\ 
\end{bmatrix}.
\end{align*}
Based on the forms of $Z^2$ in~\cref{matrices_coupling_convex} and the forms of $Z^4, Z^6$ that 
{\footnotesize
\begin{align}\label{coup_matrices_4_6}
Z^4=
\begin{bmatrix}
 5& -4 & 1 &  & &\\
 -4& 6  &-4 &1  & &\\
 1& -4 & 6 & -4 & 1&\\
  &  \ddots&  \ddots&  \ddots &  \ddots & \ddots\\
   &  & 1& -4 &  6& -3 \\
  &  & & 1 & -3 & 2\\ 
\end{bmatrix},\;
Z^6=
\begin{bmatrix}
 14& -14 & 6 & -1 & & & &\\
 -14& 20  &-15 & 6   & -1 & &&\\
 6& -15 & 20 & -15 & 6& -1&&\\
 -1&6& -15 & 20 & -15 & 6& -1&\\
&   \ddots&  \ddots&  \ddots&  \ddots &  \ddots & \ddots& \ddots\\
&& -1 & 6 & -15& 20 &  -15& 5 \\
 &&  &  -1& 6& -15 &  19& -9 \\
  && &  &-1 & 5 & -9 & 5\\
\end{bmatrix},
\end{align}}
 \hspace{-0.15cm}it can be checked from~\cref{x_star_convex}  that $x^* = \frac{B}{\sqrt{d}} {\bf 1}$, where $\bf 1$ is an all-ones vector and hence $\|x^*\|=B$. 

\vspace{0.2cm}
\noindent {\bf Step 3: characterize lower bound on $\|\nabla \Phi(x)\|$.}
\vspace{0.2cm}

Next, we characterize a lower bound on $\|\nabla \Phi(x)\|$ when the last three coordinates of $x$ are zeros, i.e., $x_{d-2}=x_{d-1}=x_d=0$. Let $\Omega = [I_{d-3}, {\bf 0}]^T$ and define $\widetilde x\in\mathbb{R}^{d-3}$ such that $\widetilde x_i=x_i$ for $i=1,...,d-3$. Then for any matrix $H$, $H\Omega$ is equivalent to removing the last three columns of $H$. Then,  based on the form of $\nabla \Phi(x)$ in~\cref{conv_gdform}, we have 
\begin{align}\label{eq:gdnorm1}
\min_{x\in\mathbb{R}^d: x_{d-2}=x_{d-1}=x_d=0} \|\nabla\Phi(x)\|^2 = \min_{\widetilde x\in\mathbb{R}^{d-3}} \|H\Omega \widetilde x -  (\beta Z^2 +\mu_y I)^{-2}\widetilde b\|^2
\end{align}
where matrix $H$ is given by 
\begin{align}\label{wideHdef}
H =  (\beta Z^2 +\mu_y I)^{-2}\underbrace{\Big (\frac{L_x\beta^2}{4} Z^6 +  \frac{L_x\beta^2\beta\mu_y}{2} Z^4 +\Big (\frac{L_y\widetilde L_{xy}^2}{4}+\frac{L_x\mu_y^2}{4}\Big) Z^2 \Big)}_{\widetilde H}.
\end{align}
Then using an approach similar to (7) in~\cite{carmon2019lower}, we have 
\begin{align}\label{eq:gdnomr2s}
\min_{\widetilde x\in\mathbb{R}^{d-3}} \|H\Omega \widetilde x -  (\beta Z^2 +\mu_y I)^{-2}\widetilde b\|^2 = \big(\widetilde b^T (\beta Z^2 +\mu_y I)^{-2} z\big)^2,
\end{align}
where $z$ is the normalized (i.e., $\|z\|=1$) solution of equation $(H\Omega)^Tz=0$. Next we characterize the solution $z$. Since $H=(\beta Z^2 +\mu_y I)^{-2}\widetilde H$, we have  
\begin{align}
(H\Omega)^Tz = (\widetilde H \Omega)^T(\beta Z^2 +\mu_y I)^{-2} z = 0.
\end{align}
Based on the definition of $\widetilde H $ in \cref{wideHdef} and the forms of $Z^2,Z^4,Z^6$ in \cref{matrices_coupling_convex} and~\cref{coup_matrices_4_6}, we have that the solution $z$ takes the form of $ z= \lambda (\beta Z^2 +\mu_y I)^{2}h$, where $\lambda$ is a factor such that $\|z\|=1$ and $h$ is a vector satisfying $h_t = t$ for $t=1,...,d$. Based on the definition of $Z^2$ in \cref{matrices_coupling_convex}, we have
\begin{align*}
1= \|z\| =& \lambda \sqrt{\sum_{i=1}^{d-2}(i\mu_y^2)^2 + ((d-1)\mu_y^2-\beta^2)^2 + (d\mu_y^2 +\beta^2+2\beta\mu_y)^2} \nonumber
\\\leq & \lambda \sqrt{\sum_{i=1}^{d-2}(i\mu_y^2)^2 + 2(d-1)^2\mu_y^4+2\beta^4 + 2d^2\mu_y^4 +2(\beta^2+2\beta\mu_y)^2} \nonumber
\\<& \lambda\sqrt{\frac{2}{3}\mu_y^4(d+1)^3 +4\beta^4+8\beta^3\mu_y+8\beta^2\mu_y^2},
\end{align*}
which further implies that 
\begin{align}\label{ggnormsacsa}
\lambda > \frac{1}{\sqrt{\frac{2}{3}\mu_y^4(d+1)^3 +4\beta^4+8\beta^3\mu_y+8\beta^2\mu_y^2}}.
\end{align}
Then, combining \cref{eq:gdnorm1}, \cref{eq:gdnomr2s} and \cref{ggnormsacsa} yields
\begin{align}\label{min_gdnorm}
\min_{x: x_{d-2}=x_{d-1}=x_d=0} \|\nabla\Phi(x)\|^2 =&\big(\widetilde b^T (\beta Z^2 +\mu_y I)^{-2} z\big)^2 = (\lambda \widetilde b^T h)^2 = \lambda^2 (\widetilde b_1 + 2\widetilde b_2+3\widetilde b_3)^2 \nonumber
\\\overset{(i)}=& \lambda^2 \frac {B^2}{4d}\Big(\frac{\widetilde L^2_{xy}L_y}{4}+\frac{L_x\mu_y^2}{4}\Big)^2 \nonumber
\\ \geq& \frac{B^2\Big(\frac{\widetilde L^2_{xy}L_y}{4}+\frac{L_x\mu_y^2}{4}\Big)^2}{\frac{8}{3}\mu_y^4d(d+1)^3 +16d\beta^4+32d\beta^3\mu_y+32d\beta^2\mu_y^2} \nonumber
\\\overset{(ii)}\geq &  \frac{B^2\Big(\frac{\widetilde L^2_{xy}L_y}{4}+\frac{L_x\mu_y^2}{4}\Big)^2}{8\mu_y^4d^4 +16d\beta^4+32d\beta^3\mu_y+32d\beta^2\mu_y^2}
\end{align}
where $(i)$ follows from the definition of $\widetilde b$ in~\cref{de:bwidetilde}, and $(ii)$ follows from $d\geq 3$. 

\vspace{0.2cm}
\noindent {\bf Step 4: characterize subspaces $\mathcal{H}_x^k$ and $\mathcal{H}_x^k$ .}
\vspace{0.2cm}

Based on the constructions of $f,g$ in \cref{con_fg}, we have 
\begin{align*}
\nabla_x f(x,y) &= \frac{L_x}{4} Z^2x,\;  \nabla_y f(x,y) = L_y y, \;\nabla_x\nabla_y g(x,y) = -\frac{\widetilde L_{xy}}{2} Z 
\nonumber
\\ \nabla_y^2g(x,y) &= \beta Z^2 +\mu_y I, \; \nabla_y g(x,y) = (\beta Z^2 + \mu_y I) y - \frac{\widetilde L_{xy}}{2} Zx + b, 
\end{align*}
which, in conjunction with \cref{hxk} and \cref{x_span}, yields 
\begin{align*}
\mathcal{H}_y^0 &= \mbox{Span}\{0\}, ...., \mathcal{H}_y^{s_0} = \mbox{Span}\{Z^{2(s_0-1)}b,...,Z^2b,b\} \nonumber
\\\mathcal{H}_x^0 &= .... \mathcal{H}_x^{s_0-1} =\mbox{Span}\{0\}, \mathcal{H}_x^{s_0} =\mbox{Span}\{Z^{2(T+s_0-2)}(Zb),....,Z^2(Zb),(Zb)\}. 
\end{align*}
Repeating the above procedure and noting that $s_{Q-1}\leq K$ yields
\begin{align}\label{eq:subsscas} 
\mathcal{H}_x^K = \mathcal{H}_x^{s_{Q-1}} &=  \mbox{Span}\{Z^{2(s_{Q-1}+QT-Q-1)}(Zb),....,Z^2(Zb),(Zb)\}\nonumber
\\&\subseteq \mbox{Span}\{Z^{2(K+QT-Q)}(Zb),....,Z^2(Zb),(Zb)\}. 
\end{align}

\noindent {\bf Step 5: characterize convergence and complexity.}
\vspace{0.2cm}

Let $M = K+QT-Q+3$ and consider an equation 
\begin{align}\label{d_eqtion}
r^4 +r \Big(\frac{2\beta^4}{\mu_y^4}+\frac{4\beta^3}{\mu_y^3}+\frac{4\beta^2}{\mu_y^2}\Big) = \frac{B^2\Big(\widetilde L^2_{xy}L_y+L_x\mu_y^2\Big)^2}{128\mu_y^4\epsilon^2},
\end{align}
where has a solution denoted as $r^*$.  
We choose $d = \lfloor r^* \rfloor$.Then, based on  \cref{min_gdnorm}, we have
\begin{align}\label{eq:quitescs}
\min_{x: x_{d-2}=x_{d-1}=x_d=0} \|\nabla\Phi(x)\|^2 \geq  \frac{B^2\Big(\frac{\widetilde L^2_{xy}L_y}{4}+\frac{L_x\mu_y^2}{4}\Big)^2}{8\mu_y^4(r^*)^4 +16r^*\beta^4+32r^*\beta^3\mu_y+32r^*\beta^2\mu_y^2} = \epsilon^2.
\end{align}
Then, to achieve $\|\nabla\Phi(x^K)\|< \epsilon$, it requires that $M> d-3$. Otherwise (i.e., if $M\leq d-3$), based on \cref{eq:subsscas} and the fact that $Zb$ has nonzeros only at the first three coordinates, we have $x^K$ has zeros at the last three coordinates and hence it follows from \cref{eq:quitescs} that 
$\|\nabla \Phi(x^K)\|\geq \epsilon$, which leads to a contradiction. Therefore, we have $M>  \lfloor r^* \rfloor-3$. Next, we characterize the total complexity.  Using the metric in~\cref{complexity_measyre}, we have 
\begin{align*} 
\mathcal{C}_{\text{\normalfont norm}}(\mathcal{A},\epsilon)\geq  \Omega(Q+QT+K) \geq  \Omega(M) \geq  \Omega(r^*).
\end{align*}
 Then, the proof is complete. 
\section{Proof of \Cref{co:co1}}
In this case, the condition number $\kappa_y$ satisfies $\kappa_y=\frac{\widetilde L_y}{\mu_y}\leq \mathcal{O}(1)$.
Then, it can be verified that $r^*$ satisfies $(r^*)^3> \Omega (\frac{2\beta^4}{\mu_y^4}+\frac{4\beta^3}{\mu_y^3}+\frac{4\beta^2}{\mu_y^2})$, and hence it follows from \cref{eq:rsoltion} that 
\begin{align*}
\mathcal{C}_{\text{\normalfont norm}}(\mathcal{A},\epsilon) \geq r^*\geq \Omega \Big(\frac{B^{\frac{1}{2}}(\widetilde L^2_{xy}L_y+L_x\mu_y^2)^{\frac{1}{2}}}{\mu_y\epsilon^{\frac{1}{2}}} \Big).
\end{align*}
\section{Proof of \Cref{co:co2}}
To prove  \Cref{co:co2}, we consider two cases $\mu_y\geq \Omega(\epsilon^{\frac{3}{2}})$ and $\mu_y\leq \mathcal{O}(\epsilon^{\frac{3}{2}})$ separately. 

{\bf Case 1: $\mu_y\geq \Omega(\epsilon^{\frac{3}{2}})$.} For this case, we have 
$\big(\frac{2\beta^4}{\mu_y^4}+\frac{4\beta^3}{\mu_y^3}+\frac{4\beta^2}{\mu_y^2}\big)\leq \mathcal{O}\big(\frac{1}{\mu_y^{3}\epsilon^{3/2}}\big)$. Then, it follows from \cref{eq:rsoltion} that $\mathcal{C}_{\text{\normalfont norm}}(\mathcal{A},\epsilon) \geq r^*\geq \Omega \big(\frac{1}{\mu_y\epsilon^{1/2}}\big)$. 

{\bf Case 2: $\mu_y\leq \mathcal{O}(\epsilon^{\frac{3}{2}})$.} For this case, first suppose $(r^*)^3\leq  \mathcal{O}\big(\frac{2\beta^4}{\mu_y^4}+\frac{4\beta^3}{\mu_y^3}+\frac{4\beta^2}{\mu_y^2}\big) $, and then it follows from \cref{eq:rsoltion}  that $r^* \geq \Omega(\frac{1}{\epsilon^{2}}) $. On the other hand, if $(r^*)^3\geq  \Omega\big(\frac{2\beta^4}{\mu_y^4}+\frac{4\beta^3}{\mu_y^3}+\frac{4\beta^2}{\mu_y^2}\big) $, then we obtain from \cref{eq:rsoltion}  that $r^*\geq \Omega(\frac{1}{\mu_y\epsilon^{1/2}})\geq \Omega(\frac{1}{\epsilon^{2}})$. Then, it concludes that $\mathcal{C}_{\text{\normalfont norm}}(\mathcal{A},\epsilon)\geq r^*\geq \Omega(\frac{1}{\epsilon^{2}})$.
Then, combining these two cases finishes the proof. 

\chapter{Proof of \Cref{chp:stoc_bilevel}}\label{appendix:stoc_bilevel}

\section{Supporting Lemmas}
 First  the Lipschitz properties in Assumption~\ref{ass:lip_stoc} imply the following lemma.
\begin{lemma}\label{le:boundv}
Suppose Assumption~\ref{ass:lip_stoc} holds. Then, the stochastic derivatives $\nabla F(z;\xi)$, $\nabla G(z;\xi)$, $\nabla_x\nabla_y G(z;\xi)$ and $\nabla_y^2 G(z;\xi)$ have bounded variances, i.e., for any $z$ and $\xi$, 
\begin{itemize}
\item $\mathbb{E}_\xi\left \|\nabla F(z;\xi)-\nabla f(z)\right\|^2 \leq M^2.$
\item $\mathbb{E}_\xi\left \|\nabla_x\nabla_y G(z;\xi)-\nabla_x\nabla_y g(z)\right\|^2 \leq L^2.$
\item $\mathbb{E}_\xi \left\|\nabla_y^2 G(z;\xi)-\nabla_y^2 g(z)\right\|^2 \leq L^2.$
\end{itemize}
\end{lemma}

\section{Proof of \Cref{prop:hessian}}
Based on the definition of $v_Q$ in~\cref{ours:est} and conditioning on $x_k,y_k^D$, we have 
\begin{align*}
\mathbb{E}&v_Q= \mathbb{E}  \eta \sum_{q=-1}^{Q-1}\prod_{j=Q-q}^Q (I - \eta \nabla_y^2G(x_k,y_k^D;\gB_j)) \nabla_y F(x_k,y_k^D;\gD_F),  \nonumber
 \\ & = \eta \sum_{q=0}^{Q} (I - \eta \nabla_y^2g(x_k,y_k^D))^q\nabla_y f(x_k,y_k^D)  \nonumber
 \\& = \eta \sum_{q=0}^{\infty} (I - \eta \nabla_y^2g(x_k,y_k^D))^q\nabla_y f(x_k,y_k^D) -\eta \sum_{q=Q+1}^{\infty} (I - \eta \nabla_y^2g(x_k,y_k^D))^q\nabla_y f(x_k,y_k^D) \nonumber
 \\& = \eta (\eta \nabla_y^2 g(x_k,y_k^D))^{-1}\nabla_y f(x_k,y_k^D) - \eta \sum_{q=Q+1}^{\infty} (I - \eta \nabla_y^2g(x_k,y_k^D))^q\nabla_y f(x_k,y_k^D), 
\end{align*}
which, in conjunction with the strong-convexity of function $g(x,\cdot)$, yields
{\small
\begin{align}\label{eq:fm}
\big\|\mathbb{E}v_Q- [\nabla_y^2 g(x_k,y^D_k)]^{-1}\nabla_y f(x_k,y_k^D) \big \| \leq \eta \sum_{q=Q+1}^{\infty}(1-\eta\mu)^{q} M\leq \frac{(1-\eta\mu)^{Q+1}M}{\mu}. 
\end{align}}
\hspace{-0.15cm}This finishes the proof for the estimation bias. 
We next prove the variance bound. 
{\small
\begin{align}\label{eq:init}
\mathbb{E} \bigg\|& \eta \sum_{q=-1}^{Q-1}\prod_{j=Q-q}^Q (I - \eta \nabla_y^2G(x_k,y_k^D;\gB_j)) \nabla_y F(x_k,y_k^D;\gD_F)-( \nabla_y^2 g(x_k,y_k^D))^{-1}\nabla_y f(x_k,y_k^D)\bigg\|^2 \nonumber
\\ \overset{(i)}\leq & 2\mathbb{E} \bigg\| \ \eta \sum_{q=-1}^{Q-1}\prod_{j=Q-q}^Q (I - \eta \nabla_y^2G(x_k,y_k^D;\gB_j))  -  ( \nabla_y^2 g(x_k,y_k^D))^{-1} \bigg \|^2M^2 + \frac{2M^2}{\mu^2D_f}  \nonumber
\\\leq& 4\mathbb{E}\bigg\|  \eta \sum_{q=-1}^{Q-1}\prod_{j=Q-q}^Q (I - \eta \nabla_y^2G(x_k,y_k^D;\gB_j)) -  \eta\sum_{q=0}^Q(I - \eta \nabla_y^2g(x_k,y_k^D))^q \bigg \|^2 M^2 \nonumber
\\&+ 4\mathbb{E}\bigg\|\eta\sum_{q=0}^Q(I - \eta \nabla_y^2g(x_k,y_k^D))^q) -  ( \nabla_y^2 g(x_k,y_k^D))^{-1} \bigg \|^2 M^2+ \frac{2M^2}{\mu^2D_f}  \nonumber
\\\overset{(ii)}\leq & 4\eta^2 \mathbb{E}\bigg\| \sum_{q=0}^{Q}\prod_{j=Q+1-q}^Q (I - \eta \nabla_y^2G(x_k,y_k^D;\gB_j)) -  \sum_{q=0}^Q(I - \eta \nabla_y^2g(x_k,y_k^D))^q \bigg \|^2M^2  \nonumber
\\&+\frac{4(1-\eta \mu)^{2Q+2}M^2}{\mu^2}+ \frac{2M^2}{\mu^2D_f}  \nonumber
\\\overset{(iii)}\leq&4\eta^2 M^2 Q\mathbb{E} \sum_{q=0}^{Q} \underbrace{\bigg\|\prod_{j=Q+1-q}^Q (I - \eta \nabla_y^2G(x_k,y_k^D;\gB_j))-  (I - \eta \nabla_y^2g(x_k,y_k^D))^q \bigg \|^2}_{M_q}   \nonumber
\\&+\frac{4(1-\eta \mu)^{2Q+2}M^2}{\mu^2}  + \frac{2M^2}{\mu^2D_f}
\end{align}
}
\hspace{-0.15cm}where $(i)$ follows from \Cref{le:boundv}, $(ii)$ follows from~\cref{eq:fm},  
 and $(iii)$ follows from the Cauchy-Schwarz inequality.

 Our next step is to upper-bound $M_q$ in~\cref{eq:init}. For simplicity, we define a general quantity $M_i$ for by replacing  $q$ in $M_q$ with $i$. Then, we have 
 
{\small\begin{align}
\mathbb{E}M_i =&\mathbb{E}\bigg\| (I - \eta \nabla_y^2g(x_k,y_k^D))\prod_{j=Q+2-i}^Q (I - \eta \nabla_y^2G(x_k,y_k^D;\gB_j)) -  (I - \eta \nabla_y^2g(x_k,y_k^D))^i \bigg \|^2 \nonumber
\\&+ \mathbb{E}\bigg\| \eta( \nabla_y^2g(x_k,y_k^D)-\nabla_y^2G(x_k,y_k^D;\gB_{Q+1-i}) ) \prod_{j=Q+2-i}^Q (I - \eta \nabla_y^2G(x_k,y_k^D;\gB_j)) \bigg \|^2 \nonumber
\\&+ 2\mathbb{E}\Big\langle(I - \eta \nabla_y^2g(x_k,y_k^D))\prod_{j=Q+2-i}^Q (I - \eta \nabla_y^2G(x_k,y_k^D;\gB_j)) -  (I - \eta \nabla_y^2g(x_k,y_k^D))^i, \nonumber
\\&\hspace{0.8cm}\eta( \nabla_y^2g(x_k,y_k^D)-\nabla_y^2G(x_k,y_k^D;\gB_{Q+1-i}) ) \prod_{j=Q+2-i}^Q (I - \eta \nabla_y^2G(x_k,y_k^D;\gB_j)) \Big\rangle \nonumber
\end{align}}
{\small
\begin{align}\label{eq:mq}
\overset{(i)}=&\mathbb{E}\bigg\| (I - \eta \nabla_y^2g(x_k,y_k^D))\prod_{j=Q+2-i}^Q (I - \eta \nabla_y^2G(x_k,y_k^D;\gB_j)) -  (I - \eta \nabla_y^2g(x_k,y_k^D))^i \bigg \|^2 \nonumber
\\&+ \mathbb{E}\bigg\|\eta( \nabla_y^2g(x_k,y_k^D)-\nabla_y^2G(x_k,y_k^D;\gB_{Q+1-i}) ) \prod_{j=Q+2-i}^Q (I - \eta \nabla_y^2G(x_k,y_k^D;\gB_j))\bigg \|^2 \nonumber
\\\overset{(ii)} \leq&(1-\eta\mu)^2\mathbb{E}M_{i-1} + \eta^2(1-\eta\mu)^{2i-2} \mathbb{E}\| \nabla_y^2g(x_k,y_k^D)-\nabla_y^2G(x_k,y_k^D;\gB_{Q+1-i}) \|^2 \nonumber
\\ \overset{(iii)}\leq&(1-\eta\mu)^2\mathbb{E} M_{i-1}+ \eta^2(1-\eta\mu)^{2i-2} \frac{L^2}{|\gB_{Q+1-i}|},
\end{align}}
\hspace{-0.15cm}where $(i)$ follows from the fact that $\mathbb{E}_{\gB_{Q+1-i}}\nabla_y^2G(x_k,y_k^D;\gB_{Q+1-i})  = \nabla_y^2g(x_k,y_k^D)$, $(ii)$ follows from the strong-convexity of function $G(x,\cdot;\xi)$, and $(iii)$ follows from~\Cref{le:boundv}.

Then, telescoping~\cref{eq:mq} over $i$ from $2$ to $q$ yields
\begin{align*}
\mathbb{E} M_q \leq L^2\eta^2(1-\eta\mu)^{2q-2}\sum_{j=1}^q \frac{1}{|\gB_{Q+1-j}|},
\end{align*}
which, combined with the choice of $|\gB_{Q+1-j}|=BQ(1-\eta\mu)^{j-1}$ for $j=1,...,Q$, yields
\begin{align}\label{eq:edhs}
\mathbb{E} M_q \leq &\eta^2(1-\eta\mu)^{2q-2}  \sum_{j=1}^q \frac{L^2}{BQ} \Big(\frac{1}{1-\eta\mu}\Big)^{j-1} \nonumber
\\=&\frac{\eta^2L^2}{BQ} (1-\eta\mu)^{2q-2} \frac{\left(\frac{1}{1-\eta\mu}\right)^{q-1}-1}{\frac{1}{1-\eta\mu} -1} \leq  \frac{\eta L^2}{(1-\eta\mu)\mu} \frac{1}{BQ}(1-\eta\mu)^q.
\end{align}
Substituting~\cref{eq:edhs} into~\cref{eq:init} yields
{\small\begin{align}
\mathbb{E} \bigg\|& \eta \sum_{q=-1}^{Q-1}\prod_{j=Q-q}^Q (I - \eta \nabla_y^2G(x_k,y_k^D;\gB_j)) \nabla_y F(x_k,y_k^D;\gD_F)-( \nabla_y^2 g(x_k,y_k^D))^{-1}\nabla_y f(x_k,y_k^D)\bigg\|^2 \nonumber
\\\leq& 4\eta^2 M^2 Q \sum_{q=0}^{Q} \frac{\eta L^2}{(1-\eta\mu)\mu} \frac{1}{BQ}(1-\eta\mu)^q+\frac{4(1-\eta \mu)^{2Q+2}M^2}{\mu^2}  + \frac{2M^2}{\mu^2D_f}  \nonumber
\\\leq &\frac{4\eta^2  L^2M^2}{\mu^2} \frac{1}{B}+\frac{4(1-\eta \mu)^{2Q+2}M^2}{\mu^2}+ \frac{2M^2}{\mu^2D_f} ,
\end{align} }
\hspace{-0.12cm}where the last inequality follows from the fact that $\sum_{q=0}^S x^{q}\leq \frac{1}{1-x}$. Then, the proof is complete.

\section{Auxiliary Lemmas}\label{se:supplemma}
We first use the following lemma to characterize the first-moment error of the gradient estimate $\widehat \nabla \Phi(x_k)$, whose form is given by~\cref{estG}. 
\begin{lemma}\label{le:first_m}
Suppose Assumptions~\ref{assum:stoc_geo},~\ref{ass:lip_stoc} and \ref{high_lip_stoc} hold.  Then, conditioning on $x_k$ and $y_k^D$, we have
{\small
\begin{align*}
\big\|\mathbb{E}\widehat \nabla \Phi(x_k)-\nabla \Phi(x_k)\big\|^2\leq 2 \Big( L+\frac{L^2}{\mu} + \frac{M\tau}{\mu}+\frac{LM\rho}{\mu^2}\Big)^2\|y_k^D-y^*(x_k)\|^2 +\frac{2L^2M^2(1-\eta \mu)^{2Q}}{\mu^2}.
\end{align*}}
\end{lemma}
\begin{proof}[\bf Proof of~\Cref{le:first_m}] To simplify notations, we define 
\begin{align}\label{def:phih}
\widetilde \nabla \Phi_D(x_k) =  \nabla_x f(x_k,y^D_k) -\nabla_x \nabla_y g(x_k,y^D_k) \big[\nabla_y^2 g(x_k,y^D_k)\big]^{-1}
\nabla_y f(x_k,y^D_k).
\end{align}
By the definition of $\widehat \nabla \Phi(x_k)$ in~\cref{estG} and conditioning on $x_k$ and $y_k^D$, we have 
\begin{align*}
\mathbb{E}\widehat \nabla \Phi(x_k) =&  \nabla_x f(x_k,y_k^D)-\nabla_x \nabla_y g(x_k,y_k^D)\mathbb{E} v_Q\nonumber
\\=&\widetilde \nabla \Phi_D(x_k) - \nabla_x \nabla_y g(x_k,y_k^D)( \mathbb{E}v_Q- [\nabla_y^2 g(x_k,y^D_k)]^{-1}\nabla_y f(x_k,y_k^D)),
\end{align*}
which further implies that 
\begin{align}\label{eq:first_eg}
\big\|\mathbb{E}\widehat \nabla& \Phi(x_k)- \nabla \Phi(x_k)  \big\|^2\nonumber
\\\leq &2\mathbb{E}\|\widetilde \nabla \Phi_D(x_k) -\nabla \Phi(x_k)\|^2 +  2\|\mathbb{E}\widehat \nabla \Phi(x_k)- \widetilde \nabla \Phi_D(x_k)\|^2 \nonumber
\\\leq&2\mathbb{E}\|\widetilde \nabla \Phi_D(x_k) -\nabla \Phi(x_k)\|^2 + 2L^2\| \mathbb{E}v_Q- [\nabla_y^2 g(x_k,y^D_k)]^{-1}\nabla_y f(x_k,y_k^D) \|^2\nonumber
\\\leq& 2\mathbb{E}\|\widetilde \nabla \Phi_D(x_k) -\nabla \Phi(x_k)\|^2  + \frac{ 2L^2M^2(1-\eta\mu)^{2Q+2}}{\mu^2},
\end{align}
where the last inequality follows from~\Cref{prop:hessian}.  
Our next step is to upper-bound the first term at the right hand side of~\cref{eq:first_eg}. Using the fact that $\big\|\nabla_y^2 g(x,y)^{-1}\big\|\leq \frac{1}{\mu}$ and based on Assumptions~\ref{ass:lip_stoc} and~\ref{high_lip_stoc}, we have
\begin{align}\label{eq:appox}
\|\widetilde \nabla \Phi_D(x_k) -\nabla \Phi(x_k)\| \leq &\| \nabla_x f(x_k,y^D_k)- \nabla_x f(x_k,y^*(x_k))\| \nonumber
\\&+\frac{L^2}{\mu}\|y_k^D-y^*(x_k)\|+ \frac{M\tau}{\mu}\|y_k^D-y^*(x_k)\| \nonumber
\\&+ LM \big\|\nabla_y^2 g(x_k,y^D_k)^{-1} \nonumber
-\nabla_y^2 g(x_k,y^*(x_k))^{-1}
\big\|
\\\leq & \Big( L+\frac{L^2}{\mu} + \frac{M\tau}{\mu}+\frac{LM\rho}{\mu^2}\Big) \|y_k^D-y^*(x_k)\|,
\end{align} 
where  the last inequality follows because $\|M_1^{-1}-M_2^{-1}\|\leq \|M_1^{-1}M_2^{-1}\|\|M_1-M_2\| $ for any  matrices $M_1$ and $M_2$. Combining~\cref{eq:first_eg} and~\cref{eq:appox}
 completes the proof.
\end{proof}
Then, we characterize the variance of the estimator $\widehat \nabla \Phi(x_k)$. 
\begin{lemma}\label{le:variancc} 
Suppose Assumptions~\ref{assum:stoc_geo},~\ref{ass:lip_stoc} and \ref{high_lip_stoc} hold. Then, we have
\begin{align*}
\mathbb{E}\|\widehat \nabla& \Phi(x_k)-\nabla \Phi(x_k)\|^2 \leq   \frac{4L^2M^2}{\mu^2D_g} + \Big(\frac{8L^2}{\mu^2} + 2\Big) \frac{M^2}{D_f}+ \frac{16\eta^2  L^4M^2}{\mu^2} \frac{1}{B} \nonumber
\\&+\frac{16 L^2M^2(1-\eta \mu)^{2Q}}{\mu^2}+ \Big( L+\frac{L^2}{\mu} + \frac{M\tau}{\mu}+\frac{LM\rho}{\mu^2}\Big)^2 \mathbb{E}\|y_k^D-y^*(x_k)\|^2.
\end{align*}
\end{lemma}
\begin{proof}[\bf Proof of~\Cref{le:variancc}] Based on the definitions of $\nabla \Phi(x_k)$ and $\widetilde \nabla \Phi_D(x_k)$ in~\cref{trueG} and~\cref{def:phih} and conditioning on $x_k$ and $y_k^D$, we have
\begin{align}\label{eq:vbbs}
\mathbb{E}\|\widehat \nabla& \Phi(x_k)-\nabla \Phi(x_k)\|^2  \nonumber
\\\overset{(i)}=& \mathbb{E} \|\widehat \nabla \Phi(x_k)- \widetilde \nabla \Phi_D(x_k)\|^2+\|\widetilde \nabla \Phi_D(x_k)-\nabla \Phi(x_k)\|^2 
\nonumber
\\\overset{(ii)}\leq &2\mathbb{E} \big\|\nabla_x \nabla_y G(x_k,y_k^D;\gD_G) v_Q -\nabla_x \nabla_y g(x_k,y^D_k) \big[\nabla_y^2 g(x_k,y^D_k)\big]^{-1}
\nabla_y f(x_k,y^D_k)\big\|^2  \nonumber
\\&+ \frac{2M^2}{D_f}+ \Big( L+\frac{L^2}{\mu} + \frac{M\tau}{\mu}+\frac{LM\rho}{\mu^2}\Big)^2 \|y_k^D-y^*(x_k)\|^2 \nonumber
\\\overset{(iii)}\leq& \frac{4M^2}{\mu^2}\mathbb{E} \|\nabla_x \nabla_y G(x_k,y_k^D;\gD_G) -\nabla_x \nabla_y g(x_k,y^D_k) \|^2 \nonumber
\\&+ 4L^2\mathbb{E}\|v_Q-\big[\nabla_y^2 g(x_k,y^D_k)\big]^{-1}
\nabla_y f(x_k,y^D_k)\|^2 \nonumber
\\&+ \Big( L+\frac{L^2}{\mu} + \frac{M\tau}{\mu}+\frac{LM\rho}{\mu^2}\Big)^2 \|y_k^D-y^*(x_k)\|^2 + \frac{2M^2}{D_f},
\end{align}
where $(i)$ follows because $\mathbb{E}_{\gD_G,\gD_H,\gD_F} \widehat \nabla \Phi(x_k) =  \widetilde \nabla \Phi_D(x_k)$, $(ii)$ follows from \Cref{le:boundv} and~\cref{eq:appox}, and $(iii)$
follows from the Young's inequality and Assumption~\ref{ass:lip_stoc}. 

 Using~\Cref{le:boundv} and~\Cref{prop:hessian} in \cref{eq:vbbs}, yields 
\begin{align}
\mathbb{E}\|\widehat \nabla \Phi(x_k)-&\nabla \Phi(x_k)\|^2 \leq   \frac{4L^2M^2}{\mu^2D_g} + \frac{16\eta^2  L^4M^2}{\mu^2} \frac{1}{B}+\frac{16(1-\eta \mu)^{2Q}L^2M^2}{\mu^2}+ \frac{8L^2M^2}{\mu^2D_f}\nonumber
\\&+ \Big( L+\frac{L^2}{\mu} + \frac{M\tau}{\mu}+\frac{LM\rho}{\mu^2}\Big)^2 \|y_k^D-y^*(x_k)\|^2 + \frac{2M^2}{D_f},
\end{align}
which, unconditioning on $x_k$ and $y_k^D$, completes the proof. 
\end{proof}
It can be seen from Lemmas~\ref{le:first_m} and~\ref{le:variancc} that the upper bounds on both the estimation error and bias depend on the tracking error $\|y_k^D-y^*(x_k)\|^2$. The following lemma provides an upper bound on such a tracking error $\|y_k^D-y^*(x_k)\|^2$. 
\begin{lemma}\label{tra_error}
Suppose Assumptions~\ref{assum:stoc_geo},~\ref{ass:lip_stoc} and~\ref{ass:bound} hold. Define constants
\begin{small}
\begin{align}\label{eq:defs}
\lambda=&  \Big(\frac{L-\mu}{L+\mu}\Big)^{2D} \Big(2+  \frac{4\beta^2L^2}{\mu^2}  \Big( L+\frac{L^2}{\mu} + \frac{M\tau}{\mu}+\frac{LM\rho}{\mu^2}\Big)^2  \Big) \nonumber
\\\Delta =&\frac{4L^2M^2}{\mu^2D_g} + \Big(\frac{8L^2}{\mu^2} + 2\Big) \frac{M^2}{D_f}+ \frac{16\eta^2  L^4M^2}{\mu^2} \frac{1}{B}+\frac{16 L^2M^2(1-\eta \mu)^{2Q}}{\mu^2} 
\nonumber
\\\omega = &\frac{4\beta^2L^2}{\mu^2} \Big(\frac{L-\mu}{L+\mu}\Big)^{2D}. 
\end{align}
\end{small}
\hspace{-0.12cm}Choose $D$ such that $\lambda<1$ and set inner-loop stepsize $\alpha=\frac{2}{L+\mu}$. Then, we have 
{\small\begin{align*}
\mathbb{E}\|&y_k^{D}-y^*(x_k) \|^2 \nonumber
\\\leq&\lambda^{k} \left( \left(\frac{L-\mu}{L+\mu}\right)^{2D}\|y_0-y^*(x_0)\|^2 + \frac{\sigma^2}{L\mu S}\right) + \omega\sum_{j=0}^{k-1}\lambda^{k-1-j} \mathbb{E}\|\nabla \Phi(x_j)\|^2 + \frac{\omega\Delta +\frac{\sigma^2}{L\mu S}}{1-\lambda}.
\end{align*}}
\end{lemma}
\begin{proof}[\bf Proof of~\Cref{tra_error}] First note that for an integer $t\leq D$
{\small
\begin{align}\label{eq:initssbilevelsacsds}
\|&y_k^{t+1}-y^*(x_k) \|^2 \nonumber
\\&= \|y_k^{t+1}-y_k^t\|^2 + 2\langle y_k^{t+1}-y_k^t,  y_k^t-y^*(x_k) \rangle + \| y_k^t-y^*(x_k)\|^2 \nonumber
\\&= \alpha^2\| \nabla_y G(x_k,y_k^{t}; \gS_t) \|^2 -  2\alpha\langle  \nabla_y G(x_k,y_k^{t}; \gS_t),  y_k^t-y^*(x_k) \rangle +\| y_k^t-y^*(x_k)\|^2.
\end{align}}
\hspace{-0.15cm}Conditioning on $y_k^t$ and taking expectation in~\cref{eq:initssbilevelsacsds}, we have 
\begin{align}\label{eq:traerr}
\mathbb{E}\|&y_k^{t+1}-y^*(x_k) \|^2  \nonumber
\\\overset{(i)}\leq& \alpha^2 \Big(\frac{\sigma^2}{S} + \|\nabla_y g(x_k,y_k^t)\|^2 \Big) - 2\alpha \langle  \nabla_y g(x_k,y_k^{t}),  y_k^t-y^*(x_k) \rangle  \nonumber
\\&+\| y_k^t-y^*(x_k)\|^2 \nonumber
\\\overset{(ii)}\leq&\frac{\alpha^2\sigma^2}{S} + \alpha^2\|\nabla_y g(x_k,y_k^t)\|^2 - 2\alpha\left(  \frac{L\mu}{L+\mu} \|y_k^t-y^*(x_k)\|^2+\frac{\|\nabla_y g(x_k,y_k^t)\|^2}{L+\mu} \right) \nonumber
\\&+\| y_k^t-y^*(x_k)\|^2  \nonumber
\\=& \frac{\alpha^2\sigma^2}{S} - \alpha\left(\frac{2}{L+\mu}-\alpha\right)\|\nabla_y g(x_k,y_k^t)\|^2 + \left(1-\frac{2\alpha L\mu}{L+\mu} \right)\|y_k^t-y^*(x_k)\|^2
\end{align}
where $(i)$ follows from the third item in Assumption~\ref{ass:lip_stoc}, and $(ii)$ follows from the strong-convexity and smoothness of $g$. Since $\alpha=\frac{2}{L+\mu}$, we obtain from~\cref{eq:traerr}~that  
\begin{align}\label{eq:eyk}
\mathbb{E}\|y_k^{t+1}&-y^*(x_k) \|^2\leq \left(\frac{L-\mu}{L+\mu} \right)^2\|y_k^t-y^*(x_k)\|^2 + \frac{4\sigma^2}{(L+\mu)^2S}.
\end{align}
Unconditioning on $y^t_k$ in \cref{eq:eyk} and telescoping~\cref{eq:eyk} over $t$ from $0$ to $D-1$ yield
\begin{align}\label{eq:yjtt}
\mathbb{E}\|y_k^{D}-y^*(x_k) \|^2 \leq &  \left(\frac{L-\mu}{L+\mu}\right)^{2D}\mathbb{E}\|y^0_k-y^*(x_k)\|^2 + \frac{\sigma^2}{L\mu S} \nonumber
\\ = & \left(\frac{L-\mu}{L+\mu}\right)^{2D}\mathbb{E}\|y^D_{k-1}-y^*(x_k)\|^2 + \frac{\sigma^2}{L\mu S},
\end{align} 
where the last inequality follows from \Cref{alg:main} that $y_k^0=y^{D}_{k-1}$. Note that 
\begin{align}\label{eq:midone}
\mathbb{E}\|y^D_{k-1}-y^*(x_k)\|^2 \leq &2\mathbb{E}\|y^D_{k-1}-y^*(x_{k-1})\|^2 + 2\mathbb{E}\|y^*(x_{k-1})-y^*(x_k)\|^2 \nonumber
\\\overset{(i)}\leq& 2\mathbb{E}\|y^D_{k-1}-y^*(x_{k-1})\|^2 +\frac{2L^2}{\mu^2} \mathbb{E}\|x_k-x_{k-1}\|^2 \nonumber
\\ \leq & 2\mathbb{E}\|y^D_{k-1}-y^*(x_{k-1})\|^2 +\frac{2\beta^2L^2}{\mu^2} \mathbb{E}\|\widehat \nabla \Phi(x_{k-1})\|^2 \nonumber
\\\leq & 2\mathbb{E}\|y^D_{k-1}-y^*(x_{k-1})\|^2  + \frac{4\beta^2L^2}{\mu^2} \mathbb{E}\|\nabla \Phi(x_{k-1})\|^2 \nonumber
\\&+\frac{4\beta^2L^2}{\mu^2} \mathbb{E}\|\widehat \nabla \Phi(x_{k-1})-\nabla \Phi(x_{k-1})\|^2,
\end{align}
where $(i)$ follows from Lemma 2.2 in \cite{ghadimi2018approximation}. Using~\Cref{le:variancc} in~\cref{eq:midone} yields
{\small
\begin{align}\label{eq:seceq}
\mathbb{E}\|&y^D_{k-1}-y^*(x_k)\|^2  \nonumber
\\\leq& \left(2+  \frac{4\beta^2L^2}{\mu^2}  \Big( L+\frac{L^2}{\mu} + \frac{M\tau}{\mu}+\frac{LM\rho}{\mu^2}\Big)^2  \right)\mathbb{E}\|y^D_{k-1}-y^*(x_{k-1})\|^2+ \frac{4\beta^2L^2}{\mu^2} \mathbb{E}\|\nabla \Phi(x_{k-1})\|^2  \nonumber
\\&+  \frac{4\beta^2L^2}{\mu^2} \left( \frac{4L^2M^2}{\mu^2D_g} + \Big(\frac{8L^2}{\mu^2} + 2\Big) \frac{M^2}{D_f}+ \frac{16\eta^2  L^4M^2}{\mu^2} \frac{1}{B}+\frac{16 L^2M^2(1-\eta \mu)^{2Q}}{\mu^2}  \right).
\end{align}}
\hspace{-0.15cm}Combining~\cref{eq:yjtt} and~\cref{eq:seceq} yields
{\small
\begin{align}\label{eq:enroll}
\mathbb{E}\|&y_k^{D}-y^*(x_k) \|^2  \nonumber
\\\leq&  \Big(\frac{L-\mu}{L+\mu}\Big)^{2D} \Big(2+  \frac{4\beta^2L^2}{\mu^2}  \Big( L+\frac{L^2}{\mu} + \frac{M\tau}{\mu}+\frac{LM\rho}{\mu^2}\Big)^2  \Big)\mathbb{E}\|y^D_{k-1}-y^*(x_{k-1})\|^2 \nonumber
\\&+ \Big(\frac{L-\mu}{L+\mu}\Big)^{2D} \frac{4\beta^2L^2}{\mu^2} \left( \frac{4L^2M^2}{\mu^2D_g} + \Big(\frac{8L^2}{\mu^2} + 2\Big) \frac{M^2}{D_f}+ \frac{16\eta^2  L^4M^2}{\mu^2} \frac{1}{B}+\frac{16 L^2M^2(1-\eta \mu)^{2Q}}{\mu^2}  \right)\nonumber
\\&+\frac{4\beta^2L^2}{\mu^2} \Big(\frac{L-\mu}{L+\mu}\Big)^{2D} \mathbb{E}\|\nabla \Phi(x_{k-1})\|^2  + \frac{\sigma^2}{L\mu S}. 
\end{align}
}
\hspace{-0.15cm}Based on the definitions of $\lambda,\omega,\Delta$ in~\cref{eq:defs}, we obtain from~\cref{eq:enroll} that 
\begin{align}\label{eq:readytote}
\mathbb{E}\|y_k^{D}-y^*(x_k) \|^2 \leq& \lambda \mathbb{E}\|y^D_{k-1}-y^*(x_{k-1})\|^2 + \omega\Delta +\frac{\sigma^2}{L\mu S}  +\omega \mathbb{E}\|\nabla \Phi(x_{k-1})\|^2. 
\end{align}
Telescoping~\cref{eq:readytote} over $k$ yields
\begin{align*}
\mathbb{E}\|&y_k^{D}-y^*(x_k) \|^2 \nonumber
\\\leq & \lambda^{k} \mathbb{E}\|y_0^D-y^*(x_0)\|^2 + \omega\sum_{j=0}^{k-1}\lambda^{k-1-j} \mathbb{E}\|\nabla \Phi(x_j)\|^2 + \frac{\omega\Delta +\frac{\sigma^2}{L\mu S}}{1-\lambda} \nonumber
\\\leq&\lambda^{k} \left( \left(\frac{L-\mu}{L+\mu}\right)^{2D}\|y_0-y^*(x_0)\|^2 + \frac{\sigma^2}{L\mu S}\right) + \omega\sum_{j=0}^{k-1}\lambda^{k-1-j} \mathbb{E}\|\nabla \Phi(x_j)\|^2 + \frac{\omega\Delta +\frac{\sigma^2}{L\mu S}}{1-\lambda}, 
\end{align*}
which completes the proof. 
\end{proof}

\section{Proof of~\Cref{th:nonconvex}}\label{mianshisimida}
We now provide the proof for~\Cref{th:nonconvex}, based on the supporting lemmas we develop in~\Cref{se:supplemma}. 

Based on the smoothness of the function $\Phi(x)$ in~\Cref{le:lipphi}, we have 
\begin{align*}
\Phi(x_{k+1}) &\leq  \Phi(x_k)  + \langle \nabla \Phi(x_k), x_{k+1}-x_k\rangle + \frac{L_\Phi}{2} \|x_{k+1}-x_k\|^2 \nonumber
\\\leq& \Phi(x_k)  - \beta \langle \nabla \Phi(x_k),\widehat \nabla \Phi(x_k)\rangle + \beta^2 L_\Phi \|\nabla\Phi(x_k)\|^2+\beta^2 L_\Phi\|\nabla\Phi(x_k)-\widehat \nabla\Phi(x_k)\|^2.
\end{align*}
For simplicity, let $\mathbb{E}_k = \mathbb{E}(\cdot\,| \,x_k,y_k^D)$. Note that we choose $\beta=\frac{1}{4L_\phi}$. Then, 
taking expectation over the above inequality, we have
\begin{align} \label{eq:jiayou}
\mathbb{E}\Phi(x_{k+1}) \leq &\mathbb{E}\Phi(x_k)  - \beta \mathbb{E}\langle \nabla \Phi(x_k),\mathbb{E}_k\widehat \nabla \Phi(x_k)\rangle + \beta^2 L_\Phi \mathbb{E}\|\nabla\Phi(x_k)\|^2 \nonumber
\\&+\beta^2 L_\Phi\mathbb{E}\|\nabla\Phi(x_k)-\widehat \nabla\Phi(x_k)\|^2 \nonumber
\\\overset{(i)}\leq& \mathbb{E}\Phi(x_k)  +\frac{\beta}{2}\mathbb{E}\|\mathbb{E}_k\widehat \nabla \Phi(x_k)-\nabla \Phi(x_k) \|^2 -\frac{\beta}{4} \mathbb{E}\|\nabla\Phi(x_k)\|^2 \nonumber
\\&+\frac{\beta}{4}\mathbb{E}\|\nabla\Phi(x_k)-\widehat \nabla\Phi(x_k)\|^2 \nonumber
\\\overset{(ii)}\leq& \mathbb{E}\Phi(x_k) -\frac{\beta}{4}\mathbb{E}\|\nabla\Phi(x_k)\|^2 +\frac{\beta L^2M^2(1-\eta \mu)^{2Q}}{\mu^2} 
\nonumber
\\&+\frac{\beta}{4}
\left(   \frac{4L^2M^2}{\mu^2D_g} + \Big(\frac{8L^2}{\mu^2} + 2\Big) \frac{M^2}{D_f}+ \frac{16\eta^2  L^4M^2}{\mu^2} \frac{1}{B}+\frac{16 L^2M^2(1-\eta \mu)^{2Q}}{\mu^2}\right) \nonumber
\\&+  \frac{5\beta}{4} \Big( L+\frac{L^2}{\mu} + \frac{M\tau}{\mu}+\frac{LM\rho}{\mu^2}\Big)^2\mathbb{E}\|y_k^D-y^*(x_k)\|^2 
\end{align}
where $(i)$ follows from Cauchy-Schwarz inequality, and $(ii)$ follows from \Cref{le:first_m} and~\Cref{le:variancc}. 
For simplicity,  let
\begin{align}\label{def:nu}
\nu= \frac{5}{4}\Big( L+\frac{L^2}{\mu} + \frac{M\tau}{\mu}+\frac{LM\rho}{\mu^2}\Big)^2.
\end{align} 
Then, applying~\Cref{tra_error} in~\cref{eq:jiayou} and using the definitions of $\omega,\Delta,\lambda$ in~\cref{eq:defs}, we have 
\begin{align}
\mathbb{E}\Phi(x_{k+1})\leq& \mathbb{E}\Phi(x_k) -\frac{\beta}{4} \mathbb{E}\|\nabla\Phi(x_k)\|^2 +\frac{\beta L^2M^2(1-\eta \mu)^{2Q}}{\mu^2} 
\nonumber
\\&+\frac{\beta}{4} \Delta + \beta\nu\lambda^{k} \left( \left(\frac{L-\mu}{L+\mu}\right)^{2D}\|y_0-y^*(x_0)\|^2 + \frac{\sigma^2}{L\mu S}\right) \nonumber
\\&+ \beta\nu\omega\sum_{j=0}^{k-1}\lambda^{k-1-j} \mathbb{E}\|\nabla \Phi(x_j)\|^2 + \frac{\beta\nu(\omega\Delta +\frac{\sigma^2}{L\mu S})}{1-\lambda}.\nonumber
\end{align}
Telescoping the above inequality over $k$ from $0$ to $K-1$ yields
\begin{align*}
\mathbb{E}\Phi(x_{K}) \leq \Phi(x_0) - &\frac{\beta}{4} \sum_{k=0}^{K-1}\mathbb{E}\|\nabla\Phi(x_k)\|^2 + \beta\nu\omega\sum_{k=1}^{K-1}\sum_{j=0}^{k-1}\lambda^{k-1-j} \mathbb{E}\|\nabla \Phi(x_j)\|^2 \nonumber
\\&+\frac{K\beta\Delta}{4} + \Big(\Big(\frac{L-\mu}{L+\mu}\Big)^{2D}\|y_0-y^*(x_0)\|^2 + \frac{\sigma^2}{L\mu S}\Big)\frac{\beta\nu}{1-\lambda} \nonumber
\\&+ \frac{K\beta L^2M^2(1-\eta \mu)^{2Q}}{\mu^2}  + \frac{K\beta\nu(\omega\Delta +\frac{\sigma^2}{L\mu S})}{1-\lambda},
\end{align*}
which, using the fact that {\small$$\sum_{k=1}^{K-1}\sum_{j=0}^{k-1}\lambda^{k-1-j} \mathbb{E}\|\nabla \Phi(x_j)\|^2\leq \left(\sum_{k=0}^{K-1}\lambda^k\right)\sum_{k=0}^{{K-1}}\mathbb{E}\|\nabla\Phi(x_k)\|^2<\frac{1}{1-\lambda}\sum_{k=0}^{{K-1}}\mathbb{E}\|\nabla\Phi(x_k)\|^2,$$} \hspace{-0.15cm}yields
\begin{align}\label{eq:opsac}
 \Big(\frac{1}{4} -&\frac{\nu\omega}{1-\lambda}\Big) \frac{1}{K}\sum_{k=0}^{K-1}\mathbb{E}\|\nabla\Phi(x_k)\|^2 \nonumber
 \\\leq &\frac{\Phi(x_0)-\inf_x\Phi(x)}{\beta K}+\frac{\nu\big((\frac{L-\mu}{L+\mu})^{2D}\|y_0-y^*(x_0)\|^2 + \frac{\sigma^2}{L\mu S}\big)}{K(1-\lambda)}+\frac{\Delta}{4} +  \frac{ L^2M^2(1-\eta \mu)^{2Q}}{\mu^2}  \nonumber
 \\&+ \frac{\nu(\omega\Delta +\frac{\sigma^2}{L\mu S})}{1-\lambda}.
\end{align}
We choose the number $D$ of inner-loop steps as 
$$D\geq \max\bigg\{\frac{\log \big(12+  \frac{48\beta^2L^2}{\mu^2} ( L+\frac{L^2}{\mu} + \frac{M\tau}{\mu}+\frac{LM\rho}{\mu^2})^2\big)}{2\log (\frac{L+\mu}{L-\mu})},\frac{\log \big(\sqrt{\beta}(L+\frac{L^2}{\mu} + \frac{M\tau}{\mu}+\frac{LM\rho}{\mu^2})\big)}{\log (\frac{L+\mu}{L-\mu})}\bigg\}.$$ Then,  since $\beta=\frac{1}{4L_\Phi}$ and $D\geq \frac{\log \big(12+  \frac{48\beta^2L^2}{\mu^2} ( L+\frac{L^2}{\mu} + \frac{M\tau}{\mu}+\frac{LM\rho}{\mu^2})^2\big)}{2\log (\frac{L+\mu}{L-\mu})}$, we have $\lambda\leq \frac{1}{6}$, and  \cref{eq:opsac} is further simplified to 
\begin{align}\label{eq:ephi}
 \Big(\frac{1}{4} -&\frac{6}{5}\nu\omega\Big) \frac{1}{K}\sum_{k=0}^{K-1}\mathbb{E}\|\nabla\Phi(x_k)\|^2 \nonumber
 \\\leq &\frac{\Phi(x_0)-\inf_x\Phi(x)}{\beta K}+\frac{2\nu\big((\frac{L-\mu}{L+\mu})^{2D}\|y_0-y^*(x_0)\|^2 + \frac{\sigma^2}{L\mu S}\big)}{K}+\frac{ \Delta}{4} +  \frac{ L^2M^2(1-\eta \mu)^{2Q}}{\mu^2}  \nonumber
 \\&+ 2\nu\Big(\omega\Delta +\frac{\sigma^2}{L\mu S}\Big).
\end{align}
By $\omega$ in~\cref{eq:defs}, $\nu$ in~\cref{def:nu} and $D\geq \frac{\log \big(12+  \frac{48\beta^2L^2}{\mu^2} ( L+\frac{L^2}{\mu} + \frac{M\tau}{\mu}+\frac{LM\rho}{\mu^2})^2\big)}{2\log (\frac{L+\mu}{L-\mu})}$, we have
\begin{align}\label{eq:opccsa}
\nu\omega = &\frac{5\beta^2L^2}{\mu^2} \Big(\frac{L-\mu}{L+\mu}\Big)^{2D}\Big( L+\frac{L^2}{\mu} + \frac{M\tau}{\mu}+\frac{LM\rho}{\mu^2}\Big)^2\nonumber
\\<& \frac{\frac{5\beta^2L^2}{\mu^2} \Big( L+\frac{L^2}{\mu} + \frac{M\tau}{\mu}+\frac{LM\rho}{\mu^2}\Big)^2}{12+  \frac{48\beta^2L^2}{\mu^2} ( L+\frac{L^2}{\mu} + \frac{M\tau}{\mu}+\frac{LM\rho}{\mu^2})^2} \leq \frac{5}{48}.
\end{align}
In addition, since $D>\frac{\log \big(\sqrt{\beta}\big(L+\frac{L^2}{\mu} + \frac{M\tau}{\mu}+\frac{LM\rho}{\mu^2}\big)\big)}{\log (\frac{L+\mu}{L-\mu})}$, we have
\begin{align}\label{eq:opccsa2}
 \nu\Big(\frac{L-\mu}{L+\mu}\Big)^{2D} =  \frac{5}{4}\Big(\frac{L-\mu}{L+\mu}\Big)^{2D}\Big( L+\frac{L^2}{\mu} + \frac{M\tau}{\mu}+\frac{LM\rho}{\mu^2}\Big)^2 <\frac{5}{4\beta}.
\end{align}
Substituting~\cref{eq:opccsa} and~\cref{eq:opccsa2} in~\cref{eq:ephi} yields 
\begin{align*}
\frac{1}{K}\sum_{k=0}^{K-1}\mathbb{E}\|\nabla\Phi(x_k)\|^2 \leq & \frac{8(\Phi(x_0)-\inf_x\Phi(x)+\frac{5}{2}\|y_0-y^*(x_0)\|^2) }{\beta K}+\Big(1+\frac{1}{K}\Big)\frac{16\nu\sigma^2}{L\mu S} \nonumber
\\&+\frac{11}{3} \Delta+  \frac{8L^2M^2}{\mu^2}(1-\eta \mu)^{2Q}, 
\end{align*}
which, in conjunction with~\cref{eq:defs} and \cref{def:nu}, yields~\cref{eq:main_nonconvex} in~\Cref{th:nonconvex}.  


Then, based on \cref{eq:main_nonconvex},  to achieve an $\epsilon$-accurate stationary point, i.e., $\mathbb{E}\|\nabla\Phi(\bar x)\|^2\leq \epsilon$ with $\bar x$ chosen from $x_0,...,x_{K-1}$ uniformly at random, it suffices to choose
\begin{align*}
K =&  \frac{32L_\Phi(\Phi(x_0)-\inf_x\Phi(x)+\frac{5}{2}\|y_0-y^*(x_0)\|^2) }{\epsilon}=\mathcal{O}\Big(\frac{\kappa^3}{\epsilon}\Big), D=\Theta(\kappa)\nonumber
\\Q = & \kappa\log \frac{\kappa^2}{\epsilon}, S = \mathcal{O}\Big(\frac{\kappa^5}{\epsilon}\Big),D_g =\mathcal{O}\left(\frac{\kappa^2}{\epsilon}\right), D_f =\mathcal{O}\left(\frac{\kappa^2}{\epsilon}\right), B =\mathcal{O}\left(\frac{\kappa^2}{\epsilon}\right).
\end{align*}
Note that the above choices of $Q$ and $B$ satisfy the condition that $B\geq \frac{1}{Q(1-\eta\mu)^{Q-1}}$ required in~\Cref{prop:hessian}. 

Then, the gradient complexity is given by $\mbox{\normalfont Gc}(F,\epsilon)=KD_f=\mathcal{O}(\kappa^5\epsilon^{-2}), \mbox{\normalfont Gc}(G,\epsilon)=KDS=\mathcal{O}(\kappa^9\epsilon^{-2}).$
In addition, the Jacobian- and Hessian-vector product complexities are given by $ \mbox{\normalfont JV}(G,\epsilon)=KD_g=\mathcal{O}(\kappa^5\epsilon^{-2})$ and 
\begin{align*}
\mbox{\normalfont HV}(G,\epsilon) = K \sum_{j=1}^Q BQ(1-\eta\mu)^{j-1}=\frac{KBQ}{\eta\mu} \leq\mathcal{O}\left( \frac{\kappa^6}{\epsilon^{2}}\log \frac{\kappa^2}{\epsilon} \right).
\end{align*}
Then, the proof is complete. 

\chapter{Objective Examples and Proof of \Cref{chp: maml}}\label{append: maml}

\section{Examples for Two Types of Objective Functions}\label{ggpopssasdasdax}
\subsection*{RL Example for Resampling Case}
RL problems are often captured by objective functions in the expectation form. Consider a RL meta learning problem,  where each task  corresponds to  a Markov decision process (MDP) with horizon $H$. Each RL task $\mathcal{T}_i$ corresponds to an initial state distribution $\rho_i$, a policy $\pi_w$ parameterized by $w$ that denotes a distribution over the action set given each state, and a transition distribution kernel $q_i(x_{t+1}|x_t,a_t)$ at time steps $t=0,...,H-1$. Then, the loss $l_i(w)$ is defined as negative total reward, i.e., 
$l_i(w):= -\mathbb{E}_{\tau\sim p_i(\cdot| w)}[ \mathcal{R}(\tau)]$,
where $\tau= (s_0,a_0,s_1,a_1,...,s_{H-1},a_{H-1})$ is a  trajectory following the distribution $p_i(\cdot | w)$, and the reward $$\mathcal{R}(\tau) := \sum_{t=0}^{H-1} \gamma^t\mathcal{R}(s_t,a_t)$$ with $\mathcal{R}(\cdot)$ given as a reward function. The estimated gradient here is  $$\nabla l_i(w; \Omega):= \frac{1}{|\Omega|} \sum_{\tau \in \Omega} g_i(w; \tau),$$ where $g_i(w; \tau)$ is an unbiased policy gradient estimator s.t. $\mathbb{E}_{\tau \sim p_i(\cdot|w)} g_i(w;\tau)= \nabla l_i(w)$, e.g, REINFORCE~\cite{williams1992simple} or G(PO)MDP~\cite{baxter2001infinite}. In addition, the estimated Hessian  is $$\nabla^2 l_i(w; \Omega):= \frac{1}{|\Omega|}\sum_{\tau\in \Omega}H_i(w; \tau)$$, where $H_i(w;\tau)$ is an unbiased policy Hessian estimator, e.g., DiCE~\cite{foerster2018dice} or LVC~\cite{rothfuss2019promp}.

\subsection*{Classification Example for Finite-Sum Case}
The risk minimization problem in classification often has a finite-sum objective function. For example, the 
mean-squared error (MSE) loss takes the form of  
\begin{align*}
(\text{Classification}):\; l_{S_i}(w):= \frac{1}{|S_i|}\sum_{(x_j, y_j)\in S_i}\|y_j - \phi (w; x_i) \|^2 \;(\text{similarly for } \, l_{T_i}(w)), 
\end{align*}
where $x_j, y_j$ are a feature-label pair and $\phi(w;\cdot)$ can be  a deep neural network parameterized by $w$.

\section{Derivation of Simplified Form of  Gradient $\nabla \mathcal{L}_i(w)$}\label{simplifeid}
First note that $\mathcal{L}_i(w_k) = l_i(\widetilde w_{k,N}^i)$ and $\widetilde w_{k,N}^i$ is obtained by the following gradient descent updates
\begin{align}\label{gd_pr}
\widetilde w^i_{k, j+1} =\widetilde w^i_{k,j} - \alpha \nabla l_i(\widetilde w^i_{k,j}), \,\,j = 0, 1,..., N-1 \, \text{ with }\, \widetilde w^i_{k,0} := w_k.
\end{align}
Then, by the chain rule, we have 
\begin{align*}
\nabla \mathcal{L}_i(w_k) = \nabla_{w_k} l_i(\widetilde w_{k,N}^i) = \prod_{j=0}^{N-1}\nabla_{\widetilde w_{k,j}^i} \left(\widetilde w_{k,j+1}^i\right) \nabla l_i(\widetilde w_{k,N}^i), 
\end{align*}
which, in conjunction with \cref{gd_pr}, implies that 
\begin{align*}
\nabla \mathcal{L}_i(w_k) &=\prod_{j=0}^{N-1}\nabla_{\widetilde w_{k,j}^i} \left(\widetilde w^i_{k,j} - \alpha \nabla l_i(\widetilde w^i_{k,j})\right) \nabla l_i(\widetilde w_{k,N}^i)
 = 
\prod_{j=0}^{N-1} \left(I - \alpha \nabla^2 l_i(\widetilde w^i_{k,j})\right) \nabla l_i(\widetilde w_{k,N}^i),
\end{align*}
which finishes the proof.

\section{Proof for Convergence in Resampling Case}\label{prop:meta_grad}
For the resampling case, we provide the proofs for Propositions~\ref{th:lipshiz},~\ref{le:distance}, \ref{th:first-est}, \ref{th:second} on the properties of meta gradient, and Theorem~\ref{th:mainonline} and Corollary~\ref{co:online} on the convergence and complexity performance of multi-step MAML. 
The proofs of these results require several technical lemmas, which we relegate to~\Cref{aux:lemma}.  

To simplify notations, we let $\bar S^i_j$ and  $\bar D^i_j$ denote the randomness over $S_{k,m}^i, D_{k,m}^i,m=0,...,j-1$ and  let $\bar S_j$ and $\bar D_j$ denote all randomness over $\bar S^i_j, \bar D^i_j, i\in \mathcal{I}$, respectively.
\subsection*{Proof of Proposition~\ref{th:lipshiz}}
	First recall that {\small$\nabla \mathcal{L}_i(w) =   \prod_{j=0}^{N-1}(I-\alpha \nabla^2 l_i(\widetilde w^i_{j}))\nabla l_i( \widetilde w^i _{N})$}. Then, we have 
	{\small
	\begin{align}\label{delff}
	\|\nabla \mathcal{L}_i(w) - \nabla \mathcal{L}_i(u)\| 
	 &\leq  \Big\| \prod_{j=0}^{N-1}(I-\alpha \nabla^2 l_i(\widetilde w^i_{j})) -  \prod_{j=0}^{N-1}(I-\alpha \nabla^2 l_i(\widetilde u^i_{j}))\Big\|\big\|\nabla l_i( \widetilde w^i _{N})\big\|   \nonumber
	\\ &+ (1+\alpha L)^N \|\nabla l_i( \widetilde w^i _{N})  - \nabla l_i( \widetilde u^i _{N}) \| \nonumber
	\\ \overset{(i)}\leq&  \Big\| \prod_{j=0}^{N-1}(I-\alpha \nabla^2 l_i(\widetilde w^i_{j})) -  \prod_{j=0}^{N-1}(I-\alpha \nabla^2 l_i(\widetilde u^i_{j}))\Big\|(1+\alpha L)^N\big\|\nabla l_i( w)\big\|   \nonumber
	\\ &+ (1+\alpha L)^N L \|\widetilde w^i _{N} - \widetilde u^i _{N} \| \nonumber
	\\ \overset{(ii)}\leq&  \underbrace{\Big\| \prod_{j=0}^{N-1}(I-\alpha \nabla^2 l_i(\widetilde w^i_{j})) -  \prod_{j=0}^{N-1}(I-\alpha \nabla^2 l_i(\widetilde u^i_{j}))\Big\|}_{V(N)}(1+\alpha L)^N\big\|\nabla l_i( w)\big\|   \nonumber
	\\ &+ (1+\alpha L)^{2N} L \|w-u\|,
	\end{align}}
\hspace{-0.15cm}where (i) follows from Lemma~\ref{le:jiw}, and (ii) follows from Lemma~\ref{d_u_w}.
	We next upper-bound the term $V(N)$ in the above inequality. 
	Specifically, define a more general quantity $V(m)$ by replacing $N$ in $V(N)$ with $m$.
	 Then, we have 
	\begin{align}\label{youyitian}
	V(m) 
	\leq& \Big\| \prod_{j=0}^{m-2}(I-\alpha \nabla^2 l_i(\widetilde w^i_{j})) \Big\|\big\|\alpha \nabla^2 l_i(\widetilde w_{m-1}^i) -\alpha \nabla^2 l_i(\widetilde u_{m-1}^i)  \big\| \nonumber
	\\&+ \Big\|\prod_{j=0}^{m-2}(I-\alpha \nabla^2 l_i(\widetilde w^i_{j})) - \prod_{j=0}^{m-2}(I-\alpha \nabla^2 l_i(\widetilde u^i_{j}))\Big\| \big\|I-\alpha\nabla^2 l_i(\widetilde u_{m-1}^i) \big\| \nonumber
	\\\leq& (1+\alpha L)^{m-1} \big\|\alpha \nabla^2 l_i(\widetilde w_{m-1}^i) -\alpha \nabla^2 l_i(\widetilde u_{m-1}^i)  \big\| \nonumber
	\\&+ (1+\alpha L)\Big\|\prod_{j=0}^{m-2}(I-\alpha \nabla^2 l_i(\widetilde w^i_{j})) - \prod_{j=0}^{m-2}(I-\alpha \nabla^2 l_i(\widetilde u^i_{j}))\Big\| \nonumber
	\\\leq & (1+\alpha L)^{m-1} \alpha \rho \|\widetilde w_{m-1}^i - \widetilde u_{m-1}^i\| + (1+\alpha L)V(m-1) \nonumber
	\\\leq & (1+\alpha L)^{m-1} \alpha \rho (1+\alpha L)^{m-1} \|w-u\| + (1+\alpha L)V(m-1). 
	\end{align}
	Telescoping~\cref{youyitian} over $m$ from $1$ to $N$ and noting $V(1) \leq \alpha \rho\|w-u\|$, we have 
	\begin{align}\label{vbbn}
	V(N)&\leq (1+\alpha L)^{N-1}V(1) + \sum_{m=0}^{N-2}\alpha \rho (1+\alpha L)^{2(N-m)-2}\|w-u\|(1+\alpha L)^m \nonumber
	\\& =  (1+\alpha L)^{N-1}\alpha \rho \|w-u\| +\alpha \rho (1+\alpha L)^N\sum_{m=0}^{N-2} (1+\alpha L)^{m}\|w-u\| \nonumber
	\\& \leq \left( (1+\alpha L)^{N-1}\alpha \rho + \frac{\rho}{L} (1+\alpha L)^N ((1+\alpha L)^{N-1}-1)\right) \|w-u\|.
	\end{align}
	Recalling the definition of $C_\mathcal{L}$ and 
	 Combining~\cref{delff},~\cref{vbbn}, we have 
	$\|\nabla \mathcal{L}_i(w) - \nabla \mathcal{L}_i(u)\| \leq \big( C_\mathcal{L} \|\nabla l_i(w)\| + (1+\alpha L)^{2N}L  \big) \|w-u\|.$ 
	We then have
	\begin{align*}
	\|\nabla \mathcal{L}(w) - \nabla \mathcal{L}(u)\| &\leq  \mathbb{E}_{i\sim p(\mathcal{T})}\|(\nabla \mathcal{L}_i(w) - \nabla \mathcal{L}_i(u))\|
	\\&\leq  \big( C_\mathcal{L}  \mathbb{E}_{i\sim p(\mathcal{T})}\|\nabla l_i(w)\| + (1+\alpha L)^{2N}L  \big) \|w-u\|,
	\end{align*}
	which finishes the proof. 

\subsection*{Proof of Proposition~\ref{le:distance}}	
	We first prove the first-moment bound. 
	Conditioning on $w_k$, we have 
	\begin{align*}
	\mathbb{E}_{\bar S^i_m}&\|w_{k,m}^i - \widetilde w_{k,m}^i\| 
	\\\overset{(i)}=& \mathbb{E}_{\bar S^i_m}\big\|w_{k,m-1}^i - \alpha \nabla l_i(w_{k,m-1}^i; S_{k,m-1}^i) - (\widetilde w_{k,m-1}^i - \alpha \nabla l_i(\widetilde w_{k,m-1}^i) )\big\|  \nonumber
	\\ \leq& \mathbb{E}_{\bar S^i_m} \|w_{k,m-1}^i - \widetilde w_{k,m-1}^i\| + \alpha \mathbb{E}_{\bar S^i_m}\big\|\nabla l_i(w_{k,m-1}^i; S_{k,m-1}^i) - \nabla l_i(w_{k,m-1}^i)\big\|\nonumber
	\\ &+ \alpha \mathbb{E}_{\bar S^i_m}\big\| \nabla l_i(w_{k,m-1}^i) - \nabla l_i(\widetilde w_{k,m-1}^i) \big\| \nonumber
	\\ \leq & \alpha \mathbb{E}_{\bar S^i_{m-1}} \Big(      \mathbb{E}_{S_{k,m-1}^i} \big(\|\nabla l_i(w_{k,m-1}^i; S_{k,m-1}^i) - \nabla l_i(w_{k,m-1}^i)\big\| \,\Big | \bar S^i_{m-2}\big)\Big)    \nonumber
	\\ & + (1+\alpha L)  \mathbb{E}_{\bar S^i_{m-1}} \|w_{k,m-1}^i - \widetilde w_{k,m-1}^i\| \nonumber
	\\\overset{(ii)}\leq & (1+\alpha L)  \mathbb{E}_{\bar S^i_{m-1}} \|w_{k,m-1}^i - \widetilde w_{k,m-1}^i\|  +  \alpha\frac{\sigma_g}{\sqrt{S}},\nonumber
	\end{align*}
	where (i) follows from~\cref{gd_w} and~\cref{es:up}, and (ii) follows from Assumption~\ref{a3}. 
	Telescoping the above inequality over $m$ from $1$ to $j$ and using  $w_{k,0}^i = \widetilde w_{k,0}^i = w_k$, we have 
	$\mathbb{E}_{\bar S^i_j}\|w_{k,j}^i - \widetilde w_{k,j}^i\|  \leq ((1+\alpha L)^j-1) \frac{\sigma_g}{L\sqrt{S}}$,
	which finishes the proof of the first-moment bound. 
	
	\vspace{0.1cm}
	We next begin to prove the second-moment bound. Conditioning on $w_k$, we have 
	{\small
	\begin{align*}
	&\mathbb{E}_{\bar S^i_m}\|w_{k,m}^i - \widetilde w_{k,m}^i\|^2 
	\\ &=   \mathbb{E}_{\bar S^i_{m-1}}\|w_{k,m-1}^i - \widetilde w_{k,m-1}^i\|^2  + 
	\alpha^2\mathbb{E}_{\bar S^i_m}\|\nabla l_i(w_{k,m-1}^i; S_{k,m-1}^i)- \nabla l_i(\widetilde w_{k,m-1}^i)\|^2
	\\ & \quad-2\alpha\mathbb{E}_{\bar S^i_{m-1}}\left(\mathbb{E}_{S_{k,m-1}^i} \langle w_{k,m-1}^i - \widetilde w_{k,m-1}^i, \nabla l_i(w_{k,m-1}^i; S_{k,m-1}^i)- \nabla l_i(\widetilde w_{k,m-1}^i)\rangle \big | \bar S^i_{m-1}\right)
	\\ &\overset{(i)}\leq   \mathbb{E}_{\bar S^i_{m-1}}\|w_{k,m-1}^i - \widetilde w_{k,m-1}^i\|^2  -2\alpha\mathbb{E}_{\bar S^i_{m-1}} \langle w_{k,m-1}^i - \widetilde w_{k,m-1}^i, \nabla l_i(w_{k,m-1}^i)- \nabla l_i(\widetilde w_{k,m-1}^i)\rangle
	\\ &\quad +
	\alpha^2\mathbb{E}_{\bar S^i_m}\left(  2\|\nabla l_i(w_{k,m-1}^i; S_{k,m-1}^i)- \nabla l_i( w_{k,m-1}^i)\|^2 + 2\|\nabla l_i( w_{k,m-1}^i)- \nabla l_i(\widetilde w_{k,m-1}^i)\|^2 \right)
	\\ &\overset{(ii)}\leq   \mathbb{E}_{\bar S^i_{m-1}}\|w_{k,m-1}^i - \widetilde w_{k,m-1}^i\|^2  +2\alpha\mathbb{E}_{\bar S^i_{m-1}} \| w_{k,m-1}^i - \widetilde w_{k,m-1}^i\|\|\nabla l_i(w_{k,m-1}^i)- \nabla l_i(\widetilde w_{k,m-1}^i)\|
	\\ & \quad+
	\alpha^2\mathbb{E}_{\bar S^i_m}\left(  2\|\nabla l_i(w_{k,m-1}^i; S_{k,m-1}^i)- \nabla l_i( w_{k,m-1}^i)\|^2 + 2\|\nabla l_i( w_{k,m-1}^i)- \nabla l_i(\widetilde w_{k,m-1}^i)\|^2 \right)
	\\ &\leq   \mathbb{E}_{\bar S^i_{m-1}}\|w_{k,m-1}^i - \widetilde w_{k,m-1}^i\|^2  +2\alpha L\mathbb{E}_{\bar S^i_{m-1}} \| w_{k,m-1}^i - \widetilde w_{k,m-1}^i\|^2
	\\ & \quad+
	2\alpha^2\mathbb{E}_{\bar S^i_{m-1}}\Big(  \frac{\sigma_g^2}{S}+ L^2\|w_{k,m-1}^i-  \widetilde w_{k,m-1}^i\|^2 \Big) 
	\\ &\leq \big(1+2\alpha L+2\alpha^2 L^2\big)  \mathbb{E}_{\bar S^i_{m-1}}\|w_{k,m-1}^i - \widetilde w_{k,m-1}^i\|^2   + \frac{2\alpha^2\sigma_g^2}{S},
	\end{align*}}
\hspace{-0.2cm}	where (i) follows from $\mathbb{E}_{S_{k,m-1}^i}  \nabla l_i(w_{k,m-1}^i; S_{k,m-1}^i)= \nabla l_i(w_{k,m-1}^i)$ and (ii) follows from $-\langle a,b\rangle\leq \|a\|\|b\|$. 
 	Noting that $w_{k,0}^i= \widetilde w_{k,0}^i = w_k$ and telescoping the above inequality over $m$ from $1$ to $j$, we obtain
	$\mathbb{E}_{\bar S^i_j}\|w_{k,j}^i - \widetilde w_{k,j}^i\|^2 \leq \left( (1+2\alpha L + 2\alpha^2L^2) ^j -1 \right)\frac{\alpha \sigma_g^2}{L(1+\alpha L) S}$.
	Then,taking the expectation over $w_k$ in the above inequality finishes the proof. 

\subsection*{Proof of Proposition~\ref{th:first-est}}
	Recall {\small $\widehat G_i(w_k)=  \prod_{j=0}^{N-1}(I - \alpha \nabla^2 l_i(w_{k,j}^i; D_{k,j}^i))\nabla l_i(w_{k,N}^i; T^i_k).$}
	Conditioning on $w_k$ yields
	{\small\begin{align}\label{gmeans}
	\mathbb{E} \widehat G_i(w_k) =& \mathbb{E}_{\bar S_N, i\sim p(\mathcal{T})} \mathbb{E}_{\bar D_N}\Big(  \prod_{j=0}^{N-1}  \big(I - \alpha \nabla^2 l_i(w_{k,j}^i; D_{k,j}^i)\big) \mathbb{E}_{T_k^i} \nabla l_i(w_{k,N}^i;  T_k^i) \big | \bar S_N, i  \Big) \nonumber
	\\ = & \mathbb{E}_{\bar S_N, i\sim p(\mathcal{T})}   \prod_{j=0}^{N-1}  \mathbb{E}_{D_{k,j}^i}\big(I - \alpha \nabla^2 l_i(w_{k,j}^i; D_{k,j}^i)\big |\bar  S_N, i \big)  \nabla l_i(w_{k,N}^i)     \nonumber
	\\ = &  \mathbb{E}_{\bar S_N, i\sim p(\mathcal{T})}  \prod_{j=0}^{N-1}  \big(I - \alpha \nabla^2 l_i(w_{k,j}^i)\big)  \nabla l_i(w_{k,N}^i)    ,
	\end{align}} 
	\hspace{-0.23cm}which, combined with {\small $\nabla \mathcal{L}(w_k)  =\mathbb{E}_{i\sim p(\mathcal{T})} \prod_{j=0}^{N-1}(I-\alpha \nabla^2 l_i(\widetilde w^i_{k,j}))\nabla l_i(\widetilde w^i _{k,N})$}, yields
	{\small	\begin{align}\label{eq:ek}
	\|\mathbb{E} \widehat G_i(w_k)&  - \nabla \mathcal{L}(w_k)\|  \nonumber
	\\ \overset{(i)}\leq & \mathbb{E}_{\bar S_N,  i\sim p(\mathcal{T})} \Big \|  \prod_{j=0}^{N-1}  \big(I - \alpha \nabla^2 l_i(w_{k,j}^i)\big)  \nabla l_i(w_{k,N}^i)  -   \prod_{j=0}^{N-1}(I-\alpha \nabla^2 l_i(\widetilde w^i_{k,j}))\nabla l_i(\widetilde w^i _{k,N})  \Big  \|  \nonumber
	\\ \leq &  \mathbb{E}_{\bar S_N,  i\sim p(\mathcal{T})} \Big \|  \prod_{j=0}^{N-1}  \big(I - \alpha \nabla^2 l_i(w_{k,j}^i)\big)  \nabla l_i(w_{k,N}^i)  -   \prod_{j=0}^{N-1}(I-\alpha \nabla^2 l_i(w^i_{k,j}))\nabla l_i(\widetilde w^i _{k,N})  \Big  \| \nonumber
	\\ &+ \mathbb{E}_{\bar S_N,  i\sim p(\mathcal{T})} \Big \|  \prod_{j=0}^{N-1}  \big(I - \alpha \nabla^2 l_i(w_{k,j}^i)\big)  \nabla l_i(\widetilde w_{k,N}^i)  -   \prod_{j=0}^{N-1}(I-\alpha \nabla^2 l_i(\widetilde w^i_{k,j}))\nabla l_i(\widetilde w^i _{k,N})  \Big  \| \nonumber
	\\ \overset{(ii)}\leq & 
	 (1+\alpha L)^N  \mathbb{E}_{\bar S_N,  i} \big\| \nabla l_i( w_{k}) \big \|   \Big \|  \prod_{j=0}^{N-1}  \big(I - \alpha \nabla^2 l_i(w_{k,j}^i)\big) -   \prod_{j=0}^{N-1}(I-\alpha \nabla^2 l_i(\widetilde w^i_{k,j})) \Big  \|   
\nonumber
	\\ &+   (1+\alpha L)^N L\mathbb{E}_{\bar S_N,  i} \big \| w_{k,N}^i  - \widetilde w^i _{k,N}  \big  \|    \nonumber
	\\\overset{(iii)}\leq & (1+\alpha L)^N  \mathbb{E}_{ i} \big\| \nabla l_i( w_{k}) \big \|   \underbrace{\mathbb{E}_{\bar S_N} \Big( \Big \|  \prod_{j=0}^{N-1}  \big(I - \alpha \nabla^2 l_i(w_{k,j}^i)\big) -   \prod_{j=0}^{N-1}(I-\alpha \nabla^2 l_i(\widetilde w^i_{k,j})) \Big  \|\, \Big |\, i \Big)}_{R(N)}    \nonumber
	\\ &+  (1+\alpha L)^N ((1+\alpha L)^N -1\big)  \frac{\sigma_g}{\sqrt{S}},  
	\end{align}}
\hspace{-0.2cm} where (i) follows from Jensen's inequality,  (ii) follows from Lemma~\ref{le:jiw}, and (iii) follows from item 1 in Proposition~\ref{le:distance}. 
	Our next step is to upper-bound the term $R(N)$. To simplify notations, we define a general quantity $R(m)$ by replacing $N$ in $R(N)$ with $m$,  and 
	we use $\mathbb{E}_{\bar S_m | i}(\cdot)$ to denote $\mathbb{E}_{\bar S_m}(\cdot| i)$. Then, we have 
	\begin{align}\label{eq:arjpp}
	R(m) \leq & \mathbb{E}_{\bar S_m| i}  \Big \|  \prod_{j=0}^{m-1}  \big(I - \alpha \nabla^2 l_i(w_{k,j}^i)\big) -   \prod_{j=0}^{m-2}(I-\alpha \nabla^2 l_i( w^i_{k,j})) (I-\alpha \nabla^2 l_i( \widetilde w^i_{k,m-1}) \Big  \| \nonumber
	\\ & + \mathbb{E}_{\bar S_m|i}  \Big \|  \prod_{j=0}^{m-2}(I-\alpha \nabla^2 l_i( w^i_{k,j})) (I-\alpha \nabla^2 l_i( \widetilde w^i_{k,m-1})  -  \prod_{j=0}^{m-1}(I-\alpha \nabla^2 l_i(\widetilde w^i_{k,j})) \Big  \| \nonumber
	\\\leq & (1+\alpha L)^{m-1} \alpha \rho\mathbb{E}_{\bar S_m|i} \|w_{k,m-1}^i - \widetilde w_{k,m-1}^i\| + (1+\alpha L) R(m-1) \nonumber
	\\\overset{(i)}\leq &  \alpha \rho (1+\alpha L)^{m-1} ( (1+\alpha L)^{m-1} -1 )\frac{\sigma_g}{L\sqrt{S}} + (1+\alpha L) R(m-1) \nonumber
	\\ \leq& \alpha \rho (1+\alpha L)^{N-1} \big( (1+\alpha L)^{N-1} -1  \big)\frac{\sigma_g}{L\sqrt{S}}  +  (1+\alpha L) R(m-1),
	\end{align}
	where (i) follows from Proposition~\ref{le:distance}. 
	Telescoping the above inequality over $m$ from $2$ to $N$ and using $R(1) =0$, we have 
	\begin{align}\label{addionl}
	R(N) \leq ((1+\alpha L)^{N-1}-1)^2 (1+\alpha L)^{N-1}\frac{\rho \sigma_g}{L^2\sqrt{S}}.
	\end{align}
	Thus, conditioning on $w_k$ and combining~\cref{addionl} and~\cref{eq:ek}, we have 
{\small	\begin{align*}
	\|\mathbb{E} \widehat G_i(w_k)  - \nabla \mathcal{L}(w_k)\|  
	 \leq & ((1+\alpha L)^{N-1}-1)^2\frac{\rho}{L} (1+\alpha L)^{2N-1}\frac{\sigma_g}{L\sqrt{S}}  \mathbb{E}_{ i\sim p(\mathcal{T})} \big(\big\| \nabla l_i( w_{k}) \big \|   \big) \nonumber
	\\ &+   \frac{(1+\alpha L)^N ((1+\alpha L)^N -1\big)\sigma_g}{\sqrt{S}}\nonumber
	\\\leq &((1+\alpha L)^{N-1}-1)^2\frac{\rho}{L} (1+\alpha L)^{2N-1}\frac{\sigma_g}{L\sqrt{S}}  \Big( \frac{\|\nabla \mathcal{L}(w_k)\| }{1-C_l} + \frac{\sigma }{1-C_l}     \Big) \nonumber
	\\ &+   \frac{(1+\alpha L)^N ((1+\alpha L)^N -1\big)\sigma_g}{\sqrt{S}}, 
	\end{align*} }
\hspace{-0.24cm}	where the last inequality follows from Lemma \ref{le:lL}. 
	Rearranging the above inequality and using $C_{\text{\normalfont err}_1} $ and $C_{\text{\normalfont err}_2}$ defined in Proposition~\ref{th:first-est} finish the proof.  

\subsection*{Proof of Proposition~\ref{th:second}}
Recall {\footnotesize$\widehat G_i(w_k)=  \prod_{j=0}^{N-1}(I - \alpha \nabla^2 l_i(w_{k,j}^i;  D_{k,j}^i))\nabla l_i(w_{k,N}^i;  T^i_k)$}.  
Conditioning on $w_k$, we have 
	\begin{align}\label{esni}
	\mathbb{E}\|&\widehat G_i(w_k)\|^2 \nonumber
	\\\leq &\mathbb{E}_{\bar S_N, i } \bigg( \mathbb{E}_{\bar D_N, T_k^i} \Big(  \Big \|\prod_{j=0}^{N-1}(I - \alpha \nabla^2 l_i(w_{k,j}^i; D_{k,j}^i))\Big\|^2 \|\nabla l_i(w_{k,N}^i; T^i_k)\|^2 \Big | \bar S_N, i         \Big)\bigg) \nonumber
	\\\leq & \underbrace{\mathbb{E}_{\bar S_N, i } \bigg( \prod_{j=0}^{N-1} \mathbb{E}_{\bar D_N} \Big(  \Big \|I - \alpha \nabla^2 l_i(w_{k,j}^i; D_{k,j}^i)\Big\|^2 \Big | \bar S_N, i \Big)}_{P} \underbrace{\mathbb{E}_{T_k^i}\Big( \|\nabla l_i(w_{k,N}^i; T^i_k)\|^2 \Big |\bar  S_N, i         \Big)}_{Q}\bigg). 
	\end{align}
	We  next upper-bound $P$ and $Q$ in~\cref{esni}. Note that $w_{k,j}^i, j=0,...,N-1$ are deterministic when conditioning on $S_N$, $i$, and $w_k$. Thus, conditioning on $S_N$, $i$, and $w_k$, we have 
	\begin{align}\label{pbound}
	\mathbb{E}_{\bar D_N}  \Big \|I - \alpha \nabla^2 l_i(w_{k,j}^i; D_{k,j}^i)\Big\|^2  = & \text{Var} \Big(  I - \alpha \nabla^2 l_i(w_{k,j}^i; D_{k,j}^i)  \Big) +\big\|I - \alpha \nabla^2 l_i(w_{k,j}^i) \big\|^2 \nonumber
	\\\leq & \frac{\alpha^2\sigma_H^2}{D} + (1+\alpha L)^2.
	\end{align}
	We next bound $Q$ term. Conditioning on $\bar S_N, i$ and $w_k$, we have 
	\begin{align}\label{etki}
	\mathbb{E}_{T_k^i} \|\nabla l_i(w_{k,N}^i; T^i_k)\|^2 \overset{(i)}\leq & 3\mathbb{E}_{T_k^i}\|\nabla l_i(w_{k,N}^i;  T^i_k) -\nabla l_i(w_{k,N}^i)\|^2 + 3\mathbb{E}_{T_k^i}\|\nabla l_i(\widetilde w_{k,N}^i) \|^2 \nonumber
	\\ &+ 3\mathbb{E}_{T_k^i}\|\nabla l_i(w_{k,N}^i) - \nabla l_i(\widetilde w_{k,N}^i)\|^2\nonumber
	\\\overset{(ii)}\leq & \frac{3\sigma_g^2}{T} + 3L^2 \|w_{k,N}^i - \widetilde w_{k,N}^i\|^2 + 3(1+\alpha L)^{2N} \|\nabla l_i(w_k)\|^2,
	\end{align}
	where (i) follows from $\|\sum_{i=1}^n a\|^2\leq n\sum_{i=1}^n\|a\|^2$, and (ii) follows from Lemma~\ref{le:jiw}. Thus, conditioning on $w_k$ and combining~\cref{esni},~\cref{pbound},~\cref{etki}, we have  
	\begin{align}
	\mathbb{E}&\|\widehat G_i(w_k)\|^2  \nonumber
	\\  &\leq  3\Big(\frac{\alpha^2\sigma_H^2}{D} + (1+\alpha L)^2\Big)^N\Big( \frac{\sigma_g^2}{T} + L^2 \mathbb{E}\|w_{k,N}^i - \widetilde w_{k,N}^i\|^2 + (1+\alpha L)^{2N} \mathbb{E}\|\nabla l_i(w_k)\|^2\Big) \nonumber
	\end{align}
	which, in conjunction with Proposition~\ref{le:distance},  
	yields 
	\begin{align}\label{ggsmida}
	\mathbb{E}\|\widehat G_i(w_k)\|^2   \leq& 3(1+\alpha L)^{2N} \big(\frac{\alpha^2\sigma_H^2}{D} + (1+\alpha L)^2\big)^N (\|\nabla l(w_k)\|^2 + \sigma^2)\nonumber
	\\&+\frac{C_{\text{\normalfont squ}_1}}{T} + \frac{C_{\text{\normalfont squ}_2}}{S}.
	\end{align} 	
Based on Lemma~\ref{le:lL} and  conditioning on $w_k$, we have  
	\begin{align*}
	\|\nabla l(w_k)\|^2 \leq \frac{2}{(1-C_l)^2} \|\nabla \mathcal{L}(w_k)\| + \frac{2C_l^2}{(1-C_l)^2} \sigma^2,
	\end{align*}
	which, in conjunction with  $\frac{2x^2}{(1-x)^2}+1 \leq \frac{2}{(1-x)^2}$ and \cref{ggsmida}, finishes the proof. 



\subsection*{Proof of Theorem~\ref{th:mainonline}}
 The proof of Theorem~\ref{th:mainonline} consists of four main steps: step $1$ of bounding an iterative meta update by the meta-gradient smoothness established by Proposition~\ref{th:lipshiz}; step $2$ of characterizing first-moment  error of the meta-gradient estimator { $\widehat G_i(w_k)$} by Proposition~\ref{th:first-est}; step $3$ of characterizing second-moment  error of the meta-gradient estimator { $\widehat G_i(w_k)$} by Proposition~\ref{th:second}; and step $4$ of combining steps 1-3, and telescoping to yield the convergence. 
 
To simplify notations, define the smoothness parameter of the meta-gradient as $$L_{w_k} = (1+\alpha L)^{2N}L + C_\mathcal{L} \mathbb{E}_{i\sim p(\mathcal{T})}\|\nabla l_i(w_k)\|,$$ where $C_\mathcal{L}$ is given in~\cref{clcl}. 	
Based on the smoothness of the gradient $\nabla \mathcal{L}(w) $ given by Proposition~\ref{th:lipshiz}, we have 
	\begin{align*}
	\mathcal{L}(w_{k+1}) \leq & \mathcal{L}(w_k) + \langle  \nabla \mathcal{L}(w)  , w_{k+1} - w_{k}    \rangle + \frac{L_{w_k}}{2} \|w_{k+1}-w_{k}\|^2 \nonumber
	\end{align*}
	The randomness from $\beta_k$ depends on $B_k^\prime$ and $D_{L_k}^i, i \in B_k^{\prime}$, and thus is independent of $S_{k,j}^i, D_{k,j}^i$ and $T_k^i$ for $i\in B_k, j=0,...,N$. Then,  taking expectation over the above inequality, conditioning on $w_k$, and recalling $e_k := \mathbb{E}\widehat G_i(w_k) - \nabla \mathcal{L}(w_k) $, we have  
	$\mathbb{E}( \mathcal{L}(w_{k+1})| w_k) \leq \mathcal{L}(w_{k}) - \mathbb{E} (\beta_k)\langle \nabla \mathcal{L}(w_{k}), \nabla \mathcal{L}(w_{k}) + e_k\rangle+\frac{L_{w_k}\mathbb{E} (\beta^2_k)\mathbb{E} \big\|  \frac{1}{B} \sum_{i\in B_k} \widehat G_i(w_k) \big \|^2}{2}.$ 
Then, applying Lemma~\ref{le:xiaodege} in the above inequality yields
	\begin{align}\label{qiangxing}
	\mathbb{E}( \mathcal{L}(w_{k+1})| w_k)
	\leq & \mathcal{L}(w_{k}) -\frac{4}{5C_\beta} \frac{1}{ L_{w_k}}  \|\nabla \mathcal{L}(w_{k})\|^2+ \frac{4}{5C_\beta} \frac{1}{ L_{w_k}} |\langle \nabla \mathcal{L}(w_{k}), e_k\rangle|  \nonumber
	\\ &+ \frac{2}{C_\beta^2} \frac{1}{L_{w_k}} \Big( \frac{1}{B}\mathbb{E}  \big\|   \widehat G_i(w_k) \big \|^2  +    \|\mathbb{E}  \widehat G_i(w_k)\|^2\Big).\nonumber
	\\\leq & \mathcal{L}(w_{k}) -\frac{4}{5C_\beta} \frac{1}{ L_{w_k}}  \|\nabla \mathcal{L}(w_{k})\|^2+ \frac{2}{5C_\beta} \frac{1}{ L_{w_k}} \| \nabla \mathcal{L}(w_{k})\|^2 +  \frac{2}{5C_\beta} \frac{\| e_k\|^2 }{ L_{w_k}} \nonumber
	\\ &+ \frac{2}{C_\beta^2} \frac{1}{L_{w_k}} \Big( \frac{1}{B}\mathbb{E}  \big\|   \widehat G_i(w_k) \big \|^2  +    \|\mathbb{E}  \widehat G_i(w_k)\|^2\Big).
	\end{align}
	Then, 
	applying Propositions~\ref{th:first-est} and~\ref{th:second} to  the above inequality yields
	\begin{align}\label{sikas}
	\mathbb{E}(& \mathcal{L}(w_{k+1})| w_k)  \nonumber
	 \\\leq & \mathcal{L}(w_{k}) -\frac{2}{5C_\beta} \frac{1}{ L_{w_k}}  \|\nabla \mathcal{L}(w_{k})\|^2+ \frac{2}{C_\beta^2} \frac{1}{L_{w_k}}  \frac{1}{B}\mathbb{E}  \big\|   \widehat G_i(w_k) \big \|^2  +    \frac{4}{C_\beta^2} \frac{1}{L_{w_k}} \|\nabla \mathcal{L}(w_{k})\|^2    \nonumber
	\\ &+ \Big(\frac{6}{5C_\beta L_{w_k}}  + \frac{12}{C_\beta^2 L_{w_k}} \Big) \Big ( \frac{C^2_{{\text{\normalfont err}}_2}  }{S}\|\nabla \mathcal{L}(w_k)\|^2 +\frac{C^2_{{\text{\normalfont err}}_1}}{S} + \frac{C^2_{{\text{\normalfont err}}_2} \sigma^2  }{S}\Big ) 
	  \nonumber
	\\ \leq & \mathcal{L}(w_{k}) - \frac{2}{C_\beta L_{w_k}} \left(   \frac{1}{5} - \left(\frac{3}{5} + \frac{6}{C_\beta}\right)\frac{C^2_{{\text{\normalfont err}}_2} }{S} - \frac{C_{\text{\normalfont squ}_3}}{C_\beta B} - \frac{2}{C_\beta}\right) \|\nabla \mathcal{L}(w_k) \|^2 \nonumber
	\\ &+ \frac{6( \frac{1}{5} +\frac{2}{C_\beta})}{C_\beta L_{w_k}S}\Big( C^2_{{\text{\normalfont err}}_1} +C^2_{{\text{\normalfont err}}_2}\sigma^2\Big)  +  \frac{2}{C_\beta^2 L_{w_k}B} \Big(  \frac{C_{\text{\normalfont squ}_1}}{T}  +  \frac{C_{\text{\normalfont squ}_2}}{S} + C_{\text{\normalfont squ}_3} \sigma^2 \Big).
	\end{align}
	Recalling {$L_{w_k} = (1+\alpha L)^{2N}L + C_\mathcal{L} \mathbb{E}_{i}\|\nabla l_i(w_k)\|$}, we have  $ L_{w_k} \geq L$ and 
	\begin{align}\label{oips}
	L_{w_k} 
	 \overset{(i)}\leq &  (1+\alpha L)^{2N}L + \frac{C_\mathcal{L}\sigma}{1-C_l}+ \frac{C_\mathcal{L}}{1-C_l} \|\nabla \mathcal{L}(w_k)\|,  
	\end{align}
	where (i) follows from Assumption~\ref{a2} and Lemma~\ref{le:lL}. 
	Combining~\cref{sikas} and~\cref{oips}, we have 
	\begin{align}\label{miops}
	\mathbb{E}( \mathcal{L}(w_{k+1})| w_k) \leq& \mathcal{L}(w_{k}) + \frac{6}{C_\beta L}\Big( \frac{1}{5} +\frac{2}{C_\beta}   \Big)\Big( C^2_{{\text{\normalfont err}}_1} +C^2_{{\text{\normalfont err}}_2}\sigma^2\Big) \frac{1}{S} \nonumber
	\\&+  \frac{2}{C_\beta^2 L} \Big(  \frac{C_{\text{\normalfont squ}_1}}{T}  +  \frac{C_{\text{\normalfont squ}_2}}{S} + C_{\text{\normalfont squ}_3} \sigma^2 \Big) \frac{1}{B} \nonumber
	\\ &	- \frac{2}{C_\beta } \frac{ \frac{1}{5} - \left(\frac{3}{5} + \frac{6}{C_\beta}\right)\frac{C^2_{{\text{\normalfont err}}_2} }{S} - \frac{C_{\text{\normalfont squ}_3}}{C_\beta B} - \frac{2}{C_\beta}}{(1+\alpha L)^{2N}L + \frac{C_\mathcal{L}\sigma}{1-C_l}+ \frac{C_\mathcal{L}}{1-C_l} \|\nabla \mathcal{L}(w_k)\|} \|\nabla \mathcal{L}(w_k) \|^2.
	\end{align}
	Based on the notations in \cref{para:com}, we rewrite~\cref{miops} as  
	\begin{align*}
	\mathbb{E}&( \mathcal{L}(w_{k+1})| w_k) \leq \mathcal{L}(w_{k}) + \frac{\xi}{S} +  \frac{\phi }{B}	-\theta \frac{  \|\nabla \mathcal{L}(w_k) \|^2}{\chi +\|\nabla \mathcal{L}(w_k) \|}.
	\end{align*}
	Unconditioning on $w_k$ in the above inequality and  
	telescoping the above inequality over $k$ from $0$ to $K-1$, we have 
	\begin{align}\label{ggopos}
	\frac{1}{K}\sum_{k=0}^{K-1} \mathbb{E}\left(\frac{ \theta \|\nabla \mathcal{L}(w_k) \|^2}{\chi +\|\nabla \mathcal{L}(w_k) \|}\right) \leq \frac{\Delta}{K} +    \frac{\xi}{S} +  \frac{\phi }{B},
	\end{align}
	where $\Delta = \mathcal{L}(w_0) - \mathcal{L}^*$. 
	Choosing $\zeta$  from $\{0,...,K-1\}$ uniformly at random, we obtain from \cref{ggopos} that 
	\begin{align}\label{medistep}
	\mathbb{E}\left(\frac{ \theta \|\nabla \mathcal{L}(w_\zeta) \|^2}{\chi +\|\nabla \mathcal{L}(w_\zeta) \|}\right) \leq \frac{\Delta}{K} +    \frac{\xi}{S} +  \frac{\phi }{B}.
	\end{align}
	Consider a function $f(x) = \frac{x^2}{c+x}, \,x>0$, where $c>0$ is a constant. Simple computation shows that $f^{\prime\prime}(x) =\frac{2c^2}{(x+c)^3}>0$. Thus, using Jensen's inequality in \cref{medistep}, we have 
	\begin{align}\label{reoolls}
	\frac{ \theta (\mathbb{E}\|\nabla \mathcal{L}(w_\zeta) \|)^2}{\chi +\mathbb{E}\|\nabla \mathcal{L}(w_\zeta) \|} \leq \frac{\Delta}{K} +    \frac{\xi}{S} +  \frac{\phi }{B}.
	\end{align}
	Rearranging the above inequality yields
	\begin{align}\label{havetogo}
	\mathbb{E}\|\nabla \mathcal{L}(w_\zeta) \|  \leq &\frac{\Delta}{\theta }\frac{1}{K} +    \frac{\xi}{\theta}\frac{1}{S} +  \frac{\phi }{\theta}\frac{1}{B} + \sqrt{\frac{\chi}{2} }\sqrt{\frac{\Delta}{\theta }\frac{1}{K} +    \frac{\xi}{\theta}\frac{1}{S} +  \frac{\phi }{\theta}\frac{1}{B}},
	\end{align}
	which finishes the proof.
\subsection*{Proof of Corollary~\ref{co:online}}
	Since $\alpha = \frac{1}{8NL}$, we have 
	\begin{align*}
	(1+\alpha L)^ N = &\big(1+ \frac{1}{8N}\big)^N = e^{N\log(1+\frac{1}{8N})} \leq e^{1/8}  < \frac{5}{4},
	(1+\alpha L)^ {2N} < e^{1/4} <  \frac{3}{2},
	\end{align*}
	which, in conjunction with~\cref{ppolp}, implies that 
	\begin{align}\label{errpp}
	C_{{\text{\normalfont err}}_1}  < \frac{5\sigma_g}{16},\quad C_{{\text{\normalfont err}}_2} < \frac{3\rho \sigma_g}{4L^2}.
	\end{align}
	Furthermore, noting that $D \geq \sigma_H^2/L^2$, we have
	\begin{align}\label{ctextp}
	C_{\text{\normalfont squ}_1} \leq &3(1+2\alpha L + 2\alpha^2 L^2)^N\sigma_g^2 <4\sigma_g^2, \; C_{\text{\normalfont squ}_2} < \frac{1.3\sigma^{2}_g}{8} < \frac{\sigma_g^2}{5},\; C_{\text{\normalfont squ}_3} \leq 11.
	\end{align}
	Based on~\cref{clcl}, we have 
	\begin{align}\label{bb:cl}
	C_{\mathcal{L}}<& \frac{75}{128}\frac{\rho}{L}<\frac{3}{5}\frac{\rho}{L} \,\text{ and } \,C_{\mathcal{L}} \overset{(i)}>  \frac{\rho}{L} ((N-1) \alpha L) > \frac{1}{16}\frac{\rho}{L},
	\end{align}
	where (i) follows from the inequality that $(1+a)^n > 1 +an$. 
	Then, using \cref{errpp}, \cref{ctextp} and~\cref{bb:cl}, we obtain from \cref{para:com} that 
	\begin{align}\label{manypara}
	\xi < &\frac{7}{500L} \Big( \frac{1}{10} + \frac{9\rho\sigma^2}{16L^4}\Big)\sigma_g^2,\quad \phi \leq \frac{1}{5000L} \Big( \frac{3\sigma_g^2}{T} + \frac{\sigma_g^2}{5S} + 11\sigma^2\Big) < \frac{1}{1000L}(\sigma_g^2 + 3\sigma^2)  \nonumber
	\\ \theta \geq & \frac{L}{60\rho} \Big( \frac{1}{5} - \frac{4}{5} \frac{9}{16} \frac{\rho^2\sigma_g^2}{L^4}\frac{1}{S}  - \frac{11}{100B} - \frac{1}{50}\Big) 
	 = \frac{L}{1500\rho},\; \chi \leq   \frac{24L^2}{\rho} +\sigma.
	\end{align}
	Then, treating $\Delta, \rho, L$ as constants and using~\cref{c:result}, we obtain
	\begin{small}
	\begin{align*}
	\mathbb{E}\|\nabla \mathcal{L}(w_\zeta) \|  \leq \mathcal{O} \Big(  \frac{1}{K} + \frac{\sigma_g^2(\sigma^2+1)}{S} &+ \frac{\sigma_g^2 +\sigma^2}{B} +\frac{\sigma^2_g}{TB}
	\\&+ \sqrt{\sigma +1}\sqrt{\frac{1}{K} + \frac{\sigma_g^2(\sigma^2+1)}{S} + \frac{\sigma_g^2 +\sigma^2}{B}+\frac{\sigma^2_g}{TB}}\Big).
	\end{align*}
	\end{small}
	\hspace{-0.15cm}Then, choosing  $S\geq C_S\sigma_g^2(\sigma^2+1)\max(\sigma,1)\epsilon^{-2}$, $B\geq C_B(\sigma_g^2+\sigma^2)\max(\sigma,1)\epsilon^{-2}$ and $TB >C_{T}\sigma_g^2\max(\sigma,1) \epsilon^{-2}$, we have 
	\begin{align*}
	\mathbb{E}\|\nabla \mathcal{L}(w_\zeta) \|  \leq  \mathcal{O}\bigg(\frac{1}{K} + \frac{1}{\epsilon^2} \Big(\frac{1}{C_S} +\frac{1}{C_B}+\frac{1}{C_{T}} \Big)+ \sqrt{\sigma} \sqrt{\frac{1}{K}+\frac{1}{\sigma\epsilon^2}\Big(\frac{1}{C_S} +\frac{1}{C_B}+\frac{1}{C_{T}} \Big)}\bigg)
	\end{align*}
	After at most $K =  C_K\max(\sigma,1)\epsilon^{-2}$ iterations, the above inequality implies, for constants $C_S, C_B,C_T$ and $C_K$ large enough, 
	$\mathbb{E}\|\nabla \mathcal{L}(w_\zeta)\| \leq \epsilon$.
	Recall that we need $|B_k^\prime| > \frac{4C^2_{\mathcal{L}}\sigma^2}{3(1+\alpha L)^{4N}L^2}$ and $|D_{L_k}^i| > \frac{64\sigma^2_g C_\mathcal{L}^2}{(1+\alpha L)^{4N}L^2}$ for building stepsize $\beta_k$ at each iteration $k$. Based on the selected parameters, we have 
	\begin{align*}
	\frac{4C^2_{\mathcal{L}}\sigma^2}{3(1+\alpha L)^{4N}L^2} \leq \frac{4\sigma^2}{3L^2} \frac{3\rho}{5L}\leq \Theta({\sigma^2}), \quad \frac{64\sigma^2_g C_\mathcal{L}^2}{(1+\alpha L)^{4N}L^2} < \Theta(\sigma_g^2),
	\end{align*}
	which implies  $|B_k^\prime| =\Theta(\sigma^2)$ and $|D_{L_k}^i| =\Theta(\sigma^2_g)$. Then, since  the batch size $D =\Theta(\sigma_H^2/L^2)$, the total number of gradient computations at each meta iteration $k$ is given by 
	$B (NS+T) + |B_k^\prime||D_{L_k}^i|\leq \mathcal{O}(N\epsilon^{-4}+\epsilon^{-2}    )$.
	The total number of Hessian computations at each meta iteration is  
	$BND \leq \mathcal{O}(N\epsilon^{-2}). $
	This completes the proof. 

\section{Proof for Convergence in Finite-Sum Case}

For the finite-sum case, we provide the proofs for Propositions~\ref{finite:lip},~\ref{finite:seconderr} on the properties of meta gradient, and Theorem~\ref{mainth:offline} and  Corollary~\ref{co:mainoffline} on the convergence and complexity of multi-step MAML. The proofs of these results rely on several technical lemmas, which we relegate to~\Cref{aux:lemma_finite}.  
\subsection*{Proof of Proposition~\ref{finite:lip}}
By the definition of $\nabla \mathcal{L}_i(\cdot)$, we have 
	\begin{align}\label{lopasv}
	\|\nabla \mathcal{L}_i(w) & -\nabla \mathcal{L}_i(u) \| \nonumber
	\\\leq &\Big\|\prod_{j=0}^{N-1}(I - \alpha \nabla^2 l_{S_i}(\widetilde w_{j}^i))\nabla l_{T_i}(\widetilde w_{N}^i) -\prod_{j=0}^{N-1}(I - \alpha \nabla^2 l_{S_i}(\widetilde u_{j}^i))\nabla l_{T_i}(\widetilde w_{N}^i)\Big\|  \nonumber
	\\ & + \Big\|\prod_{j=0}^{N-1}(I - \alpha \nabla^2 l_{S_i}(\widetilde u_{j}^i))\nabla l_{T_i}(\widetilde w_{N}^i) -\prod_{j=0}^{N-1}(I - \alpha \nabla^2 l_{S_i}(\widetilde u_{j}^i))\nabla l_{T_i}(\widetilde u_{N}^i)\Big\| \nonumber
	\\ \leq &\underbrace{ \Big\|\prod_{j=0}^{N-1}(I - \alpha \nabla^2 l_{S_i}(\widetilde  w_{j}^i)) -\prod_{j=0}^{N-1}(I - \alpha \nabla^2 l_{S_i}(\widetilde u_{j}^i))\Big\|}_{A}  \|\nabla l_{T_i}(\widetilde  w_{N}^i)\|  \nonumber
	\\& + (1+\alpha L)^N \|\nabla l_{T_i}(\widetilde w_{N}^i)- \nabla l_{T_i}(\widetilde  u_{N}^i)\|.
	\end{align}
	We next upper-bound $A$ in the above inequality. Specifically,  we have
	\begin{align}\label{alegeq}
	A \leq &  \Big\|\prod_{j=0}^{N-1}(I - \alpha \nabla^2 l_{S_i}(\widetilde w_{j}^i)) -\prod_{j=0}^{N-2}(I - \alpha \nabla^2 l_{S_i}(\widetilde w_{j}^i))(I - \alpha \nabla^2 l_{S_i}(\widetilde u_{N-1}^i))\Big\| \nonumber
	\\ &+\Big\| \prod_{j=0}^{N-2}(I - \alpha \nabla^2 l_{S_i}(\widetilde w_{j}^i))(I - \alpha \nabla^2 l_{S_i}(\widetilde u_{N-1}^i))-\prod_{j=0}^{N-1}(I - \alpha \nabla^2 l_{S_i}(\widetilde u_{j}^i))\Big\|\nonumber
	\\ \leq &\Big(   (1+\alpha  L)^{N-1}\alpha \rho  + \frac{\rho}{L} (1+\alpha L)^N \big( (1+\alpha L)^{N-1} -1 \big)\Big)\|w-u\|,
	\end{align}
	where the last inequality uses an approach similar to \cref{vbbn}. 
	Combining~\cref{lopasv} and \cref{alegeq} yields
	\begin{align}\label{inops}
	\|\nabla \mathcal{L}_i(w) & -\nabla \mathcal{L}_i(u) \| \nonumber
\\	 \leq& \big(   (1+\alpha  L)^{N-1}\alpha \rho  + \frac{\rho}{L} (1+\alpha L)^N \big( (1+\alpha L)^{N-1} -1 \big)\big)\|w-u\|  \|\nabla l_{T_i}(\widetilde w_{N}^i)\| \nonumber
	\\ &+ (1+\alpha L)^NL \|\widetilde w_{N}^i- \widetilde u_{N}^i\|.
	\end{align}
	To upper-bound $ \|\nabla l_{T_i}(\widetilde w_{N}^i)\| $ in~\cref{inops},  using the mean value theorem, we have
	\begin{align}\label{lowni}
	\|\nabla l_{T_i}(\widetilde w_{N}^i)\|  = &  \Big\|\nabla l_{T_i} (w-\sum_{j=0}^{N-1}\alpha \nabla l_{S_i}(\widetilde w_j^i))\Big\|\nonumber
	\\ \overset{(i)}\leq & \|\nabla l_{T_i} (w)\| + \alpha L \sum_{j=0}^{N-1} (1+\alpha L)^j\big\| \nabla l_{S_i}(w)   \big\| \nonumber
	\\ \overset{(ii)}\leq & (1+\alpha L)^N  \|\nabla l_{T_i} (w)\|  + \big( (1+\alpha L)^N-1 \big)b_i,
	\end{align}
	where (i) follows from Lemma~\ref{finite:gbd}, and (ii) follows from Assumption~\ref{assum:vaoff}. In addition, using an approach similar to Lemma~\ref{d_u_w}, we have
	\begin{align}\label{wnos}
	\|\widetilde w_{N}^i- \widetilde u_{N}^i\| \leq (1+\alpha L)^N \|w-u\|.
	\end{align}
	Combining~\cref{inops}, \cref{lowni} and \cref{wnos} yields
	{\small \begin{align*}
	\|\nabla &\mathcal{L}_i(w)  -\nabla \mathcal{L}_i(u) \| 
	\\ \leq& \Big(   (1+\alpha  L)^{N-1}\alpha \rho  + \frac{\rho}{L} (1+\alpha L)^N \big( (1+\alpha L)^{N-1} -1 \big)\Big)(1+\alpha L)^N \|\nabla l_{T_i} (w)\|\|w-u\| \nonumber
	\\&+ \Big(   (1+\alpha  L)^{N-1}\alpha \rho  + \frac{\rho}{L} (1+\alpha L)^N \big( (1+\alpha L)^{N-1} -1 \big)\Big)\big( (1+\alpha L)^N-1 \big)b_i\|w-u\| \nonumber
	\\ &+ (1+\alpha L)^{2N}L \|w- u\|,
	\end{align*}}
	\hspace{-0.23cm}which, in conjunction with $C_b$ and $C_\mathcal{L}$ given in \cref{cl1ss}, yields 
	\begin{align*}
	\|\nabla \mathcal{L}_i(w)  -\nabla \mathcal{L}_i(u) \| \leq \big((1+\alpha L)^{2N}L + C_bb_i + C_{\mathcal{L}} \|\nabla l_{T_i}(w)\|  \big)\|w-u\|.
	\end{align*}
	Based on the above inequality and  Jensen's inequality, we 
	finish the proof.

\subsection*{Proof of Proposition~\ref{finite:seconderr}}
	Conditioning on $w_k$, we have 
	\begin{align*}
	\mathbb{E}\|\widehat G_i(w_k)\|^2 = &\mathbb{E} \Big\| \prod_{j=0}^{N-1}(I - \alpha \nabla^2 l_{S_i}(w_{k,j}^i))\nabla l_{T_i}(w_{k,N}^i)  \Big\|^2 \leq  (1+\alpha L)^{2N} \mathbb{E} \|\nabla l_{T_i}(w_{k,N}^i)\|^2,
	\end{align*}
	which, using an approach similar to \cref{lowni}, yields
	\begin{align}\label{gwkopo}
	\mathbb{E}\|&\widehat G_i(w_k)\|^2\nonumber
	\\ \leq&  (1+\alpha L)^{2N} 2(1+\alpha L)^{2N} \mathbb{E} \|\nabla l_{T_i}(w_k)\|^2 + 2(1+\alpha L)^{2N} \big( (1+\alpha L)^N -1\big)^2 \mathbb{E}_i b_i^2 \nonumber
	\\ \leq  & 2(1+\alpha L)^{4N} (\|\nabla l_{T}(w_k)\|^2 + \sigma^2)+ 2(1+\alpha L)^{2N} \big( (1+\alpha L)^N -1\big)^2 \widetilde b \nonumber
	\\ \overset{(i)}\leq & 2(1+\alpha L)^{4N} \Big(  \frac{2}{C_1^2} \|\nabla l_{T}(w_k)\|^2 + \frac{2C_2^2}{C_1^2} + \sigma^2 \Big) + 2(1+\alpha L)^{2N} \big( (1+\alpha L)^N -1\big)^2 \widetilde b \nonumber
	\\ \leq & \frac{4(1+\alpha L)^{4N}}{C_1^2}\|\nabla l_{T}(w_k)\|^2 + \frac{4(1+\alpha L)^{4N}C_2^2}{C_1^2} + 2(1+\alpha L)^{4N}(\sigma^2 + \widetilde b),
	\end{align}
	where (i) follows from Lemma~\ref{twc1c2}, and constants $C_1$ and $C_2$ are given by \cref{c1c2}. Noting that $C_2=\big( (1+\alpha L)^{2N}-1  \big)\sigma + (1+\alpha L)^N \big((1+\alpha L)^N -1 \big) b < \big( (1+\alpha L)^{2N}-1  \big)(\sigma +b)$ and using the definitions of $A_{\text{\normalfont squ}_1}, A_{\text{\normalfont squ}_2}$ in \cref{wocaopp}, we finish the proof. 
\subsection*{Proof of Theorem~\ref{mainth:offline}}
	Based on the smoothness of $\nabla \mathcal{L}(\cdot)$ established in Proposition~\ref{finite:lip}, we have
	\begin{align*}
	\mathcal{L}(w_{k+1}) 
	 \leq &\mathcal{L}(w_k) -\beta_k\Big \langle \nabla \mathcal{L}(w_k), \frac{1}{B}\sum_{i\in B_k} \widehat G_i(w_k)\Big\rangle + \frac{L_{w_k}\beta_k^2}{2}\Big\|\frac{1}{B}\sum_{i\in B_k} \widehat G_i(w_k)\Big\|^2 \nonumber
	\end{align*}
	Taking the conditional expectation given $w_k$ over the above inequality and noting that the randomness over $\beta_k$ is independent of the randomness over $ \widehat G_i(w_k)$,  we have 
	\begin{align}\label{wk1k1}
	\mathbb{E}	(\mathcal{L}(w_{k+1})  | w_k)  \leq &\mathcal{L}(w_{k}) - \frac{1}{C_\beta}\mathbb{E}\Big(\frac{1}{\hat L_{w_k}} \,\Big |\, w_k\Big) \|\nabla \mathcal{L}(w_{k}) \|^2 \nonumber
	\\&+  \frac{L_{w_k}}{2C_\beta^2}\mathbb{E}\Big(\frac{1}{\hat L^2_{w_k}} \,\Big |\, w_k\Big) \mathbb{E} \Big(  \Big\|\frac{1}{B}\sum_{i\in B_k} \widehat G_i(w_k)\Big\|^2  \Big| w_k   \Big).
	\end{align}
	Note that, conditioning on $w_k$,
	\begin{align}\label{bbe1b}
	\mathbb{E}  \Big\|\frac{1}{B}\sum_{i\in B_k} \widehat G_i(w_k)\Big\|^2   
	\leq & \frac{1}{B}\big(  A_{\text{\normalfont squ}_1} \|\nabla \mathcal{L}(w_k)\|^2 + A_{\text{\normalfont squ}_2}   \big)  + \|\nabla \mathcal{L}(w_k)\|^2
	\end{align}
	where the inequality follows from Proposition~\ref{finite:seconderr}. Then, combining~\cref{bbe1b},~\cref{wk1k1} and applying Lemma~\ref{le:betak}, we have 
	\begin{align}\label{lwkpkps}
	\mathbb{E}	(\mathcal{L}(w_{k+1})  | w_k)
	 \leq & \mathcal{L}(w_{k})  - \Big(  \frac{1}{L_{w_k}C_\beta}  - \frac{\frac{A_{\text{\normalfont squ}_1}}{B}+1}{L_{w_k}C_\beta^2}\Big)\|\nabla \mathcal{L}(w_k)\|^2 + \frac{A_{\text{\normalfont squ}_2}}{L_{w_k}C_\beta^2b}.
	\end{align}
	Recalling that $L_{w_k} = (1+\alpha L)^{2N}L + C_b b +  C_\mathcal{L} \mathbb{E}_{i\sim p(\mathcal{T})}\|\nabla l_{T_i}(w_k)\|$ and conditioning on $w_k$, we have  $L_{w_k}\geq L$ and 
	\begin{align}\label{lkulpi}
	L_{w_k} \leq & (1+\alpha L)^{2N}L + C_b b +  C_\mathcal{L} (\|\nabla l_{T}(w_k)\| + \sigma) \nonumber
	\\\overset{(i)}\leq&(1+\alpha L)^{2N}L + C_b b + C_\mathcal{L}\Big( \frac{C_2}{C_1} +\sigma\Big)  + \frac{C_\mathcal{L}}{C_1} \|\nabla \mathcal{L}(w_k)\|,
	\end{align}
	where $(i)$ follows from Lemma~\ref{twc1c2}. Combining~\cref{lkulpi} and~\cref{lwkpkps} yields
	\begin{small}
	\begin{align}\label{b2n8} 
	\mathbb{E}	(&\mathcal{L}(w_{k+1})  | w_k) \nonumber
	\\\leq&  \mathcal{L}(w_{k})  -\frac{\Big(  \frac{1}{C_\beta}  - \frac{1}{C_\beta^2}\Big( \frac{A_{\text{\normalfont squ}_1}}{B}+1\Big)\Big)\|\nabla \mathcal{L}(w_k)\|^2 }{(1+\alpha L)^{2N}L + C_b b + C_\mathcal{L}\Big( \frac{C_2}{C_1} +\sigma\Big)  + \frac{C_\mathcal{L}}{C_1} \|\nabla \mathcal{L}(w_k)\|} + \frac{1}{LC_\beta^2} \frac{A_{\text{\normalfont squ}_2}}{B} \nonumber
	\\ = &\mathcal{L}(w_{k})  -\frac{\frac{C_1}{C_\mathcal{L}} \Big(  \frac{1}{C_\beta}  - \frac{1}{C_\beta^2}\Big( \frac{A_{\text{\normalfont squ}_1}}{B}+1\Big)\Big)\|\nabla \mathcal{L}(w_k)\|^2 }{\frac{C_1}{C_\mathcal{L}} (1+\alpha L)^{2N}L + \frac{bC_1C_b  }{C_\mathcal{L}} +C_2 +C_1\sigma  + \|\nabla \mathcal{L}(w_k)\|} + \frac{1}{LC_\beta^2} \frac{A_{\text{\normalfont squ}_2}}{B} \nonumber
	\\= &\mathcal{L}(w_{k})  -\frac{\frac{C_1}{C_\mathcal{L}} \Big(  \frac{1}{C_\beta}  - \frac{1}{C_\beta^2}\Big( \frac{A_{\text{\normalfont squ}_1}}{B}+1\Big)\Big)\|\nabla \mathcal{L}(w_k)\|^2 }{\frac{C_1}{C_\mathcal{L}} (1+\alpha L)^{2N}L + \frac{bC_1C_b  }{C_\mathcal{L}} +(1+\alpha L)^N((1+\alpha L)^{2N}-1)b  + \|\nabla \mathcal{L}(w_k)\|} +  \frac{A_{\text{\normalfont squ}_2}}{ LC_\beta^2 B},
	\end{align}
	\end{small}
	\hspace{-0.12cm}where the last equality follows from the definitions of $C_1,C_2$ in \cref{c1c2}. 
Combining the definitions in~\cref{offline:constants}  with \cref{b2n8} and taking the expectation over  $w_k$,  we have
	\begin{align*}
	\mathbb{E}\frac{\theta \|\nabla \mathcal{L}(w_k)\|^2}{\xi + \|\nabla \mathcal{L}(w_k)\|} \leq \mathbb{E}( \mathcal{L}(w_{k})  - \mathcal{L}(w_{k+1})     ) + \frac{\phi}{B}.
	\end{align*}
	Telescoping the above bound over $k$ from $0$ to $K-1$ and choosing $\zeta$  from $\{0,...,K-1\}$ uniformly at random, we have
	\begin{align}\label{oppps}
	\mathbb{E}\frac{\theta \|\nabla \mathcal{L}(w_\zeta)\|^2}{\xi + \|\nabla \mathcal{L}(w_\zeta)\|}  \leq \frac{\Delta}{K} +\frac{\phi}{B}. 
	\end{align}
	Using an approach similar to \cref{reoolls}, we obtain from~\cref{oppps} that 
	\begin{align*}
	\frac{	(\mathbb{E}\|\nabla \mathcal{L}(w_\zeta)\|)^2}{\xi + 	\mathbb{E}\|\nabla \mathcal{L}(w_\zeta)\|}  \leq \frac{\Delta}{\theta K} +\frac{\phi}{\theta B},
	\end{align*}
	which further implies that 
	\begin{align}\label{iolscasa}
	\mathbb{E}\|\nabla \mathcal{L}(w_\zeta)\| \leq \frac{\Delta}{2\theta K} +\frac{\phi}{2\theta B} + \sqrt{ \xi \Big(\frac{\Delta}{\theta K} +\frac{\phi}{\theta B}\Big) + \Big(\frac{\Delta}{2\theta K} +\frac{\phi}{2\theta B}\Big)^2 },
	\end{align}
	which finishes the proof. 
\subsection*{Proof of Corollary~\ref{co:mainoffline}}
	Since $\alpha = \frac{1}{8NL}$, we have $(1+\alpha L)^{4N}< e^{0.5}<2$, and thus
	\begin{align}\label{afterpuck}
	A_{\text{\normalfont squ}_1}  &<  32, \; A_{\text{\normalfont squ}_2} < 8(\sigma +b)^2 + 4(\sigma^2+\widetilde b),\nonumber
	\\C_{\mathcal{L}} &< \Big(\frac{5\rho}{32NL} + \frac{\rho}{L}\frac{5}{16}\Big) \frac{5}{4} < \frac{5\rho}{8L}, \; C_{\mathcal{L}} > \frac{\rho}{L} \alpha L (N-1)>\frac{\rho}{16L},\nonumber
	\\C_b &< \frac{15}{32}\frac{\rho}{L}\frac{1}{4}< \frac{\rho}{8L},
	\end{align}
	which, in conjunction with \cref{offline:constants}, yields
	\begin{align}\label{fini:offpara}
	\theta \geq &  \frac{1}{80} \frac{4L}{5\rho} \Big( 1- \frac{33}{80}\Big) \geq \frac{L}{200\rho}, \; \phi \leq \frac{2(\sigma +b)^2 + (\sigma^2+\widetilde b)}{1600L},\; \xi \leq  \frac{24L^2}{\rho} + \frac{37b}{16}. 
	\end{align}
	Combining~\cref{fini:offpara} and \cref{iopnn} yields
	\begin{align*}
	\mathbb{E}\|\nabla \mathcal{L}(w_\zeta)\| \leq &\frac{\Delta}{2\theta K} +\frac{\phi}{2\theta B} + \sqrt{ \xi \Big(\frac{\Delta}{\theta K} +\frac{\phi}{\theta B}\Big) + \Big(\frac{\Delta}{2\theta K} +\frac{\phi}{2\theta B}\Big)^2 } \nonumber
	\\ \leq & \mathcal{O}\Big(  \frac{1}{K} +\frac{\sigma^2}{B} +\sqrt{\frac{1}{K} +\frac{\sigma^2}{B} }  \Big).
	\end{align*}
	Then, based on the parameter selection that $B\geq C_B\sigma^2\epsilon^{-2}$ and after at most $K=C_k\epsilon^{-2}$ iterations, we have 
	\begin{align*}
	\mathbb{E}\|\nabla \mathcal{L}(w_\zeta)\| \leq \mathcal{O}\Big(\big(\frac{1}{C_B}+\frac{1}{C_k}\big)\frac{1}{\epsilon^2} + \frac{1}{\epsilon}\sqrt{\big(\frac{1}{C_B}+\frac{1}{C_k}\big)}\Big).
	\end{align*}	
	Then, for $C_B,C_K$ large enough, we obtain from the above inequality that 
$	\mathbb{E}\|\nabla \mathcal{L}(w_\zeta)\| \leq \epsilon.$
	Thus, the total number of gradient computations is given by $B(T+NS)=\mathcal{O}(\epsilon^{-2}(T+NS)).$ Furthermore, the  total number of Hessian computations is given by $BNS =\mathcal{O}(NS\epsilon^{-2}) $
at each iteration.  
	Then, the proof is complete.

\section{Auxiliary Lemmas for MAML in Resampling Case}
\label{aux:lemma}

In this section, we derive some useful lemmas  to prove the propositions given in  \Cref{theory:online} on the properties of the meta gradient and the main results Theorem~\ref{th:mainonline} and Corollary~\ref{co:online}. 

The first lemma provides a bound on the difference between $\|\widetilde w_j^i - \widetilde u_j^i\|$ for $j=0,...,N,  i\in\mathcal{I}$, where $\widetilde w_j^i,\, j=0,...,N, i\in\mathcal{I}$ are given through the {\em gradient descent} updates in~\cref{gd_w} and  $\widetilde u_j^i,\, j=0,...,N$ are defined in the same way. 
\begin{lemma}\label{d_u_w}
	For any $i\in\mathcal{I}$, $j=0,...,N$ and $w,u \in \mathbb{R}^d$, we have 
	\begin{align*}
	\left\|\widetilde w_j^i -\widetilde  u_j^i\right\| \leq (1+\alpha L)^j \|w-u\|. 
	\end{align*}
\end{lemma}
\begin{proof}
	Based on the updates that $\widetilde w_m^i = \widetilde w_{m-1}^i - \alpha\nabla l_i(\widetilde w_{m-1}^i)$ and $\widetilde u_m^i = \widetilde u_{m-1}^i - \alpha\nabla l_i(\widetilde u_{m-1}^i)$, we obtain, for any $i \in\mathcal{I}$,
	\begin{align*}
	\|\widetilde w_m^i - \widetilde u_m^i\|  =& \|\widetilde w_{m-1}^i - \alpha\nabla l_i(\widetilde w_{m-1}^i) -\widetilde u_{m-1}^i + \alpha\nabla l_i(\widetilde u_{m-1}^i) \| \nonumber
	\\ \overset{(i)}\leq & \|\widetilde w_{m-1}^i -\widetilde u_{m-1}^i\| + \alpha L\|\widetilde w_{m-1}^i - \widetilde u_{m-1}^i\|  \nonumber
	\\ \leq & (1+\alpha L)  \|\widetilde w_{m-1}^i -\widetilde u_{m-1}^i\|,
	\end{align*} 
	where (i) follows from the triangle inequality. Telescoping the above inequality over $m$ from $1$ to $j$, we obtain 
	\begin{align*}
	\left\|\widetilde w_j^i - \widetilde u_j^i\right\| \leq (1+\alpha L)^j \|\widetilde w^i_0-\widetilde u^i_0\|, 
	\end{align*}
	which, in conjunction with the fact that $\widetilde w_0^i = w$ and $\widetilde u_0^i = u$, finishes the proof.
\end{proof}
The following lemma provides an upper bound on $\|\nabla l_i(\widetilde w_j^i)\|$ for all $i\in\mathcal{I}$ and $j=0,..., N$, where $\widetilde w_j^i$ is defined in the same way as in Lemma~\ref{d_u_w}. 
\begin{lemma}\label{le:jiw}
	For any $i\in\mathcal{I}$,  $j=0,...,N$ and $w \in \mathbb{R}^d$, we have 
	\begin{align*}
	\|\nabla l_i(\widetilde w_j^i)\| \leq (1+\alpha L)^j \|\nabla l_i(w)\|.
	\end{align*}
\end{lemma}
\begin{proof}
	For $m\geq1$, we have 
	\begin{align*}
	\|\nabla l_i(\widetilde w_m^i)\| = & \|\nabla l_i(\widetilde w_m^i) - \nabla l_i(\widetilde w_{m-1}^i) + \nabla l_i(\widetilde w_{m-1}^i)\|
	\\\leq & \|\nabla l_i(\widetilde w_m^i) - \nabla l_i(\widetilde w_{m-1}^i) \| +  \|\nabla l_i(\widetilde w_{m-1}^i)\|
	\\\leq & L\|\widetilde w_m^i - \widetilde w_{m-1}^i\| +\|\nabla l_i(\widetilde w_{m-1}^i)\|\leq (1+\alpha L)\|\nabla l_i(\widetilde w_{m-1}^i)\|,
	\end{align*}
	where the last inequality follows from the update $\widetilde w_m^i =  \widetilde w_{m-1}^i - \alpha \nabla l_i(\widetilde w_{m-1}^i)$. Then, telescoping the above inequality over $m$ from $1$ to $j$ yields
	\begin{align*}
	\|\nabla l_i(\widetilde w_j^i)\| \leq (1+\alpha L)^j \|\nabla l_i(\widetilde w_0^i)\|,
	\end{align*}
	which, combined with the fact that $\widetilde w_0^i = w$, finishes the proof.  
\end{proof}
The following lemma gives an upper bound on  the quantity $\big\|I - \prod_{j=0}^m(I - \alpha V_j)\big\|$ for all matrices  $V_j \in \mathbb{R}^{d\times d},j=0,...,m$ that satisfy $\|V_j\|\leq L$.
\begin{lemma}\label{le:prd}
	For all matrices $V_j \in \mathbb{R}^{d\times d}, j =0,..., m$ that satisfy $\|V_j\|\leq L$, we have 
	\begin{align*}
	\Big\|I - \prod_{j=0}^m(I - \alpha V_j)\Big\| \leq (1+\alpha L)^{m+1} - 1.
	\end{align*}
\end{lemma}
\begin{proof}
	First note that the product $ \prod_{j=0}^m(I - \alpha V_j) $ can be expanded as  
	\begin{align}
	\prod_{j=0}^m(I - \alpha V_j) = I - \sum_{j=0}^m \alpha V_j + \sum_{0\leq p < q \leq m} \alpha^2 V_p V_q+\cdots+(-1)^{m+1} \alpha^{m+1}\prod_{j=0}^m V_j.\nonumber
	\end{align}
	Then, by using $\|V_j\|\leq L$ for $j=0,...,m$, we have 
	\begin{align}
	\Big\|I - \prod_{j=0}^m(I - \alpha V_j)\Big\|  \leq & \Big\|\sum_{j=0}^m \alpha V_j \Big\| + \Big\|\sum_{0\leq p < q \leq m} \alpha^2 V_pV_q \Big\| + \cdots + \Big\| \alpha^{m+1}\prod_{j=0}^m V_j\Big\| \nonumber
	\\\leq & {\rm C}^1_{m+1} \alpha L + {\rm C}_{m+1}^2 (\alpha L)^2 + \cdots + {\rm C}_{m+1}^{m+1} (\alpha L)^{m+1}\nonumber
	\\ = & (1+\alpha L)^{m+1} - 1, \nonumber
	\end{align}
	where the notion $C_n^k$ denotes the number of $k$-element subsets of a set of size $n$. Then, the proof is complete. 
\end{proof}
Recall the gradient { $\nabla \mathcal{L}_i(w) = \prod_{j=0}^{N-1}(I-\alpha \nabla^2 l_i(\widetilde w^i_{j}))\nabla l_i( \widetilde w^i _{N})$}, where {  $ \widetilde w_{j}^i, i\in \mathcal{I}, j=0,..., N$} are given by the gradient descent steps in \cref{gd_w} and $\widetilde w_{0}^i = w$ for all tasks $i \in \mathcal{I}$.  
Next, we provide an upper bound on the difference { $\|\nabla l_i(w) - \nabla \mathcal{L}_i(w)\|$}.
\begin{lemma}\label{le:fF}
	For any $i \in\mathcal{I}$ and $w \in \mathbb{R}^d$, we have 
	\begin{align*}
	\|\nabla l_i(w) - \nabla \mathcal{L}_i(w)\| \leq C_l \|\nabla l_i(w)\|,
	\end{align*}
	where $C_l$ is a positive constant given by 
	\begin{align}\label{eq:cfn}
	C_l = (1+\alpha L)^{2N} - 1 > 0.
	\end{align}
\end{lemma}
\begin{proof}
	First note that $\widetilde w_N^i$ can be rewritten as $\widetilde w_N^i = w - \alpha \sum_{j=0}^{N-1} \nabla l_i\big(\widetilde w_j^i\big)$. Then, based on  the mean value theorem (MVT) for vector-valued functions~\cite{mcleod1965mean}, we have, there exist constants $r_t, t=1,...,d$ satisfying $\sum_{t=1}^d r_t =1$ and vectors $w_t^\prime\in\mathbb{R}^d, t=1,...,d$ such that     
	\begin{align}\label{mvts}
	\nabla l_i( \widetilde w^i _{N}) =& \nabla l_i\Big( w - \alpha \sum_{j=0}^{N-1} \nabla l_i\big(\widetilde w_j^i\big)\Big)= \nabla l_i(w) + \Big(\sum_{t=1}^dr_t\nabla^2 l_i (w_t^\prime)\Big) \Big(-\alpha \sum_{j=0}^{N-1} \nabla l_i\big(\widetilde w_j^i\big)\Big) \nonumber
	\\ = &  \Big(I- \alpha\sum_{t=1}^dr_t\nabla^2 l_i (w_t^\prime)\Big)\nabla l_i(w) - \alpha \sum_{t=1}^dr_t\nabla^2 l_i (w_t^\prime) \sum_{j=1}^{N-1} \nabla l_i\big(\widetilde w_j^i\big).
	\end{align}
	For simplicity, we define  $K(N):=  \prod_{j=0}^{N-1}(I-\alpha \nabla^2 l_i(\widetilde w^i_{j}))$. Then, using~\cref{mvts} yields 
	\begin{small}
	\begin{align}
	\|\nabla &l_i(w) - \nabla \mathcal{L}_i(w)\| = \|\nabla l_i(w) - K(N)\nabla l_i(\widetilde w_N^i)\| \nonumber
	\\ =& \Big\|\nabla l_i(w) - K(N)\Big(I- \alpha\sum_{t=1}^dr_t\nabla^2 l_i (w_t^\prime)\Big)\nabla l_i(w)  + \alpha K(N)\sum_{t=1}^dr_t\nabla^2 l_i (w_t^\prime)\sum_{j=1}^{N-1} \nabla l_i\big(\widetilde w_j^i\big)\Big\| \nonumber
	\\ \leq & \Big\|\Big(I - K(N)\Big(I- \alpha\sum_{t=1}^dr_t\nabla^2 l_i (w_t^\prime)\Big)\Big)\nabla l_i(w) \Big\| + \Big\| \alpha K(N)\sum_{t=1}^dr_t\nabla^2 l_i (w_t^\prime)\sum_{j=1}^{N-1} \nabla l_i\big(\widetilde w_j^i\big)\Big\| \nonumber
	\\ \overset{(i)}\leq & \Big\|\Big(I - K(N)\Big(I- \alpha\sum_{t=1}^dr_t\nabla^2 l_i (w_t^\prime)\Big)\Big)\nabla l_i(w) \Big\| + \alpha L (1+\alpha L)^N \sum_{j=1}^{N-1}  \Big\| \nabla l_i\big(\widetilde w_j^i\big)\Big\| \nonumber
	\\\overset{(ii)}\leq & \Big\|I - K(N)\Big(I- \alpha\sum_{t=1}^dr_t\nabla^2 l_i (w_t^\prime)\Big)\Big\|\|\nabla l_i(w)\| + \alpha L (1+\alpha L)^N \sum_{j=1}^{N-1}  (1+\alpha L)^j \|\nabla l_i(w)\| \nonumber
	\\\overset{(iii)}\leq & ((1+\alpha L)^{N+1}-1)\|\nabla l_i(w)\| + (1+\alpha L)^{N+1} ((1+\alpha L)^{N-1}-1)\|\nabla l_i(w)\|  \nonumber
	\\ = & ((1+\alpha L)^{2N} - 1)\|\nabla l_i(w)\|,\nonumber
	\end{align}
	\end{small}
\hspace{-0.12cm}	where (i) follows from the fact that $\|\nabla^2 l_i(u)\| \leq L$ for any $u\in \mathbb{R}^d$ and $\sum_{t=1}^d r_t =1$,  and the inequality that $\|\sum_{j=1} ^n a_j\|\leq \sum_{j=1} ^n\|a_j\|$,  (ii) follows from  Lemma~\ref{le:jiw}, and (iii) follows from Lemma~\ref{le:prd}.
\end{proof}

Recall that the expected value of the gradient of the loss $\nabla l(w):=\mathbb{E}_{i\sim p(\mathcal{T})} \nabla l_i(w)$ and the objective function $\nabla \mathcal{L}(w): = \nabla \mathcal{L}_i(w)$. Based on the above lemmas, we next provide an upper bound on $\|\nabla l(w)\|$ using $\|\nabla \mathcal{L}(w)\|$. 
\begin{lemma}\label{le:lL}
	For any $w\in\mathbb{R}^d$, we have 
	\begin{align*}
	\|\nabla l(w)\| \leq \frac{1}{1-C_l} \|\nabla \mathcal{L}(w)\| + \frac{C_l}{1-C_l} \sigma,
	\end{align*}
	where the constant $C_l$ is given by 
     \begin{align*}
	C_l = (1+\alpha L)^{2N} - 1.
	\end{align*}
\end{lemma}
\begin{proof}
	Based on the definition of $\nabla l(w)$, we have 
	\begin{align}
	\|\nabla l(w)\| =& \|\mathbb{E}_{i\sim p(\mathcal{T})} (\nabla l_i(w) -\nabla \mathcal{L}_i(w) +\nabla \mathcal{L}_i(w) )\| \nonumber
	\\\leq& \|\mathbb{E}_{i\sim p(\mathcal{T})} \nabla \mathcal{L}_i(w) \| + \|\mathbb{E}_{i\sim p(\mathcal{T})} (\nabla l_i(w) -\nabla \mathcal{L}_i(w)  ) \|   \nonumber
	\\\leq & \|\nabla \mathcal{L}(w) \| + \mathbb{E}_{i\sim p(\mathcal{T})} \|\nabla l_i(w) -\nabla \mathcal{L}_i(w)   \|   \nonumber
	\\\overset{(i)}\leq & \|\nabla \mathcal{L}(w) \| + C_l\mathbb{E}_{i\sim p(\mathcal{T})}\|\nabla l_i(w)   \|  \nonumber
	\\ \overset{(ii)} \leq & \|\nabla \mathcal{L}(w) \| +  C_l(\|\nabla l(w)   \| + \sigma), \nonumber
	\end{align}
	where (i) follows from Lemma~\ref{le:fF}, and (ii) follows from Assumption~\ref{a2}. Then, rearranging the above inequality completes the proof.  
\end{proof}
Recall from~\cref{hatlw} that we choose the meta stepsize $\beta_k = \frac{1}{C_\beta \widehat L_{w_k}} $, where $C_\beta$ is a positive constant and { $ \widehat L_{w_k} = (1+\alpha L)^{2N}L + C_\mathcal{L} \frac{1}{|B_k^\prime|}\sum_{i\in B_k^\prime}\|\nabla l_i(w_k; D_{L_k}^i)\|$}. Using an approach similar to Lemma 4.11 in~\cite{fallah2020convergence}, we establish the following lemma to provide the first- and second-moment bounds for $\beta_k$.
%
\begin{lemma}\label{le:xiaodege} 
Suppose that  Assumptions~\ref{assum:smooth},~\ref{a2} and~\ref{a3} hold. 
	Set the meta stepsize $\beta_k = \frac{1}{C_\beta \widehat L_{w_k}} $ with  $\widehat L_{w_k}$  given by~\cref{hatlw}, where  $|B_k^\prime| > \frac{4C^2_{\mathcal{L}}\sigma^2}{3(1+\alpha L)^{4N}L^2}$ and $|D_{L_k}^i| > \frac{64\sigma^2_g C_\mathcal{L}^2}{(1+\alpha L)^{4N}L^2}$ for all $i \in B_k^\prime$. Then, conditioning on $w_k$, we have 
\begin{align*}
\mathbb{E} \beta_k  \geq  \frac{4}{C_\beta} \frac{1}{5 L_{w_k}},\quad \mathbb{E}\beta^2_k  \leq  \frac{4}{C_\beta^2} \frac{1}{L^2_{w_k}}, 
\end{align*}
where $L_{w_k} = (1+\alpha L)^{2N}L + C_\mathcal{L} \mathbb{E}_{i\sim p(\mathcal{T})}\|\nabla l_i(w_k)\|$ with  $C_\mathcal{L}$ given in~\cref{clcl}. 
\end{lemma}

\begin{proof}
	Let $\widetilde L_{w_k} = 4L +\frac{4C_\mathcal{L}}{(1+\alpha L)^{2N}} \frac{1}{|B_k^\prime|}\sum_{i\in B_k^\prime}\|\nabla l_i(w_k;  D_{L_k}^i)\|$. Note that 
	  $|B_k^\prime| > \frac{4C^2_{\mathcal{L}}\sigma^2}{3(1+\alpha L)^{4N}L^2}$ and $|D_{L_k}^i| > \frac{64\sigma^2_g C_\mathcal{L}^2}{(1+\alpha L)^{4N}L^2}, \,i \in B_k^\prime$. Then, using an approach similar to (61) in \cite{fallah2020convergence} and conditioning on $w_k$, we have 
	\begin{align}\label{1toL}
	\mathbb{E} \Big(  \frac{1}{\widetilde L^2_{w_k} }\Big) \leq \frac{\sigma_\beta^2/(4L)^2 + \mu_\beta^2/(\mu_\beta)^2}{\sigma_\beta^2 + \mu_\beta^2},
	\end{align}
	where $\sigma^2_\beta$ and $\mu_\beta$ are the variance and mean of  $\frac{4C_\mathcal{L}}{(1+\alpha L)^{2N}} \frac{1}{|B_k^\prime|}\sum_{i\in B_k^\prime}\|\nabla l_i(w_k; D_{L_k}^i)\|$. Using an approach similar to (62) in \cite{fallah2020convergence}, conditioning on $w_k$ and using $|D_{L_k}^i| > \frac{64\sigma^2_g C_\mathcal{L}^2}{(1+\alpha L)^{4N}L^2}$, we have 
	\begin{align}\label{lluu}
	\frac{C_\mathcal{L}}{(1+\alpha L)^{2N}} \mathbb{E}_i\|\nabla l_i(w_k)\| - L \leq \mu_\beta  \leq \frac{C_\mathcal{L}}{(1+\alpha L)^{2N}} \mathbb{E}_i\|\nabla l_i(w_k)\| + L,
	\end{align}
	which implies that $\mu_\beta +5L \geq \frac{4}{(1+\alpha L)^{2N}}L_{w_k}$, and thus using \cref{1toL} yields
	\begin{align}\label{lopops}
	\frac{16}{(1+\alpha L)^{4N}}L^2_{w_k}	\mathbb{E} \Big(  \frac{1}{\widetilde L^2_{w_k} }\Big) \leq \frac{\mu_\beta^2(25/16+ \sigma_\beta^2/(8L^2))+25\sigma_\beta^2/8}{\sigma_\beta^2 + \mu_\beta^2 }.
	\end{align}
	Furthermore, conditioning on $w_k$,  $\sigma_\beta$ is bounded by 
	\begin{align}\label{signbeta}
	\sigma_\beta^2 =& \frac{16 C^2_\mathcal{L}}{(1+\alpha L)^{4N}|B^\prime_k|} \text{Var} (\|\nabla l_i(w_k; D_{L_k}^i)\|) \nonumber
	\\ \leq & \frac{16 C^2_\mathcal{L}}{(1+\alpha L)^{4N}|B^\prime_k|} \Big(\sigma^2 + \frac{\sigma_g^2}{|D_{L_k}^i|}\Big) \nonumber
	\\\overset{(i)}\leq & \frac{16 C^2_\mathcal{L}\sigma^2}{(1+\alpha L)^{4N}|B^\prime_k|}  + \frac{L^2}{4|B_k^\prime|} \overset{(ii)}\leq  12L^2 + \frac{1}{4}L^2 < \frac{25}{2} L^2,
	\end{align} 
	where (i) follows from $|D_{L_k}^i| > \frac{64\sigma^2_g C_\mathcal{L}^2}{(1+\alpha L)^{4N}L^2}, \,i \in B_k^\prime$ and (ii) follows from $|B_k^\prime| > \frac{4C^2_{\mathcal{L}}\sigma^2}{3(1+\alpha L)^{4N}L^2}$ and $|B_k^\prime| \geq 1$. Then, plugging \cref{signbeta} in~\cref{lopops}, we then have 
$	\frac{16}{(1+\alpha L)^{4N}}L^2_{w_k}	\mathbb{E} \Big(  \frac{1}{\widetilde L^2_{w_k} }\Big) \leq \frac{25}{8}.$
	Then, noting that $\beta_k  = \frac{4}{C_\beta (1+\alpha L)^{2N} \widetilde L_{w_k}}$, using the above inequality and conditioning on $w_k$,  we have 
	\begin{align}\label{secondmm}
	\mathbb{E}\beta^2_k =  \frac{16}{C^2_\beta (1+\alpha L)^{4N}}\mathbb{E}  \left(\frac{1}{ \widetilde L^2_{w_k}} \right) \leq \frac{25}{8C_\beta^2} \frac{1}{L^2_{w_k}} < \frac{4}{C_\beta^2} \frac{1}{L^2_{w_k}}.
	\end{align}
	In addition, by Jensen's inequality and conditioning on $w_k$, we have 
	\begin{align}\label{onemo}
	\mathbb{E} \beta_k =&\frac{4}{C_\beta (1+\alpha L)^{2N} } \mathbb{E}\Big(\frac{1}{ \widetilde L_{w_k}}\Big) \geq  \frac{4}{C_\beta (1+\alpha L)^{2N} } \frac{1}{ \mathbb{E}\widetilde L_{w_k}} =  \frac{4}{C_\beta (1+\alpha L)^{2N} } \frac{1}{4L + \mu_\beta} \nonumber
	\\\overset{(i)}\geq &   \frac{4}{C_\beta  } \frac{1}{4L(1+\alpha L)^{2N} +L_{w_k}}  \overset{(ii)}\geq \frac{4}{C_\beta} \frac{1}{5 L_{w_k}}, 
	\end{align}
	where (i) follows from \cref{lluu} and (ii) follows from the fact $L_{w_k} >  (1+\alpha L)^{2N}L$.
\end{proof}

\section{Auxiliary Lemmas for MAML in Finite-Sum Case}\label{aux:lemma_finite}
In this section, we provide some useful lemmas to prove the propositions in  \Cref{theory:offline} on properties of the meta gradient and the main results Theorem~\ref{mainth:offline} and Corollary~\ref{co:mainoffline}.

The following lemma provides an upper bound on $\|l_{S_i}(\widetilde w^i_j)\|$  for all $i\in\mathcal{I}$ and $j=0,..., N$, where $\widetilde w_j^i$ is defined by~\cref{innerfinite} with $\widetilde w_0^i=w$. 
\begin{lemma}\label{finite:gbd}
	For any $i\in\mathcal{I}$,  $j=0,...,N$ and $w \in \mathbb{R}^d$, we have 
	\begin{align*}
	\|\nabla l_{S_i}(\widetilde w_j^i)\| \leq (1+\alpha L)^j \|\nabla l_{S_i}(w)\|.
	\end{align*}
\end{lemma}
\begin{proof}
	The proof is similar to that of Lemma~\ref{le:jiw}, and thus omitted.  
\end{proof}
We next provide a bound on  $\|\nabla l_{T_i}(w) - \nabla \mathcal{L}_i(w) \|$, where $$\nabla \mathcal{L}_i(w) = \prod_{j=0}^{N-1}(I - \alpha \nabla^2 l_{S_i}(w_{j}^i))\nabla l_{T_i}(w_{N}^i).$$
\begin{lemma}\label{tiis}
	For any $i \in\mathcal{I}$ and $w \in \mathbb{R}^d$, we have 
	\begin{align*}
	\|\nabla l_{T_i}(w) - \nabla \mathcal{L}_i(w)\| \leq \big( (1+&\alpha L)^N -1  \big)\|\nabla l_{T_i}(w)\| 
	\\&+ (1+\alpha L)^N  \big( (1+\alpha L)^N -1  \big) \|\nabla l_{S_i}(w)\|.
	\end{align*}
\end{lemma}
\begin{proof}
	Using the mean value theorem (MVT), we have, there exist constants $r_t, t=1,...,d$ satisfying $\sum_{t=1}^d r_t =1$ and vectors $w_t^\prime\in\mathbb{R}^d, t=1,...,d$ such that     

	\begin{align*}
	\nabla l_{T_i}(\widetilde w_N^i) = & \nabla  l_{T_i} \big(w - \alpha \sum_{j=0}^{N-1} \nabla l_{S_i}(\widetilde w_j^i) \big) = \nabla l_{T_i}(w) + \sum_{t=1}^dr_t \nabla^2  l_{T_i} (w_t^\prime)\big(  - \alpha \sum_{j=0}^{N-1} \nabla l_{S_i}(\widetilde w_j^i)      \big) \nonumber
	\\ = &\nabla l_{T_i}(w) - \alpha\sum_{t=1}^dr_t \nabla^2  l_{T_i} (w_t^\prime)\sum_{j=0}^{N-1} \nabla l_{S_i}(\widetilde w_j^i).    
	\end{align*}
	Based on the above equality, we have
	\begin{small}
	\begin{align}
	\|\nabla &l_{T_i}(w) - \nabla \mathcal{L}_i(w)\| \nonumber
	\\=& \Big\|\nabla l_{T_i}(w) - \prod_{j=0}^{N-1}(I - \alpha \nabla^2 l_{S_i}(\widetilde w_{j}^i))\nabla l_{T_i}(\widetilde w_{N}^i)\Big\| \nonumber
	\\ =& \Big\|I - \prod_{j=0}^{N-1}(I - \alpha \nabla^2 l_{S_i}(\widetilde w_{j}^i)) \Big\|\|\nabla l_{T_i}(w)\|\nonumber
	\\&\hspace{1.5cm}+ \Big\| \prod_{j=0}^{N-1}(I - \alpha \nabla^2 l_{S_i}(\widetilde w_{j}^i)) \alpha\sum_{t=1}^dr_t \nabla^2  l_{T_i} (w_t^\prime) \sum_{j=0}^{N-1} \nabla l_{S_i}(\widetilde w_j^i)\Big\| \nonumber 
	\\ \overset{(i)}\leq & \big((1+\alpha L)^N -1\big) \|\nabla l_{T_i}(w)\| + \alpha L(1+\alpha L)^N\sum_{j=0}^{N-1} \|\nabla l_{S_i}(\widetilde w_j^i)\| \nonumber
	\\\overset{(ii)} \leq & \big((1+\alpha L)^N -1\big) \|\nabla l_{T_i}(w)\| + \alpha L(1+\alpha L)^N\sum_{j=0}^{N-1} (1+\alpha L)^j\|\nabla l_{S_i}(w)\| \nonumber
	\\ = & \big((1+\alpha L)^N -1\big) \|\nabla l_{T_i}(w)\| + (1+\alpha L)^N \big((1+\alpha L)^N-1\big)\|\nabla l_{S_i}(w)\|, \nonumber
	\end{align}
	\end{small}
	\hspace{-0.12cm}where (i) follows from Lemma~\ref{le:prd} and $\|\sum_{t=1}^dr_t \nabla^2  l_{T_i} (w_t^\prime)\|\leq \sum_{t=1}^dr_t\| \nabla^2  l_{T_i} (w_t^\prime)\|\leq L$, and (ii) follows from Lemma~\ref{finite:gbd}. Then, the proof is complete. 
\end{proof}	
Recall that $\nabla l_T(w) = \mathbb{E}_{i\sim p(\mathcal{T})} \nabla l_{T_i}(w)$, $\nabla \mathcal{L}(w) =  \mathbb{E}_{i\sim p(\mathcal{T})} \nabla \mathcal{L}_i(w)$ and $b = \mathbb{E}_{i\sim p(\mathcal{T})} [b_i]$. 
The following lemma provides an upper bound on $\|\nabla l_T(w)\|$.	
\begin{lemma}\label{twc1c2}
	For any $i \in\mathcal{I}$ and $w \in \mathbb{R}^d$, we have 
	\begin{align}
	\|\nabla l_T(w)\| \leq \frac{1}{C_1} \|\nabla \mathcal{L}(w)\| + \frac{C_2}{C_1},
	\end{align}
	where constants $C_1, C_2>0$ are given by 
	\begin{align}\label{c1c2}
	C_1 =& 2-(1+\alpha L)^{2N}, \nonumber
	\\C_2 =& \big( (1+\alpha L)^{2N}-1  \big)\sigma + (1+\alpha L)^N \big((1+\alpha L)^N -1 \big) b.
	\end{align}
\end{lemma}	
\begin{proof}
	First note that 
	\begin{align*}
	\|\nabla l_T(w)\| 
	=& \|\mathbb{E}_i (\nabla l_{T_i}(w) -\nabla \mathcal{L}_i(w) ) +\nabla \mathcal{L}(w)\|  \nonumber
	\\ \leq & \|\nabla \mathcal{L}(w)\| + \mathbb{E}_i \| \nabla l_{T_i}(w) -\nabla \mathcal{L}_i(w) \| \nonumber
	\\ \overset{(i)}\leq &  \|\nabla \mathcal{L}(w)\| + \mathbb{E}_i  ( ( (1+\alpha L)^N -1  )\|\nabla l_{T_i}(w)\| \nonumber
	\\&+ (1+\alpha L)^N  ( (1+\alpha L)^N -1  ) \|\nabla l_{S_i}(w)\|  ) \nonumber
	\\ \overset{(ii)}\leq &  \|\nabla \mathcal{L}(w)\| +   \big( (1+\alpha L)^N -1  \big)\big( \|\nabla l_{T}(w)\| +\sigma \big) 
	\\&+ (1+\alpha L)^N  \big( (1+\alpha L)^N -1  \big) (\mathbb{E}_i \|\nabla l_{T_i}(w)\| + \mathbb{E}_ib_i) \nonumber
	\\ \leq & \|\nabla \mathcal{L}(w)\| +   \big( (1+\alpha L)^N -1 + (1+\alpha L)^N((1+\alpha L)^N-1) \big) \|\nabla l_{T}(w)\|  \nonumber
	\\ &+ ((1+\alpha L)^N -1)\sigma + (1+\alpha L)^N((1+\alpha L)^N -1)(\sigma + b) \nonumber
	\\ \leq & \|\nabla \mathcal{L}(w)\| +   \big( (1+\alpha L)^{2N} -1 \big) \|\nabla l_{T}(w)\|  
	\\&+ ((1+\alpha L)^{2N} -1)\sigma + (1+\alpha L)^N((1+\alpha L)^N -1)b
	\end{align*}
	where (i) follows from Lemma~\ref{tiis}, (ii) follows from Assumption~\ref{assum:vaoff}. Based on the definitions of $C_1$ and $C_2$ in~\cref{c1c2}, the proof is complete. 
\end{proof}

The following lemma provides the first- and second-moment bounds on $1/\hat L_{w_k}$, where  
$\hat L_{w_k} =(1+\alpha L)^{2N}L + C_b b +  C_\mathcal{L}\frac{\sum_{i \in B_k^\prime}\|\nabla l_{T_i}(w_k)\|}{|B_k^\prime|}$.
\begin{lemma}\label{le:betak}
	If the batch size $|B_k^\prime| \geq \frac{2C^2_\mathcal{L}\sigma^2}{( C_b b + (1+\alpha L)^{2N} L)^2}$, conditioning on $w_k$, we have
	\begin{align*}
	\mathbb{E} \Big( \frac{1}{\hat L_{w_k}}  \Big) \geq \frac{1}{L_{w_k}}, \quad \mathbb{E} \Big( \frac{1}{\hat L^2_{w_k}}  \Big) \leq \frac{2}{L^2_{w_k}} 
	\end{align*}
	where $L_{w_k}$ is given by $$L_{w_k}= (1+\alpha L)^{2N}L + C_b b +  C_\mathcal{L} \mathbb{E}_{i\sim p(\mathcal{T})}\|\nabla l_{T_i}(w_k)\|.$$
\end{lemma}
\begin{proof}
	Conditioning on $w_k$ and using an approach similar to \cref{1toL}, we have 	
	\begin{align}
	\mathbb{E} \Big( \frac{1}{\hat L^2_{w_k}}  \Big) \leq \frac{\sigma_\beta^2 / \big(  C_b b + (1+\alpha L)^{2N} L  \big)^2 + \mu^2_\beta / (\mu_\beta + C_b b + (1+\alpha L)^{2N} L)^2}{\sigma_\beta^2 + \mu_\beta^2},
	\end{align}
	where $\mu_\beta$ and $\sigma^2_\beta$ are the  mean and variance of variable $\frac{C_\mathcal{L}}{|B_k^\prime|}\sum_{i \in B_k^\prime}\|\nabla l_{T_i}(w_k)\|$. Noting that $\mu_\beta = C_\mathcal{L} \mathbb{E}_{i\sim p(\mathcal{T})}\|\nabla l_{T_i}(w_k)\|$, we have $L_{w_k} = (1+\alpha L)^{2N}L + C_b b +  \mu_\beta$, and thus 
	\begin{align}\label{halfgo}
	L^2_{w_k}\mathbb{E} \Big( \frac{1}{\hat L^2_{w_k}}  \Big) \leq \frac{\sigma_\beta^2\frac{((1+\alpha L)^{2N}L + C_b b +  \mu_\beta)^2}{ \big(  C_b b + (1+\alpha L)^{2N} L  \big)^2 }+ \mu^2_\beta }{\sigma_\beta^2 + \mu_\beta^2} \leq \frac{2\sigma_\beta^2+\mu_\beta^2 + \frac{2\sigma_\beta^2\mu_\beta^2}{ \big(  C_b b + (1+\alpha L)^{2N} L  \big)^2}}{\sigma_\beta^2 + \mu_\beta^2},
	\end{align}
	where the last inequality follows from $(a+b)^2\leq 2a^2+2b^2$.
	Note that, conditioning on $w_k$, 
	\begin{align*}
	\sigma_\beta^2 = \frac{C^2_\mathcal{L} }{|B_k^\prime|}    \text{Var} \|\nabla l_{T_i}(w_k)\| \leq  \frac{C^2_\mathcal{L} }{|B_k^\prime|}   \sigma^2,
	\end{align*}
	which, in conjunction with $|B_k^\prime| \geq \frac{2C^2_\mathcal{L}\sigma^2}{( C_b b + (1+\alpha L)^{2N} L)^2}$, yields
	\begin{align}\label{sigbbs}
	\frac{2\sigma_\beta^2}{ \big(  C_b b + (1+\alpha L)^{2N} L  \big)^2} \leq 1.
	\end{align}
	Combining~\cref{sigbbs} and~\cref{halfgo} yields $$\mathbb{E} \Big( \frac{1}{\hat L^2_{w_k}}  \Big) \leq \frac{2}{L^2_{w_k}}. $$ 
	In addition, conditioning on $w_k$, we have
	\begin{align}
	\mathbb{E} \Big( \frac{1}{\hat L_{w_k}}  \Big)  \overset{(i)}\geq\frac{1}{	\mathbb{E}  \hat L_{w_k}}  = \frac{1}{L_{w_k}},
	\end{align}
	where (i) follows from  Jensen's inequality. Then, the proof is complete. 
\end{proof}

\chapter{Experimental Details and Proof of \Cref{chp:anil}}\label{appendix:anil}

\section{Further Specification of Experiments}\label{appen:exp}
Following~\cite{learn2learn2019}, we consider a 5-way 5-shot task on both the FC100 and miniImageNet datasets, where we evaluate the model's ability to discriminate $5$ unseen classes, given only $5$ labelled samples per class. We adopt Adam~\cite{kingma2014adam} as the optimizer for the meta outer-loop update, and adopt the cross-entropy loss to measure the error between the predicted and true labels.

\subsection*{Introduction of FC100 and miniImageNet datasets}

{\bf FC100 dataset.} The FC100 dataset~\cite{oreshkin2018tadam} is generated from CIFAR-100~\cite{krizhevsky2009learning}, and consists of $100$ classes with each class containing $600$ images of size $32\time 32$. Following recent works~\cite{oreshkin2018tadam,lee2019meta}, we split these $100$ classes into $60$ classes for meta-training, $20$ classes for meta-validation, and $20$ classes for meta-testing.  

\vspace{0.2cm}
{\noindent\bf miniImageNet dataset.} The miniImageNet dataset~\cite{vinyals2016matching} consists of $100$ classes randomly chosen from ImageNet~\cite{russakovsky2015imagenet}, where each class contains $600$ images of size $84\times 84$. Following the repository~\cite{learn2learn2019}, we partition these classes into $64$ classes for meta-training, $16$ classes for meta-validation, and $20$ classes for meta-testing.

\subsection*{Model Architectures and Hyper-Parameter Setting}
We adopt the following four model architectures depending on the dataset and the geometry of the inner-loop loss. The hyper-parameter configuration for each architecture is also provided as follows. 

\vspace{0.1cm}
{\noindent\bf Case 1: FC100 dataset, strongly-convex inner-loop loss.} Following~\cite{learn2learn2019}, we use a $4$-layer CNN of four convolutional blocks, where each block sequentially consists of  a $3\times 3$ convolution with a padding of $1$ and a stride of $2$, batch normalization, ReLU activation, and $2\times 2$
max pooling. Each convolutional layer has $64$ filters. This model is trained with an inner-loop stepsize of $0.005$, an outer-loop (meta) stepsize of $0.001$, and a mini-batch size of $B=32$. We set the regularization parameter $\lambda$ of the $L^2$ regularizer  to be $\lambda =5$.

\vspace{0.1cm}
{\noindent \bf Case 2: FC100 dataset, nonconvex inner-loop loss.} We adopt a $5$-layer CNN 
 with the first four convolutional layers the same as in {\bf Case 1}, followed by ReLU activation, and a full-connected layer with size of $256\times \text{ways}$.  This model is trained with an inner-loop stepsize of $0.04$, an outer-loop (meta) stepsize of $0.003$, and a mini-batch size of $B=32$. 
 
 \vspace{0.1cm}
 {\noindent\bf Case 3: miniImageNet  dataset, strongly-convex inner-loop loss.} Following~\cite{raghu2019rapid}, we use a $4$-layer CNN of four convolutional blocks, where each block sequentially consists of  a $3\times 3$ convolution with $32$ filters, batch normalization, ReLU activation, and $2\times 2$
max pooling. We choose an inner-loop stepsize of $0.002$, an outer-loop (meta) stepsize of $0.002$, and a mini-batch size of $B=32$, and set the regularization parameter $\lambda$ of the $L^2$ regularizer  to be $\lambda =0.1$.
 
 \vspace{0.1cm}
 {\noindent\bf Case 4: miniImageNet dataset, nonconvex inner-loop loss.} We adopt a $5$-layer CNN 
 with the first four convolutional layers the same as in {\bf Case 3}, followed by ReLU activation, and a full-connected layer with size of $128\times \text{ways}$. We choose an inner-loop stepsize of $0.02$, an outer-loop (meta) stepsize of $0.003$, and a mini-batch size of $B=32$.

 
\section{Proof of~\Cref{le:gd_form} }
We first prove the form of the partial gradient $\frac{\partial L_{\mathcal{D}_i}( w^i_{k,N}, \phi_k)}{\partial w_k}$. Using the chain rule, we have 
\begin{align}\label{eq:form1}
\frac{\partial L_{\mathcal{D}_i}( w^i_{k,N}, \phi_k)}{\partial w_k} &= \frac{\partial w_{k,N}^i(w_k,\phi_k)}{\partial w_k} \nabla_{w} L_{\mathcal{D}_i} (w_{k,N}^i,\phi_k) + \frac{\partial \phi_k}{\partial w_k} \nabla_{\phi} L_{\mathcal{D}_i} (w_{k,N}^i,\phi_k) \nonumber
\\& = \frac{\partial w_{k,N}^i(w_k,\phi_k)}{\partial w_k} \nabla_{w} L_{\mathcal{D}_i} (w_{k,N}^i,\phi_k), 
\end{align}
where the last equality follows from the fact that $\frac{\partial \phi_k}{\partial w_k}  = 0$. Recall that the gradient updates in~\Cref{alg:anil} are given by 
\begin{align}\label{eq:gd_d2}
w_{k,m+1}^i = w_{k,m}^i - \alpha \nabla_{w} L_{\mathcal{S}_i} (w_{k,m}^i,\phi_k),\, m=0,1,...,N-1,
\end{align}
where $w_{k,0}^i=w_k $ for all $i$. 
Taking derivatives w.r.t. $w_k$ in~\cref{eq:gd_d2} yields
\begin{align}\label{eq:gd_up}
\frac{\partial w_{k,m+1}^i}{\partial w_k} = &\frac{\partial w_{k,m}^i}{\partial w_k} - \alpha \frac{\partial w_{k,m}^i}{\partial w_k}\nabla^2_{w} L_{\mathcal{S}_i}(w_{k,m}^i,\phi_k)-  \underbrace{\alpha \frac{\partial \phi_k}{\partial w_k}\nabla_\phi\nabla_{w} L_{\mathcal{S}_i}(w_{k,m}^i,\phi_k)}_{0}. 
\end{align}
Telescoping \cref{eq:gd_up} over $m$ from $0$ to $N-1$ yields
\begin{align*}
\frac{\partial w_{k,N}^i}{\partial w_k} = \prod_{m=0}^{N-1}(I - \alpha \nabla_w^2L_{\mathcal{S}_i}(w_{k,m}^i,\phi_k)),
\end{align*}
which, in conjunction~\cref{eq:form1}, yields the first part in~\Cref{le:gd_form}. 

For the second part, using chain rule, we have
\begin{align}\label{eq:sec}
\frac{\partial L_{\mathcal{D}_i}( w^i_{k,N}, \phi_k)}{\partial \phi_k} = \frac{\partial w^i_{k,N}}{\partial \phi_k}\nabla_w L_{\mathcal{D}_i}( w^i_{k,N}, \phi_k) + \nabla_\phi L_{\mathcal{D}_i}( w^i_{k,N}, \phi_k).
\end{align}
Taking derivates w.r.t. $\phi_k$ in \cref{eq:gd_d2} yields 
\begin{align*}
\frac{\partial w_{k,m+1}^i}{\partial \phi_k} = &\frac{\partial w_{k,m}^i}{\partial \phi_k} -\alpha \Big( \frac{\partial w_{k,m}^i}{\partial \phi_k}\nabla^2_w  L_{\mathcal{S}_i}(w_{k,m}^i,\phi_k) +\nabla_\phi \nabla_w L_{\mathcal{S}_i} (w_{k,m}^i,\phi_k)\Big) \nonumber
\\= &\frac{\partial w_{k,m}^i}{\partial \phi_k} (I - \alpha \nabla^2_w  L_{\mathcal{S}_i}(w_{k,m}^i,\phi_k)) - \alpha \nabla_\phi \nabla_w L_{\mathcal{S}_i} (w_{k,m}^i,\phi_k).
\end{align*}
Telescoping the above equality over $m$ from $0$ to $N-1$ yields
\begin{align*}
\frac{\partial w_{k,N}^i}{\partial \phi_k} = \frac{\partial w_{k,0}^i}{\partial \phi_k} &\prod_{m=0}^{N-1} (I - \alpha \nabla^2_w  L_{\mathcal{S}_i}(w_{k,m}^i,\phi_k) )  \nonumber
\\&- \alpha\sum_{m=0}^{N-1}\nabla_\phi\nabla_w L_{\mathcal{S}_i}(w_{k,m}^i,\phi_k)\prod_{j=m+1}^{N-1} (I-\alpha \nabla^2_{w} L_{\mathcal{S}_i}(w_{k,j}^i,\phi_k)),
\end{align*}
which, in conjunction with the fact that $\frac{\partial w_{k,0}^i}{\partial \phi_k}=\frac{\partial w_k}{\partial \phi_k}=0$ and~\cref{eq:sec}, yields the second part.

\section{Proof for Strongly-Convex Inner Loop}\label{append:str}
\subsection*{Auxiliary Lemma}
The following lemma characterizes a bound on the difference between $w_{t}^{i}(w_1,\phi_1)$ and $w_{t}^{i}(w_2,\phi_2)$, where $w_{t}^{i}(w,\phi)$ corresponds to the $t^{th}$ inner-loop iteration starting from the initialization point $(w,\phi)$.
\begin{lemma}\label{le:support}
Choose $\alpha$ such that $1-2\alpha\mu+\alpha^2L^2>0$. Then, 
for any two points $(w_1,\phi_1),(w_2,\phi_2)\in\mathbb{R}^n$, we have 
\begin{align*}
\big\|w_{t}^{i}(w_1,\phi_1)- w_{t}^{i}(w_2,\phi_2)\big\| \leq (1-2\alpha\mu+\alpha^2L^2)^{\frac{t}{2}}\|w_1-w_2\| + \frac{\alpha L\|\phi_1-\phi_2\|}{1-\sqrt{1-2\alpha\mu+\alpha^2L^2}}.
\end{align*}
\end{lemma}
\begin{proof}
Based on the updates in~\cref{inner:gd}, we have 
\begin{align}
w_{m+1}^i(w_1,\phi_1)-&w_{m+1}^i(w_2,\phi_2) = w_{m}^i(w_1,\phi_1)-w_{m}^i(w_2,\phi_2)  \nonumber
\\&-\alpha \big(\nabla_w L_{\mathcal{S}_i}(w_m^i(w_1,\phi_1),\phi_1)-\nabla_w L_{\mathcal{S}_i}(w_m^i(w_2,\phi_2),\phi_1)\big)  \nonumber
\\&+ \alpha \big(\nabla_w L_{\mathcal{S}_i}(w_m^i(w_2,\phi_2),\phi_2)-\nabla_w L_{\mathcal{S}_i}(w_m^i(w_2,\phi_2),\phi_1)\big),\nonumber
\end{align}
which, together with the triangle inequality and~\Cref{assm:smooth}, yields
\begin{small}
\begin{align}\label{eq:tri}
\|&w_{m+1}^i(w_1,\phi_1)-w_{m+1}^i(w_2,\phi_2) \|
\nonumber
\\&\leq \underbrace{\Big\|w_{m}^i(w_1,\phi_1)-w_{m}^i(w_2,\phi_2) -\alpha \big(\nabla_w L_{\mathcal{S}_i}(w_m^i(w_1,\phi_1),\phi_1)-\nabla_w L_{\mathcal{S}_i}(w_m^i(w_2,\phi_2),\phi_1)\big)\Big\|}_{P}  \nonumber
\\&\;\;\;\;+ \alpha L\|\phi_1-\phi_2\|.
\end{align} 
\end{small}
\hspace{-0.12cm}Our next step is to upper-bound the term $P$ in~\cref{eq:tri}. Note that 
\begin{small}
\begin{align}\label{eq:p}
P^2 =& \|w_{m}^i(w_1,\phi_1)-w_{m}^i(w_2,\phi_2) \|^2 + \alpha^2\|\nabla_w L_{\mathcal{S}_i}(w_m^i(w_1,\phi_1),\phi_1)-\nabla_w L_{\mathcal{S}_i}(w_m^i(w_2,\phi_2),\phi_1)\|^2 \nonumber
\\&-2\alpha \Big\langle w_{m}^i(w_1,\phi_1)-w_{m}^i(w_2,\phi_2),  \nabla_w L_{\mathcal{S}_i}(w_m^i(w_1,\phi_1),\phi_1)-\nabla_w L_{\mathcal{S}_i}(w_m^i(w_2,\phi_2),\phi_1)  \Big\rangle \nonumber
\\\leq& (1+\alpha^2L^2-2\alpha\mu) \|w_{m}^i(w_1,\phi_1)-w_{m}^i(w_2,\phi_2) \|^2, 
\end{align}
\end{small}
\hspace{-0.12cm} where the last inequality follows from the strong-convexity of the loss function $L_{\mathcal{S}_i}(\cdot,\phi)$ that for any $w,w^\prime$ and $\phi$,
\begin{align*}
\langle w - w^\prime, \nabla_w L_{\mathcal{S}_i}(w,\phi) - \nabla_w L_{\mathcal{S}_i}(w^\prime,\phi)\rangle \geq  \mu \|w-w^\prime\|^2.
\end{align*}
Substituting~\cref{eq:p} into~\cref{eq:tri} yields  
\begin{align}
\|w_{m+1}^i(w_1,\phi_1)-w_{m+1}^i(w_2,\phi_2) \|\leq& \sqrt{1+\alpha^2L^2-2\alpha\mu}\|w_{m} ^i(w_1,\phi_1)-w_{m}^i(w_2,\phi_2) \| \nonumber
\\&+ \alpha L\|\phi_1-\phi_2\|.
\end{align}
Telescoping the above inequality over $m$ from $0$ to $t-1$ completes the proof.
\end{proof}

\subsection*{Proof of~\Cref{le:strong-convex} }
Using an approach similar to the proof of~\Cref{le:gd_form}, we have
\begin{align}\label{eq:ainiyo}
\frac{\partial L_{\mathcal{D}_i}( w^i_{N}, \phi)}{\partial w} =& \prod_{m=0}^{N-1}(I - \alpha \nabla_w^2L_{\mathcal{S}_i}(w_{m}^i,\phi)) \nabla_{w} L_{\mathcal{D}_i} (w_{N}^i,\phi). 
\end{align}
Let $w_m^i(w,\phi)$ denote the $m^{th}$ inner-loop iteration starting from $(w,\phi)$. Then, we have
{\small 
\begin{align}\label{eq:initss}
\Big\| &\frac{\partial L_{\mathcal{D}_i}( w^i_N,\phi)}{\partial w} \Big |_{(w_1,\phi_1)} -  \frac{\partial L_{\mathcal{D}_i}( w^i_N,\phi)}{\partial w} \Big |_{(w_2,\phi_2)} \Big\| \nonumber
\\ \leq & \underbrace{\Big\|\prod_{m=0}^{N-1}(I - \alpha \nabla_w^2L_{\mathcal{S}_i}(w_{m}^i(w_2,\phi_2),\phi_2)) \Big\|\Big\|\nabla_{w} L_{\mathcal{D}_i} (w_{N}^i(w_1,\phi_1),\phi_1)-\nabla_{w} L_{\mathcal{D}_i} (w_{N}^i(w_2,\phi_2),\phi_2)\Big\|}_{P} \nonumber
\\ &+\Big\|\prod_{m=0}^{N-1}(I - \alpha \nabla_w^2L_{\mathcal{S}_i}(w_{m}^i(w_1,\phi_1),\phi_1)) \nabla_{w} L_{\mathcal{D}_i} (w_{N}^i(w_1,\phi_1),\phi_1)\nonumber
\\&\hspace{1.2cm}\underbrace{\hspace{0.7cm}-\prod_{m=0}^{N-1}(I - \alpha \nabla_w^2L_{\mathcal{S}_i}(w_{m}^i(w_2,\phi_2),\phi_2)) \nabla_{w} L_{\mathcal{D}_i} (w_{N}^i(w_1,\phi_1),\phi_1)\Big\|}_{Q}, 
\end{align}
}
\hspace{-0.12cm}where $w_m^i(w,\phi)$ is obtained through the following gradient descent steps
\begin{align}\label{eq:updates}
w_{t+1}^i(w,\phi) = w_{t}^i(w,\phi) - \alpha \nabla_{w} L_{\mathcal{S}_i} (w_{t}^i(w,\phi),\phi),\, t=0,...,m-1\;\text{and} \;w_0^i(w,\phi)=w.
\end{align}
We next upper-bound the term $P$ in~\cref{eq:initss}. Based on the strongly-convexity of the function $L_{\mathcal{S}_i}(\cdot,\phi)$, we have $\big\|I-\alpha\nabla_w^2L_{\mathcal{S}_i}(\cdot,\phi)\big\|\leq 1-\alpha \mu$, and hence 
\begin{align}\label{eq:pss}
P &\leq (1-\alpha \mu)^{N}\big\|\nabla_{w} L_{\mathcal{D}_i} (w_{N}^i(w_1,\phi_1),\phi_1)-\nabla_{w} L_{\mathcal{D}_i} (w_{N}^i(w_2,\phi_2),\phi_2)\big\| \nonumber
\\\overset{(i)}\leq&(1-\alpha \mu)^{N} L\big(\|w_{N}^i(w_1,\phi_1)-w_{N}^i(w_2,\phi_2)\| +\|\phi_1-\phi_2\|\big) \nonumber
\\\overset{(ii)}\leq& (1-\alpha\mu)^{N}L\Big( (1-2\alpha\mu+\alpha^2L^2)^{\frac{N}{2}}\|w_1-w_2\| + \frac{\alpha L\|\phi_1-\phi_2\|}{1-\sqrt{1-2\alpha\mu+\alpha^2L^2}} + \|\phi_1-\phi_2\|  \Big) \nonumber
\\\overset{(iii)}\leq &  (1-\alpha\mu)^{\frac{3N}{2}} L \|w_1-w_2\| +  (1-\alpha\mu)^NL\left(\frac{2L}{\mu}+1\right)\|\phi_1-\phi_2\|,
\end{align}
where $(i)$ follows from~\Cref{assm:smooth}, (ii) follows from~\Cref{le:support}, and $(iii)$ follows from the fact that  $\alpha\mu = \frac{\mu^2}{L^2} =\alpha^2L^2$ and $\sqrt{1-x}\leq 1-\frac{1}{2}x$. 

To upper-bound the term $Q$ in~\cref{eq:initss}, we have
\begin{small}
\begin{align}\label{eq:qml}
Q\leq M\underbrace{\bigg\|\prod_{m=0}^{N-1}(I - \alpha \nabla_w^2L_{\mathcal{S}_i}(w_{m}^i(w_1,\phi_1),\phi_1)) -\prod_{m=0}^{N-1}(I - \alpha \nabla_w^2L_{\mathcal{S}_i}(w_{m}^i(w_2,\phi_2),\phi_2))  \bigg\|}_{P_{N-1}}.
\end{align}
\end{small}
\hspace{-0.12cm}To upper-bound $P_{N-1}$ in~\cref{eq:qml}, we define a more general quantity $P_t$ by replacing $N-1$ with $t$ in~\cref{eq:qml}. Using the triangle inequality, we have
\begin{align}\label{eq:telob}
P_t &\leq \alpha(1-\alpha\mu)^{t}\| \nabla_w^2L_{\mathcal{S}_i}(w_{t}^i(w_1,\phi_1),\phi_1)) -\nabla_w^2L_{\mathcal{S}_i}(w_{t}^i(w_2,\phi_2),\phi_2)) \| + (1-\alpha\mu) P_{t-1} \nonumber
\\\leq& (1-\alpha\mu) P_{t-1} + \alpha \rho(1-\alpha\mu)^{\frac{3t}{2}}  \|w_1-w_2\| +(1-\alpha\mu)^t\alpha\rho\frac{2L+\mu}{\mu}\|\phi_1-\phi_2\|.
\end{align}
Telescoping~\cref{eq:telob} over $t$ from $1$ to $N-1$ yields
\begin{align*}
P_{N-1}\leq &(1-\alpha\mu)^{N-1} P_0 + \sum_{t=1}^{N-1}  \alpha \rho(1-\alpha\mu)^{\frac{3t}{2}}  \|w_1-w_2\|(1-\alpha \mu)^{N-1-t} \nonumber
\\&+\sum_{t=1}^{N-1}(1-\alpha\mu)^t\alpha\rho\left(\frac{2L}{\mu}+1\right)\|\phi_1-\phi_2\|(1-\alpha \mu)^{N-1-t}, 
\end{align*}
which, in conjunction with $P_0\leq \alpha\rho(\|w_1-w_2\|+\|\phi_1-\phi_2\|)$, yields
\begin{align*}
P_{N-1}\leq& (1-\alpha\mu)^{N-1} \alpha\rho(\|w_1-w_2\|+\|\phi_1-\phi_2\|) +(1-\alpha\mu)^{N-1}\frac{ \alpha \rho \|w_1-w_2\| \sqrt{1-\alpha\mu}}{1-\sqrt{1-\alpha\mu}} \nonumber
\\&+ \alpha\rho\left(\frac{2L}{\mu}+1\right)\|\phi_1-\phi_2\|(N-1)(1-\alpha\mu)^{N-1} \nonumber
\\\leq& \frac{2\rho}{\mu}(1-\alpha\mu)^{N-1}\|w_1-w_2\|+  \alpha\rho\left(\frac{2L}{\mu}+1\right)\|\phi_1-\phi_2\|N(1-\alpha\mu)^{N-1}, 
\end{align*}
which, in conjunction with \cref{eq:qml}, yields 
\begin{align}\label{eq:qss}
Q\leq  \frac{2\rho M}{\mu}&(1-\alpha\mu)^{N-1}\|w_1-w_2\| \nonumber
\\&+  \alpha\rho M\left(\frac{2L}{\mu}+1\right)\|\phi_1-\phi_2\|N(1-\alpha\mu)^{N-1}.
\end{align}
Substituting~\cref{eq:pss} and~\cref{eq:qss} into~\cref{eq:initss} yields
\begin{align}\label{eq:mamamiya}
\Big\| \frac{\partial L_{\mathcal{D}_i}( w^i_N,\phi)}{\partial w} \Big |_{(w_1,\phi_1)} &-  \frac{\partial L_{\mathcal{D}_i}( w^i_N,\phi)}{\partial w} \Big |_{(w_2,\phi_2)} \Big\| \nonumber
\\\leq \Big((1-\alpha\mu)^{\frac{3N}{2}} L+& \frac{2\rho M}{\mu}(1-\alpha\mu)^{N-1}\Big) \|w_1-w_2\|  \nonumber
\\+  \Big((1-\alpha\mu)^NL&+\alpha\rho MN(1-\alpha\mu)^{N-1}\Big)\left(\frac{2L}{\mu}+1\right)\|\phi_1-\phi_2\|.
\end{align} 
Based on the definition $ L^{meta}( w,\phi)= \mathbb{E}_i L_{\mathcal{D}_i}( w^i_N,\phi)$ and using the Jensen's inequality, we have 
\begin{align}\label{eq:jensen}
 \Big\| \frac{\partial L^{meta}( w,\phi)}{\partial w} \big |_{(w_1,\phi_1)} -&  \frac{\partial L^{meta}( w,\phi)}{\partial w} \big |_{(w_2,\phi_2)} \Big\|  \nonumber
 \\&\leq \mathbb{E}_i  \Big\| \frac{\partial L_{\mathcal{D}_i}( w^i_N,\phi)}{\partial w} \Big |_{(w_1,\phi_1)} -  \frac{\partial L_{\mathcal{D}_i}( w^i_N,\phi)}{\partial w} \Big |_{(w_2,\phi_2)} \Big\|. 
\end{align}
Combining~\cref{eq:mamamiya} and~\cref{eq:jensen} completes the proof of the first item. 

We next prove the Lipschitz property of the partial gradient $\frac{\partial L_{\mathcal{D}_i}( w^i_N, \phi)}{\partial \phi}$. For notational convenience, we define several quantities below.
\begin{align}\label{eq:notations}
Q_m(w,\phi) &= \nabla_\phi\nabla_w L_{\mathcal{S}_i}(w_{m}^i(w,\phi),\phi), \;U_m(w,\phi)  = \prod_{j=m+1}^{N-1}(I-\alpha\nabla_w^2L_{\mathcal{S}_i}(w_{j}^i(w,\phi),\phi)), \nonumber
\\V_m(w,\phi) & = \nabla_w L_{\mathcal{D}_i}(w_N^i(w,\phi),\phi),
\end{align}
where we let $w_m^i(w,\phi)$ denote the $m^{th}$ inner-loop iteration starting from $(w,\phi)$.
Using an approach similar to the proof for~\Cref{le:gd_form}, we have
\begin{align}\label{eq:ainiyo11}
\frac{\partial L_{\mathcal{D}_i}( w^i_{N}, \phi)}{\partial \phi} =& -\alpha \sum_{m=0}^{N-1}\nabla_\phi\nabla_w L_{\mathcal{S}_i}(w_{m}^i,\phi) \prod_{j=m+1}^{N-1}(I-\alpha\nabla_w^2L_{\mathcal{S}_i}(w_{j}^i,\phi))\nabla_w L_{\mathcal{D}_i}(w_{N}^i,\phi) \nonumber
\\&+\nabla_\phi L_{\mathcal{D}_i}(w_{N}^i,\phi).
\end{align}
 Then, we have 
\begin{align}\label{eq:sect}
\Big\|&\frac{\partial L_{\mathcal{D}_i}( w^i_N, \phi)}{\partial \phi}\Big |_{(w_1,\phi_1)}  - \frac{\partial L_{\mathcal{D}_i}( w^i_N, \phi)}{\partial \phi} \Big |_{(w_2,\phi_2)}\Big\| \nonumber
\\&\;\leq \alpha\sum_{m=0}^{N-1}\|Q_m(w_1,\phi_1)U_m(w_1,\phi_1)V_m(w_1,\phi_1)-Q_m(w_2,\phi_2)U_m(w_2,\phi_2)V_m(w_2,\phi_2)\|  \nonumber
\\&\;\quad+\|\nabla_\phi L_{\mathcal{D}_i}(w_{N}^i(w_1,\phi_1),\phi_1)-\nabla_\phi L_{\mathcal{D}_i}(w_{N}^i(w_2,\phi_2),\phi_2)\|.
\end{align}
Using the triangle inequality, we have 
\begin{align}\label{R1R2R3} 
 \|Q_m(w_1,&\phi_1)U_m(w_1,\phi_1)V_m(w_1,\phi_1)-Q_m(w_2,\phi_2)U_m(w_2,\phi_2)V_m(w_2,\phi_2)\| \nonumber
 \\\leq &\underbrace{\|Q_m(w_1,\phi_1)-Q_m(w_2,\phi_2)\|\|U_m(w_1,\phi_1)\|\|V_m(w_1,\phi_1)\|}_{R_1}  \nonumber
 \\&+ \underbrace{\|Q_m(w_2,\phi_2)\|\|U_m(w_1,\phi_1)-U_m(w_2,\phi_2)\|\|V_m(w_1,\phi_1)\|}_{R_2}\nonumber
 \\&+\underbrace{\|Q_m(w_2,\phi_2)\|\|U_m(w_2,\phi_2)\|\|V_m(w_1,\phi_1)-V_m(w_2,\phi_2)\|}_{R_3}.
\end{align} 
Combining~\cref{eq:sect} and~\cref{R1R2R3}, we have  
\begin{align}\label{eq:initllls}
&\Big\|\frac{\partial L_{\mathcal{D}_i}( w^i_N, \phi)}{\partial \phi}\Big |_{(w_1,\phi_1)}  - \frac{\partial L_{\mathcal{D}_i}( w^i_N, \phi)}{\partial \phi} \Big |_{(w_2,\phi_2)}\Big\| \nonumber
\\&\;\;\leq \alpha\sum_{m=0}^{N-1}(R_1+R_2+R_3) +\|\nabla_\phi L_{\mathcal{D}_i}(w_{N}^i(w_1,\phi_1),\phi_1)-\nabla_\phi L_{\mathcal{D}_i}(w_{N}^i(w_2,\phi_2),\phi_2)\|.
\end{align}
To upper-bound $R_1$, we have 
\begin{align}\label{sR1}
R_1&\leq \tau (\|w_m^i(w_1,\phi_1)-w_m^i(w_2,\phi_2)\| + \|\phi_1-\phi_2\|) (1-\alpha \mu)^{N-m-1} M \nonumber
\\\leq& \tau M(1-\alpha\mu)^{N-\frac{m}{2}-1}  \|w_1-w_2\| + \tau M\Big(\frac{2L}{\mu}+1\Big) (1-\alpha \mu)^{N-m-1}\|\phi_1-\phi_2\|,
\end{align}
where the second inequality follows from~\Cref{le:support}. 
 
For $R_2$, based on Assumptions~\ref{assm:smooth} and~\ref{assm:second}, we have
\begin{align}\label{eq:rfist}
R_2\leq LM\|U_m(w_1,\phi_1)-U_m(w_2,\phi_2)\|.
\end{align}
Using the definitions of $U_m(w_1,\phi_1)$ and $U_m(w_2,\phi_2)$ in~\cref{eq:notations} and using the triangle inequality, we have 
\begin{align*}
\|U_m&(w_1,\phi_1)-U_m(w_2,\phi_2)\|  \nonumber
\\\leq& \alpha \|\nabla_w^2L_{\mathcal{S}_i}(w_{m+1}^i(w_1,\phi_1),\phi_1)-\nabla_w^2L_{\mathcal{S}_i}(w_{m+1}^i(w_2,\phi_2),\phi_2)\| \|U_{m+1}(w_1,\phi_1)\| \nonumber
\\&+ \|I-\alpha\nabla_w^2L_{\mathcal{S}_i}(w_{m+1}^i(w_1,\phi_1),\phi_1)\| \|U_{m+1}(w_1,\phi_1)-U_{m+1}(w_2,\phi_2)\| \nonumber
\\\leq&\alpha \rho (1-\alpha \mu)^{N-m-2}(\|w_{m+1}^i(w_1,\phi_1)-w_{m+1}^i(w_2,\phi_2)\| + \|\phi_1-\phi_2\|) \nonumber
\\&+ (1-\alpha \mu)\|U_{m+1}(w_1,\phi_1)-U_{m+1}(w_2,\phi_2)\| \nonumber
\\\leq&\alpha \rho (1-\alpha \mu)^{N-m-2}\Big((1-\alpha\mu)^{\frac{m+1}{2}}  \|w_1-w_2\| +  \Big(\frac{2L}{\mu}+1\Big)\|\phi_1-\phi_2\|\Big) \nonumber
\\&\hspace{3cm}+ (1-\alpha \mu)\|U_{m+1}(w_1,\phi_1)-U_{m+1}(w_2,\phi_2)\|,
\end{align*} 
where the last  inequality follows from~\Cref{le:support}.  
Telescoping the above inequality over $m$ yields
\begin{align}
\|&U_m(w_1,\phi_1)-U_m(w_2,\phi_2)\| \nonumber
\\&\leq (1-\alpha\mu)^{N-m-2}\|U_{N-2}(w_1,\phi_1)-U_{N-2}(w_2,\phi_2)\|  \nonumber
\\&\quad+\sum_{t=0}^{N-m-3} (1-\alpha\mu)^{t}\alpha \rho (1-\alpha \mu)^{N-m-t-2} \nonumber
\\&\hspace{3cm}\times\Big((1-\alpha\mu)^{\frac{m+t+1}{2}}  \|w_1-w_2\| +  \Big(\frac{2L}{\mu}+1\Big)\|\phi_1-\phi_2\|\Big),\nonumber
\end{align}
which, in conjunction with~\cref{eq:notations}, yields
\begin{align}\label{eq:ggsmida}
\|U_m(w_1,\phi_1)&-U_m(w_2,\phi_2)\| \leq \left( \frac{\alpha\rho}{1-\alpha\mu}+\frac{2\rho}{\mu}\right)(1-\alpha\mu)^{N-1-\frac{m}{2}}\|w_1-w_2\|  \nonumber
\\+&\alpha(N-1-m)\left( \rho+\frac{2\rho L}{\mu} \right)(1-\alpha\mu)^{N-2-m}\|\phi_1-\phi_2\|.
\end{align}
Combining~\cref{eq:rfist} and~\cref{eq:ggsmida} yields
\begin{align}\label{eq:r2bb}
R_2\leq& LM \left( \frac{\alpha\rho}{1-\alpha\mu}+\frac{2\rho}{\mu}\right)(1-\alpha\mu)^{N-1-\frac{m}{2}}\|w_1-w_2\|  \nonumber
\\&\quad+\alpha LM(N-1-m)\left( \rho+\frac{2\rho L}{\mu} \right)(1-\alpha\mu)^{N-2-m}\|\phi_1-\phi_2\|.
\end{align}
For $R_3$, using the triangle inequality, we have 
\begin{align}\label{eq:r3bb}
R_3&\leq L(1-\alpha \mu)^{N-m-1} L(\|w_N^i(w_1,\phi_1)-w_N^i(w_2,\phi_2)\|+\|\phi_1-\phi_2\|) \nonumber
\\\leq& L^2(1-\alpha\mu)^{\frac{3N}{2}-m-1}\|w_1-w_2\| + L^2\left(\frac{2L}{\mu}+1\right)(1-\alpha\mu)^{N-1-m}\|\phi_1-\phi_2\|.
\end{align}
where the last inequality follows from~\Cref{le:support}.

Combine $R_1,R_2$ and $R_3$ in~\cref{sR1},~\cref{eq:r2bb} and~\cref{eq:r3bb}, we have
\begin{align}\label{eq:youdianda}
\sum_{m=0}^{N-1}&(R_1+R_2+R_3) \leq  \frac{2\tau M}{\alpha\mu}(1-\alpha\mu)^{\frac{N-1}{2}}  \|w_1-w_2\| + \frac{\tau M}{\alpha\mu}\Big(\frac{2L}{\mu}+1\Big) \|\phi_1-\phi_2\| \nonumber
\\ &+ \frac{2LM}{\alpha\mu} \left( \frac{\alpha\rho}{1-\alpha\mu}+\frac{2\rho}{\mu}\right)(1-\alpha\mu)^{\frac{N-1}{2}}\|w_1-w_2\| +\frac{\alpha LM}{\alpha^2\mu^2}\left( \rho+\frac{2\rho L}{\mu} \right)\|\phi_1-\phi_2\| \nonumber
\\&+\frac{L^2}{\alpha\mu} (1-\alpha\mu)^{\frac{N}{2}}\|w_1-w_2\| + \frac{L^2}{\alpha\mu}\left(\frac{2L}{\mu}+1\right)\|\phi_1-\phi_2\|.
\end{align}
In addition, note that 
\begin{align}\label{eq:ggpopdasda}
 \|\nabla_\phi L_{\mathcal{D}_i}(w_{N}^i(w_1,\phi_1),&\phi_1)-\nabla_\phi L_{\mathcal{D}_i}(w_{N}^i(w_2,\phi_2),\phi_2)\| \nonumber
 \\&\leq (1-\alpha\mu)^{\frac{N}{2}} L \|w_1-w_2\| +  L\left(\frac{2L}{\mu}+1\right)\|\phi_1-\phi_2\|.
 \end{align}
 Combining~\cref{eq:initllls},~\cref{eq:youdianda}, and~\cref{eq:ggpopdasda} yields
\begin{align}\label{eq:opissas}
\Big\|&\frac{\partial L_{\mathcal{D}_i}( w^i_N, \phi)}{\partial \phi}\Big |_{(w_1,\phi_1)}  - \frac{\partial L_{\mathcal{D}_i}( w^i_N, \phi)}{\partial \phi} \Big |_{(w_2,\phi_2)}\Big\| \nonumber
\\ &\hspace{2cm}\leq   \left(  L+\frac{2\tau M}{\mu}+ \frac{2LM}{\mu} \left( \frac{\alpha\rho}{1-\alpha\mu}+\frac{2\rho}{\mu}\right) +\frac{L^2}{\mu} \right)(1-\alpha\mu)^{\frac{N-1}{2}}\|w_1-w_2\| \nonumber
\\&\hspace{2.2cm}+ \left(L+\frac{\tau M}{\mu} + \frac{ LM\rho}{\mu^2}+\frac{L^2}{\mu} \right)\left(\frac{2L}{\mu}+1\right)\|\phi_1-\phi_2\|,
\end{align}
which, using an approach similar to~\cref{eq:jensen}, completes the proof. 
\subsection*{Proof of~\Cref{th:strong-convex} }
For notational convenience, we define
\begin{align}\label{eq:ddfinetes}
g_{w}^i(k) &=  \frac{\partial L_{\mathcal{D}_i}( w^i_{k,N},\phi_k)}{\partial {w_k}},\quad g_{\phi}^i(k)  = \frac{\partial L_{\mathcal{D}_i}( w^i_{k,N},\phi_k)}{\partial {\phi_k}}, \nonumber
\\L_w&=(1-\alpha\mu)^{\frac{3N}{2}} L+ \frac{2\rho M}{\mu}(1-\alpha\mu)^{N-1},L_w^\prime= (L+\alpha\rho MN)(1-\alpha\mu)^{N-1}\Big(\frac{2L}{\mu}+1\Big),\nonumber
\\  L_\phi&= \left(  L+\frac{2\tau M}{\mu}+ \frac{2LM}{\mu} \left( \frac{\alpha\rho}{1-\alpha\mu}+\frac{2\rho}{\mu}\right) +\frac{L^2}{\mu} \right)(1-\alpha\mu)^{\frac{N-1}{2}}, \nonumber
\\L_\phi^\prime &=\left(L+\frac{\tau M}{\mu} + \frac{ LM\rho}{\mu^2}+\frac{L^2}{\mu} \right)\left(\frac{2L}{\mu}+1\right).
\end{align}
Then, the updates of~\Cref{alg:anil} are given by 
\begin{align}\label{alg:sgdpf}
w_{k+1}=  w_{k} - \frac{\beta_w}{B}\sum_{i\in\mathcal{B}_k} g_{w}^i(k)\,\text{ and }\, \phi_{k+1}= \phi_{k} - \frac{\beta_\phi}{B}\sum_{i\in\mathcal{B}_k}g_{\phi}^i(k).
\end{align}
Based on~\cref{eq:mamamiya} and~\cref{eq:opissas} in the proof of~\Cref{le:strong-convex}, we have 
\begin{small}
\begin{align*}
L^{meta}(w_{k+1},\phi_k) \leq & L^{meta}(w_k,\phi_k) + \big\langle \frac{\partial L^{meta}(w_k,\phi_k)}{\partial w_k}, w_{k+1}-w_k  \big\rangle + \frac{L_w}{2} \|w_{k+1}-w_k\|^2, \nonumber
\\ L^{meta}(w_{k+1},\phi_{k+1}) \leq & L^{meta}(w_{k+1},\phi_k) + \big\langle \frac{\partial L^{meta}(w_{k+1},\phi_k)}{\partial \phi_k}, \phi_{k+1}-\phi_k  \big\rangle + \frac{L^\prime_\phi}{2} \|\phi_{k+1}-\phi_k\|^2.
\end{align*}
\end{small}
\hspace{-0.12cm}Adding the above two inequalities, we have
\begin{align}\label{eq:gg1anilscasa}
L^{meta}(w_{k+1},\phi_{k+1}) \leq&  L^{meta}(w_k,\phi_k) + \left\langle \frac{\partial L^{meta}(w_k,\phi_k)}{\partial w_k}, w_{k+1}-w_k  \right\rangle + \frac{L_w}{2} \|w_{k+1}-w_k\|^2 \nonumber
\\ &+\left\langle \frac{\partial L^{meta}(w_{k},\phi_k)}{\partial \phi_k}, \phi_{k+1}-\phi_k  \right\rangle + \frac{L_\phi^\prime}{2} \|\phi_{k+1}-\phi_k\|^2  \nonumber
\\&+\left\langle \frac{\partial L^{meta}(w_{k+1},\phi_k)}{\partial \phi_k}-\frac{\partial L^{meta}(w_{k},\phi_k)}{\partial \phi_k}, \phi_{k+1}-\phi_k  \right\rangle. 
\end{align}
Based on the Cauchy-Schwarz inequality, we have 
\begin{align}\label{eq:gg2}
\Big\langle \frac{\partial L^{meta}(w_{k+1},\phi_k)}{\partial \phi_k}-&\frac{\partial L^{meta}(w_{k},\phi_k)}{\partial \phi_k}, \phi_{k+1}-\phi_k  \Big\rangle  \nonumber
\\&\leq  L_\phi\|w_{k+1}-w_k\|\|\phi_{k+1}-\phi_k\|\nonumber
\\&\leq  \frac{L_\phi}{2}\|w_{k+1}-w_k\|^2 + \frac{L_\phi}{2}\|\phi_{k+1}-\phi_k\|^2.
\end{align}
Combining~\cref{eq:gg1anilscasa} and~\cref{eq:gg2}, we have   
\begin{align*}
L^{meta}(w_{k+1},\phi_{k+1}) \leq&  L^{meta}(w_k,\phi_k) + \left\langle \frac{\partial L^{meta}(w_k,\phi_k)}{\partial w_k}, w_{k+1}-w_k  \right\rangle  \nonumber
\\ &+ \frac{L_w+L_\phi }{2}\|w_{k+1}-w_k\|^2+\left\langle \frac{\partial L^{meta}(w_{k},\phi_k)}{\partial \phi_k}, \phi_{k+1}-\phi_k  \right\rangle  \nonumber
\\&+ \frac{L_\phi+L_\phi^\prime}{2}\|\phi_{k+1}-\phi_k\|^2,  \nonumber
\end{align*}
which, in conjunction with the updates in~\cref{alg:sgdpf}, yields 
\begin{align}\label{eq:wellplay}
L^{meta}&(w_{k+1},\phi_{k+1}) \nonumber
\\\leq&  L^{meta}(w_k,\phi_k) - \left\langle \frac{\partial L^{meta}(w_k,\phi_k)}{\partial w_k}, \frac{\beta_w}{B}\sum_{i\in\mathcal{B}_k} g_{w}^i (k) \right\rangle + \frac{L_w+L_\phi }{2} \Big\| \frac{\beta_w}{B}\sum_{i\in\mathcal{B}_k} g_{w}^i (k)\Big\|^2 \nonumber
\\ &-\left\langle \frac{\partial L^{meta}(w_{k},\phi_k)}{\partial \phi_k},\frac{\beta_\phi}{B}\sum_{i\in\mathcal{B}_k}g_{\phi}^i(k)  \right\rangle + \frac{L_\phi+L_\phi^\prime}{2} \Big\|\frac{\beta_\phi}{B}\sum_{i\in\mathcal{B}_k}g_{\phi}^i(k)\Big\|^2.
\end{align}
Let $\mathbb{E}_k=\mathbb{E}(\cdot | w_k,\phi_k)$ for simplicity. Then, conditioning on $w_k,\phi_k$, and taking expectation over \cref{eq:wellplay}, we have 
\begin{align}\label{eq:diedie}
\mathbb{E}_k L^{meta}(w_{k+1}&,\phi_{k+1})  \nonumber
\\\overset{(i)}\leq&  L^{meta}(w_k,\phi_k) - \beta_w\left\| \frac{\partial L^{meta}(w_k,\phi_k)}{\partial w_k}\right\|^2 + \frac{L_w+L_\phi}{2} \mathbb{E}_k\Big\| \frac{\beta_w}{B}\sum_{i\in\mathcal{B}_k} g_{w}^i(k) \Big\|^2 \nonumber
\\ &-\beta_\phi\left\| \frac{\partial L^{meta}(w_{k},\phi_k)}{\partial \phi_k}  \right\| + \frac{L_\phi+L_\phi^\prime}{2} \mathbb{E}_k\Big\|\frac{\beta_\phi}{B}\sum_{i\in\mathcal{B}_k}g_{\phi}^i(k)\Big\|^2 \nonumber
\\\leq & L^{meta}(w_k,\phi_k) - \beta_w\left\| \frac{\partial L^{meta}(w_k,\phi_k)}{\partial w_k}\right\|^2 + \frac{(L_w+L_\phi)\beta_w^2}{2B}\mathbb{E}_k\big\|g_{w}^i(k) \big\|^2 \nonumber
\\ &+ \frac{L_\phi+L_w}{2}\beta_w^2 \left\| \frac{\partial L^{meta}(w_k,\phi_k)}{\partial w_k}\right\|^2-\beta_\phi\left\| \frac{\partial L^{meta}(w_{k},\phi_k)}{\partial \phi_k}  \right\|^2 \nonumber
\\&+ \frac{L_\phi+L_\phi^\prime}{2} \left(\frac{\beta_\phi^2}{B}\mathbb{E}_k\big\|g_{\phi}^i(k) \big\|^2 + \beta_\phi^2 \left\| \frac{\partial L^{meta}(w_k,\phi_k)}{\partial \phi_k}\right\|^2\right),
\end{align}
where $(i)$ follows from the fact that $\mathbb{E}_k g_w^i(k)= \frac{\partial L^{meta}(w_k,\phi_k)}{\partial w_k}$ and $\mathbb{E}_k g_\phi^i(k)= \frac{\partial L^{meta}(w_k,\phi_k)}{\partial \phi_k}$.

 Our next
step is to upper-bound $\mathbb{E}_k\big\|g_{w}^i(k) \big\|^2$ and $\mathbb{E}_k\big\|g_{\phi}^i(k) \big\|^2$ in~\cref{eq:diedie}. Based on the definitions of $g_{w}^i(k)$ in~\cref{eq:ddfinetes} and using the explicit forms of the meta gradients in~\Cref{le:gd_form}, we have
\begin{align}\label{eq:omgs}
\mathbb{E}_k\big\|g_{w}^i(k) \big\|^2 \leq& \mathbb{E}_k\Big\|\prod_{m=0}^{N-1}(I - \alpha \nabla_w^2L_{\mathcal{S}_i}(w_{k,m}^i,\phi_k)) \nabla_{w} L_{\mathcal{D}_i} (w_{k,N}^i,\phi_k)\Big\|^2 \nonumber\\\leq &(1-\alpha\mu)^{2N} M^2.
\end{align}
Using an approach similar to~\cref{eq:omgs}, we have 
\begin{align}\label{eq:anothergg}
\mathbb{E}_k\big\|&g_{\phi}^i(k) \big\|^2  \nonumber
\\ \leq& 2\mathbb{E}_k\bigg\|\alpha \sum_{m=0}^{N-1}\nabla_\phi\nabla_w L_{\mathcal{S}_i}(w_{k,m}^i,\phi_k) \prod_{j=m+1}^{N-1}(I-\alpha\nabla_w^2L_{\mathcal{S}_i}(w_{k,j}^i,\phi_k))\nabla_w L_{\mathcal{D}_i}(w_{k,N}^i,\phi_k)\bigg\|^2 \nonumber
\\&+2\|\nabla_\phi L_{\mathcal{D}_i}(w_{k,N}^i,\phi_k)\|^2 \nonumber
\\ \leq& 2\alpha^2 L^2 M^2 \mathbb{E}_k\Big(\sum_{m=0}^{N-1} (1-\alpha \mu)^{N-1-m}\Big)^2 +2 M^2 \nonumber
\\< & \frac{2L^2M^2}{\mu^2}+2M^2 <  2M^2 \left( \frac{L^2}{\mu^2}+1 \right).
\end{align}
Substituting~\cref{eq:omgs} and~\cref{eq:anothergg} into~\cref{eq:diedie} yields
\begin{align}\label{eq:jingyi}
\mathbb{E}_k L^{meta}&(w_{k+1},\phi_{k+1}) \leq L^{meta}(w_k,\phi_k) - \left(\beta_w-\frac{L_w+L_\phi}{2}\beta_w^2\right)\left\| \frac{\partial L^{meta}(w_k,\phi_k)}{\partial w_k}\right\|^2 \nonumber
\\&+ \frac{(L_w+L_\phi)\beta_w^2}{2B}(1-\alpha\mu)^{2N} M^2-\left(\beta_\phi -\frac{L_\phi+L_\phi^\prime}{2}\beta_\phi^2 \right)\left\| \frac{\partial L^{meta}(w_{k},\phi_k)}{\partial \phi_k}  \right\|^2\nonumber
\\ & + \frac{ (L_\phi+L_\phi^\prime) \beta_\phi^2}{B} M^2 \left( \frac{L^2}{\mu^2}+1 \right).
\end{align}
Let $\beta_w=\frac{1}{L_w+L_\phi}$ and $\beta_\phi = \frac{1}{L_\phi+L_\phi^\prime}$. Then, 
unconditioning on $w_k$ and $\phi_k$ and telescoping~\cref{eq:jingyi} over $k$ from $0$ to $K-1$ yield 
\begin{align}
&\frac{\beta_w}{2}\frac{1}{K}\sum_{k=0}^{K-1}\mathbb{E}\left\| \frac{\partial L^{meta}(w_k,\phi_k)}{\partial w_k}\right\|^2 + \frac{\beta_\phi}{2}\frac{1}{K}\sum_{k=0}^{K-1}\mathbb{E}\left\| \frac{\partial L^{meta}(w_{k},\phi_k)}{\partial \phi_k}  \right\|^2 \nonumber
\\& \leq  \frac{L^{meta}(w_0,\phi_0)-\min_{w,\phi}L^{meta}(w,\phi) }{K}+ \frac{\beta_w M^2(1-\alpha\mu)^{2N}}{2B}+\frac{  \beta_\phi M^2\big( \frac{L^2}{\mu^2}+1 \big)}{B}.
\end{align}
Let $\Delta = L^{meta}(w_0,\phi_0)-\min_{w,\phi}L^{meta}(w,\phi)$ and let $\xi$ be chosen from $\{0,...,K-1\}$ uniformly at random. Then, we have 
\begin{align*}
\mathbb{E}\left\| \frac{\partial L^{meta}(w_\xi,\phi_\xi)}{\partial w_\xi}\right\|^2  \leq& \frac{2\Delta (L_w+L_\phi)}{K} + \frac{(1-\alpha\mu)^{2N} M^2}{B}+\frac{L_w+L_\phi}{L_\phi+L_\phi^\prime}\frac{2 }{B} M^2 \left( \frac{L^2}{\mu^2}+1 \right), \nonumber
\\\mathbb{E}\left\| \frac{\partial L^{meta}(w_\xi,\phi_\xi)}{\partial \phi_\xi}\right\|^2  \leq &\frac{2\Delta (L_\phi+L_\phi^\prime)}{K} +\frac{L_\phi+L_\phi^\prime}{L_w+L_\phi}\frac{1}{B}(1-\alpha\mu)^{2N} M^2+\frac{ 2}{B} M^2 \left( \frac{L^2}{\mu^2}+1 \right),
\end{align*}
which, in conjunction with the definitions of $L_\phi,L_\phi^\prime$ and $L_w$ in~\cref{eq:ddfinetes} and $\alpha=\frac{\mu}{L^2}$, yields 
\begin{align*}
\mathbb{E}\left\| \frac{\partial L^{meta}(w_\xi,\phi_\xi)}{\partial w_\xi}\right\|^2  \leq& \mathcal{O}\Bigg( \frac{  \frac{1}{\mu^2}\left(1-\frac{\mu^2}{L^2}\right)^{\frac{N}{2}}}{K}      +\frac{\frac{1}{\mu} \left(1-\frac{\mu^2}{L^2}\right)^{\frac{N}{2}}}{B}     \Bigg), \nonumber
\\\mathbb{E}\left\| \frac{\partial L^{meta}(w_\xi,\phi_\xi)}{\partial \phi_\xi}\right\|^2  \leq & \mathcal{O}\Bigg(\frac{ \frac{1}{\mu^2}\left(1-\frac{\mu^2}{L^2}\right)^{\frac{N}{2}}+\frac{1}{\mu^3}}{K} +\frac{\frac{1}{\mu}\left(1-\frac{\mu^2}{L^2}\right)^{\frac{3N}{2}}+\frac{1}{\mu^2}}{B}\Bigg).
\end{align*}
To achieve an $\epsilon$-stationary point, i.e., {\small $\mathbb{E}\left\| \frac{\partial L^{meta}(w,\phi)}{\partial w}\right\|^2 <\epsilon,\mathbb{E}\left\| \frac{\partial L^{meta}(w,\phi)}{\partial w}\right\|^2 <\epsilon$}, ANIL requires at most 
\begin{align*}
KBN=&\mathcal{O}\left(\frac{L^2}{\mu^2}\left(1-\frac{\mu^2}{L^2}\right)^{\frac{N}{2}}+\frac{L^3}{\mu^3}\right)\left(\frac{L}{\mu}\left(1-\frac{\mu^2}{L^2}\right)^{\frac{3N}{2}}+\frac{L^2}{\mu^2}\right)\epsilon^{-2}
\\\leq&\mathcal{O}\left(\frac{N}{\mu^4}\left(1-\frac{\mu^2}{L^2}\right)^{\frac{N}{2}}+\frac{N}{\mu^5}\right)\epsilon^{-2}
\end{align*}
 gradient evaluations in $w$,  $KB=\mathcal{O}\Big(\mu^{-4}\left(1-\frac{\mu^2}{L^2}\right)^{N/2}+\mu^{-5}\Big)\epsilon^{-2}$ gradient evaluations in $\phi$,  and $KBN=\mathcal{O}\Big(\frac{N}{\mu^{4}}\left(1-\frac{\mu^2}{L^2}\right)^{N/2}+\frac{N}{\mu^{5}}\Big)\epsilon^{-2}$ evaluations of second-order derivatives.

\section{Proof for Nonconvex Inner Loop}\label{appen:smooth_nonconvex}
\subsection*{Proof of~\Cref{le:smooth_nonconvex}}
Based on the explicit forms of the meta gradient in~\cref{eq:ainiyo} and using an approach similar to~\cref{eq:initss}, we have 
\begin{align}\label{eq:def}
 &\Big\| \frac{\partial L_{\mathcal{D}_i}( w^i_N,\phi)}{\partial w} \Big |_{(w_1,\phi_1)} -  \frac{\partial L_{\mathcal{D}_i}( w^i_N,\phi)}{\partial w} \Big |_{(w_2,\phi_2)} \Big\| \nonumber
\\ &= \Big\|\prod_{m=0}^{N-1}(I - \alpha \nabla_w^2L_{\mathcal{S}_i}(w_{m}^i(w_1,\phi_1),\phi_1)) \nabla_{w} L_{\mathcal{D}_i} (w_{N}^i(w_1,\phi_1),\phi_1)\nonumber
\\&\hspace{2cm}-\prod_{m=0}^{N-1}(I - \alpha \nabla_w^2L_{\mathcal{S}_i}(w_{m}^i(w_2,\phi_2),\phi_2)) \nabla_{w} L_{\mathcal{D}_i} (w_{N}^i(w_2,\phi_2),\phi_2)\Big\|, 
\end{align}
where $w_m^i(w,\phi)$ is obtained through the gradient descent steps in~\cref{eq:updates}.

Using the triangle inequality in~\cref{eq:def} yields
\begin{align}\label{eq:diff} 
\Big\| &\frac{\partial L_{\mathcal{D}_i}( w^i_N,\phi)}{\partial w} \Big |_{(w_1,\phi_1)} -  \frac{\partial L_{\mathcal{D}_i}( w^i_N,\phi)}{\partial w} \Big |_{(w_2,\phi_2)} \Big\| \leq  \Big\|\prod_{m=0}^{N-1}(I - \alpha \nabla_w^2L_{\mathcal{S}_i}(w_{m}^i(w_2,\phi_2),\phi_2)) \Big\| \nonumber
\\&\quad\quad\times\Big\|\nabla_{w} L_{\mathcal{D}_i} (w_{N}^i(w_1,\phi_1),\phi_1)-\nabla_{w} L_{\mathcal{D}_i} (w_{N}^i(w_2,\phi_2),\phi_2)\Big\| \nonumber
\\ &\quad\quad+\Big\|\prod_{m=0}^{N-1}(I - \alpha \nabla_w^2L_{\mathcal{S}_i}(w_{m}^i(w_1,\phi_1),\phi_1)) \nabla_{w} L_{\mathcal{D}_i} (w_{N}^i(w_1,\phi_1),\phi_1)\nonumber
\\&\quad\hspace{2cm}-\prod_{m=0}^{N-1}(I - \alpha \nabla_w^2L_{\mathcal{S}_i}(w_{m}^i(w_2,\phi_2),\phi_2)) \nabla_{w} L_{\mathcal{D}_i} (w_{N}^i(w_1,\phi_1),\phi_1)\Big\|. 
\end{align}
Our next two steps are to upper-bound the two terms at the right hand side of~\cref{eq:diff}, respectively.

{Step 1: Upper-bound the first term at the right hand side of~\cref{eq:diff}.}
\begin{small}
\begin{align}\label{eq:first}
&\Big\|\prod_{m=0}^{N-1}(I - \alpha \nabla_w^2L_{\mathcal{S}_i}(w_{m}^i(w_2,\phi_2),\phi_2)) \Big\|\Big\|\nabla_{w} L_{\mathcal{D}_i} (w_{N}^i(w_1,\phi_1),\phi_1)-\nabla_{w} L_{\mathcal{D}_i} (w_{N}^i(w_2,\phi_2),\phi_2)\Big\| \nonumber
\\&\overset{(i)}\leq (1+\alpha L)^N \Big\|\nabla_{w} L_{\mathcal{D}_i} (w_{N}^i(w_1,\phi_1),\phi_1)-\nabla_{w} L_{\mathcal{D}_i} (w_{N}^i(w_2,\phi_2),\phi_2)\Big\|\nonumber
\\&\overset{(ii)}\leq (1+\alpha L)^N L(\|w_{N}^i(w_1,\phi_1)-w_{N}^i(w_2,\phi_2)\|+\|\phi_1-\phi_2\|),
\end{align}
\end{small}
\hspace{-0.12cm}  where $(i)$ follows from the fact that $\| \nabla_w^2L_{\mathcal{S}_i}(w_{m}^i(w_2,\phi_2),\phi_2)\|\leq L$, and $(ii)$ follows from~\Cref{assm:smooth}. Based on the gradient descent steps in~\cref{eq:updates}, we have, for any $0\leq m\leq N-1$,  
  \begin{align}
 & w_{m+1}^i(w_1,\phi_1) - w_{m+1}^i(w_2,\phi_2)  \nonumber
 \\ &\quad= w_{m}^i(w_1,\phi_1) -w_{m}^i(w_2,\phi_2) - \alpha\big(\nabla_{w} L_{\mathcal{S}_i} (w_{m}^i(w_1,\phi_1),\phi_1) -\nabla_{w} L_{\mathcal{S}_i} (w_{m}^i(w_2,\phi_2),\phi_2)\big).\nonumber
  \end{align}
Based on the above equality, we further obtain 
\begin{align*}
\|&w_{m+1}^i(w_1,\phi_1) - w_{m+1}^i(w_2,\phi_2)\|
\\&\leq \|w_{m}^i(w_1,\phi_1) -w_{m}^i(w_2,\phi_2) \| + \alpha \|\nabla_{w} L_{\mathcal{S}_i} (w_{m}^i(w_1,\phi_1),\phi_1) -\nabla_{w} L_{\mathcal{S}_i} (w_{m}^i(w_2,\phi_2),\phi_2)\| \nonumber 
\\&\leq(1+\alpha L)\|w_{m}^i(w_1,\phi_1) -w_{m}^i(w_2,\phi_2) \|  + \alpha L \|\phi_1-\phi_2\|,\nonumber
\end{align*}
where the last inequality follows from~\Cref{assm:smooth}.
Telescoping the above inequality over $m$ from $0$ to $N-1$ yields
\begin{align}\label{eq:wn}
\|w_{N}^i(w_1,\phi_1) - &w_{N}^i(w_2,\phi_2)\|  \nonumber
\\&\leq (1+\alpha L)^N \|w_1-w_2\| + ((1+\alpha L)^N-1)\|\phi_1-\phi_2\|.
\end{align}
Combining~\cref{eq:first} and~\cref{eq:wn} yields
\begin{align}\label{upb1}
&\Big\|\prod_{m=0}^{N-1}(I - \alpha \nabla_w^2L_{\mathcal{S}_i}(w_{m}^i(w_2,\phi_2),\phi_2)) \Big\|\Big\|\nabla_{w} L_{\mathcal{D}_i} (w_{N}^i(w_1,\phi_1),\phi_1)-\nabla_{w} L_{\mathcal{D}_i} (w_{N}^i(w_2,\phi_2),\phi_2)\Big\| \nonumber
\\&\leq (1+\alpha L)^{2N} L(\|w_1-w_2\|+\|\phi_1-\phi_2\|).
\end{align}

{ Step 2: Upper-bound the second term at the right hand side of~\cref{eq:diff}.}

Based on item 2 in~\Cref{assm:smooth}, we have that $\|\nabla_{w} L_{\mathcal{D}_i}(\cdot,\cdot)\|\leq M$. Then, the second term at the right hand side of~\cref{eq:diff} is further upper-bounded by 
\begin{align}\label{m:quantity}
M\underbrace{\bigg\|\prod_{m=0}^{N-1}(I - \alpha \nabla_w^2L_{\mathcal{S}_i}(w_{m}^i(w_1,\phi_1),\phi_1)) -\prod_{m=0}^{N-1}(I - \alpha \nabla_w^2L_{\mathcal{S}_i}(w_{m}^i(w_2,\phi_2),\phi_2))\bigg\|}_{P_{N-1}}. 
\end{align}
In order to upper-bound $P_{N-1}$ in~\cref{m:quantity}, we define a more general quantity $P_t$ by replacing $N-1$ with $t$ in~\cref{m:quantity}. Based on the triangle inequality, we have 
\begin{align}
P_t \leq& \alpha \bigg\|\prod_{m=0}^{t-1}(I - \alpha \nabla_w^2L_{\mathcal{S}_i}(w_{m}^i,\phi_1))\bigg\|\Big\|  \nabla_w^2L_{\mathcal{S}_i}(w_{t}^i(w_1,\phi_1),\phi_1) - \nabla_w^2L_{\mathcal{S}_i}(w_{t}^i(w_2,\phi_2),\phi_2)\Big\| \nonumber
\\ &+ P_{t-1} \Big\|I - \alpha \nabla_w^2L_{\mathcal{S}_i}(w_{t}^i(w_2,\phi_2),\phi_2)\Big\| \nonumber
\\\leq & \alpha(1+\alpha L)^t \rho (\|w_{t}^i(w_1,\phi_1) -w_{t}^i(w_2,\phi_2)\| + \|\phi_1-\phi_2\|) + (1+\alpha L)P_{t-1}\nonumber
\\\overset{(i)}\leq &\alpha \rho(1+\alpha L)^{2t}  (\|w_1-w_2\|+\|\phi_1-\phi_2\|) + (1+\alpha L) P_{t-1}, \nonumber
\end{align} 
where $(i)$ follows from~\cref{eq:wn}.  Rearranging the above inequality, we have  
\begin{align}\label{eq:idc}
P_t - &\frac{\rho}{L}(1+\alpha L)^{2t+1}  (\|w_1-w_2\|+\|\phi_1-\phi_2\|)  \nonumber
\\&\leq (1+\alpha L)(P_{t-1}-\frac{\rho}{L}(1+\alpha L)^{2t-1}(\|w_1-w_2\|+\|\phi_1-\phi_2\|)).
\end{align}
Telescoping \cref{eq:idc} over $t$ from $1$ to $N-1$ yields
\begin{align}
P_{N-1} - \frac{\rho}{L}(1+\alpha L)^{2N-1}  &(\|w_1-w_2\|+\|\phi_1-\phi_2\|)   \nonumber
\\&\leq (1+\alpha L)^{N}\Big(P_{0}-\frac{\rho}{L}(1+\alpha L)(\|w_1-w_2\|+\|\phi_1-\phi_2\|) \Big),  \nonumber
\end{align}
which, combined  with $P_{0}=\alpha\|\nabla_w^2L_{\mathcal{S}_i}(w_1,\phi_1)-\nabla_w^2L_{\mathcal{S}_i}(w_2,\phi_2)\|\leq \alpha\rho(\|w_1-w_2\|+\|\phi_1-\phi_2\|)$, yields
\begin{align}\label{eq:pn1}
P_{N-1} - \frac{\rho}{L}(1+\alpha L)^{2N-1}  &(\|w_1-w_2\|+\|\phi_1-\phi_2\|)   \nonumber
\\&\leq (1+\alpha L)^{N}\Big(\frac{\rho}{L}(\|w_1-w_2\|+\|\phi_1-\phi_2\|) \Big)\nonumber
\\&\leq \frac{\rho}{L}(1+\alpha L)^{2N-1}  (\|w_1-w_2\|+\|\phi_1-\phi_2\|), 
 \end{align}
 where the last inequality follows because $N\geq 1$. 
 Combining \cref{m:quantity} and~\cref{eq:pn1}, the second term at the right hand side of~\cref{eq:diff} is upper-bounded by 
 \begin{align}\label{upb2}
\frac{ 2M\rho}{L}(1+\alpha L)^{2N-1}  &(\|w_1-w_2\|+\|\phi_1-\phi_2\|).
 \end{align}
 {Step 3: Combine two bounds in Steps 1 and 2.}
 
 Combining~\cref{upb1},~\cref{upb2}, and using $\alpha <\mathcal{O}(\frac{1}{N})$, we have 
 \begin{align}\label{eq:j1}
 \Big\| \frac{\partial L_{\mathcal{D}_i}( w^i_N,\phi)}{\partial w} &\Big |_{(w_1,\phi_1)} -  \frac{\partial L_{\mathcal{D}_i}( w^i_N,\phi)}{\partial w} \Big |_{(w_2,\phi_2)} \Big\| \nonumber
\\ &\leq \Big(1+\alpha L+\frac{2M\rho}{L}\Big) (1+\alpha L)^{2N-1} L(\|w_1-w_2\|+\|\phi_1-\phi_2\|) \nonumber
\\&\leq \text{poly}(M,\rho,\alpha,L) N(\|w_1-w_2\|+\|\phi_1-\phi_2\|),
\end{align}
which, using an approach similar to~\cref{eq:jensen}, completes the proof of the first item in~\Cref{le:smooth_nonconvex}.
We next prove the Lipschitz property of the partial gradient $\frac{\partial L_{\mathcal{D}_i}( w^i_N, \phi)}{\partial \phi}$. Using an approach similar to~\cref{eq:sect} and~\cref{R1R2R3}, we have 
\begin{align}\label{eq:qopqop}
\Big\|&\frac{\partial L_{\mathcal{D}_i}( w^i_N, \phi)}{\partial \phi}\Big |_{(w_1,\phi_1)}  - \frac{\partial L_{\mathcal{D}_i}( w^i_N, \phi)}{\partial \phi} \Big |_{(w_2,\phi_2)}\Big\| \nonumber
\\&\leq \alpha\sum_{m=0}^{N-1}(R_1+R_2+R_3) +\|\nabla_\phi L_{\mathcal{D}_i}(w_{N}^i(w_1,\phi_1),\phi_1)-\nabla_\phi L_{\mathcal{D}_i}(w_{N}^i(w_2,\phi_2),\phi_2)\|,
\end{align}
where $R_1,R_2$ and $R_3$ are defined in~\cref{R1R2R3}.

To upper-bound $R_1$ in the above inequality, we have
\begin{align}\label{R1}
R_1\overset{(i)}\leq &\tau (\|w_m^i(w_1,\phi_1)-w_m^i(w_2,\phi_2)\| + \|\phi_1-\phi_2\|) (1+\alpha L)^{N-m-1} M \nonumber
\\\overset{(ii)}\leq& \tau M (1+\alpha L)^{N-1} (\|w_1-w_2\| + \|\phi_1-\phi_2\|),
\end{align}
where $(i)$ follows from Assumptions~\ref{assm:smooth} and~\ref{assm:second} and $(ii)$ follows from~\cref{eq:wn}.

 For $R_2$, using the triangle inequality, we have
\begin{align}\label{eq:impor}
\|U_m&(w_1,\phi_1)-U_m(w_2,\phi_2)\|  \nonumber
\\\leq& \alpha \|\nabla_w^2L_{\mathcal{S}_i}(w_{m+1}^i(w_1,\phi_1),\phi_1)-\nabla_w^2L_{\mathcal{S}_i}(w_{m+1}^i(w_2,\phi_2),\phi_2)\| \|U_{m+1}(w_1,\phi_1)\| \nonumber
\\&+ \|I-\alpha\nabla_w^2L_{\mathcal{S}_i}(w_{m+1}^i(w_1,\phi_1),\phi_1)\| \|U_{m+1}(w_1,\phi_1)-U_{m+1}(w_2,\phi_2)\| \nonumber
\\\leq&\alpha \rho (1+\alpha L)^{N-m-2}(\|w_{m+1}^i(w_1,\phi_1)-w_{m+1}^i(w_2,\phi_2)\| + \|\phi_1-\phi_2\|) \nonumber
\\&+ (1+\alpha L)\|U_{m+1}(w_1,\phi_1)-U_{m+1}(w_2,\phi_2)\| \nonumber
\\\leq&\alpha \rho (1+\alpha L)^{N-1}(\|w_1-w_2\| + \|\phi_1-\phi_2\|) \nonumber
\\&+ (1+\alpha L)\|U_{m+1}(w_1,\phi_1)-U_{m+1}(w_2,\phi_2)\|.
\end{align}
Telescoping the above inequality over $m$ yields
\begin{align*}
&\|U_m(w_1,\phi_1)-U_m(w_2,\phi_2)\| +\frac{\rho}{L} (1+\alpha L)^{N-1}(\|w_1-w_2\| + \|\phi_1-\phi_2\|)\nonumber
 \\&\leq (1+\alpha L)^{N-m-2}\Big(\|U_{N-2}(w_1,\phi_1)-U_{N-2}(w_2,\phi_2)\|
 \\&\hspace{3.5cm}+\frac{\rho}{L} (1+\alpha L)^{N-1}(\|w_1-w_2\| + \|\phi_1-\phi_2\|)\Big), \nonumber
\end{align*}
which, in conjunction with 
\begin{align*}
\|U_{N-2}(w_1,\phi_1)&-U_{N-2}(w_2,\phi_2)\|
\\=& \alpha\|\nabla_w^2L_{\mathcal{S}_i}(w_{N-1}^i(w_1,\phi_1),\phi_1)-\nabla_w^2L_{\mathcal{S}_i}(w_{N-1}^i(w_2,\phi_2),\phi_2)\| \nonumber
\\\leq& \alpha\rho (1+\alpha L)^{N-1}(\|w_1-w_2\|+\|\phi_1-\phi_2\|),
\end{align*}
yields that
\begin{align}\label{eq:um1}
\|U_m(w_1,\phi_1)-U_m(w_2,\phi_2)\| \leq& \big(\alpha \rho+\frac{\rho}{L}\big)(1+\alpha L)^{2N-m-3}(\|w_1-w_2\|+\|\phi_1-\phi_2\|) \nonumber
\\& -\frac{\rho}{L}(1+\alpha L)^{N-1} (\|w_1-w_2\|+\|\phi_1-\phi_2\|).
\end{align}
Based on~\Cref{assm:smooth}, we have $\|Q_m(w_2,\phi_2)\|\leq L$ and $\|V_m(w_1,\phi_1)\|\leq M$, which, combined with \cref{eq:um1} and the definition of $R_2$  in~\cref{R1R2R3}, yields
\begin{align}\label{R2}
R_2\leq& ML\Big(\alpha \rho+\frac{\rho}{L}\Big)(1+\alpha L)^{2N-m-3}(\|w_1-w_2\|+\|\phi_1-\phi_2\|)\nonumber
\\&-M\rho (1+\alpha L)^{N-1} (\|w_1-w_2\|+\|\phi_1-\phi_2\|).
\end{align}

For  $R_3$, using~\Cref{assm:smooth}, we have 
\begin{align}\label{R3}
R_3\leq& L (1+\alpha L)^{N-m-1} \| \nabla_w L_{\mathcal{D}_i}(w_N^i(w_1,\phi_1),\phi_1)- \nabla_w L_{\mathcal{D}_i}(w_N^i(w_2,\phi_2),\phi_2)\| \nonumber
\\\leq & L^2 (1+\alpha L)^{2N-m-1} (\|w_1-w_2\|+\|\phi_1-\phi_2\|),  
\end{align}
where the last inequality follows from \cref{eq:wn}. Combining~\cref{R1},~\cref{R2} and~\cref{R3} yields
\begin{align}\label{eq:quv}
 R_1+R_2+R_3\leq & M(\tau-\rho) (1+\alpha L)^{N-1} (\|w_1-w_2\| + \|\phi_1-\phi_2\|)  \nonumber
 \\&+M\rho(1+\alpha L)^{2N-m-2}(\|w_1-w_2\|+\|\phi_1-\phi_2\|)\nonumber
 \\&+L^2 (1+\alpha L)^{2N-m-1} (\|w_1-w_2\|+\|\phi_1-\phi_2\|).  
\end{align}
 Combining~\cref{eq:qopqop},~\cref{eq:quv}, and using~\cref{eq:wn} and $\alpha <\mathcal{O}(\frac{1}{N})$, we have 
 \begin{align}\label{eq:winnsa}
 \Big\|&\frac{\partial L_{\mathcal{D}_i}( w^i_N, \phi)}{\partial \phi}\Big |_{(w_1,\phi_1)}  - \frac{\partial L_{\mathcal{D}_i}( w^i_N, \phi)}{\partial \phi} \Big |_{(w_2,\phi_2)}\Big\| \nonumber
 \\&\leq \Big(\alpha M(\tau-\rho)  N(1+\alpha L)^{N-1} + \Big(L+\frac{\rho M}{L}\Big)(1+\alpha L)^{2N}\Big)(\|w_1-w_2\|+\|\phi_1-\phi_2\|)\nonumber
\\&\leq\text{\normalfont poly}(M,\rho,\tau,\alpha,L) N(\|w_1-w_2\| + \|\phi_1-\phi_2\|),
 \end{align}
 which, using an approach similar to~\cref{eq:jensen}, finishes the proof of the second item in~\Cref{le:smooth_nonconvex}.
\subsection*{Proof of \Cref{th:nonconvexcaonidaxeas}}
For notational convenience, we define 
\begin{align}\label{eq:definitions}
g_{w}^i(k) &=  \frac{\partial L_{\mathcal{D}_i}( w^i_{k,N},\phi_k)}{\partial {w_k}},\quad g_{\phi}^i(k)  = \frac{\partial L_{\mathcal{D}_i}( w^i_{k,N},\phi_k)}{\partial {\phi_k}}, \nonumber
\\ L_w &=\big(L+\alpha L^2+2M\rho\big) (1+\alpha L)^{2N-1}, \nonumber
\\L_\phi & = \alpha M(\tau-\rho)  N(1+\alpha L)^{N-1} + \left(L+\frac{\rho M}{L}\right)(1+\alpha L)^{2N}.
\end{align}
Based on~\cref{eq:j1} and~\cref{eq:winnsa} in the proof of~\Cref{le:smooth_nonconvex}, we have 
\begin{small}
\begin{align*}
L^{meta}(w_{k+1},\phi_k) \leq & L^{meta}(w_k,\phi_k) + \left\langle \frac{\partial L^{meta}(w_k,\phi_k)}{\partial w_k}, w_{k+1}-w_k  \right\rangle + \frac{L_w}{2} \|w_{k+1}-w_k\|^2, \nonumber
\\ L^{meta}(w_{k+1},\phi_{k+1}) \leq & L^{meta}(w_{k+1},\phi_k) + \left\langle \frac{\partial L^{meta}(w_{k+1},\phi_k)}{\partial \phi_k}, \phi_{k+1}-\phi_k  \right\rangle + \frac{L_\phi}{2} \|\phi_{k+1}-\phi_k\|^2.
\end{align*}
\end{small}
\hspace{-0.15cm}Adding the above two inequalities, and using an approach similar to~\cref{eq:wellplay}, we have 
\begin{align}\label{eq:noexp}
L^{meta}&(w_{k+1},\phi_{k+1}) \nonumber
\\\leq&  L^{meta}(w_k,\phi_k) - \left\langle \frac{\partial L^{meta}(w_k,\phi_k)}{\partial w_k}, \frac{\beta_w}{B}\sum_{i\in\mathcal{B}_k} g_{w}^i (k) \right\rangle + \frac{L_w+L_\phi }{2} \Big\| \frac{\beta_w}{B}\sum_{i\in\mathcal{B}_k} g_{w}^i (k)\Big\|^2 \nonumber
\\ &-\left\langle \frac{\partial L^{meta}(w_{k},\phi_k)}{\partial \phi_k},\frac{\beta_\phi}{B}\sum_{i\in\mathcal{B}_k}g_{\phi}^i(k)  \right\rangle + L_\phi \Big\|\frac{\beta_\phi}{B}\sum_{i\in\mathcal{B}_k}g_{\phi}^i(k)\Big\|^2.
\end{align}
Let $\mathbb{E}_k=\mathbb{E}(\cdot | w_k,\phi_k)$. Then, conditioning on $w_k,\phi_k$, taking expectation over~\cref{eq:noexp} and using an approach similar to~\cref{eq:diedie}, we have 
\begin{align}\label{eq:finals}
\mathbb{E}_k L^{meta}(w_{k+1},&\phi_{k+1}) \nonumber
\\\leq&  L^{meta}(w_k,\phi_k) - \beta_w\left\| \frac{\partial L^{meta}(w_k,\phi_k)}{\partial w_k}\right\|^2 + \frac{(L_w+L_\phi)\beta_w^2}{2B}\mathbb{E}_k\big\|g_{w}^i(k) \big\|^2 \nonumber
\\ &+ \frac{L_\phi+L_w}{2}\beta_w^2 \left\| \frac{\partial L^{meta}(w_k,\phi_k)}{\partial w_k}\right\|^2-\beta_\phi\left\| \frac{\partial L^{meta}(w_{k},\phi_k)}{\partial \phi_k}  \right\|^2 \nonumber
\\&+L_\phi\left(\frac{\beta_\phi^2}{B}\mathbb{E}_k\big\|g_{\phi}^i(k) \big\|^2 + \beta_\phi^2 \left\| \frac{\partial L^{meta}(w_k,\phi_k)}{\partial \phi_k}\right\|^2\right).
\end{align}
 
 Our next
step is to upper-bound $\mathbb{E}_k\big\|g_{w}^i(k) \big\|^2$ and $\mathbb{E}_k\big\|g_{\phi}^i(k) \big\|^2$ in~\cref{eq:finals}. Based on the definitions of $g_{w}^i(k)$ in~\cref{eq:definitions} and~\Cref{le:gd_form}, we have
\begin{align}\label{eq:gwi}
\mathbb{E}_k\big\|g_{w}^i(k) \big\|^2 \leq& \mathbb{E}_k\left\| \frac{\partial L_{\mathcal{D}_i}( w^i_{k,N}, \phi_k)}{\partial w_k} \right\|^2 \nonumber
\\=& \mathbb{E}_k\left\|\prod_{m=0}^{N-1}(I - \alpha \nabla_w^2L_{\mathcal{S}_i}(w_{k,m}^i,\phi_k)) \nabla_{w} L_{\mathcal{D}_i} (w_{k,N}^i,\phi_k)\right\|^2 \nonumber
\\\leq &  \mathbb{E}_k(1+\alpha L)^{2N} M^2 = (1+\alpha L)^{2N} M^2.
\end{align}
Using an approach similar to~\cref{eq:gwi}, we have
\begin{align}\label{eq:gphi}
\mathbb{E}_k&\big\|g_{\phi}^i(k) \big\|^2 \nonumber
\\ \leq& 2\mathbb{E}_k\bigg\|\alpha \sum_{m=0}^{N-1}\nabla_\phi\nabla_w L_{\mathcal{S}_i}(w_{k,m}^i,\phi_k) \prod_{j=m+1}^{N-1}(I-\alpha\nabla_w^2L_{\mathcal{S}_i}(w_{k,j}^i,\phi_k))\nabla_w L_{\mathcal{D}_i}(w_{k,N}^i,\phi_k)\bigg\|^2 \nonumber
\\&+2\|\nabla_\phi L_{\mathcal{D}_i}(w_{k,N}^i,\phi_k)\|^2 \nonumber
\\ \leq& 2\alpha^2 L^2 M^2 \mathbb{E}_k\Big(\sum_{m=0}^{N-1} (1+\alpha L)^{N-1-m}\Big)^2 +2 M^2 \nonumber
\\< & 2M^2 (1+\alpha L)^{N}-1)^2 +2M^2 <  2M^2 (1+\alpha L)^{2N}.
\end{align}
Substituting~\cref{eq:gwi} and~\cref{eq:gphi} into~\cref{eq:finals}, we have
\begin{align}\label{eq:mid}
\mathbb{E}_k L^{meta}(w_{k+1},&\phi_{k+1}) \leq L^{meta}(w_k,\phi_k) - \left(\beta_w-\frac{L_w+L_\phi}{2}\beta_w^2\right)\left\| \frac{\partial L^{meta}(w_k,\phi_k)}{\partial w_k}\right\|^2 \nonumber
\\&+ \frac{(L_w+L_\phi)\beta_w^2}{2B}(1+\alpha L)^{2N}M^2-\big(\beta_\phi -L_\phi\beta_\phi^2 \big)\left\| \frac{\partial L^{meta}(w_{k},\phi_k)}{\partial \phi_k}  \right\|^2\nonumber
\\ & + \frac{2 L_\phi \beta_\phi^2}{B}  (1+\alpha L)^{2N}M^2.
\end{align}
Set  $\beta_w =\frac{1}{L_w+L_\phi}$ and $\beta_\phi =  \frac{1}{2L_\phi}$. Then, unconditioning on $w_k,\phi_k$ in~\cref{eq:mid}, we have
\begin{align*}
\mathbb{E} L^{meta}(w_{k+1},\phi_{k+1}) \leq& \mathbb{E}L^{meta}(w_k,\phi_k) - \frac{\beta_w}{2}\mathbb{E}\left\| \frac{\partial L^{meta}(w_k,\phi_k)}{\partial w_k}\right\|^2 + \frac{\beta_w}{2B}(1+\alpha L)^{2N}M^2\nonumber
\\ &-\frac{\beta_\phi}{2}\mathbb{E}\left\| \frac{\partial L^{meta}(w_{k},\phi_k)}{\partial \phi_k}  \right\|^2 + \frac{\beta_\phi}{B}  (1+\alpha L)^{2N}M^2.
\end{align*}
Telescoping the above equality over $k$ from $0$ to $K-1$ yields
\begin{align}\label{eq:conv}
\frac{\beta_w}{2}\frac{1}{K}&\sum_{k=0}^{K-1}\mathbb{E}\left\| \frac{\partial L^{meta}(w_k,\phi_k)}{\partial w_k}\right\|^2 + \frac{\beta_\phi}{2}\frac{1}{K}\sum_{k=0}^{K-1}\mathbb{E}\left\| \frac{\partial L^{meta}(w_{k},\phi_k)}{\partial \phi_k}  \right\|^2 \nonumber
\\& \leq  \frac{L^{meta}(w_0,\phi_0)-\min_{w,\phi}L^{meta}(w,\phi) }{K}+ \frac{\beta_w+2\beta_\phi}{2B}(1+\alpha L)^{2N}M^2.
\end{align}
Let $\Delta = L^{meta}(w_0,\phi_0)-\min_{w,\phi}L^{meta}(w,\phi)>0$ and let $\xi$ be chosen from $\{0,...,K-1\}$ uniformly at random. Then,~\cref{eq:conv} further yields
\begin{align*}
\mathbb{E}\left\| \frac{\partial L^{meta}(w_\xi,\phi_\xi)}{\partial w_\xi}\right\|^2  \leq& \frac{2\Delta (L_w+L_\phi)}{K} + \frac{1+\frac{L_w+L_\phi}{L_\phi}}{B}(1+\alpha L)^{2N}M^2 \nonumber
\\\mathbb{E}\left\| \frac{\partial L^{meta}(w_\xi,\phi_\xi)}{\partial \phi_\xi}\right\|^2  \leq &\frac{4\Delta L_\phi}{K} + \frac{2+\frac{2L_\phi}{L_w+L_\phi}}{B}(1+\alpha L)^{2N}M^2,
\end{align*}
which, in conjunction with the definitions of $L_w$ and $L_\phi$ in~\cref{eq:definitions} and using $\alpha<\mathcal{O}(\frac{1}{N})$, yields 
\begin{align}
\mathbb{E}\left\| \frac{\partial L^{meta}(w_\xi,\phi_\xi)}{\partial w_\xi}\right\|^2  \leq& \mathcal{O}\bigg(  \frac{N}{K} + \frac{N}{B}  \bigg), \nonumber
\\\mathbb{E}\left\| \frac{\partial L^{meta}(w_\xi,\phi_\xi)}{\partial \phi_\xi}\right\|^2  \leq &\mathcal{O}\bigg(  \frac{N}{K} + \frac{N}{B}  \bigg).
\end{align}
To achieve an $\epsilon$-stationary point, i.e., {\small $\mathbb{E}\left\| \frac{\partial L^{meta}(w,\phi)}{\partial w}\right\|^2 <\epsilon,\mathbb{E}\left\| \frac{\partial L^{meta}(w,\phi)}{\partial w}\right\|^2 <\epsilon$}, $K$ and $B$ need to be at most $\mathcal{O}(N\epsilon^{-2})$, which, in conjunction with the gradient forms in~\Cref{le:gd_form}, completes  the complexity results.

\bibliographystyle{plain}
\bibliography{bibfile}

\end{document}